\documentclass{article}

\PassOptionsToPackage{numbers}{natbib}

\usepackage[final]{neurips_2022}

\usepackage[utf8]{inputenc} 
\usepackage[T1]{fontenc}    
\usepackage{url}            
\usepackage{booktabs}       
\usepackage{amsfonts}       
\usepackage{nicefrac}       
\usepackage{microtype}      
\usepackage{wrapfig}

\usepackage{graphicx}
\usepackage{subcaption}

\usepackage{amsmath}
\usepackage{amsthm}
\usepackage{amsfonts}
\usepackage{amssymb}
\usepackage{bm}
\usepackage{graphicx}
\usepackage{sidecap}
\usepackage{mathrsfs}
\usepackage{multirow, multicol}
\usepackage{caption}
\usepackage{tabularx}
\usepackage{numprint}
\usepackage{wrapfig}
\usepackage{booktabs} 
\usepackage{rotating}

\usepackage{array,tabularx,multirow,hhline,booktabs,colortbl}
\usepackage[table]{xcolor}
\usepackage{xcolor, hyperref}

\usepackage{natbib} 
\newtheorem{theorem}{Theorem}[section]
\newtheorem{lemma}[theorem]{Lemma}

\newcommand{\ba}{\bm{a}}

\newcommand{\bu}{\bm{u}}
\newcommand{\bw}{\bm{w}}
\newcommand{\bx}{\bm{x}}
\newcommand{\bX}{\bm{X}}

\newcommand{\bZ}{\bm{Z}}
\newcommand{\by}{\bm{y}}
\newcommand{\be}{\bm{e}}

\newcommand{\bgamma}{\bm{\gamma}}
\newcommand{\bGamma}{\bm{\Gamma}}

\DeclareMathOperator*{\argmin}{\arg\!\min}

\def\ceil#1{\lceil #1 \rceil}
\def\floor#1{\lfloor #1 \rfloor}
\def\1{\bm{1}}

\DeclareMathAlphabet{\mathsfit}{\encodingdefault}{\sfdefault}{m}{sl}
\SetMathAlphabet{\mathsfit}{bold}{\encodingdefault}{\sfdefault}{bx}{n}

\def\vert#1{\lvert #1 \rvert}
\def\Vert#1{\lVert #1 \rVert}

\newcommand{\textssm}[1]{#1}

\newcommand{\modelfont}{\renewcommand*\familydefault{\sfdefault}\normalfont}
\newcommand{\prow}[0]{\quad\mathrel{\raisebox{-0.75ex}{\dots}}}
\newcommand{\risklabel}[0]{{\color{black}\textbf{RISK}}}
\newcommand{\scorelabel}[0]{{\color{black}\textbf{SCORE}}}

\definecolor{predcolor}{gray}{0.95}
\definecolor{scorecolor}{gray}{0.95}
\definecolor{riskcolor}{gray}{0.95}
\definecolor{transparentcolor}{gray}{0.95}

\newcommand{\instruction}[2]{{\color{white}\phantom{\textbf{add points from rows {#1} to {#2}}}}}

\newcommand{\risktable}[0]{\par\vspace{0.5em}\scriptsize\centering\renewcommand{\arraystretch}{1.25}\modelfont}

\usepackage{algorithm}
\usepackage{algpseudocode}
\usepackage{numprint}
\usepackage{siunitx}
\usepackage{dsfont}
\usepackage{hyperref}       
\def\ourmethod{FasterRisk}
\def\ourmethodLR{SparseBeamLR}
\def\ourmethodExpand{ExpandSuppBy1}
\def\ourmethodPool{CollectSparseDiversePool}
\def\ourmethodRound{AuxiliaryLossRounding}
\def\scaleAndRound{StarRaySearch}
\def\Wtmp{\mathcal{W}_{\textrm{tmp}}}

\usepackage[toc,page,header]{appendix}
\usepackage{minitoc}

\newcommand*\samethanks[1][\value{footnote}]{\footnotemark[#1]}

\title{FasterRisk: Fast and Accurate Interpretable Risk Scores}

\author{%
  Jiachang Liu\textsuperscript{1}\thanks{These authors contributed equally.} \quad Chudi Zhong\textsuperscript{1}\samethanks \quad Boxuan Li\textsuperscript{1} \quad
  Margo Seltzer\textsuperscript{2} \quad
  Cynthia Rudin\textsuperscript{1}\\
  \textsuperscript{1} Duke Univeristy \textsuperscript{2} University of British Columbia\\
  \texttt{\{jiachang.liu, chudi.zhong, boxuan.li\}@duke.edu}\\
  \texttt{mseltzer@cs.ubc.ca, cynthia@cs.duke.edu}
}

\begin{document}

\maketitle


\doparttoc 
\faketableofcontents 


\begin{abstract}
Over the last century, \textit{risk scores} have been the most popular form of predictive model used in healthcare and criminal justice. Risk scores are sparse linear models with integer coefficients; often these models can be memorized or placed on an index card. Typically, risk scores have been created either without data or by rounding logistic regression coefficients, but these methods do not reliably produce high-quality risk scores. Recent work used mathematical programming, which is computationally slow. We introduce an approach for efficiently producing a collection of high-quality risk scores learned from data. Specifically, our approach  produces a pool of almost-optimal sparse continuous solutions, each with a different support set, using a beam-search algorithm. Each of these continuous solutions is transformed into a separate risk score through a ``star ray'' search, where a range of multipliers are considered before rounding the coefficients sequentially to maintain low logistic loss. Our algorithm returns all of these high-quality risk scores for the user to consider. This method completes within minutes and can be valuable in a broad variety of applications. 
\end{abstract}
\section{Introduction}
\normalsize
\textit{Risk scores} are sparse linear models with integer coefficients that predict risks. They are arguably the most popular form of predictive model for high stakes decisions through the last century and are the standard form of model used in criminal justice \citep{austin2010kentucky,latessa2009creation} and medicine \citep{,Kessler2005,moreno2005,Than2014,Six2008,Ustun2017Medicine}.
\begin{wrapfigure}{r}{.6\textwidth}
\vspace{-4.3mm}
\centering
{
\resizebox{0.6\columnwidth}{!}{\begin{tabular}{|l l  r | l |}
   \hline
1.   & \textssm{Oval Shape}  & -2 points & $\phantom{+}\prow{}$ \\ 
  2. & \textssm{Irregular Shape}  & 4 points & $+\prow{}$ \\ 
  3. & \textssm{Circumscribed Margin}  & -5 points & $+\prow{}$ \\ 
  4. & \textssm{Spiculated Margin}  & 2 points & $+\prow{}$ \\ 
  5. & \textssm{Age $\geq$ 60}   & 3 points & $+\phantom{\prow{}}$ \\[0.1em]
   \hline
 & \instruction{1}{5} & \scorelabel{} & $=\phantom{\prow{}}$ \\ 
   \hline
\end{tabular}}
}
{%
\risktable{}
\begin{tabular}{|r|c|c|c|c|c|c|}
\hline \rowcolor{scorecolor}\scorelabel{} & -7     & -5     & -4     & -3     & -2     & -1  \\
\hline 
\rowcolor{riskcolor}\risklabel{} & 6.0\%  & 10.6\% & 13.8\% & 17.9\% & 22.8\% & 28.6\%  \\ 
\hline
\hline \rowcolor{scorecolor}\scorelabel{} & 0 & 1      & 2      & 3      & 4      & $\geq$ 5  \\
\hline 
\rowcolor{riskcolor}\risklabel{} & 35.2\% & 42.4\% & 50.0\% & 57.6\% & 64.8\% & 71.4\% \\ 
\hline
\end{tabular}
\normalsize
}
\caption{Risk score on the mammo dataset \citep{elter2007prediction}, whose population is biopsy patients. It predicts the risk of malignancy of a breast lesion. Risk score is from \ourmethod{} on a fold of a 5-CV split. The AUCs  on the training and test sets are 0.854 and 0.853, respectively.\vspace*{-10pt}
}
\label{fig:ExampleRiskScore}
\end{wrapfigure}
Their history dates back to at least the  criminal justice work of Burgess \citep{burgess1928factors}, where, based on their criminal history and demographics,  individuals were assigned integer point scores between 0 and 21 that determined the probability of their ``making good or of failing upon parole.''
Other famous risk scores are arguably the most widely-used predictive models in healthcare. These include the APGAR score \citep{APGAR1953}, developed in 1952 and given to newborns, and the CHADS$_2$ score \citep{gage2001validation}, which estimates stroke risk for atrial fibrillation patients. 
Figures \ref{fig:ExampleRiskScore} and \ref{fig:RiskScore_mammo_modelSize_5_pool_9_figure2} show example risk scores, which estimate risk of a breast lesion being malignant.

\begin{wrapfigure}{r}{.6\textwidth}
\vspace{-13.5mm}
\centering
{
\resizebox{0.6\columnwidth}{!}{\begin{tabular}{|llr|l|}
\hline
1. & Irregular Shape & 4 points & $\phantom{+}\prow{}$ \\
2. & Circumscribed Margin & -5 points  & $+\prow{}$ \\
3. & SpiculatedMargin & 2 points & $+\prow{}$\\ 
4. & Age $\geq$ 45 & 1 point\text{ } & $+\prow{}$\\
5. & Age $\geq$ 60 & 3 points & $+\phantom{\prow{}}$ \\[0.1em]
\hline
    & \instruction{1}{5} & \scorelabel{} & $=\phantom{\prow{}}$ \\ 
\hline
\end{tabular}}
}
{
\risktable{}
\begin{tabular}{|r|c|c|c|c|c|c|c|}
\hline \rowcolor{scorecolor}\scorelabel{} & -5 & -4 & -3 & -2 & -1 & 0  \\
\hline \rowcolor{riskcolor}\risklabel{}  & 7.3\% & 9.7\% & 12.9\% & 16.9\% & 21.9\% & 27.8\% \\ 
\hline
\hline \rowcolor{scorecolor}\scorelabel{} & 1 & 2 & 3 & 4 & 5 & 6 \\
\hline \rowcolor{riskcolor}\risklabel{}  & 34.6\% & 42.1\% & 50.0\% & 57.9\% & 65.4\% & 72.2\% \\ 
\hline


\end{tabular}
}
\medskip\caption{A second risk score on the mammo dataset on the same fold as in Figure~\ref{fig:ExampleRiskScore}. The AUCs on the training and test sets are 0.855 and 0.859, respectively. \ourmethod{} can produce a diverse pool of high-quality models. Users can choose a model that best fits with their domain knowledge.\vspace*{-3pt}
}

\label{fig:RiskScore_mammo_modelSize_5_pool_9_figure2}
\end{wrapfigure}

Risk scores have the benefit of being easily memorized; usually their names reveal the full model -- for instance, the factors in CHADS$_2$ are past \textbf{C}hronic heart failure, \textbf{H}ypertension, \textbf{A}ge$\geq$75 years, \textbf{D}iabetes, and past \textbf{S}troke (where past stroke receives \textbf{2} points and the others each receive 1 point). For risk scores, counterfactuals are often trivial to compute, even without a calculator. Also, checking that the data and calculations are correct is easier with risk scores than with other approaches. In short, risk scores have been created by humans for a century to support a huge spectrum of applications~\cite{allam2020scoring, lee2022risk, shang2020scoring, wasilewski2020covid, xie2020autoscore, zhang2020novel}, because humans find them easy to understand.

Traditionally, risk scores have been created in two main ways: (1) without data, with expert knowledge only (and validated only afterwards on data), and (2) using a semi-manual process involving manual feature selection and rounding of logistic regression coefficients. That is, these approaches rely heavily on domain expertise and rely little on data. Unfortunately, the alternative of building a model \textit{directly} from data leads to computationally hard problems: optimizing risk scores over a global objective on data is NP-hard, because in order to produce integer-valued scores, the feasible region must be the integer lattice. There have been only a few approaches to design risk scores automatically \citep{billiet2017interval,billiet2016interval,carrizosaDILSVM13,chevaleyre2013rounding, ErtekinRu15,sokolovska2017fused, sokolovska2018provable,ustun2017optimized,ustun2019learning,ustun2013supersparse}, but each of these has a flaw that limits its use in practice: the optimization-based approaches use mathematical programming solvers (which require a license) that are slow and scale poorly, and the other methods are randomized greedy algorithms, producing fast but much lower-quality solutions. We need an approach that exhibits the best of both worlds: speed fast enough to operate in a few minutes on a laptop and optimization/search capability as powerful as that of the mathematical programming tools. Our method, \ourmethod{}, lies at this intersection.
It is fast enough to enable interactive model design and can rapidly produce a large pool of models from which users can choose rather than producing only a single model. 

One may wonder why simple rounding of $\ell_1$-regularized logistic regression coefficients does not yield sufficiently good risk scores. Past works \citep{ustun2015slim,ustun2019learning} explain this as follows: the sheer amount of $\ell_1$ regularization needed to get a very sparse solution leads to large biases and worse loss values, and rounding goes against the performance gradient. For example, consider the following coefficients from $\ell_1$ regularization: [1.45, .87, .83, .47, .23, .15, ... ]. This model is worse than its unregularized counterpart due to the bias induced by the large $\ell_1$ term. Its rounded solution is [1,1,1,0,0,0,..], which leads to even worse loss. Instead, one could multiply all the coefficients by a constant and then round, but which constant is best? There are an infinite number of choices. Even if some value of the multiplier leads to minimal loss due to rounding, the bias from the $\ell_1$ term still limits the quality of the solution.

The algorithm presented here does not have these disadvantages. The steps are: (1) Fast subset search with $\ell_0$ optimization (avoiding the bias from $\ell_1$). This requires the solution of an NP-hard problem, but our fast subset selection algorithm is able to solve this quickly. We proceed from this accurate sparse continuous solution, preserving both sparseness and accuracy in the next steps. (2) Find a pool of diverse continuous sparse solutions that are almost as good as the solution found in (1) but with different support sets. (3) A ``star ray'' search, where we search for feasible integer-valued solutions along multipliers of each item in the pool from (2). By using multipliers, the search space resembles the rays of a star, because it extends each coefficient in the pool outward from the origin to search for solutions. To find integer solutions, 
we perform a local search (a form of sequential rounding). This method yields high performance solutions: we provide a theoretical upper bound on the loss difference between the continuous sparse solution and the rounded integer sparse solution.

Through extensive experiments, we show that our proposed method is computationally fast and produces high-quality integer solutions. 
This work thus provides valuable and novel tools to create risk scores for professionals in many different fields, such as healthcare, finance, and criminal justice.

\textbf{Contributions}: Our contributions include the three-step framework for producing risk scores, a beam-search-based algorithm for logistic regression with bounded coefficients (for Step 1), the search algorithm to find pools of diverse high-quality continuous solutions (for Step 2), the star ray search technique using multipliers (Step 3), and a theorem guaranteeing the quality of the star ray search.

\section{Related Work}
\label{sec:related_work}
\textit{Optimization-based approaches:}
Risk scores, which model $P(y=1|\bx)$, are different from threshold classifiers, which  predict either $y=1$ or $y=-1$ given $\bx$. Most work in the area of optimization of integer-valued sparse linear models focuses on classifiers, not risk scores \citep{billiet2017interval,billiet2016interval,carrizosaDILSVM13, sokolovska2017fused,sokolovska2018provable,ustun2015slim,ustun2013supersparse,zeng2017interpretable}. This difference is important, because a classifier generally cannot be calibrated well for use in risk scoring: only its single decision point is optimized. Despite this, several works use the hinge loss to calibrate predictions  \citep{billiet2016interval,carrizosaDILSVM13,sokolovska2017fused}. All of these optimization-based algorithms use mathematical programming solvers (i.e., integer programming solvers), which tend to be slow and cannot be used on larger problems. However, they can handle both feature selection and integer constraints.

To directly optimize risk scores, typically the logistic loss is used. The RiskSLIM algorithm \citep{ustun2019learning} optimizes the logistic loss regularized with $\ell_0$ regularization, subject to integer constraints on the coefficients. RiskSLIM uses callbacks to a MIP solver, alternating between solving linear programs and using branch-and-cut to divide and reduce the search space. The branch-and-cut procedure needs to keep track of unsolved nodes, whose number increases exponentially with the size of the feature space. Thus, RiskSLIM's major challenge is scalability.

\textit{Local search-based approaches:} As discussed earlier, a natural way to produce a scoring system or risk score is by selecting features manually and rounding logistic regression coefficients or hinge-loss solutions to integers \citep{chevaleyre2013rounding,cole1993algorithm, ustun2019learning}. 
While rounding is fast, rounding errors can cause the solution quality to be much worse than that of the optimization-based approaches.  
Several works have proposed improvements over traditional rounding.
In Randomized Rounding \citep{chevaleyre2013rounding},  each coefficient is rounded up or down randomly, based on its continuous coefficient value.
However, randomized rounding does not seem to perform well in practice.
Chevaleyre~\citep{chevaleyre2013rounding} also proposed Greedy Rounding, where coefficients are rounded sequentially.
While this technique aimed to provide theoretical guarantees for the hinge loss, we identified a serious flaw in the argument, rendering the bounds incorrect (see Appendix \ref{app:ChevaleyreWrongProof}).
The RiskSLIM paper \citep{ustun2019learning} proposed SequentialRounding, which, at each iteration, chooses a coefficient to round up or down, making the best choice according to the regularized logistic loss. 
This gives better solutions than other types of rounding, because the coefficients are considered together through their performance on the loss function, not independently. 

A drawback of SequentialRounding is that it considers rounding up or down only to the nearest integer from the continuous solution.
By considering \textit{multipliers}, we consider a much larger space of possible solutions. The idea of multipliers (i.e., ``scale and round'') is used for medical scoring systems  \citep{cole1993algorithm}, though, as far as we know, it has been used only with traditional rounding rather than SequentialRounding, which could easily lead to poor performance, and we have seen no previous work that studies how to perform scale-and-round in a systematic, computationally efficient way. While the general idea of scale-and-round seems simple, it is not: there are an infinite number of possible multipliers, and, for each one, a number of possible nearby integer coefficient vectors that is the size of a hypercube, expanding exponentially in the search space. 

\textit{Sampling Methods:} The Bayesian method of Ertekin et al.~\citep{ErtekinRu15} samples scoring systems, favoring those that are simpler and more accurate, according to a prior. 
``Pooling'' \cite{ustun2019learning} creates multiple models through sampling along the regularization path of ElasticNet. As discussed, when regularization is tuned high enough to induce sparse solutions, it results in substantial bias and low-quality solutions  (see~\citep{ustun2015slim,ustun2019learning} for numerous experiments on this point). Note that there is a literature on finding diverse solutions to mixed-integer optimization problems~\citep[e.g.,][]{Ahanor2022}, but it focuses only on linear objective functions.

\section{Methodology}
\label{sec:methodology}

Define dataset $\mathcal{D}=\{1, \bx_i,y_i\}_{i=1}^n$ (1 is a static feature corresponding to the intercept) and scaled dataset as $\frac{1}{m}\times \mathcal{D}=\{\frac{1}{m}, \frac{1}{m}\bx_i,y_i\}_{i=1}^n$, for a real-valued $m$.
Our goal is to produce high-quality risk scores within a few minutes on a small personal computer. We start with an optimization problem similar to RiskSLIM's \citep{ustun2019learning}, which minimizes the logistic loss subject to sparsity constraints and integer coefficients:
\begin{eqnarray} \label{obj:sparse_integer}
    \min_{\bw,w_0} L(\bw,w_0,\mathcal{D}),& \textrm{ where }L(\bw,w_0,\mathcal{D}) = \sum_{i=1}^n \log(1+\exp(-y_i (\bx_i^T \bw + w_0)))\\\nonumber
    \textrm{ such that } &\|\bw\|_0 \leq k \textrm{ and } \bw\in\mathbb{Z}^p,\;\; \forall j\in[1,..,p]\; w_j\in [-5,5],\;\; w_0 \in \mathbb{Z}. 
\end{eqnarray}
In practice, the range of these box constraints $[-5, 5]$ is user-defined and can be different for each coefficient. (We use 5 for ease of exposition.)
The sparsity constraint $\|\bw\|_0\leq k$ or integer constraints $\bw \in \mathbb{Z}^p$ make the problem NP-hard, and this is a difficult mixed-integer nonlinear program. 
Transforming the original features to all possible dummy variables, which is a standard type of preprocessing \cite[e.g.,][]{LiuEtAl2022}, changes the model into a (flexible) generalized additive model; such models can be
as accurate as the best machine learning models
\citep{ustun2019learning,WangHanEtAl2022}. Thus, we generally process variables in $\bx$ to be binary.

To make the solution space substantially larger than $[-5,-4, ..., 4,5]^p$, we use \textit{multipliers}.
The problem becomes:
\begin{eqnarray} \label{obj:sparse_integer_mult}
    \lefteqn{\hspace*{-30pt}\min_{\bw,w_0,m} L\left(\bw, w_0,\frac{1}{m}\mathcal{D}\right), \textrm{ where }L\left(\bw,  w_0,\frac{1}{m}\mathcal{D}\right) = \sum_{i=1}^n \log\left(1+\exp\left(-y_i \frac{\bx_i^T \bw + w_0}{m}\right)\right)}\\
    \nonumber
   & \textrm{ such that } \|\bw\|_0 \leq k, \bw\in\mathbb{Z}^p,\;\; \forall j\in[1,..,p],\; w_j\in [-5,5],\;\; w_0 \in \mathbb{Z}, \;\; m>0. 
\end{eqnarray}
Note that the use of multipliers does not weaken the interpretability of the risk score: the user still sees integer risk scores composed of values $w_j\in \{-5,-4,..,4,5\}$, $w_0 \in \mathbb{Z}$. Only the risk conversion table is calculated differently, as $P(Y=1|\bx)=1/(1+ e^{-f(\bx)})$ where $f(\bx)=\frac{1}{m}(\bw^T\bx + w_0)$.


\begin{algorithm}[t]
\caption{FasterRisk($\mathcal{D}$,$k$,$C$,$B$,$\epsilon$,$T$,$N_m$) $\rightarrow \{(\bw^{+t},w_0^{+t},m_{t})\}_t$ }
\label{alg:overall_algorithm}
\textbf{Input:} dataset $\mathcal{D}$ (consisting of feature matrix $\bX \in \mathbb{R}^{n \times p}$ and labels $\by \in \mathbb{R}^{n}$), sparsity constraint $k$,  coefficient constraint $C=5$, beam search size $B=10$, tolerance level $\epsilon=0.3$, number of attempts $T=50$, number of multipliers to try $N_{m}=20$.\\
\textbf{Output:} a pool $P$ of scoring systems $\{(\bw^t, w_0^t), m^t\}$ where $t$ is the index enumerating all found scoring systems with $\lVert \bw^t\rVert_0 \leq k$ and $\lVert \bw^t \rVert_{\infty} \leq C$ and $m^t$ is the corresponding multiplier.
\begin{algorithmic}[1]
    \State Call Algorithm \ref{alg:sparse_beam_lr} \ourmethodLR{}$(\mathcal{D},k,C,B)$ to find a high-quality solution $(\bw^*,w_0^*)$ to the sparse logistic regression problem with continuous coefficients satisfying a box constraint, i.e., solve Problem~\eqref{eqn:alg1obj}. (Algorithm \ourmethodLR{} will call Algorithm \ourmethodExpand{} as a subroutine, which grows the solution by beam search.)
    \State Call Algorithm \ref{alg:collectSparseLevelSet} \ourmethodPool{$((\bw^*, w_0^*),\epsilon,T)$}, which solves Problem \eqref{eqn:alg1objpool}. Place its output $\{(\bw^t,w_0^t)\}_t$ in pool $P =\{\bw^*, w_0^*\}$. $P\leftarrow P\cup  \{(\bw^t,w_0^t)\}_t$.
    \State Send each member $t$ in the pool $P$, which is $(\bw^t,w_0^t)$, to Algorithm \ref{alg:scale_and_round} \scaleAndRound{} $(\mathcal{D}, (\bw^t, w_0^t), C, N_m)$ to perform a line search among possible multiplier values and obtain an integer solution $(\bw^{+t}, w_0^{+t})$ with multiplier $m_t$. Algorithm \ref{alg:scale_and_round} calls Algorithm \ref{alg:rounding_method} \ourmethodRound{} which conducts the rounding step.
    \State Return the collection of risk scores  $\{(\bw^{+t},w_0^{+t},m_{t})\}_t$. If desired, return only the best model according to the logistic loss.
\end{algorithmic}
\end{algorithm}

Our method proceeds in three steps, as outlined in Algorithm \ref{alg:overall_algorithm}. In the first step, it approximately solves the following \textbf{sparse logistic regression} problem with a box constraint (but not integer constraints), detailed in Section \ref{sec:sparse_continuous_solution} and Algorithm \ref{alg:sparse_beam_lr}.
\begin{equation}\label{eqn:alg1obj}
    (\bw^*,  w^*_0)\in \argmin_{\bw, w_0} L(\bw, w_0, \mathcal{D}),
    \;\lVert \bw\rVert_0 \leq k, \bw\in\mathbb{R}^p,  \forall j \in [1, ..., p],\; \bw_j\in [-5,5], w_0 \in \mathbb{R}. 
\end{equation}
The algorithm gives an accurate and sparse real-valued solution $(\bw^{*},w_0^*)$.

The second step produces \textbf{many near-optimal sparse logistic regression solutions}, again without integer constraints, detailed in Section \ref{sec:pool_continuous_solutions} and Algorithm \ref{alg:collectSparseLevelSet}. Algorithm \ref{alg:collectSparseLevelSet}
uses $(\bw^*, w_0^*)$ from the first step to find a set $\{(\bw^t, w_0^t)\}_t$ such that for all $t$ and a given threshold $\epsilon_{\bw}$:
\begin{eqnarray}\label{eqn:alg1objpool}
    (\bw^t,  w^t_0) \;\textrm{ obeys }\; L(\bw^t, w^t_0, \mathcal{D}) \leq L(\bw^*, w_0^*, \mathcal{D}) \times (1+\epsilon_{\bw^*}) \\ \;\;\;\;\; \lVert \bw^t \rVert_0 \leq k, \; \bw^t \in \mathbb{R}^p,\;\;  \forall j \in [1, ..., p],\; w^t_j\in [-5,5], w^t_0 \in \mathbb{R}. \nonumber
\end{eqnarray}
After these steps, we have a pool of almost-optimal sparse logistic regression models. In the third step, for each coefficient vector in the pool, we \textbf{compute a risk score}. It is a feasible integer solution 
$(\bw^{+t}, w^{+t}_0)$ to the following, which includes a positive multiplier $m^t > 0$:
\begin{eqnarray}\label{eqn:alg2obj}
\lefteqn{L\left(\bw^{+t}, w^{+t}_0, \frac{1}{m^t} \mathcal{D}\right)\leq L(\bw^{t}, w^t_0, \mathcal{D})+\epsilon_t,} 
 \\\nonumber
&&\bw^{+t} \in\mathbb{Z}^p,\;\; \forall j \in [1,...,p], w_j^{+t}\in [-5,5], w^{+t}_0 \in \mathbb{Z},
\end{eqnarray}
where we derive a tight theoretical upper bound on $\epsilon_t$.
 A detailed solution to \eqref{eqn:alg2obj} is shown in Algorithm \ref{alg:rounding_method} in Appendix \ref{app:additional_algorithmic_charts}. We solve the optimization problem for a large range of multipliers in Algorithm \ref{alg:scale_and_round} for each coefficient vector in the pool, choosing the best multiplier for each coefficient vector.
 This third step yields a large collection of risk scores, all of which are approximately as accurate as the best sparse logistic regression model that can be obtained. All steps in this process are fast and scalable.

\begin{algorithm}[th]
\caption{\ourmethodLR{}($\mathcal{D}$,$k$,$C$,$B$) $\rightarrow~(\bw, w_0)$}
\label{alg:sparse_beam_lr}
\textbf{Input:} dataset $\mathcal{D}$, sparsity constraint $k$, coefficient constraint $C$, and beam search size $B$.\\
\textbf{Output:} a sparse continuous coefficient vector $(\bw, w_0)$ with $\lVert \bw \rVert_0 \leq k, \Vert{\bw}_{\infty} \leq C$.
\begin{algorithmic}[1]
    \State Define $N_+$ and $N_-$ as numbers of positive and negative labels, respectively.
    \State $w_0 \leftarrow \log(-N_+/N_-), \bw \leftarrow \mathbf{0}$ \Comment{Initialize the intercept and coefficients.}
    \State $\mathcal{F} \leftarrow \emptyset$ \Comment{Initialize the collection of found supports as an empty set}
    \State $(\mathcal{W}, \mathcal{F}) \leftarrow \text{\ourmethodExpand}(\mathcal{D},(\bw, w_0), \mathcal{F},B)$. \Comment{Returns $\leq B$ models of support 1}
    \For{$t = 2, ..., k$} \Comment{Beam search to expand the support}
        \State $\Wtmp \leftarrow \emptyset$
        \For{ $(\bw', w'_0) \in \mathcal{W}$} \Comment{Each of these has support $t-1$}
            \State $(\mathcal{W}', \mathcal{F}) \leftarrow \text{\ourmethodExpand} (\mathcal{D},(\bw', w'_0),\mathcal{F},B)$.
            \Comment{Returns $\leq B$ models with supp. $t$.}
            \State $\Wtmp \leftarrow \Wtmp \cup \mathcal{W}'$
        \EndFor
        \State Reset $\mathcal{W}$ to be the $B$ solutions in $\Wtmp$ with the smallest logistic loss values.
    \EndFor
    \State Pick $(\bw, w_0)$ from $ \mathcal{W}$ with the smallest logistic loss.
    \State Return $(\bw, w_0)$.
\end{algorithmic}
\end{algorithm}


\subsection{High-quality Sparse Continuous Solution} \label{sec:sparse_continuous_solution}

There are many different approaches for sparse logistic regression, including $\ell_1$ regularization \citep{tibshirani1996regression}, ElasticNet \citep{zou2005regularization}, $\ell_0$ regularization \citep{dedieu2020learning, LiuEtAl2022}, and orthogonal matching pursuit (OMP) \citep{elenberg2018restricted,lozano2011group}, but none of these approaches seem to be able to handle both the box constraints and the sparsity constraint in Problem \ref{eqn:alg1obj}, so we developed a new approach. This approach, in Algorithm \ref{alg:sparse_beam_lr}, \ourmethodLR{}, uses beam search for best subset selection: each iteration contains several coordinate descent steps to determine whether a new variable should be added to the support, and it clips coefficients to the box $[-5,5]$ as it proceeds. Hence the algorithm is able to determine, before committing to the new variable, whether it is likely to decrease the loss while obeying the box constraints. This beam search algorithm for solving \eqref{eqn:alg1obj} implicitly uses the assumption that one of the best models of size $k$ implicitly contains variables of one of the best models of size $k-1$.
This type of assumption has been studied in the sparse learning literature \cite{elenberg2018restricted} (Theorem 5). However, we are not aware of any other work that applies box constraints or beam search for sparse logistic regression.
In Appendix \ref{app:additional_experimental_results}, we show that our method produces better solutions than the OMP method presented in \cite{elenberg2018restricted}.


Algorithm \ref{alg:sparse_beam_lr} calls the \ourmethodExpand{} Algorithm, which has two major steps. The detailed algorithm can be found in Appendix \ref{app:additional_algorithmic_charts}. For the first step, given a solution $\bw$, we perform optimization on each single coordinate $j$ outside of the current support $supp(\bw)$:
\begin{align}
\label{eqn:expandSuppBy1_obj}
    d^*_j \in \argmin_{d \in [-5, 5]} L(\bw + d \be_j, w_0, \mathcal{D}) \text{ for } \forall j \text{ where } w_j = 0.
\end{align}
Vector $\be_j$ is 1 for the $j$th coordinate and 0 otherwise.
We find $d^*_j$ for each $j$ through an iterative thresholding operation, which is done on all coordinates in parallel, iterating several ($\sim 10$) times: 
\begin{align}
\label{eqn:expandSuppBy1_analytical_form}
    \text{ for iteration $i$: } d_j \leftarrow \textrm{Threshold}(j, d_j, \bw, w_0, \mathcal{D}):=\min(\max(c_{d_j}, -5), 5),
\end{align}
where $c_{d_j} = d_j - \frac{1}{l_j}\nabla_j L(\bw + d_j \be_j, w_0, \mathcal{D})$, and $l_j$ is a Lipschitz constant on coordinate $j$~\cite{LiuEtAl2022}. Importantly, we can perform Equation \ref{eqn:expandSuppBy1_analytical_form} on all $j$ where $w_j=0$ in parallel using matrix form.

For the second step, after the parallel single coordinate optimization is done, we pick the top $B$ indices ($j$'s) with the smallest logistic losses $L(\bw + d^*_j \be_j)$ and fine tune on the new support:
\begin{align}
\label{eqn:expandSuppBy1_finetune_obj}
    \bw_{\text{new}}^j, {w_0}_{\text{new}}^j \in \argmin_{\ba \in [-5, 5]^p, b} L(\ba, b, \mathcal{D}) \text{ with } supp(\ba) = supp(\bw) \cup \{j\}.
\end{align}
This can be done again using a variant of Equation \ref{eqn:expandSuppBy1_analytical_form} iteratively on all the coordinates in the new support.
We get $B$ pairs of $(\bw_{\text{new}}^j, {w_0}_{\text{new}}^j)$ through this \ourmethodExpand{} procedure, and the collection of these pairs form the set $\mathcal{W}'$ in Line 8 of Algorithm \ref{alg:sparse_beam_lr}.

At the end, Algorithm~\ref{alg:sparse_beam_lr} (\ourmethodLR{}) returns the best model with the smallest logistic loss found by the beam search procedure. This model satisfies both the sparsity and box constraints.


\subsection{Collect Sparse Diverse Pool (Rashomon Set)}
\label{sec:pool_continuous_solutions}

We now collect the sparse diverse pool. In Section \ref{sec:sparse_continuous_solution}, our goal was to find a sparse model $(\bw^*, w_0^*)$ with the smallest logistic loss. For high dimensional features or in the presence of highly correlated features, there could exist many sparse models with almost equally good performance \citep{breiman-cultures}. This set of models is also known as the Rashomon set. Let us find those and turn them into risk scores.
We first predefine a tolerance gap level $\epsilon$ (hyperparameter, usually set to $0.3$). Then, we delete a feature with index $j_{-}$ in the support $\textrm{supp}(\bw^*)$ and add a new feature with index $j_{+}$. We select each new index to be $j_+$ whose logistic loss is within the tolerance gap:
\begin{align}
\label{eqn:collect_sparse_diverset_pool_obj}
\textrm{Find all $j_+$ s.t. }
    \min_{a \in [-5, 5]} L(\bw^* - w^*_{j-} \be_{j-}+ a \be_{j+}, w_0, \mathcal{D}) \leq L(\bw^*, w_0^*, \mathcal{D}) (1+\epsilon).
\end{align}
We fine-tune the coefficients on each of the new supports and then save the new solution in our pool. Details can be found in Algorithm \ref{alg:collectSparseLevelSet}. Swapping one feature at a time is computationally efficient, and our experiments show it produces sufficiently diverse pools over many datasets.
We call this method the \ourmethodPool{} Algorithm.

\subsection{``Star Ray'' Search for Integer Solutions}
\label{sec:round}

The last challenge is how to get an integer solution from a continuous solution. To achieve this, we use a ``star ray'' search that searches along each ``ray'' of the star, extending each continuous solution outward from the origin using many values of a multiplier, as shown in Algorithm \ref{alg:scale_and_round}. 
The star ray search provides much more flexibility in finding a good integer solution than simple rounding.
The largest multiplier $m_{\max}$ is set to $5 / \max_j(\vert{w_j^*})$ which will take one of the coefficients to the boundary of the box constraint at 5. 
We set the smallest multiplier to be 1.0 and pick $N_m$ (usually 20) equally spaced points from $[m_{\min}, m_{\max}]$. If $m_{\max}=1$, we set $m_{\min}=0.5$ to allow shrinkage of the coefficients. We scale the coefficients and datasets with each multiplier and round the coefficients to integers using the sequential rounding technique in Algorithm \ref{alg:rounding_method}. For each continuous solution (each ``ray'' of the ``star''), we report the integer solution and multiplier with the smallest logistic loss. This process yields our collection of risk scores. Note here that a standard line search along the multiplier does not work, because the rounding error is highly non-convex.

\begin{algorithm}[th]
\caption{\scaleAndRound{}($\mathcal{D},(\bw, w_0),C,N_m)~\rightarrow~(\bw^+, w^+_0), m$}
\label{alg:scale_and_round}
\textbf{Input:} dataset $\mathcal{D}$, a sparse continuous solution $(\bw, w_0)$, coefficient constraint $C$, and number of multipliers to try $N_m$.\\
\textbf{Output:} a sparse integer solution $(\bw^+, w^+_0)$ with $\Vert{\bw^+}_{\infty} \leq C$ and  multiplier $m$.
\begin{algorithmic}[1]
    \State Define $m_{\max} \leftarrow C / \max \vert{\bw}$ as discussed in Section \ref{sec:round}. If $m_{\max} =1$, set $m_{\min} \leftarrow 0.5$; if $m_{\max} > 1$, set $m_{\min} \leftarrow 1$. 
    \State Pick $N_m$ equally spaced multiplier values $m_l \in [m_{\min}, m_{\max}]$ for $l\in [1, ..., N_m]$ and call this set $\mathcal{M} = \{m_l\}_l$.
    \State Use each multiplier to scale the good continuous solution $(\bw, w_0)$, to obtain $(m_l\bw$, $m_l w_0)$, which is a good continuous solution to the rescaled dataset 
    $\frac{1}{m_l} \mathcal{D}$. 
    \State Send each rescaled solution $(m_l\bw$, $m_lw_0)$ and its rescaled dataset $\frac{1}{m_l} \mathcal{D}$ to Algorithm \ref{alg:rounding_method} \ourmethodRound{$(\frac{1}{m_l}\mathcal{D},m_l\bw,m_lw_0)$}
    for rounding. It returns  $(\bw^{+l},w_0^{+l},m_l)$, where $(\bw^{+l},w_0^{+l})$ is close to $(m_l\bw$, $m_lw_0)$, and where $(\bw^{+l},w_0^{+l})$ on $\frac{1}{m_l}\mathcal{D}$ has a small logistic loss. 
    \State Evaluate the logistic loss to pick the best multiplier $l^* \in \argmin_{l} L(\bw^{+l},w_0^{+l},\frac{1}{m^{l}}\mathcal{D})$
    \State Return $(\bw^{+l^*},w_0^{+l^*})$ and $m_{l^*}$.
\end{algorithmic}
\end{algorithm}

We briefly discuss how the sequential rounding technique works. Details of this method can be found in Appendix \ref{app:additional_algorithmic_charts}. We initialize $\bw^+ = \bw$. Then we round the fractional part of $\bw^+$ one coordinate at a time. At each step, some of the $w_j^+$'s are integer-valued (so $w_j^+-w_j$ is nonzero) and we pick the coordinate and rounding operation (either floor or ceil) based on which can minimize the following objective function, where we will round to an integer at coordinate $r^*$:
\begin{align}
\label{eqn:auxilliaryloss_rounding_obj}
    r^*, v^* \in \argmin_{r, v}\sum_{i=1}^n l_i^2 \left(x_{ir}(v-w_r) + \sum_{j \neq r} x_{ij} (w^+_j - w_j) \right)^2, \\
    \text{subject to } r \in \{j \mid w^+_j \notin \mathbb{Z}\} \text{ and } v \in \{\floor{w^+_r}, \ceil{w^+_r}\}, \nonumber
\end{align}
where $l_i$ is the Lipschitz constant restricted to the rounding interval and can be computed as $l_i = 1/(1+\exp(y_i \bx_i^T \bgamma_i))$ with $\gamma_{ij} = \floor{w_j}$ if $y_i x_{ij} > 0$ and $\gamma_{ij} = \ceil{w_{j}}$ otherwise. (The Lipschitz constant here is much smaller than the one in Section \ref{sec:sparse_continuous_solution} due to the interval restriction.)
After we select $r^*$ and find value $v^*$, we update $\bw^+$ by setting $w^+_{r^*} = v^*$. We repeat this process until $\bw^+$ is on the integer lattice: $\bw^+ \in \mathbb{Z}^p$.
The objective function in Equation \ref{eqn:auxilliaryloss_rounding_obj} can be understood as an auxiliary upper bound of the logistic loss. 
Our algorithm provides an upper bound on the difference between the logistic losses of the continuous solution and the final rounded solution before we start the rounding algorithm (Theorem \ref{theorem:upperBound_auxiliaryLossRounding} below).
Additionally, during the sequential rounding procedure, we do not need to perform expensive operations such as logarithms or exponentials as required by the logistic loss function; the bound and auxiliary function require only sums of squares, not logarithms or exponentials.
Its derivation and proof are in Appendix \ref{app:theoretical_upper_bound_derivation}.

\begin{theorem} \label{theorem:upperBound_auxiliaryLossRounding}
Let $\bw$ be the real-valued coefficients for the logistic regression model with objective function $L(\bw) = \sum_{i=1}^n \log(1+\exp(-y_i\bx_i^T \bw))$ (the intercept is incorporated). Let $\bw^+$ be the integer-valued coefficients returned by the \ourmethodRound{} method. Furthermore, let $u_j=w_j - \lfloor w_j \rfloor$. Let $l_i = 1/(1+\exp(y_i \bx_i^T \bgamma_i))$ with $\gamma_{ij} = \floor{w_j}$ if $y_i x_{ij} > 0$ and $\gamma_{ij} = \ceil{w_{j}}$ otherwise. Then, we have an upper bound on the difference between the loss $L(\bw)$ and the loss $L(\bw^+)$:
\begin{equation}
    L(\bw^+) - L(\bw) \leq \sqrt{n \sum_{i=1}^n \sum_{j=1}^p (l_i x_{ij})^2 u_j (1-u_j)}.
\end{equation}
\end{theorem}

\textbf{Note.} \textit{Our method has a higher prediction capacity than RiskSLIM: its search space is much larger. } \\
Compared to RiskSLIM, our use of the multiplier permits a number of solutions that grows exponentially in $k$ as we increase the multiplier. To see this, consider that for each support of $k$ features, since logistic loss is convex, it contains a hypersphere in coefficient space. The volume of that hypersphere is (as usual) $V=\frac{\pi^{k/2}}{\Gamma(\frac{k}{2}+1)}r^k$ where $r$ is the radius of the hypersphere. If we increase the multiplier to 2, the grid becomes finer by a factor of 2, which is equivalent to increasing the radius by a factor of 2. Thus, the volume increases by a factor of $2^k$. In general, for maximum multiplier $m$, the search space is increased by a factor of $m^k$ over RiskSLIM.
\section{Experiments}
\label{sec:experiments}

\begin{figure}[t]
    \centering
    \includegraphics[width=\textwidth]{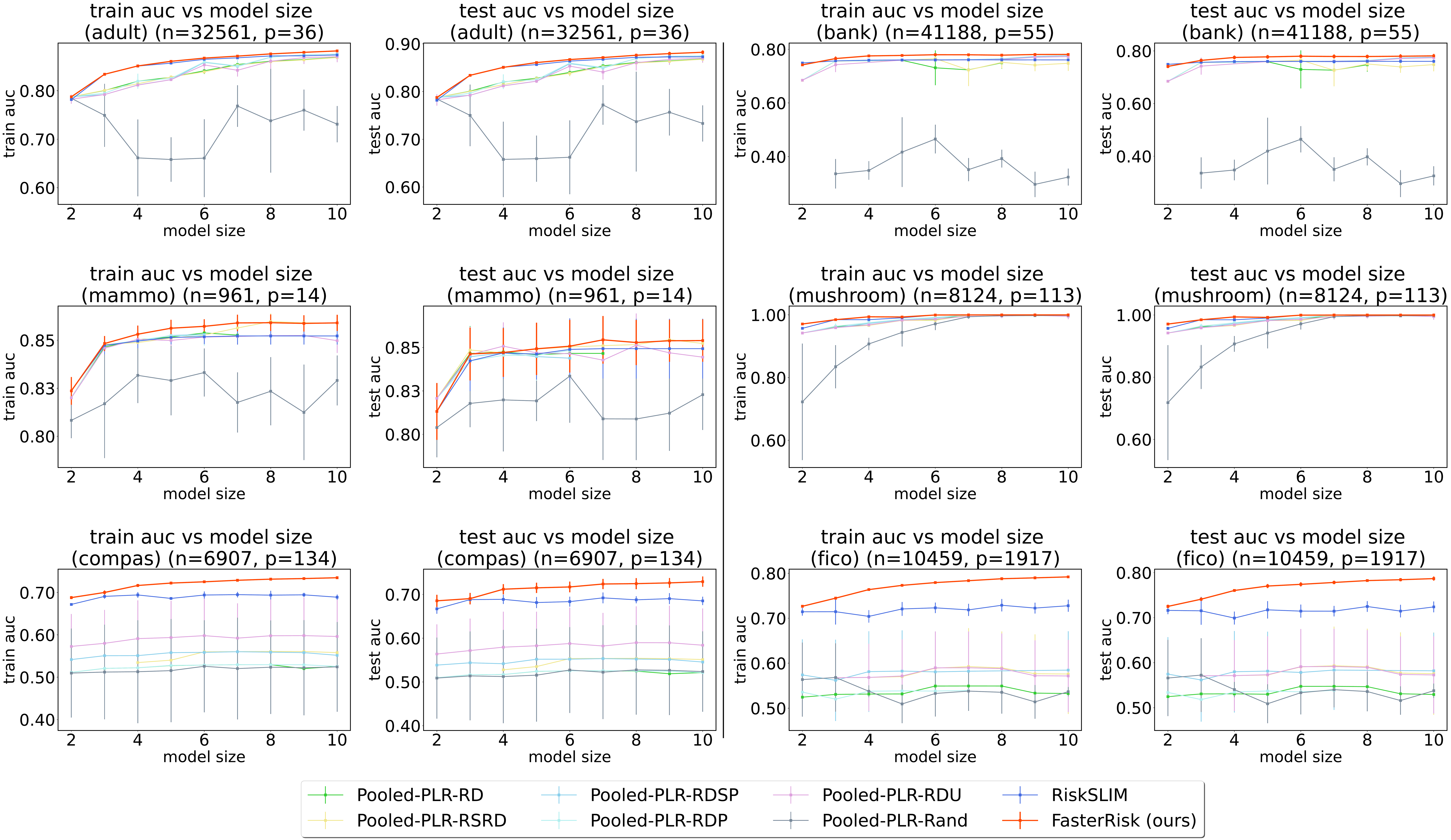}
    \caption{Performance comparison. \ourmethod{} outperforms all baselines due to its larger hypothesis space.
    On the datasets with highly-correlated variables such as COMPAS and FICO (both in the bottom row), \ourmethod{} outperforms other methods by a large margin.
    }
    \label{fig:solution_quality}
\end{figure}

\begin{figure}[htbp]
    \centering
    \includegraphics[width=\textwidth]{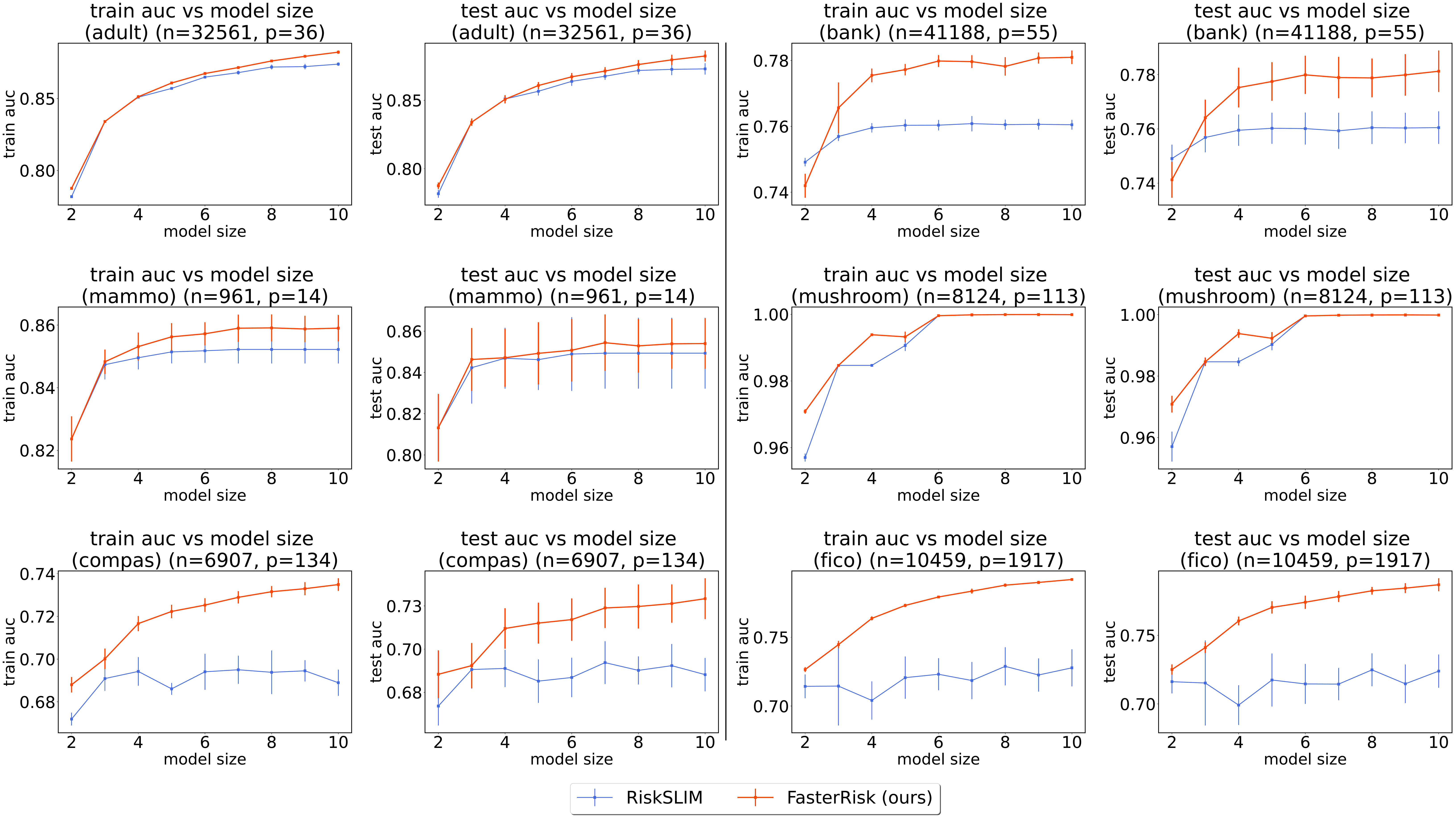}
    
    \caption{Performance comparison between \ourmethod{} and RiskSLIM.
    }
    \label{fig:direct_comparison_with_RiskSLIM}
\end{figure}

We experimentally focus on two questions:
(1) How good is \ourmethod's solution quality compared to baselines? (\S \ref{subsec:experiments_solution_quality})
(2) How fast is \ourmethod{} compared with the state-of-the-art? (\S \ref{subsec:experiments_run_time})
In the appendix, we address three more questions:
(3) How much do the sparse beam search, diverse pools, and multipliers contribute to our solution quality?
(\ref{appendix:ablation_study})
(4) How well-calibrated are the models produced by FasterRisk? (\ref{appendix:calibration_curves}) (5)  How sensitive is FasterRisk to each of the hyperparameters in the algorithm? (\ref{appendix:hyperparameter_perturbation_study})

We compare with RiskSLIM (the current state-of-the-art), as well as algorithms Pooled-PLR-RD, Pooled-PLR-RSRD, Pooled-PRL-RDSP, Pooled-PLR-Rand and Pooled-PRL-RDP. These algorithms were all previously shown to be inferior to RiskSLIM \citep{ustun2019learning}. These methods first find a pool of sparse continuous solutions using different regularizations of ElasticNet (hence the name ``Pooled Penalized Logistic Regression'' -- Pooled-PLR) and then round the coefficients with different techniques. Details are in Appendix \ref{app:baseline_specifications}.
The best solution is chosen from this pool of integer solutions that obeys the sparsity and box constraints and has the smallest logistic loss.
We also compare with the baseline AutoScore~\cite{xie2020autoscore}. However, on some datasets, the results produced by AutoScore are so poor that they distort the AUC scale, so we show those results only in Appendix~\ref{app:comparison_with_baseline}. As there is no publicly available code for any of \cite{chevaleyre2013rounding, ErtekinRu15, sokolovska2017fused, sokolovska2018provable}, they do not appear in the experiments.
For each dataset, we perform 5-fold cross validation and report training and test AUC.
Appendix~\ref{app:experimental_setups} presents details of the datasets, experimental setup, evaluation metrics, loss values, and computing platform/environment.
More experimental results appear in Appendix~\ref{app:additional_experimental_results}.

\subsection{Solution Quality}
\label{subsec:experiments_solution_quality}

We first evaluate \ourmethod's solution quality.
Figure \ref{fig:solution_quality} shows the training and test AUC on six datasets (results for training loss appear in Appendix \ref{app:additional_experimental_results}).
\textbf{\ourmethod{} (the \textcolor{red}{red} line) outperforms all baselines, consistently obtaining the highest AUC scores on both the training and test sets.}
Notably, our method obtains better results than RiskSLIM, which uses a mathematical solver and is the current state-of-the-art method for scoring systems. This superior performance is due to the use of multipliers, which increases the complexity of the hypothesis space. 
Figure \ref{fig:direct_comparison_with_RiskSLIM} provides a more detailed comparison between \ourmethod{} and RiskSLIM. One may wonder whether running RiskSLIM longer would make this MIP-based method comparable to our \ourmethod{}, since the current running time limit for RiskSLIM is only 15 minutes. We extended RiskSLIM's running time limit up to 1 hour and show the comparison in Appendix~\ref{app:running_riskslim_longer}; \ourmethod{} still outperforms RiskSLIM by a large margin.

\ourmethod{} performs significantly better than the other baselines for two reasons. First, the continuous sparse solutions produced by  ElasticNet are low quality for very sparse models. Second, it is difficult to obtain an exact model size by controlling $\ell_1$ regularization. For example, Pooled-PLR-RD and Pooled-PLR-RDSP do not have results for model size $10$ on the mammo datasets, because no such model size exists in the pooled solutions after rounding.

\begin{figure}[htbp]     
    \centering
    \includegraphics[width=\textwidth]{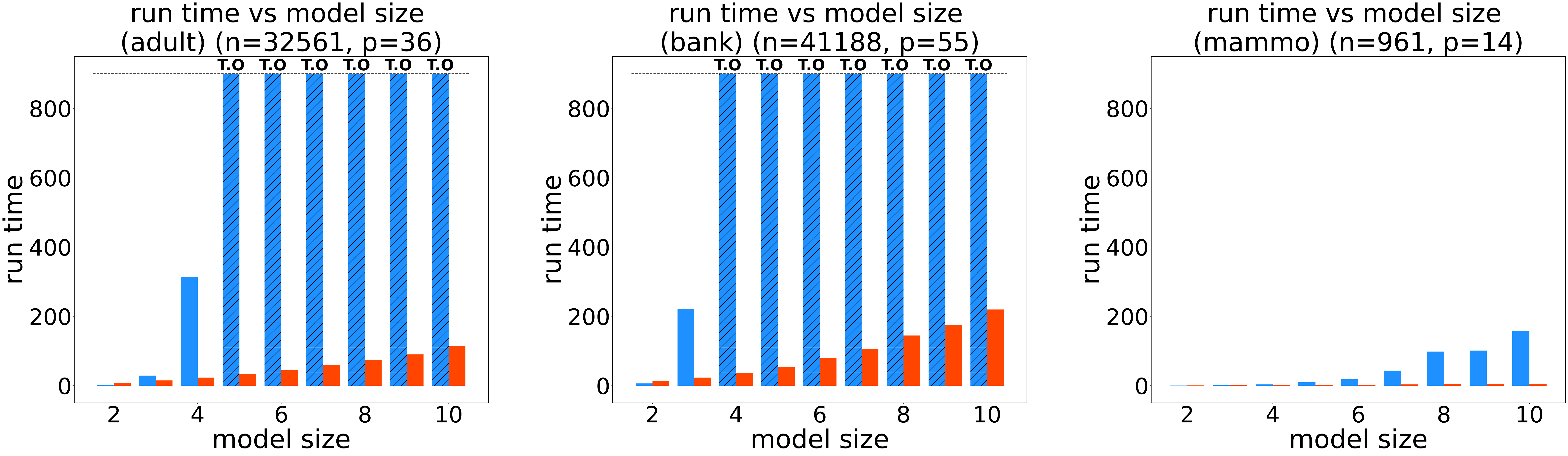}
    \includegraphics[width=\textwidth]{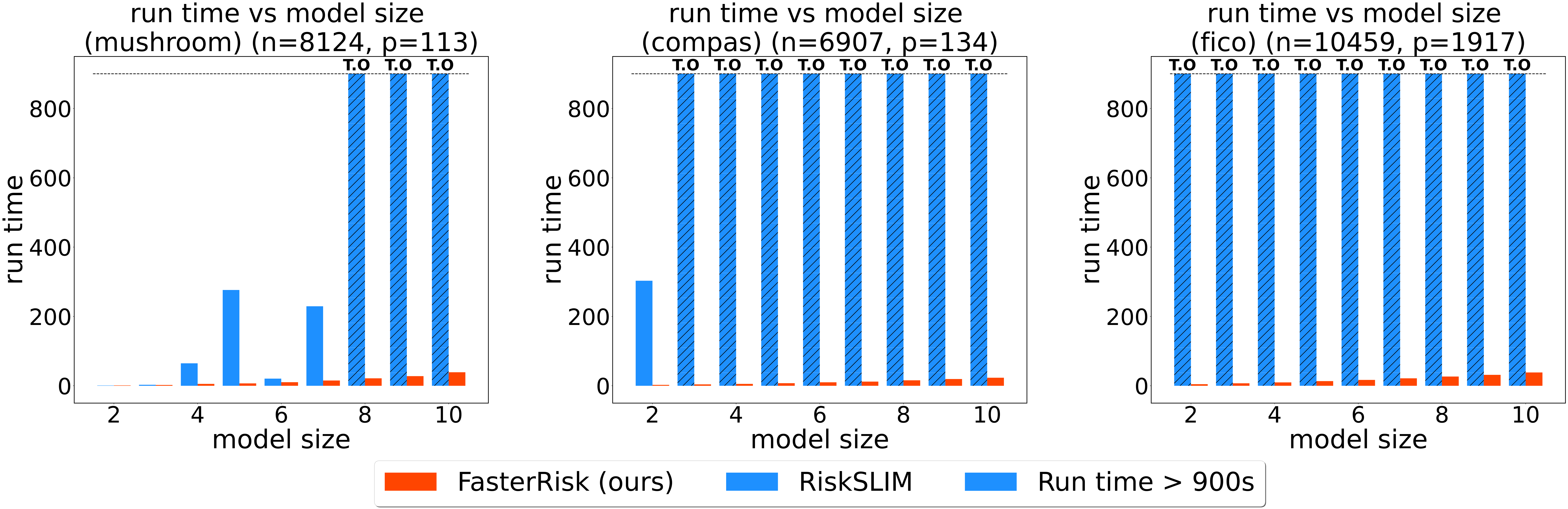}
    
    \caption{Runtime Comparison. Runtime (in seconds) versus model size for our method \ourmethod{} (in \textcolor{red}{red}) and the RiskSLIM (in \textcolor{blue}{blue}). The \textcolor{blue}{shaded blue} bars indicate cases that timed out (T.O.) at 900 seconds. 
    }
    \label{fig:time_comparison}
\end{figure}

\subsection{Runtime Comparison}
\label{subsec:experiments_run_time}
The major drawback of RiskSLIM is its limited scalability. Runtime is important to allow interactive model development and to handle larger datasets.
Figure \ref{fig:time_comparison} shows that \textbf{\ourmethod{} (\textcolor{red}{red} bars) is significantly faster than RiskSLIM (\textcolor{blue}{blue} bars) in general}.
We ran these experiments with a 900 second (15 minute) timeout.
RiskSLIM finishes running on the small dataset mammo, but it times out on the larger datasets, timing out on models larger than 4 features for adult, larger than 3 features for bank, larger than 7 features for mushroom, larger than 2 features for COMPAS, and larger than 1 feature for FICO.
RiskSLIM times out early on COMPAS and FICO datasets, suggesting that the MIP-based method struggles with high-dimensional and highly-correlated features.
Thus, we see that \ourmethod{} tends to be both faster and more accurate than RiskSLIM.

\subsection{Example Scoring Systems}
The main benefit of risk scores is their interpretability. We place a few example risk scores in Table \ref{fig:MoreExampleRiskScore} to allow the reader to judge for themselves. More risk scores examples can be found in Appendix \ref{app:risk_score_models_with_different_sizes}.
Additionally, we provide a pool of solutions for the top 12 models on the bank, mammo, and Netherlands datasets in Appendix \ref{app:examples_from_the_pool}. Prediction performance is generally not the only criteria users consider when deciding to deploy a model. Provided with a pool of solutions that perform equally well, a user can choose the one that best incorporates domain knowledge \citep{XinEtAl2022}. After the pool of models is generated, interacting with the pool is essentially computationally instantaneous.
Finally, we can reduce some models to relatively prime coefficients or transform some features for better interpretability. Examples of such transformations are given in Appendix~\ref{app:model_reduction}.
\begin{table}[htbp]
\small
\centering
\begin{subtable}[t]{0.48\linewidth}
{
\begin{tabular}{|llr|ll|}
\hline
1. &no high school diploma & -4 points &   & ... \\
2. &high school diploma only & -2 points  & + & ... \\
3. &age 22 to 29 & -2 points   & + & ... \\ 
4. &any capital gains & 3 points & + & ...\\
5. &married & 4 points & + & ...\\\hline
     &     & \multicolumn{1}{l|}{\textbf{SCORE}} & = &     \\ \hline
\end{tabular}
}
{
\risktable{}
\begin{tabular}{|r|c|c|c|c|c|}
\hline \rowcolor{scorecolor}\scorelabel{} & $<$-4 & -3 & -2 & -1 & 0 \\
\hline \rowcolor{riskcolor}\risklabel{}  & $<$1.3\% & 2.4\% & 4.4\% & 7.8\% & 13.6\%\\ 
\hline
\hline
\rowcolor{scorecolor}\scorelabel{} & 1 & 2 & 3 & 4 & 7  \\\hline
\rowcolor{riskcolor}\risklabel{} & 22.5\% & 35.0\% & 50.5\% & 65.0\% & 92.2\%  \\ 
\hline
\end{tabular}
}
\caption{\ourmethod{} models for the adult dataset, predicting salary$>50$K.
}
\end{subtable}
\hfill
\begin{subtable}[t]{0.45\linewidth}
{
\begin{tabular}{|llr|ll|}
\hline
1. &odor$=$almond & -5 points &   & ... \\
2. &odor$=$anise & -5 points  & + & ... \\
3. &odor$=$none & -5 points   & + & ... \\ 
4. &odor$=$foul & 5 points & + & ...\\
5. &gill size$=$broad & -3 points & + & ...\\\hline
     &     & \multicolumn{1}{l|}{\textbf{SCORE}} & = &     \\ \hline
\end{tabular}
}
{
\risktable{}
\begin{tabular}{|r|c|c|c|c|}
\hline \rowcolor{scorecolor}\scorelabel{} & -8 & -5 & -3 & $\geq$2\\
\hline \rowcolor{riskcolor}\risklabel{} & 1.62\% & 26.4\% & 73.6\% & >99.8\% \\ 
\hline
 \multicolumn{5}{c}{} \\
 \multicolumn{5}{c}{}  \\[1.5mm] 

\end{tabular}
}

\caption{\ourmethod{} model for the mushroom dataset, predicting whether a mushroom is poisonous.
}
\end{subtable}
\caption{Example \ourmethod{} models}
\label{fig:MoreExampleRiskScore}
\end{table}

\section{Conclusion} 
\ourmethod{} produces a collection of high-quality risk scores within minutes. Its performance owes to three key ideas: a new algorithm for sparsity- and box-constrained continuous models, using a pool of diverse solutions, and the use of the star ray search, which leverages multipliers and a new sequential rounding technique. \ourmethod{} is suitable for high-stakes decisions, and permits domain experts a collection of interpretable models to choose from.


\section*{Code Availability}
Implementations of \ourmethod{} discussed in this paper are available at \url{https://github.com/jiachangliu/FasterRisk}. 

\section*{Acknowledgements}
The authors acknowledge funding from the National Science Foundation under grants IIS-2147061 and IIS-2130250, National Institute on Drug Abuse under grant R01 DA054994, Department of Energy under grants DE-SC0021358 and DE-SC0023194, and National Research Traineeship Program under NSF grants DGE-2022040 and CCF-1934964.
We acknowledge the support of the Natural Sciences and Engineering Research Council of Canada (NSERC).
Nous remercions le Conseil de recherches en sciences naturelles et en génie du Canada (CRSNG) de son soutien.

\bibliography{bibliography}
\bibliographystyle{plainnat}



\appendix

\newpage
\addcontentsline{toc}{section}{Appendix} 
\part{Appendix to FasterRisk: Fast and Accurate Interpretable Risk Scores} 
\parttoc 

\newpage
\section{Additional Algorithms}
\label{app:additional_algorithmic_charts}

\subsection{Expand Support by One More Feature}
\begin{algorithm}[ht]
\caption{\ourmethodExpand}
\textbf{Input:} Dataset $\mathcal{D}$, coefficient constraint $C$, and beam search size $B$, current coefficient vector $(\bw, w_0)$, and a set of found supports $\mathcal{F}$.\\
\textbf{Output:} a collection of solutions $\mathcal{W} = \{(\bu^t, u^t_0) \}$ with $\lVert \bu^t \rVert_0 = \Vert{\bw}_0 + 1, \Vert{\bu^t}_{\infty} \leq C$ for $\forall t$. All of these solutions include the support of $(\bw, w_0)$ plus one more feature. None of the solutions have the same support as any element of $\mathcal{F}$, meaning we do not discover the same support set multiple times. We also output the updated $\mathcal{F}$.
\begin{algorithmic}[1]
    \State Let $\mathcal{S}^c \leftarrow \{j \mid w_j = 0\}$ \Comment{Non-support of the given solution}
    \State $\bw' \leftarrow \mathbf{0}$
    \For{$p = 1, ..., 10$} \Comment{10 steps of parallel coordinate descent with projection}
        \State $w'_{j} \leftarrow w'_j - \nabla_{j} L(\bw + w'_j \be_j, w_0) / l_j$ for $\forall j \in \mathcal{S}^c$ \Comment{$l_j$ is the smallest Lipschitz constant on coordinate $j$ with $L(\bw + w_j' \be_j + d \be_j) - L(\bw + w_j' \be_j) \leq l_j d$ for any $d \in \mathbb{R}$.}
        \State $w'_{j} \leftarrow \textrm{Clip}(w'_j, -C, C)$ for $\forall j \in \mathcal{S}^c$ \Comment{$\textrm{Clip}(x, a, b) = \max(a, \min(x, b))$}
    \EndFor
    \State Pick the $B$ coords ($j$'s) in $\mathcal{S}^c$ with smallest logistic loss $L(\mathcal{D},\bw + \be_j w'_j, w_0)$, call this set $\mathcal{J}'$. \Comment{We will use these supports, which include the support of $\bw$ plus one more.}
    \State $\mathcal{W} \leftarrow \emptyset$
    \For{$j \in \mathcal{J}'$} \Comment{Optimize on the top $B$ coordinates}
        \State If $\textrm{supp}(\bw+\be_j w'_j) \in \mathcal{F}$, continue. \Comment{We've already seen this support, so skip.}
        \State $\mathcal{F} \leftarrow \mathcal{F} \cup \{supp(\bw + \be_j w'_j)\}$. \Comment{Add new support to $\mathcal{F}$.}
        \State $(\bw'', w''_0) \in \argmin_{\bu, u_0} L(\mathcal{D},\bu, u_0)$ with $\textrm{supp}(\bu) = \textrm{supp}(\bw + \be_j w'_j)$ and $\Vert{\bu}_{\infty} \leq C$. \Comment{Fine tune on the newly expanded support using 100$\times|$support$|$ coordinate descent steps and clip operation, or until convergence; use $(\bw+\be_j w'_j, w_0)$ as a warm start for computational efficiency}
        \State $\mathcal{W} \leftarrow \mathcal{W} \cup \{(\bw'', w''_0)\}$
    \EndFor
    \State Return $\mathcal{W}$ and $\mathcal{F}$.
\end{algorithmic}
\end{algorithm}

\subsection{Collect Sparse Diverse Pool}
\begin{algorithm}[ht]
\caption{\ourmethodPool{}}
\label{alg:collectSparseLevelSet}
\textbf{Input:} Dataset $\mathcal{D}$, a coefficient vector $(\bw, w_0)$, an optimality gap tolerance $\epsilon$, and the number of attempts $T$.\\
\textbf{Output:} a set $\mathcal{S}$ containing good sparse continuous solutions.
\begin{algorithmic}[1]
    \State $\mathcal{S} \leftarrow \{ (\bw, w_0) \}$ \Comment{Initialize the sparse level set}
    \State $L^* \leftarrow L(\mathcal{D}, \bw, w_0)$ \Comment{Get the current loss}
    \State $\mathcal{J} \leftarrow \{j \mid w_j \neq 0\}$ \Comment{Get the current support}
    \For{ $j_{-} \in \mathcal{J}$ } \Comment{Remove a feature in the support}
        \State Pick the $T$ coords ($j_+$'s) in $[1, ..., p] \setminus \mathcal{J}$ with the biggest magnitudes of partial derivative $\nabla_{j^+}L(\mathcal{D},\bw - \be_{j_-} w'_j, w_0)$, call this set $\mathcal{J}_+$.
        \For{ $j_{+} \in \mathcal{J}_+$ } \Comment{Put a new feature into the support}
            \State $(\bw'', w_0'') \in \argmin_{\bw', w'_0} L (\mathcal{D}, \bw', w'_0)$ where $w'_j = 0$ if $j \in [1,...,p] \setminus \mathcal{J} \cup \{j_{-}\} $ \Comment{Fit on the new support. Problem is convex. We use coordinate descent for this.}
            \State $L_{\text{swap}} \leftarrow L (\mathcal{D}, \bw'', w''_0)$\Comment{Loss of newly formed and optimized coefficient vector}
            \If{$L_{\text{swap}} \leq (1+\epsilon) L^*$} \Comment{If its loss is good enough, include it in $\mathcal{S}$}
                \State $\mathcal{S} \leftarrow \mathcal{S} \cup \{(\bw'', w''_0) \}$ \Comment{Expand the sparse level set if loss is within the gap}
            \EndIf
        \EndFor
    \EndFor
    \State \textbf{return} $\mathcal{S}$
\end{algorithmic}
\end{algorithm}

\newpage
\subsection{Round Continuous Coefficients to Integers}
\begin{algorithm}[ht]
\caption{\ourmethodRound} \label{alg:rounding_method}
\textbf{Input:} Dataset $\mathcal{D} = {(\bx_i, y_i)}_{i=1}^n$, a sparse continuous solution $(\bw, w_0)$, where $\bw \in \mathbb{R}^p, w_0 \in \mathbb{R}$.

\textbf{Output:} an integer-valued solution $(\bw^+, w^+_0)$, where $\bw^+ \in \mathbb{Z}^p, w^+_0 \in \mathbb{Z}$.

\begin{algorithmic}[1]
    \State $\bw^c \leftarrow [w_0, \bw]$, and $\bx_i \leftarrow [1, \bx_i]$ for $\forall i$. \Comment{Concatenate to incorporate the intercept}
    \State $\bw^+ \leftarrow \bw^c$
    \State $\mathcal{J} \leftarrow \{j: \lceil w^+_j \rceil \neq \lfloor w^+_j \rfloor\}$ \Comment{Feature indices with fractional coefficients}

    \State $\bGamma \leftarrow [\floor{\bw^+}; \floor{\bw^+}; ...; \floor{\bw^+}]^T$ \Comment{n rows of $\floor{\bw^+}$}
    \State Define a new matrix $\bZ$ with entries $Z_{ij} = y_i x_{ij}$
    \State $\bGamma \leftarrow \bGamma + \mathbf{1}_{\bZ \leq 0}$. \Comment{See Theorem\ref{theorem:upperBound_auxiliaryLossRounding} and Second Inequality (Lipschitz continuity). This line performs the calculation: $\gamma_{ij} = \floor{w_j}$ if $y_i x_{ij} > 0$ and $\gamma_{ij} = \ceil{w_{j}}$ otherwise.}
    \For{$i$ = 1 to $n$}
        \State $l_i \leftarrow 1/(1+\exp(y_i \sum_{j=1}^p x_{ij} \Gamma_{ij}))$ \Comment{Chosen so we can calculate local Lipschitz constant}
    \EndFor

    \While{$\mathcal{J} \neq \emptyset$}\Comment{We iteratively round more coeffs in $\bw^+$ until fractional coeffs are gone.}
        \For{$j \in \mathcal{J}$} \Comment{Try rounding both up and down for each $j$}
            \State $\bw^{+j, up} \leftarrow (w^+_1, ..., \lceil w^+_j\rceil, ... w^+_{p+1})^T$, $\quad \quad \;\bw^{+j, down} \leftarrow (w^+_1, ..., \lfloor w^+_j\rfloor, ... w^+_{p+1})^T$
            \State $U^{j, up} \leftarrow \sum_{i=1}^n (l_i \bx_i^T (\bw^{+j, up} - \bw^c))^2$, $\; \; \, \quad U^{j, down} \leftarrow \sum_{i=1}^n (l_i \bx_i^T (\bw^{+j, down} - \bw^c))^2$
        \EndFor \\
        \Comment{Now find the best $j$ and whether to round up or down.}
        \State $U^{up} \leftarrow \min_{j\in \mathcal{J}} U^{j, up}$, $\quad U^{down} \leftarrow \min_{j\in \mathcal{J}} U^{j, down}$
        \If{$U^{up} \leq U^{down}$}
            \State $j' \leftarrow \argmin_{j\in \mathcal{J}} U^{j, up}$, $\mathcal{J} \leftarrow \mathcal{J} \setminus \{j'\}$
            \State $w^+_{j'} \leftarrow \lceil w^+_{j'}\rceil$ \Comment{Round up}
        \Else
            \State $j' \leftarrow \argmin_{j\in \mathcal{J}} U^{j, down}$, $\mathcal{J} \leftarrow \mathcal{J} \setminus \{j'\}$
            \State $w^+_{j'} \leftarrow \lfloor w^+_{j'}\rfloor$ \Comment{Round down}
        \EndIf
    \EndWhile
    \State $w^+_0\leftarrow w^+[1], \bw^+ \leftarrow   \bw^+[2:\textit{end}]$ \Comment{Separate the intercept and the coefficients}
    \State Return $(\bw^+, w^+_0)$
\end{algorithmic}
\end{algorithm}

\section{Comments on Proof of Chevaleyre \textit{et al.}}
\label{app:ChevaleyreWrongProof}

Chevaleyre \textit{et al.}~\cite{chevaleyre2013rounding} proposed Greedy Rounding, where coefficients are rounded sequentially.
While this technique provides theoretical guarantees for greedy rounding for the hinge loss, we identified a serious flaw in their argument, rendering the bounds incorrect.
We elaborate on this matter in this appendix.

The flaw is in the proof of Lemma 7.
The proof essentially shows that for each sample $i$, there is at least one $a$ (from the set $\{0,1\}$) such that the inequality holds. However, the same $a$ that works for sample $i=3$ is not guaranteed to work for sample $i=5$ for the inequality. It is not clear whether there exists one $a$ that make all inequalities (for all samples $i$ in $[1, ..., m]$) hold at the time.

To paraphrase, for each sample $i$, the proof shows that we can pick a set of $a$ (either $\{0\}$, $\{1\}$, or $\{0, 1\}$) so that the inequality holds individually. However, we can not rule out the case that intersection of these individual sets is empty.

Without this extra argument, there is a gap between the statement of Lemma 7 and the proof of Lemma 7. Then, the bound for the greedy algorithm in Theorem 8 will not hold in the paper.

\section{Theoretical Upper Bound for the Rounding Method, Algorithm \ref{alg:rounding_method}}
\label{app:theoretical_upper_bound_derivation}

The following theorem (as also shown in the main paper) states that we can provide an upper bound on the difference of the total loss between the integer solution $\bw^+$ given by Algorithm \ref{alg:rounding_method} and the real-valued solution $\bw$.

\textbf{Theorem} \ref{theorem:upperBound_auxiliaryLossRounding} (Loss incurred from rounding)
Let $\bw$ be the real-valued coefficients for the logistic regression model with objective function $L(\bw) = \sum_{i=1}^n \log(1+\exp(-y_i\bx_i^T \bw))$. Let $\bw^+$ be the integer-valued coefficients returned by the Auxiliaryloss Rounding method, Algorithm \ref{alg:rounding_method}. Furthermore, let $u_j=w_j - \lfloor w_j \rfloor$. Let $l_i = 1/(1+\exp(y_i \bx_i^T \bgamma_i))$ with $\gamma_{ij} = \floor{w_j}$ if $y_i x_{ij} > 0$ and $\gamma_{ij} = \ceil{w_{j}}$ otherwise. Then, we have an upper bound on the difference between the loss $L(\bw)$ and the loss $L(\bw^+)$:
\begin{equation}
    L(\bw^+) - L(\bw) \leq \sqrt{n \sum_{i=1}^n \sum_{j=1}^p l^2_i x_{ij}^2 u_j (1-u_j)}.
\end{equation}

To prove Theorem \ref{theorem:upperBound_auxiliaryLossRounding}, we need to use the following Lemma \ref{lemma:upperBound_auxiliaryLossRounding_per_step}, which states that during each successive step of rounding a real-valued coefficient to the integer value, the deviation can be characterized and bounded by the data features and the real-valued coefficient.
\begin{lemma} \label{lemma:upperBound_auxiliaryLossRounding_per_step}
Suppose we have rounded the first $k-1$ real-valued coefficients to integers. Then for the $k$-th real-valued coefficient, if we set $w^+_k = \argmin_{v \in \{\floor{w_k}, \ceil{w_k}\}} \sum_{i=1}^n l_i^2 (\sum_{j=1}^{k-1}x_{ij}(w^+_j -w_j) + x_{ik}(v - w_k))^2$, then we have
\begin{equation}
    \sum_{i=1}^n l_i^2 \left(\sum_{j=1}^k x_{ij}(w^+_j - w_j)\right)^2 \leq \sum_{i=1}^n l_i^2 \left(\sum_{j=1}^{k-1} x_{ij}(w^+_j - w_j)\right)^2 + \sum_{i=1}^n l_i^2 x_{ik}^2 (1-u_k) u_k
\end{equation}
where $u_k = w_k - \lfloor w_k\rfloor$.
\end{lemma}

\begin{proof}
Let $z_k$ be a binomial random variable so that $z_k = 1$ with probability $u_j$ and $z_k=0$ with probability $1-u_k$. For notational convenience, let us define the function $f(v) := \sum_{i=1}^n l_i^2 \left(\sum_{j=1}^{k-1} x_{ij} (w^+_j -w_j)+ x_{ik}( v - w_k)\right)^2$. Then $f(\floor{w_k} + z_k)$ is a random variable, and the input to function $f(\cdot)$, which is $\floor{w_k} + z_k$, takes on values either $\floor{w_k}$ or $\ceil{w_k}$.

The expectation of this random variable is
\begin{align*}
    & \mathbb{E}_{z_k} \left[ f(\floor{w_k} + z_k) \right] \\
    =& \mathbb{E}_{z_k} \left[\sum_{i=1}^n l_i^2 \left(\sum_{j=1}^{k-1} x_{ij} (w^+_j -w_j)+ x_{ik}(\floor{ w_k} + z_k - w_k)\right)^2 \right] \\
    =& \sum_{i=1}^n l_i^2 \; \mathbb{E}_{z_k} \left[ \left(\sum_{j=1}^{k-1} x_{ij} (w^+_j -w_j)+ x_{ik}( \floor{ w_k} + z_k - w_k)\right)^2 \right] \tag*{ \textit{  \# move  $\mathbb{E(\cdot)}$  inside the $\sum(\cdot)$ } } \\
    =& \sum_{i=1}^n l_i^2 \; \mathbb{E}_{z_k} \left[ \left(\sum_{j=1}^{k-1} x_{ij} (w^+_j -w_j)+ x_{ik}( z_k - u_k)\right)^2 \right] \tag*{\textit{\# substitute with } $u_k = w_k - \floor{w_k}$ }\\
    =& \sum_{i=1}^n l_i^2 \; \left[ \left(\sum_{j=1}^{k-1} x_{ij} (w^+_j -w_j) \right)^2  + 2x_{ik} \left (\sum_{j=1}^{k-1} x_{ij}(w^+_j - w_j) \right) \mathbb{E}_{z_k} \left[z_k - u_k \right] \right. \\
    & \quad \quad \quad \left. + x_{ik}^2 \mathbb{E}_{z_k} \left[ (z_k - u_k)^2 \right] \vphantom{\sum_{j=1}^{k-1}} \right]. 
    \tag*{ \textit{\# expand the square  term } }
\end{align*}
Notice that because $\mathbb{P}(z_k=1)=u_k$,  $\mathbb{P}(z_k=0)=1-u_k$, we have
\begin{align*}
    \mathbb{E}_{z_k} \left[ z_k - u_k\right] = (1-u_k) u_k + (0-u_k) (1-u_k) = 0
\end{align*}
and
\begin{align*}
    \mathbb{E}_{z_k} \left[ (z_k - u_k)^2\right] = (1-u_k)^2 u_k + (0-u_k)^2 (1-u_k) = u_k (1-u_k). \tag*{ \textit{\# similar as above}}
\end{align*}
Therefore, we have
\begin{align*}
    &\mathbb{E}_{z_k} \left[ f(\floor{w_k} + z_k) \right] \\
    =& \sum_{i=1}^n l_i^2 \; \left[ \left(\sum_{j=1}^{k-1} x_{ij} (w^+_j -w_j) \right)^2  \right. \\
    & \quad \quad \quad \quad \quad \left. + 2x_{ik} \left (\sum_{j=1}^{k-1} x_{ij}(w^+_j - w_j) \right) \mathbb{E}_{z_k} \left[z_k - u_k \right] + x_{ik}^2\mathbb{E}_{z_k} \left[(z_k - u_k)^2 \right] \vphantom{\sum_{j=1}^{k-1}} \right] \\
    =& \sum_{i=1}^n l_i^2 \; \left[ \left(\sum_{j=1}^{k-1} x_{ij} (w^+_j -w_j) \right)^2  + x_{ik}^2 u_k (1-u_k) \right] \tag*{\textit{\# plug in the two expectations above }}\\
    =& \sum_{i=1}^n  l_i^2 \; \left(\sum_{j=1}^{k-1} x_{ij} (w^+_j -w_j) \right)^2 + \sum_{i=1}^n l_i^2 \; x_{ik}^2 u_k (1-u_k). \tag*{ \textit{\# split into two summation terms} }
\end{align*}

Since the expectation of $f(\floor{w_k} + z_k )$ is equal to $\sum_{i=1}^n  l_i^2 \; \left(\sum_{j=1}^{k-1} x_{ij} (w^+_j -w_j) \right)^2  + \sum_{i=1}^n l_i^2 \; x_{ik}^2 u_k (1-u_k)$, there exists a $z_k' \in \{0, 1\}$ such that
\begin{align}
\label{eq:per_step_inequality}
    f(\floor{w_k} + z_k') \leq \sum_{i=1}^n  l_i^2 \; \left(\sum_{j=1}^{k-1}  x_{ij} (w^+_j -w_j) \right)^2  + \sum_{i=1}^n l_i^2 \; x_{ik}^2 u_k (1-u_k).
\end{align}
Note that $\floor{w_k} + z_k'$ is the minimizer of $f(\cdot)$ because the other input value $\floor{w_k} + 1 - z_k'$ will take the value $f(\floor{w_k} + 1 - z_k')$, which is greater than or equal to the expectation $\mathbb{E}_{z_k} [f(\floor{w_k} + z_k)]$.

If we round $w_k$ to an integer by setting $w^+_k = \floor{w_k} + z_k'$, then $w^+_k =\argmin_{v \in \{\floor{w_k}, \ceil{w_k}\}} f(v)$.
We now have:
\begin{align*}
    \sum_{i=1}^n l_i^2 \left(\sum_{j=1}^k x_{ij}(w^+_j - w_j)\right)^2 =& \min_{v \in \{\floor{w_k}, \ceil{w_k} \}} f(v) \tag*{ \textit{\# definition of $w_k^+ $} and $f(\cdot)$}\\
    =&  \min_{c \in \{0, 1\}} f(\floor{w_k} + c) \tag*{ \textit{\# substitute $v=\floor{w_k} + c$}}\\
    =& f(\floor{w_k} + z'_k) \tag*{ \textit{\# $\floor{w_k} + z'_k$ is the minimizer of $f(\cdot)$}}\\
    \leq & \sum_{i=1}^n l_i^2 \left(\sum_{j=1}^{k-1} x_{ij}(w^+_j - w_j)\right)^2 + \sum_{i=1}^n l_i^2 x_{ik}^2 (1-u_k) u_k, \tag*{\textit{\# Inequality \ref{eq:per_step_inequality}}}
\end{align*}
thus completing our proof.

\end{proof}

Now we can use Lemma \ref{lemma:upperBound_auxiliaryLossRounding_per_step} to prove Theorem \ref{theorem:upperBound_auxiliaryLossRounding}.

\noindent\textit{Proof of Theorem \ref{theorem:upperBound_auxiliaryLossRounding}.}
For simplicity, let us first consider the case where we round coefficients sequentially from $w^+_1$ to $w^+_p$. We claim that if at each step $r$, we round $w^+_r = \argmin_{v \in \{\floor{w_r}, \ceil{w_r}\}} \sum_{i=1}^n l_i^2 \left(\sum_{j=1}^{l-1}x_{ij}(w^+_j -w_j) + x_{ir}(v - w_r)\right)^2$, then for $\forall k \in [1, ..., p]$
\begin{equation}
    \sum_{i=1}^n l_i^2 \left(\sum_{j=1}^k x_{ij}(w^+_j -w_j)\right)^2 \leq \sum_{i=1}^n \sum_{j=1}^k l_i^2 x_{ij}^2 u_j (1-u_j).
\label{ineq:chain_bound}
\end{equation}
We prove this by the principle of induction. Suppose for step $k-1$, we have
\begin{align*}
    \sum_{i=1}^n l_i^2 \left(\sum_{j=1}^{k-1} x_{ij}(w^+_j -w_j)\right)^2 \leq \sum_{i=1}^n \sum_{j=1}^{k-1} l_i^2 x_{ij}^2 u_j (1-u_j).
\end{align*}
Then, according to Lemma \ref{lemma:upperBound_auxiliaryLossRounding_per_step} and the previous line, we have
\begin{align*}
    \sum_{i=1}^n l_i^2 \left(\sum_{j=1}^{k} x_{ij}(w^+_j -w_j)\right)^2 &\leq
    \sum_{i=1}^n \sum_{j=1}^{k-1} l_i^2 x_{ij}^2 u_j (1-u_j) + \sum_{i=1}^n l_i^2 x_{ik}^2 u_k (1-u_k) \tag*{\textit{\# Lemma \ref{lemma:upperBound_auxiliaryLossRounding_per_step}}}\\
    &= \sum_{i=1}^n \sum_{j=1}^{k} l_i^2 x_{ij}^2 u_j (1-u_j). \tag*{\textit{\# use a single sum $\sum_{i=1}^n (\cdot)$}}
\end{align*}
For the base step $k=1$, Lemma \ref{lemma:upperBound_auxiliaryLossRounding_per_step} also implies that
\begin{align*}
    \sum_{i=1}^n l_i^2 (x_{i1}(w^+_1 - w_1))^2 \leq \sum_{i=1}^n l_i^2 x_{i1}^2 u_1 (1-u_1).
\end{align*}
Thus, Inequality (\ref{ineq:chain_bound}) works for all $k$.
If we let $k=p$, we have
\begin{align}
    \sum_{i=1}^n l_i^2 \left( \sum_{j=1}^p x_{ij}(w^+_j -w_j) \right)^2 \leq \sum_{i=1}^n \sum_{j=1}^p l_i^2 x_{ij}^2 u_j (1-u_j).
\label{ineq:chain_bound_last}
\end{align}

Also, notice that this inequality holds for sequential rounding of any permutation of the feature indices $[1, ..., p]$, and the rounding order of the \ourmethodRound{} method is one specific feature order. Therefore, the Inequality (\ref{ineq:chain_bound_last}) works for the \ourmethodRound{} method as well.

Lastly, we use Inequality \ref{ineq:chain_bound_last} to derive an upper bound on the logistic loss of the \ourmethodRound{} method.
Recall that our objective is:
\begin{equation}
    L(\bw) = \sum_{i=1}^n \log(1+\exp(-y_i\bx_i^T \bw)).
\end{equation}
The loss difference between the rounded solution and the real-valued solution can be bounded as follows:
\begin{align}
    L(\bw^+) - L(\bw^*) &\leq \sum_{i=1}^n \left[\log(1+\exp(-y_i \bx_i^T \bw^+)) - \log(1+\exp(-y_i \bx_i^T \bw))\right] \nonumber \\
    &\leq \sum_{i=1}^n \vert{ l_i (y_i \bx_i^T \bw^+ - y_i \bx_i^T \bw ) } \tag*{\textit{\# Lipschitz continuity, see details below }}\\
    &= \sum_{i=1}^n \vert{ l_i y_i \bx_i^T (\bw^+-\bw) } \tag*{\textit{\# pull out common factor }} \nonumber \\
    &= \sum_{i=1}^n \vert{ l_i \bx_i^T (\bw^+-\bw) } \tag*{\textit{\# since $\vert{y_i}=1$ }} \nonumber \\
    &\leq \sum_{i=1}^n \sqrt{l_i^2 \left(\sum_{j=1}^p x_{ij}(w^+_j - w_j)\right)^2} \tag*{\textit{\# rewrite $\vert{\cdot}$ in terms of $\sqrt{\cdot}$ } } \nonumber \\
    &\leq \sqrt{n \sum_{i=1}^n l_i^2 \left(\sum_{j=1}^p x_{ij}(w^+_j - w_j)\right)^2}  \tag*{\textit{\# Jensen's Inequality, see details below}}
\end{align}
There are two inequalities we need to elaborate in details, the second and the last inequalities (Lipschitz continuity and Jensen's Inequaltiy).

\textit{Second Inequality (Lipschitz continuity):}

The second inequality holds because the logistic loss $g(a)=\log(1+\exp(-a))$ is Lipschitz continuous. If the Lipschitz constant is $l$, then we have $|g(a)-g(b)| \leq l \; |a-b|$. We now explain how we derive the Lipschitz constant $l_i = 1/(1+\exp(y_i \bx_i^T \bgamma_i))$ with $\gamma_{ij} = \floor{w_j}$ if $y_i x_{ij} > 0$ and $\gamma_{ij} = \ceil{w_{j}}$ as stated in Theorem \ref{theorem:upperBound_auxiliaryLossRounding}.

Since the logistic loss function $g(\cdot)$ is differentiable, the smallest Lipschitz constant of the function $g(\cdot)$ is $l_{\min} (g) = \sup_{a \in \textrm{Domain}(g)}  |g'(a)|$. To see this, by the definition of the Lipschitz constant, we have $\frac{|g(a) - g(b)|}{|a-b|} \leq l$. If we take the limit $b \rightarrow a$, the inequality still holds, $\lim_{b \rightarrow a} \frac{|g(a) - g(b)|}{|a-b|} \leq l$. The left hand side converges to the absolute value of the derivative of $g(\cdot)$ at $a$. Therefore, we have $|g'(a)| \leq l$. Since this works for all $a$, and we want to find the smallest Lipschitz value, we have $l_{\min} (g) = \sup_{a \in \textrm{Domain}(g)}  |g'(a)|$.

For the logistic loss $g(a) = \log(1+e^{-a})$, the absolute value of the derivative is $ |g'(a)| = \frac{1}{1+e^a}$. Thus, if $a$ is lower-bounded so that $a \geq a_1$, the smallest Lipschitz constant of the logistic loss is $l_{\min} (g) = \frac{1}{1+e^{a_1}}$.

We can apply this fact to calculate a smaller Lipschitz constant for each sample's term. If $\gamma_{ij} := \floor{w_j}$ if $y_i x_{ij} > 0$ and $\gamma_{ij} := \ceil{w_{j}}$ otherwise, then
\begin{align*}
    y_i \bx_i \bw^+ \geq y_i \bx_i^T \bgamma_i, \text{ and } |g'(y_i \bx_i \bw^+)| \leq 1/(1+\exp(y_i \bx_i^T \bgamma_i)).
\end{align*}
Therefore, $l_i=1/(1+\exp(y_i \bx_i^T \bgamma_i))$ is a valid Lipschitz constant for the $i$-th sample.

\textit{Last Inequality (Jensen's inequality):}

Jensen's Inequality states that $\mathbb{E}_z [\phi(g(z))] \geq \phi(\mathbb{E}_z [g(z)])$ for any convex function $\phi(\cdot)$. For this specific problem, let $\phi(b) = - \sqrt{b}$ and let $g(z) = l_i^2 (\sum_{j=1}^p x_{ij} (w^+_j - w_j))^2$ for a particular $i$ with probability $\frac{1}{n}$. Then, we have
\begin{align*}
    \sum_{i=1}^n \sqrt{l_i^2 \left(\sum_{j=1}^p x_{ij}(w^+_j - w_j)\right)^2} &= n \sum_{i=1}^n \frac{1}{n} \sqrt{l_i^2 \left(\sum_{j=1}^p x_{ij}(w^+_j - w_j)\right)^2} \tag*{\textit{\# multiply and divide by $n$}}\\
    &= -n \mathbb{E}_z [\phi(g(z))] \tag*{\textit{\# definition of $\phi(\cdot)$, $g(\cdot)$}, and $\mathbb{E}(\cdot)$}\\
    &\leq -n \phi(\mathbb{E}_z [g(z)]) \tag*{\textit{\# Jensen's Inequality}} \\
    &= n \sqrt{\frac{1}{n} \sum_{i=1}^n l_i^2 \left(\sum_{j=1}^p x_{ij} (w^+_j - w_j)\right)^2} \tag*{\textit{\# write out $\phi(\cdot)$, $g(\cdot)$}, and $\mathbb{E}(\cdot)$ explicitly} \\
    &= \sqrt{n \sum_{i=1}^n l_i^2 \left(\sum_{j=1}^p x_{ij} (w^+_j - w_j)\right)^2}. \tag*{\textit{\# move $n$ inside $\sqrt{\cdot}$}}
\end{align*}
Therefore, using Inequality \ref{ineq:chain_bound_last}, we can now bound the loss difference between the rounded solution and the real-valued solution as stated in Theorem \ref{theorem:upperBound_auxiliaryLossRounding}:
\begin{align*}
    L(\bw^+) - L(\bw) &\leq \sqrt{n \sum_{i=1}^n l_i^2 \left(\sum_{j=1}^p x_{ij}(w^+_j - w_j)\right)^2} \\
    &\leq \sqrt{n \sum_{i=1}^n \sum_{j=1}^p l_i^2 x_{ij}^2 u_j (1-u_j)}.
\end{align*} \qed 

\section{Experimental Setup}
\label{app:experimental_setups}

\subsection{Dataset Information}
\label{app:dataset_information}
The dataset names, data source, number of samples and features, and the classification tasks can be found in Table \ref{tab:data_info}. The datasets with results shown in the main paper (adult, bank, breastcancer, mammo, mushroom, spambase) were directly downloaded from this link: \url{https://github.com/ustunb/risk-slim/tree/master/examples/data}.
The COMPAS dataset can be downloaded from this link: \url{https://github.com/propublica/compas-analysis/blob/master/compas-scores-two-years.csv}.
The FICO dataset can be requested and downloaded from this website: \url{https://community.fico.com/s/explainable-machine-learning-challenge}.
The Netherlands dataset is available through Data Archiving and Networked Services \url{https://easy.dans.knaw.nl/ui/datasets/id/easy-dataset:78692}.

For our experiments on the COMPAS, FICO, and Netherlands datasets, we convert the continuous features into a set of highly correlated dummy variables, with all entries equal to 1 or 0. By conducting experiments on these three datasets, we can test how well \ourmethod{} works for highly correlated features. We use the preprocessing steps as explained in Section C2 of \cite{LiuEtAl2022}. We list the key preprocessing steps below.

\textbf{COMPAS:} In addition to the label \textit{``two\_year\_recid''}, we use features \textit{``sex'', ``age'', ``juv\_fel\_count'', ``juv\_misd\_count'', ``juv\_other\_count'', ``priors\_count''}, and \textit{``c\_charge\_degree''}.

\textbf{FICO:} All continuous features are used.

\textbf{Netherlands:} In addition to the label \textit{``recidivism\_in\_4y''}, we use features \textit{``sex'', ``country of birth'', ``log \# of previous penal cases'', ``11-20 previous case'', and ``$>$20 previous case'', ``age in years'', ``age at first penal case''}, and \textit{``offence type''}.

For each continuous variable $x_{\cdot,j}$, it is converted into a set of highly correlated dummy variables $\tilde{x}_{\cdot,j,\theta} = \bm{1}_{[x_{\cdot,j}\leq \theta]}$, where $\theta$ are all unique values that have appeared in feature column $j$. For Netherlands, special preprocessing steps are performed for \textit{``age in years''} (which is real-valued, not integer) and \textit{``age at first penal case''}. Instead of considering all unique values in the feature column, we consider 1000 quantiles. 

\begin{table}[ht]
    \centering
    \begin{tabular}{|l|l|c|c|l|}\hline
    Dataset & Source & N & P & Classification task\\\hline
    adult & \cite{kohavi1996scaling} & 32561 & 36 & Predict if a U.S. resident earns more than \$50,000  \\\hline
    bank & \cite{moro2014data} & 41188 & 55 & Predict if a person opens account after marketing call \\\hline
    breastcancer & \cite{mangasarian1995breast} & 683 & 9 & Detect breast cancer using a biopsy   \\ \hline
    mammo & \cite{elter2007prediction} & 961 & 14 & Detect breast cancer using a mammogram\\\hline
    mushroom & \cite{schlimmer1987concept} & 8124 & 113 & Predict if a mushroom is poisonous\\\hline
    spambase & \cite{cranor1998spam} & 4601 & 57 & Predict if an e-mail is spam\\\hline
    COMPAS & \cite{LarsonMaKiAn16} & 6907 & 134 & Predict if someone will be arrested $\leq$ 2 years of release\\\hline
    FICO & \cite{fico} & 10459 & 1917 & Predict if someone will default on a loan\\\hline
    Netherlands & \cite{tollenaar2013method} & 20000 & 2024 & Predict if someone will have any charge within 4 years\\\hline
    
    \end{tabular}
    \caption{Dataset information. Breastcancer and spambase datasets have real-valued features. All other datasets have binary (0 or 1) features.}
    \label{tab:data_info}
\end{table}

\subsection{Computing Platform}
\label{app:computing_platform}
We ran all experiments on a TensorEX TS2-673917-DPN Intel Xeon Gold 6226 Processor with 2.7Ghz (768GB RAM 48 cores). For all experiments, we used only two cores because we observed using more cores did not improve the computational speed further.

\subsection{Baselines}
\label{app:baseline_specifications}
We compare with several baselines in our experiments.

\textbf{RiskSLIM} The current state-of-the-art method is RiskSLIM. We installed this package from the following GitHub link: \url{https://github.com/ustunb/risk-slim}\footnote{The license for this package is BSD 3-Clause license. The license can be viewed on the GitHub page.}. RiskSLIM uses the IBM CPLEX MIP solver to do the optimization. The CPLEX version we used is 12.8.

\textbf{Pooled Approaches} For other baselines, we first found a pool of continuous sparse solutions by the ElasticNet \cite{zou2005regularization} method and then rounded the coefficients to integers with different rounding techniques. Because ElasticNet has $\ell_1$ and $\ell_2$ penalties, we call this method the penalized logistic regression (PLR) approach. The best integer solution was selected from this pool based on which solution produces the smallest logistic loss while obeying the sparsity constraint and box constraints. These baselines correspond to the Pooled Approaches in Section 5.1 of \cite{ustun2019learning}, where 
Figure 11 and Figure 12 clearly show that pooled approaches are much better than traditional approaches. We include Unit Weighting and Rescaled Rounding as two additional rounding methods.
The details of the pooled approach and the rounding techniques can be found in Section 5.1 of \cite{ustun2019learning}.

The ElasticNet method tries to solve the following optimization problem:
\begin{align}
    \min_{\bw } &\frac{1}{2n}\sum_{i=1}^{n} \log(1+\exp(-  y_i \bx_i^T \bw ))+ \lambda \cdot (\alpha \Vert{\bw}_1 + (1-\alpha) \Vert{\bw}_2^2)
\end{align}
where $\alpha \in $ [0,1] is a hyperparameter. By controlling $\alpha$, we choose the best model over Ridge ($\alpha$ =0), Lasso ($\alpha$=1), and Elastic net ($0<\alpha<1$). We generated 1,100 models using the glmnet package\footnote{We installed the package from the following GitHub link: \url{https://github.com/bbalasub1/glmnet_python} The package contains  GNU license, which can be viewed on the GitHub website.}. To do this, we first choose 11 values of $\alpha \in \{$0, 0.1, 0.2, ..., 0.9, 1.0$\}$. For each given $\alpha$, the package then internally and automatically selects 100 $\lambda$'s equi-spaced on the logarithmic scale between $\lambda_{\min}$ and $\lambda_{\max}$ (the smallest value for $\lambda$ such that all the coefficients are zero). We call this part the \textit{Pooled-PLR} (Pooled Penalized Logistic Regression).


To convert each continuous sparse model to an integer sparse model, we applied the following rounding methods:
\begin{itemize}
    \item 1) Pooled-PLR-RD: For each of the 1,100 PLR models in the pool, we first truncated all the coefficients (except the intercept $\beta_0$) to be within the range [-5,5] and did simple rounding:
$\beta_j=\lceil \min(\max(\beta_j,-5),5)\rfloor$, and $\beta_0=\lceil \beta_0 \rfloor$. The $\lceil \cdot \rfloor$ operation is defined as $\lceil a \rfloor = \ceil{a}$ if $\vert{a - \ceil{a}} < \vert{a - \floor{a}}$ and $\lceil a \rfloor = \floor{a}$ otherwise.

    \item 2) Pooled-PLR-RDU: For each solution, we rounded each of its coefficients to be $\pm$1 based on its signs: $w_j=\text{sign}(w_j) \mathds{1}_{[w_j \neq 0]}$ and $w_0=\lceil w_0 \rfloor$ This rounding technique is known as unit weighting or the Burgess method.

    \item 3) Pooled-PLR-RSRD: For each solution, we rescaled its coefficients by a factor $\gamma$ so that $\gamma w_{\max} = \pm 5$ and then rounded each rescaled coefficient to the nearest integer: $w_j=\lceil \gamma w_j \rfloor, \gamma=\frac{5}{\max_j |\lambda_j|}$ and $w_0=\lceil w_0 \rfloor $. 

    \item 4) Pooled-PLR-Rand: For each model in the pool, for each coefficient, denote its fractional part by $u_j = w_j - \floor{w_j}$. We rounded each coefficient up to $\ceil{w_j}$ with probability $u_j$ and down to $\ceil{w_j}$ with probability $1-u_j$. After all rounding was done, we selected the best model in the pool. 
    
    \item 5) Pooled-PLR-RDP: For each model in the pool, we iterated through each coefficient $\beta_j$ and calculated the loss for both $\lceil \beta_j \rceil$ and $\lfloor \beta_j \rfloor$ and selected the rounding that minimizes the loss.  This is called Sequential Rounding in \cite{ustun2019learning}.
    
    \item 6) Pooled-PLR-RDSP: we first rounded through Sequential Rounding (Method 5, just above), and then we applied Discrete Coordinate Descent (DCD) \cite{ustun2019learning} to iteratively improve the loss by adjusting one coefficient at a time.
    At each round, DCD selects the coefficient and its new value that decreases the logistic loss the most.
\end{itemize}

As mentioned earlier, after we get the 1,100 integer sparse models via each rounding technique, we selected the best model from the pool based on which solution has the smallest logistic loss.

\subsection{Hyperparameters Specification}
\label{sec:hyperparameter_specification}
We used the default values in Algorithm \ref{alg:overall_algorithm} for all hyperparameters. We reiterate the hyperparameters used in the experiments below.
\begin{itemize}
    \item beam search size: $B = 10$.
    \item tolerance level for sparse diverse pool: $\epsilon = 0.3$ (or 30\%).
    \item number of attempts to try for sparse diverse pool: $T=50$.
    \item number of multipliers to try: $N_m = 20$.
\end{itemize}
Performance is not particularly sensitive to these choices (see Appendix \ref{appendix:hyperparameter_perturbation_study}). If $T$, $N_m$, $B$ are chosen too large, the algorithm will take longer to execute.

\newpage
\section{Additional Experimental Results}
\label{app:additional_experimental_results}
\subsection{Additional Results on Solution Quality}

In addition to the six datasets we show in the main paper, we provide results on the breastcancer, spambase, and Netherlands datasets (see Section \ref{app:dataset_information} for more data information).
The comparison of solution quality is shown in Figure \ref{fig:AUC_breastcancer_spambase_netherlands}. We see that \ourmethod{} outperforms both RiskSLIM and other pooled approaches, even with high dimensional feature spaces and in the presence of highly correlated features (the Netherlands dataset).

\begin{figure}[ht] 
    \centering
    \includegraphics[width=\textwidth]{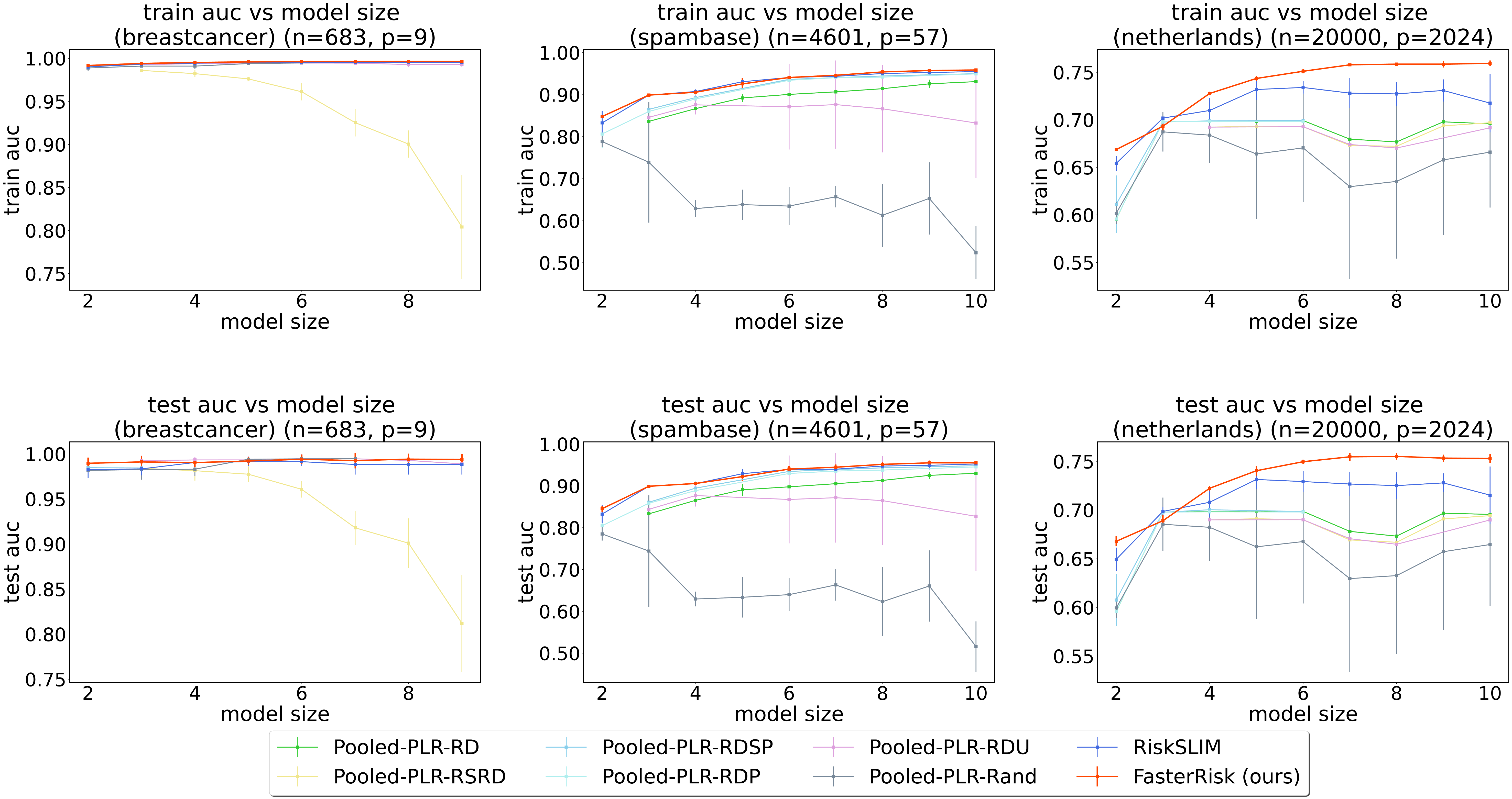}
   
    \caption{Performance comparison on the breastcancer, spambase, and Netherlands datasets. Top row is training AUC (higher is better) and bottom row is test AUC (higher is better).}
    \label{fig:AUC_breastcancer_spambase_netherlands}
\end{figure}

\clearpage
\subsection{Additional Results on Direct Comparison with RiskSLIM}
As RiskSLIM provides state-of-the-art performance, we compare it to \ourmethod{} in isolation to highlight the differences between the two approaches/algorithms.
The results are shown in Figure \ref{fig:AUC_compas_fico_netherlands_berk_ours} on the breastcancer, spambase, and Netherlands datasets. 

\begin{figure}[h] 
    \centering
    \includegraphics[width=\textwidth]{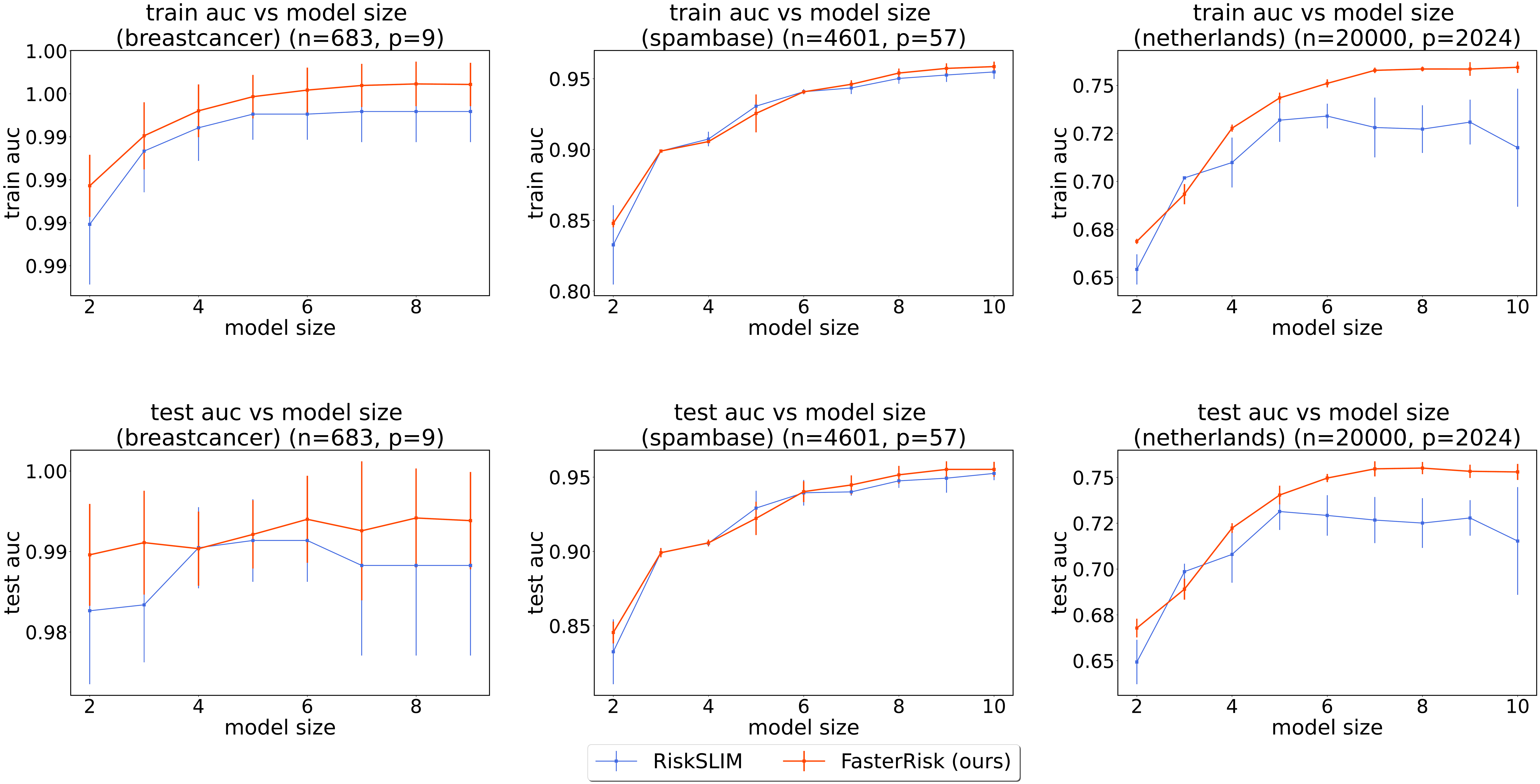}
   
    \caption{Detailed performance comparison between FasterRisk and RiskSLIM on breastcancer, spambase, and Netherlands. Top row is training AUC (higher is better) and bottom row is test AUC (higher is better).
    We can improve \ourmethod{}'s results on the spambase dataset by increasing the beam size in the algorithm. See Figure~\ref{fig:hyperparameter_beam_size_breastcancer_spambase_netherlands} for the perturbation study on this hyperparameter.}
    \label{fig:AUC_compas_fico_netherlands_berk_ours}
\end{figure}

\subsection{Additional Results on Running Time}
We also provide a runtime comparison between RiskSLIM and \ourmethod{} in Figure \ref{fig:breastcancer_spambase_netherlands_runtimeCoparison}. Except for the small dataset breastcancer, RiskSLIM timed out in all other instances. In contrast, \ourmethod{} finishes running under 50s or 100s on all cases, showing great scalability, even in high dimensional feature space and in presence of highly correlated features (the Netherlands dataset).

\begin{figure}[ht] 
    \centering
    \includegraphics[width=\textwidth]{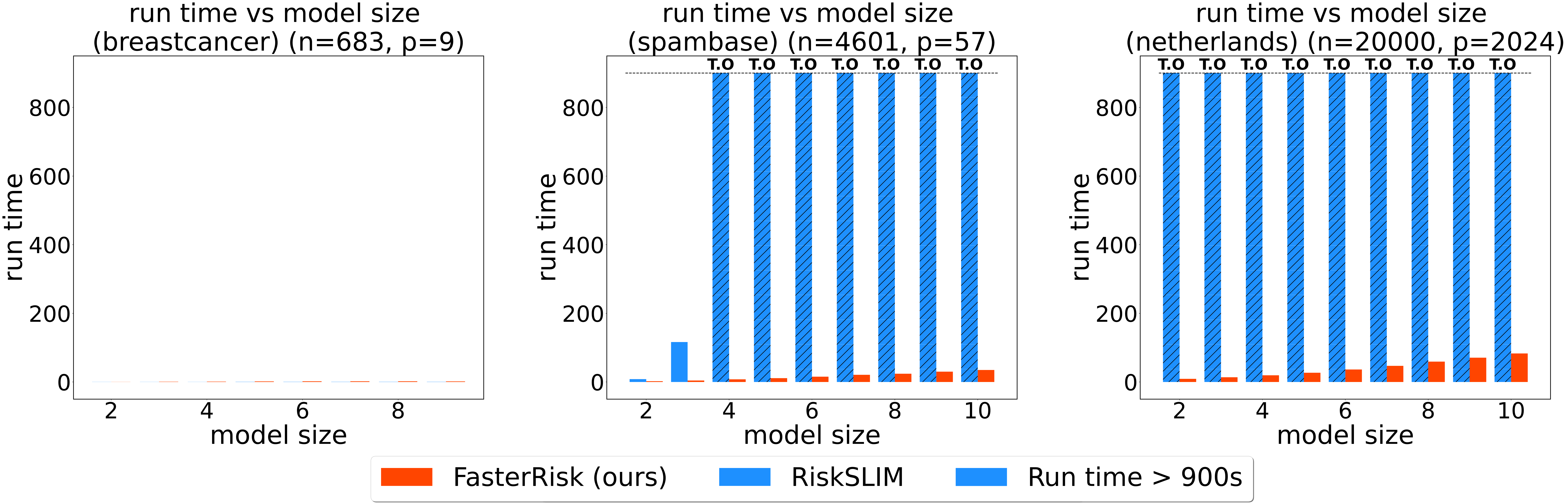}
   
    \caption{Runtime Comparison. Runtime (in seconds) versus model size for our method FasterRisk
(in \textcolor{red}{red}) and the RiskSLIM (in \textcolor{blue}{blue}). The shaded blue bars indicate cases that timed out (``T.O'') at 900
seconds.}
    \label{fig:breastcancer_spambase_netherlands_runtimeCoparison}
\end{figure}

\newpage
\subsection{Ablation Study of the Proposed Techniques}
\label{appendix:ablation_study}

We investigate how each component of \ourmethod{}, including sparse beam search, diverse pool, and multipliers, contribute to solution quality.
We quantify the contribution of each part of the algorithm by means of an ablation study in which we run variations of \ourmethod{}, each with a single component disabled.

The results are shown in Figure \ref{fig:ablation_study_adult_bank_mammo}-\ref{fig:ablation_study_breastcancer_spambase_netherlands}.
``no beam search'' means that the beam size is 1, so we expand the support by picking the next feature based on which new feature can induce the smallest logistic loss via the single coordinate optimization.
``no sparse diverse'' means that the sparse diverse pool contains only the solution by  Algorithm \ref{alg:sparse_beam_lr}, the \ourmethodLR{} method.
``no multiplier'' means that there is no ``star ray search'' of the multiplier. There is no scaling of coefficients or the data, so we think of this as setting multiplier to $1$.

The ablation study shows that different parts of our algorithm provide the biggest benefit to different data sets — that is, there is no single component of the algorithm that uniformly assists with performance; instead, the combination of these techniques, working in concert, is responsible.
We provide the detailed analysis of the contributions for each specific dataset in the figure captions.


\begin{figure}[ht] 
    \centering
    \includegraphics[width=\textwidth]{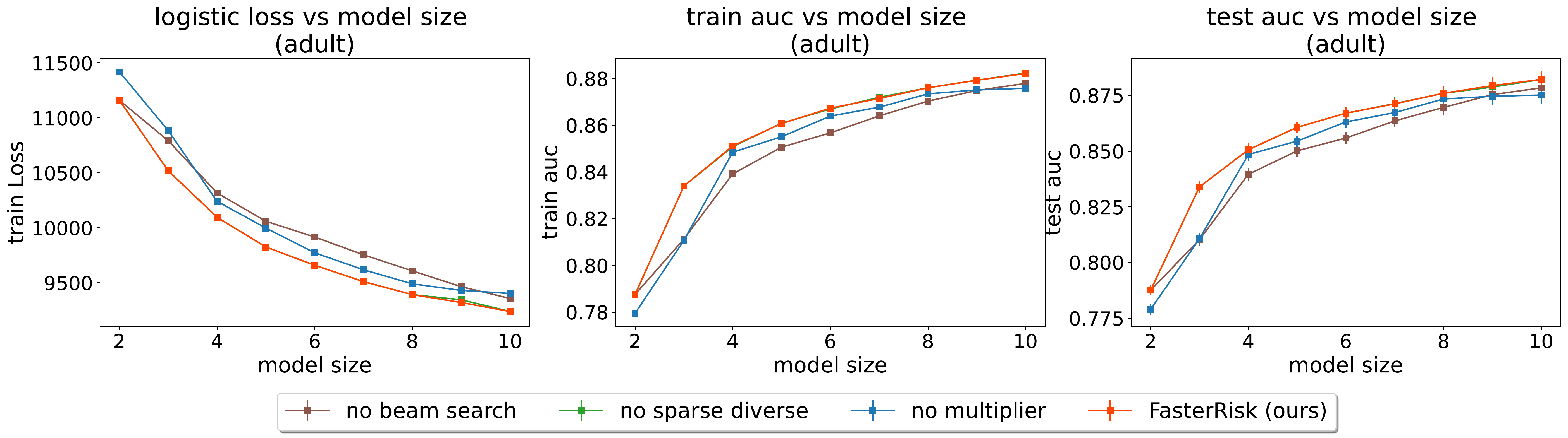}
    
    \includegraphics[width=\textwidth]{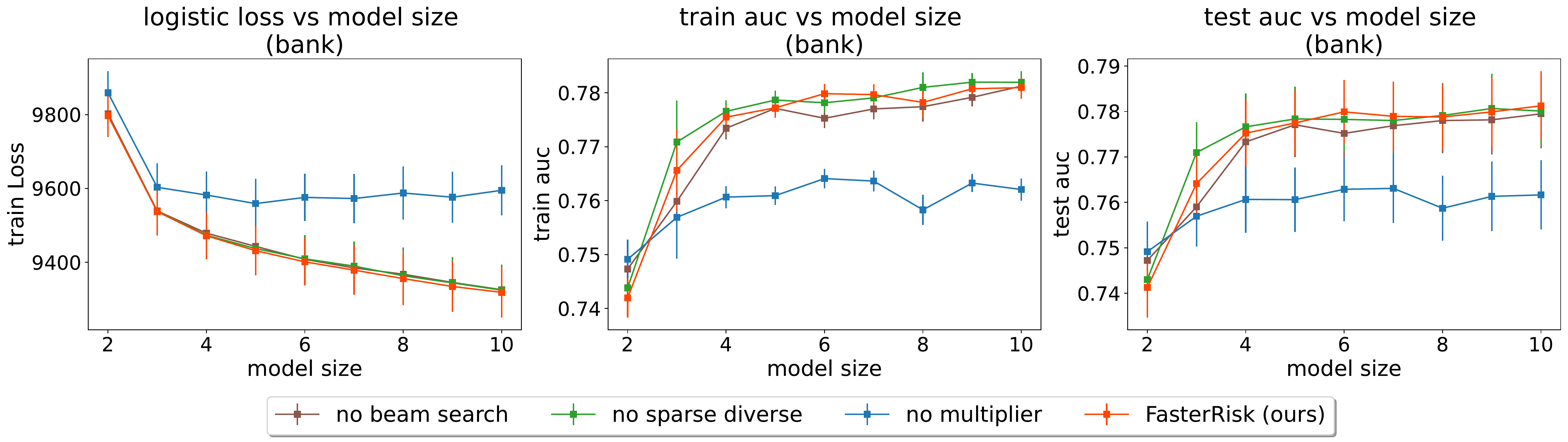}
    
    \includegraphics[width=\textwidth]{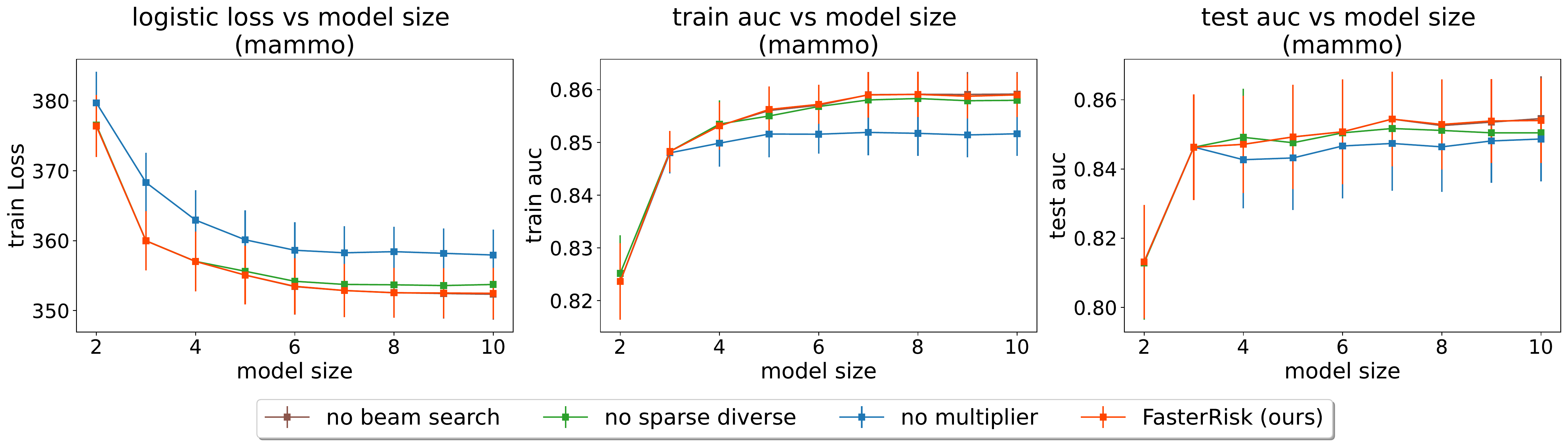}

    \caption{Ablation study on the adult, bank, and mammo datasets. Left column is loss (lower is better), middle column is training AUC (higher is better) and right column is test AUC (higher is better).
    The ``beam search'' method is particularly helpful on the adult dataset. The use of ``multiplier'' is particularly helpful on all three datasets. The ``diverse pool'' technique is somewhat helpful on the mammo dataset. More significant contributions from ``diverse pool'' can be found in Figure \ref{fig:ablation_study_breastcancer_spambase_netherlands} and Figure~\ref{fig:ablation_study_mushroom_compas_fico}.}
    \label{fig:ablation_study_adult_bank_mammo}
\end{figure}

\newpage
\begin{figure}[ht] 
    \centering
    \includegraphics[width=\textwidth]{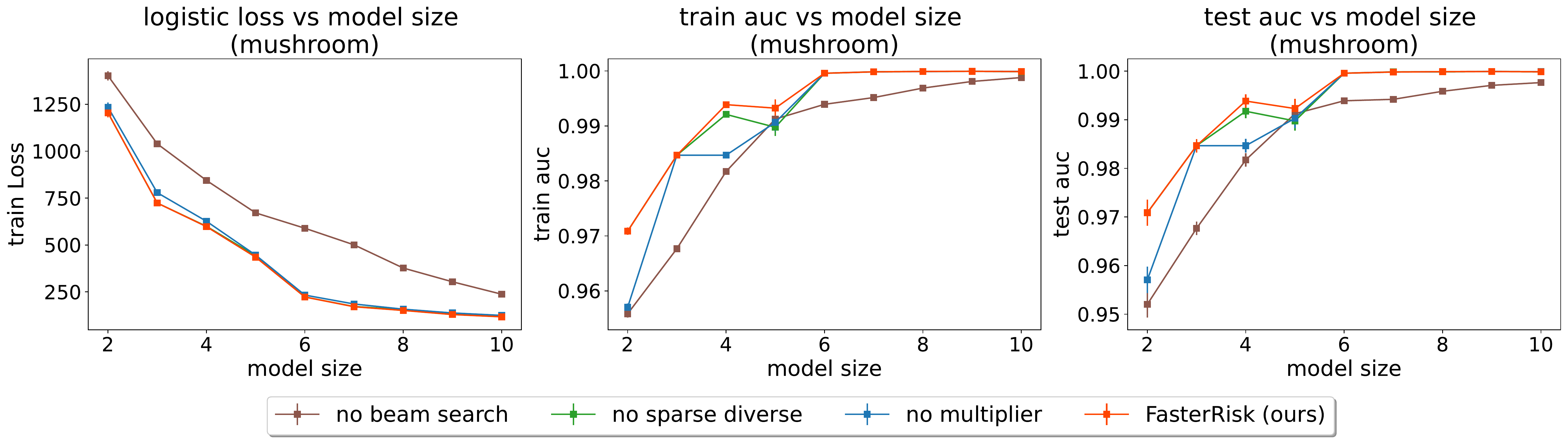}
    \includegraphics[width=\textwidth]{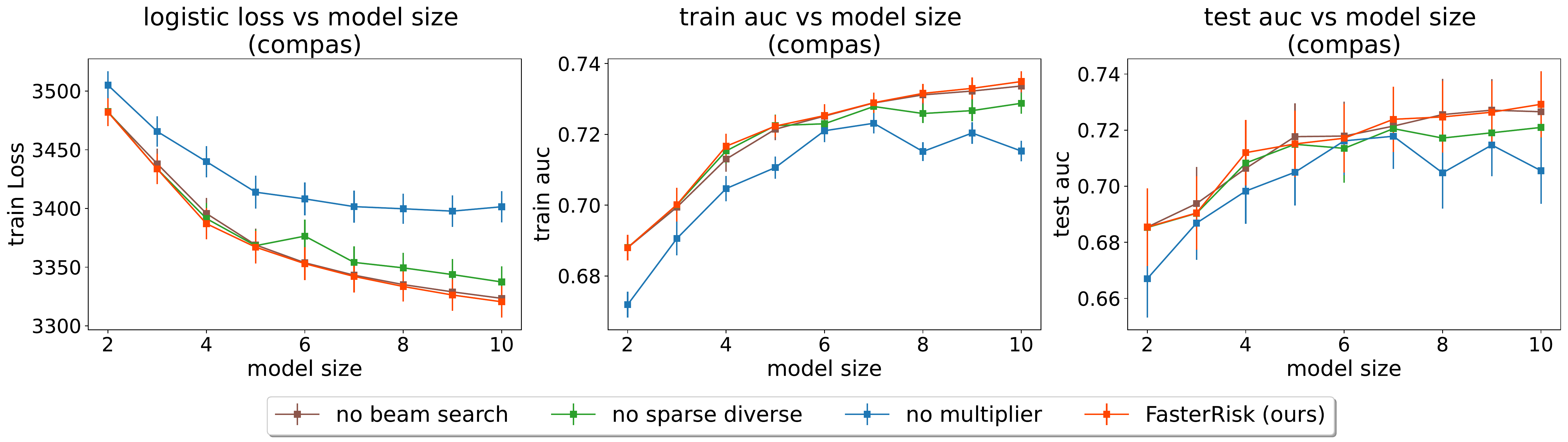}
    \includegraphics[width=\textwidth]{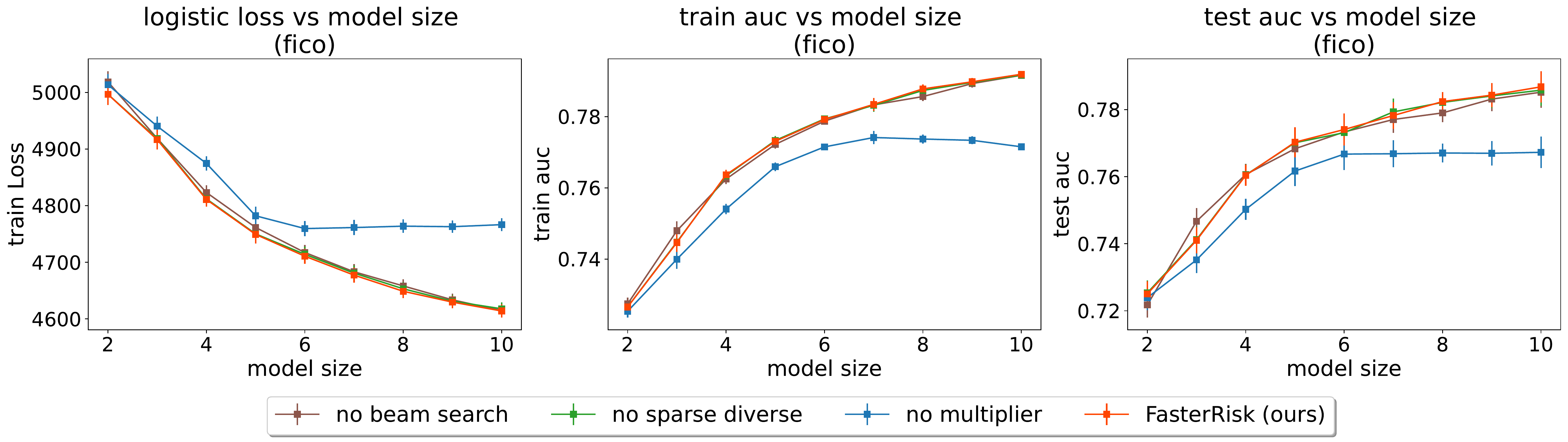}
    \caption{Ablation study on the mushroom, COMPAS, and FICO datasets. Left column is loss (lower is better), middle column is training AUC (higher is better) and right column is test AUC (higher is better).
    The ``beam search'' method is particularly helpful on the mushroom dataset. The use of ``multiplier'' is particularly helpful on the COMPAS and FICO datasets. The ``diverse pool'' technique is particularly helpful on the COMPAS dataset.
    }
    \label{fig:ablation_study_mushroom_compas_fico}
\end{figure}

\newpage
\begin{figure}[ht] 
    \centering
    \includegraphics[width=\textwidth]{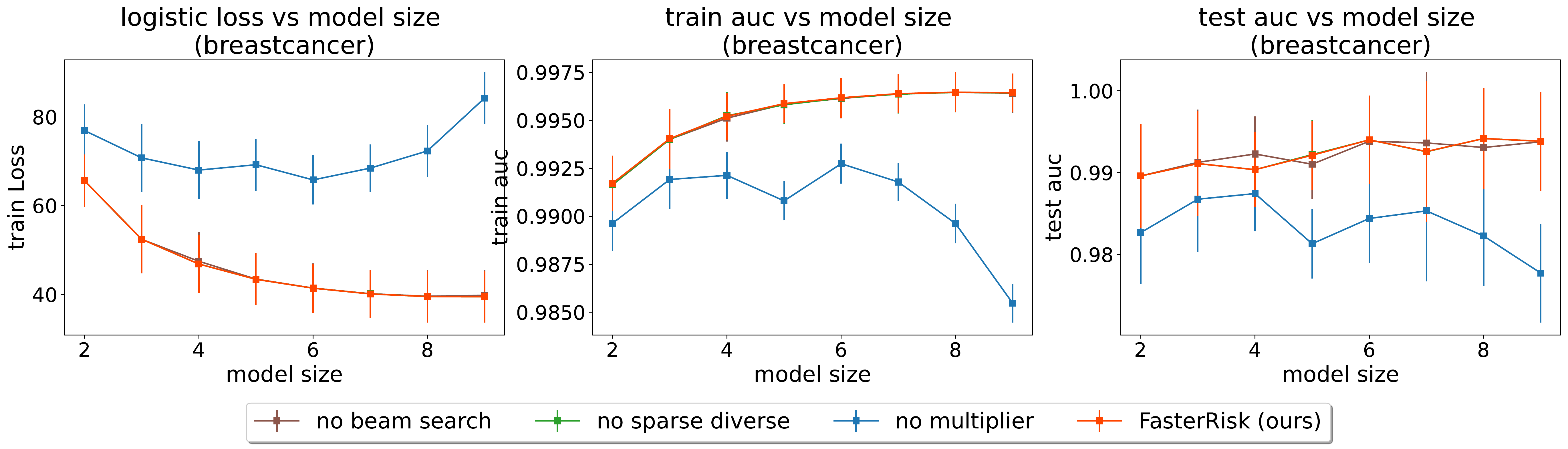}
    \includegraphics[width=\textwidth]{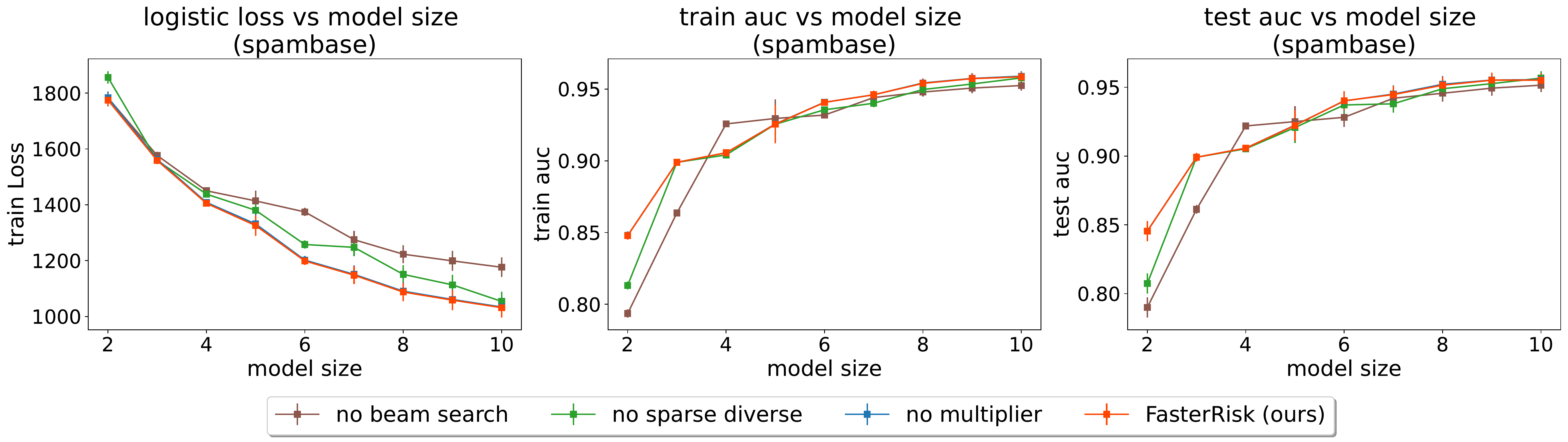}    
    \includegraphics[width=\textwidth]{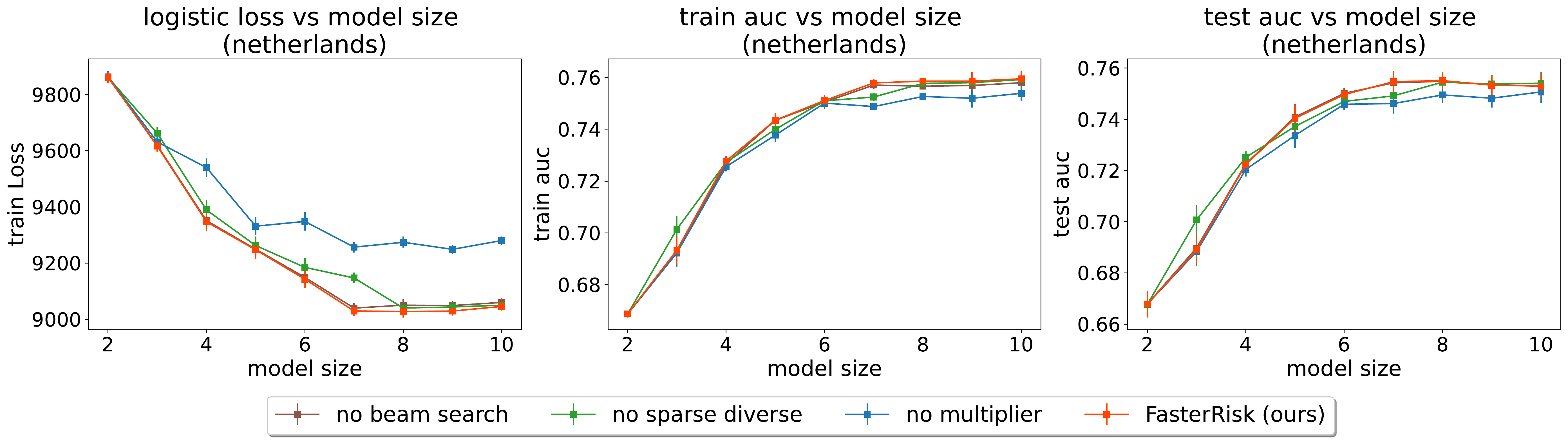}
    \caption{Ablation study on the COMPAS, FICO, and Netherlands datasets. Left column is loss (lower is better), middle column is training AUC (higher is better) and right column is test AUC (higher is better).
    The ``beam search'' method is particularly helpful on the spambase dataset. The use of ``multiplier'' is particularly helpful on breastcancer and netherlands datasets. The ``diverse pool'' technique is particularly helpful on the spambase and Netherlands datasets.
    }
    \label{fig:ablation_study_breastcancer_spambase_netherlands}
\end{figure}

\newpage
\subsection{Training Losses of FasterRisk vs. RiskSLIM}

In the main paper, due to the page limit, we have only compared the training and test AUCs between RiskSLIM and our \ourmethod{}. Here, we provide the comparison of training loss (logistic loss) between these two methods.
The results are shown in Figure \ref{fig:loss_comparison}.
We can see that \ourmethod{} outperforms RiskSLIM in almost all model size instances and datasets.

\begin{figure}[ht] 
    \centering
    \includegraphics[width=\textwidth]{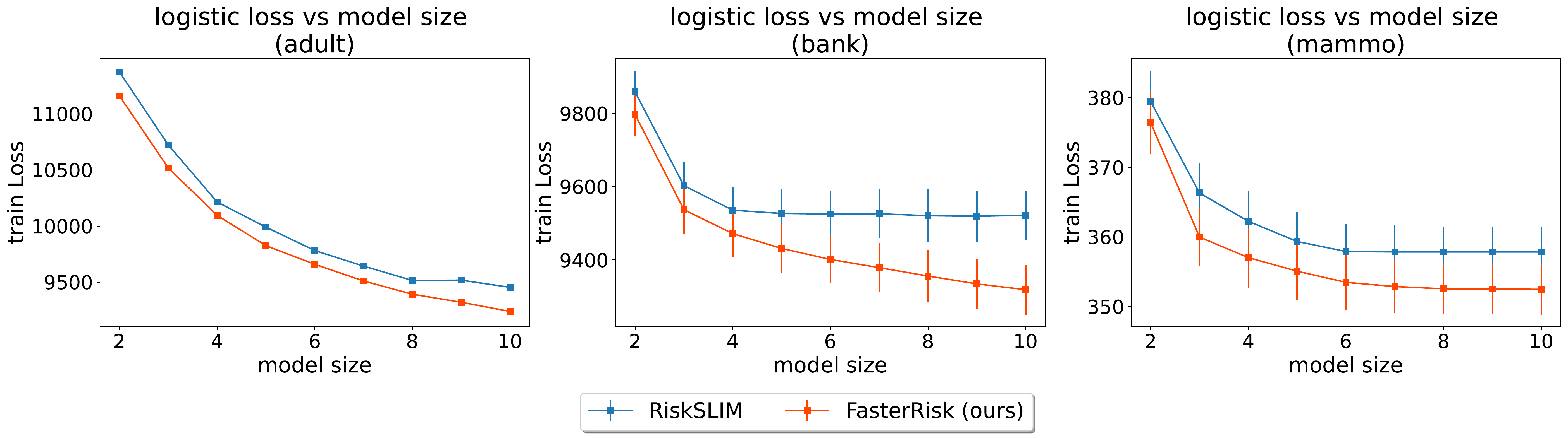}
    \includegraphics[width=\textwidth]{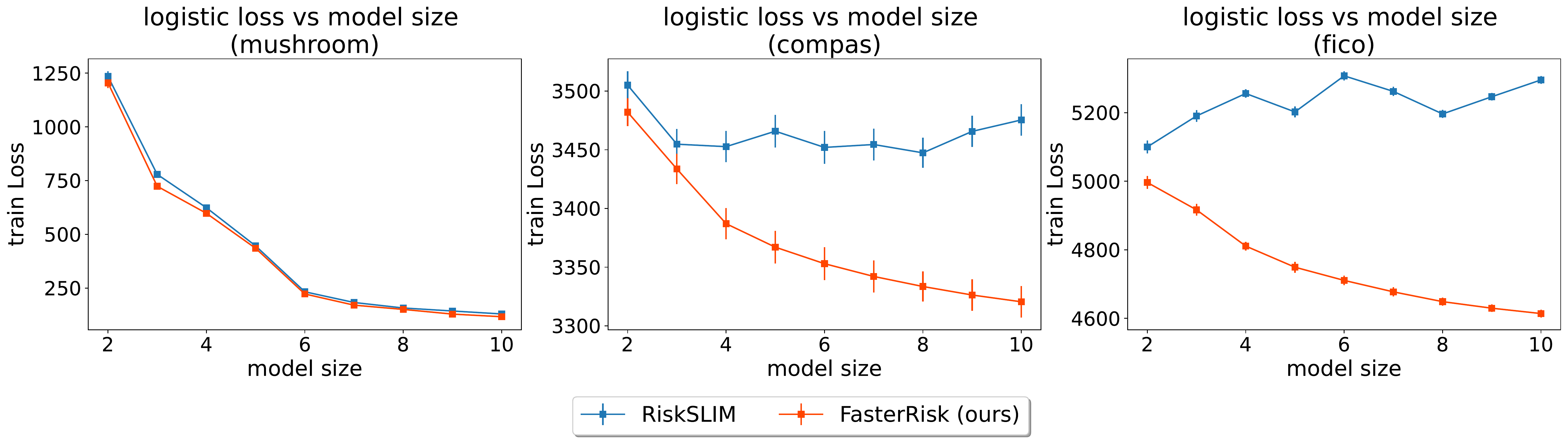}
    \includegraphics[width=\textwidth]{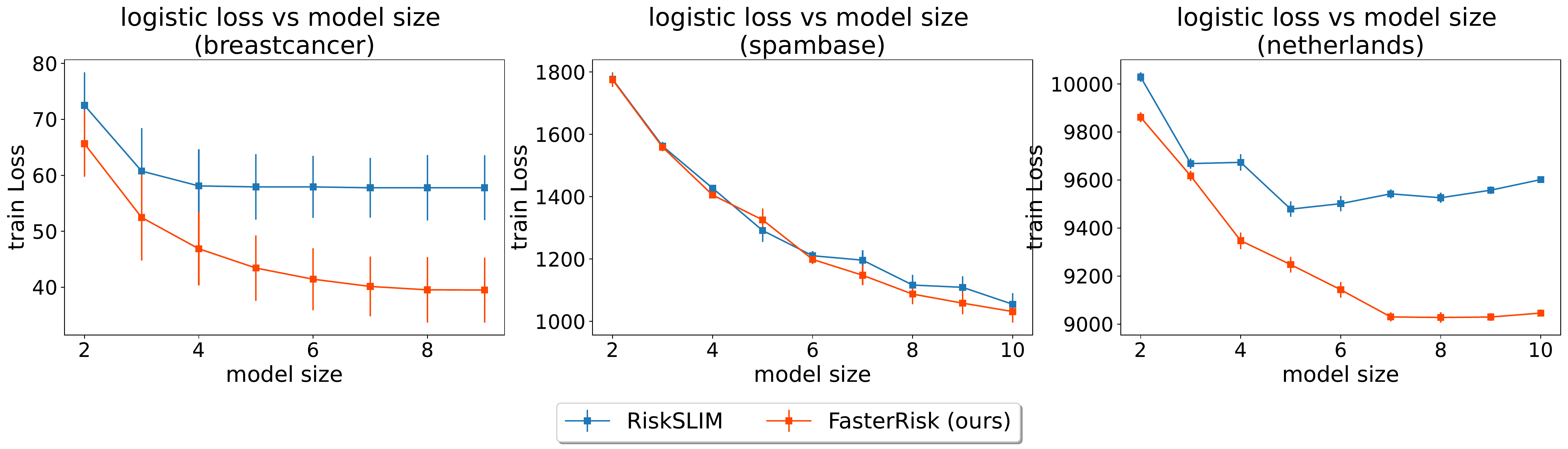}
    
    \caption{Training loss between RiskSLIM and our \ourmethod{} methods. (lower is better)}
    \label{fig:loss_comparison}
\end{figure}

\newpage
\subsection{Comparison of \ourmethodLR{} with OMP and fastSparse}
\label{app:compare_sparseBeamLR_with_OMP_and_fastSparse}


We next study how effective \ourmethodLR{} is in producing continuous sparse coefficients under the $\ell_0$ sparsity and box constraints. We compare with two existing methods, OMP \cite{elenberg2018restricted} and fastSparse \cite{LiuEtAl2022}.
OMP stands for Orthogonal Matching Pursuit, which expands the support by selecting the next feature with the largest magnitude of partial derivative.
fastSparse tries to solve the logistic loss objective with an $\ell_0$ regularization. For fastSparse, we use the default $\lambda_0$ values (coefficient for the $\ell_0$ regularization) internally selected by the software. Specifically, the software first apply a large $\lambda_0$ value to produce a super-sparse solution (with support size equal to 1 or close to 1). Then, in the solution path, the $\lambda_0$ value is sequentially decreased until the produced sparse model violates the model size constraint.

The results are shown in Figure \ref{fig:continuous_beamSearch_investigation_adult_bank_mammo}-\ref{fig:continuous_beamSearch_investigation_breastcancer_spambase_netherlands}.
Although OMP and fastSparse can somtimes produce high-quality solutions on some model size instances and datasets, \ourmethodLR{} is the only method that consistently produces high-quality sparse solutions in all cases.

OMP's solution quality is usually worse than that of \ourmethodLR{}, and OMP could not produce coefficients that satisfy the box constraints on the mushroom and spambase datasets.

fastSparse also cannot produce coefficients that satisfy the box constraints on the mushroom and spambase datasets. Additionally, it is hard to control the $\lambda_0$ regularization to produce the exact model size desired. In Figure \ref{fig:continuous_beamSearch_investigation_mushroom_compas_fico}, we do not obtain any model with model size equal to $9$ or $10$ in the solution path.

The limitations of OMP and fastSparse stated above are our main motivations for developing the \ourmethodLR{} method.

\begin{figure}[ht] 
    \centering
    \includegraphics[width=\textwidth]{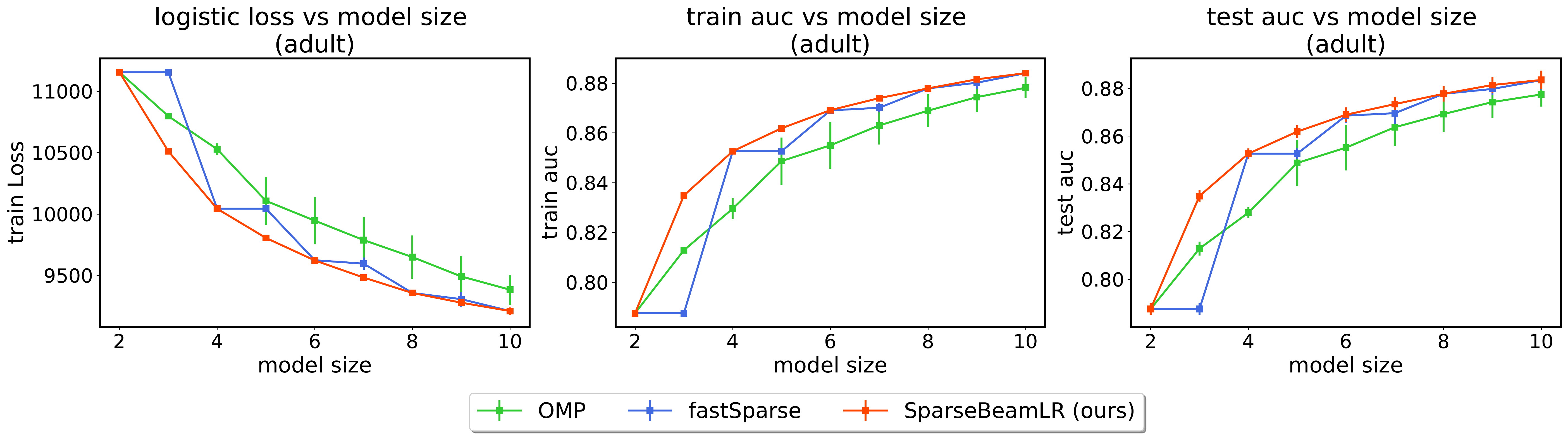}
    \includegraphics[width=\textwidth]{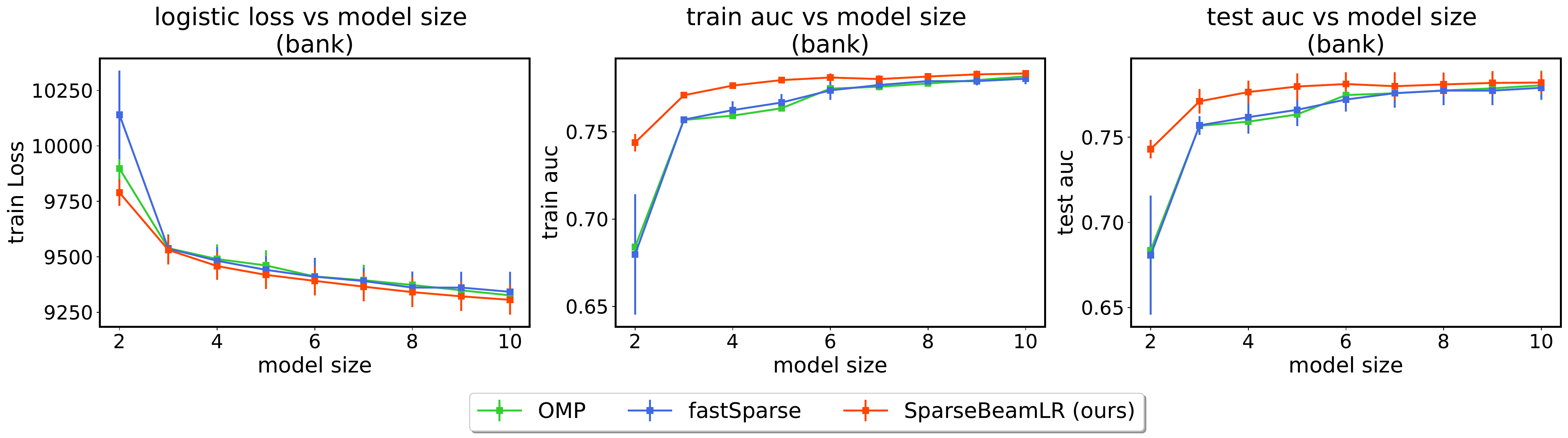}
    \includegraphics[width=\textwidth]{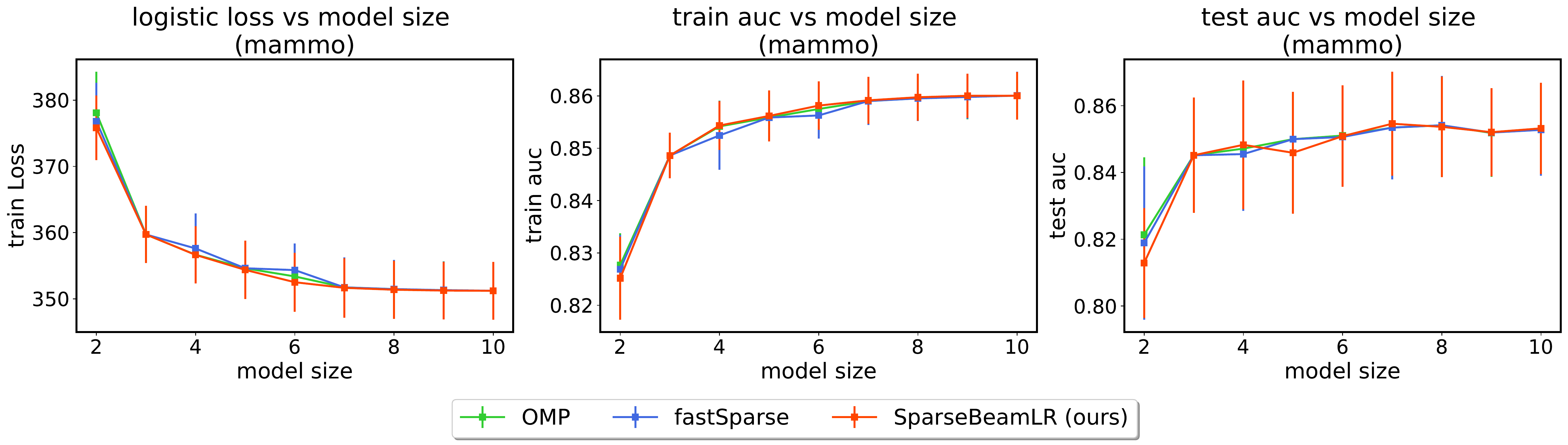}
    
    \caption{Sparse continuous solutions on the adult, bank, and mammo datasets. Left column is loss (lower is better), middle column is training AUC (higher is better) and right column is test AUC (higher is better).
    \ourmethodLR{} consistently produces high-quality continuous sparse solutions.}
    \label{fig:continuous_beamSearch_investigation_adult_bank_mammo}
\end{figure}

\newpage
\begin{figure}[ht] 
    \centering
    \includegraphics[width=\textwidth]{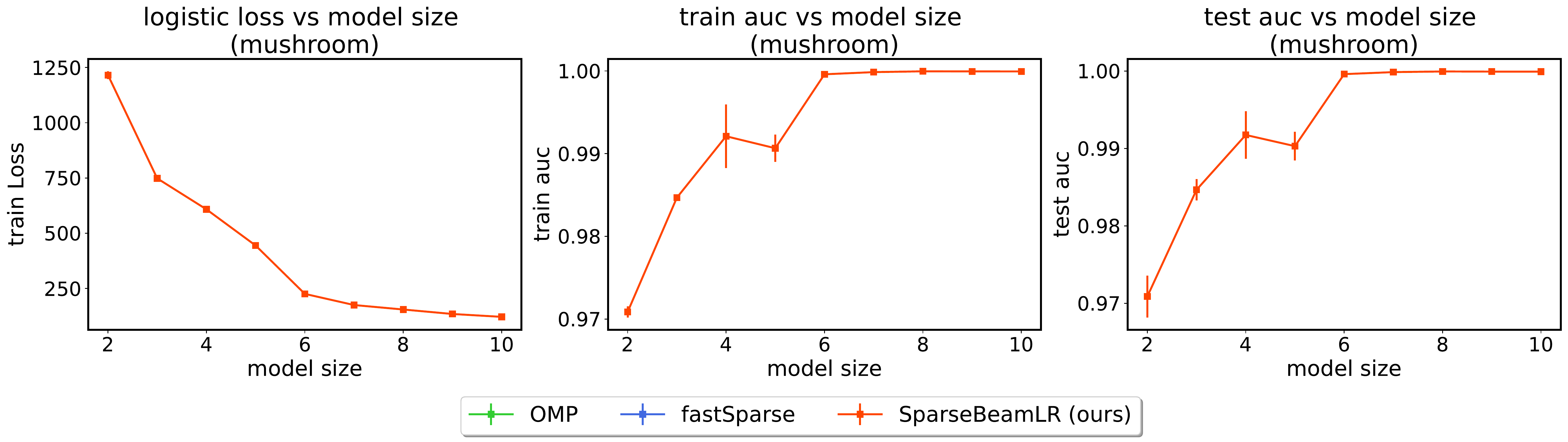}
    \includegraphics[width=\textwidth]{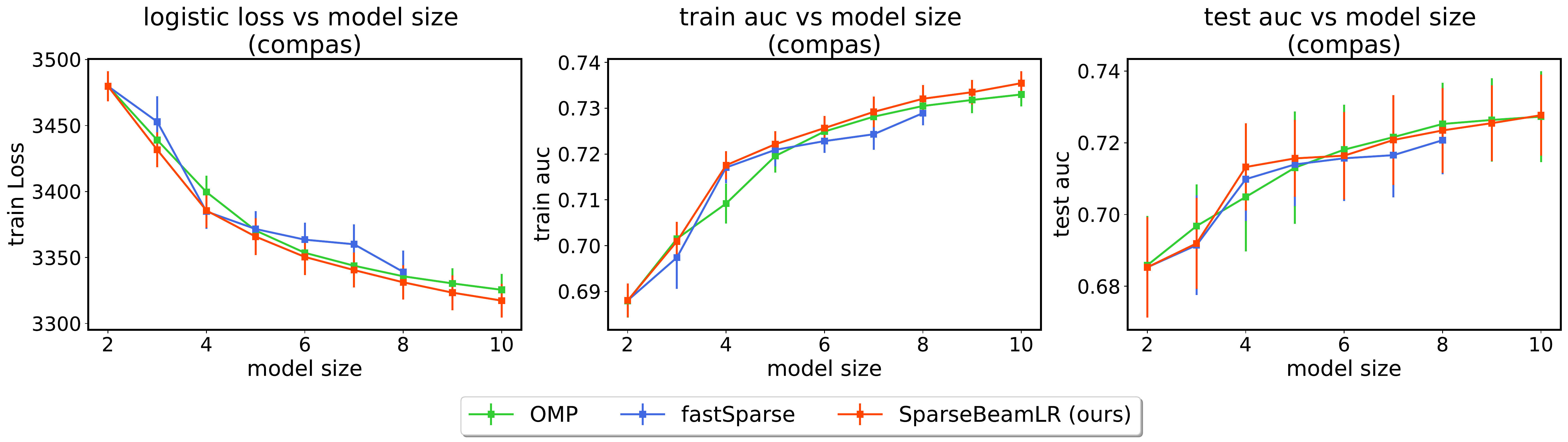}
    \includegraphics[width=\textwidth]{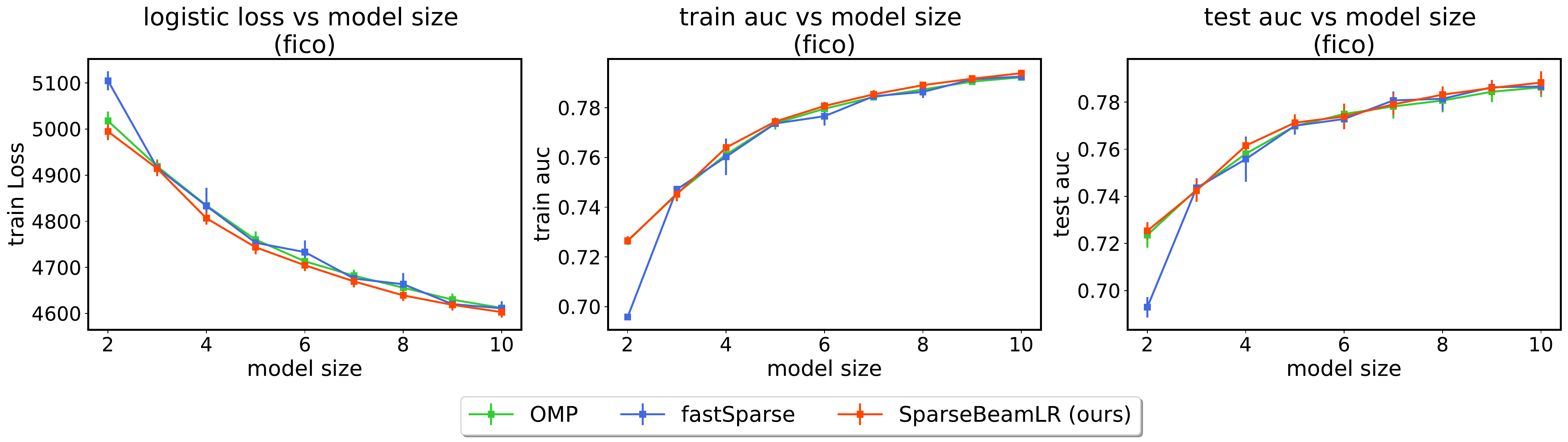}
    
    \caption{Sparse continuous solutions on the mushroom, COMPAS, and FICO datasets. Left column is loss (lower is better), middle column is training AUC (higher is better) and right column is test AUC (higher is better).
    The solution coefficients by the OMP and fastSparse methods violate the box constraints on the mushroom dataset, so we omit the results in the plot. fastSparse cannot obtain solutions with model size equal to $9$ or $10$ on the COMPAS dataset, so we do not show those points in the plot.}
    \label{fig:continuous_beamSearch_investigation_mushroom_compas_fico}
\end{figure}

\newpage
\begin{figure}[ht] 
    \centering
    \includegraphics[width=\textwidth]{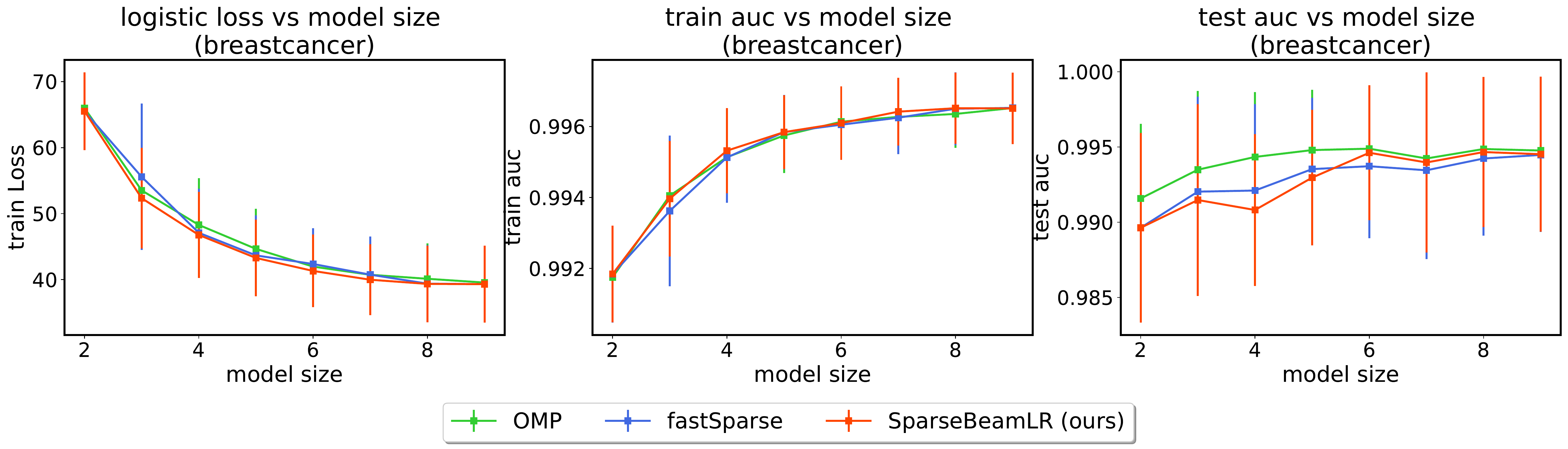}
    \includegraphics[width=\textwidth]{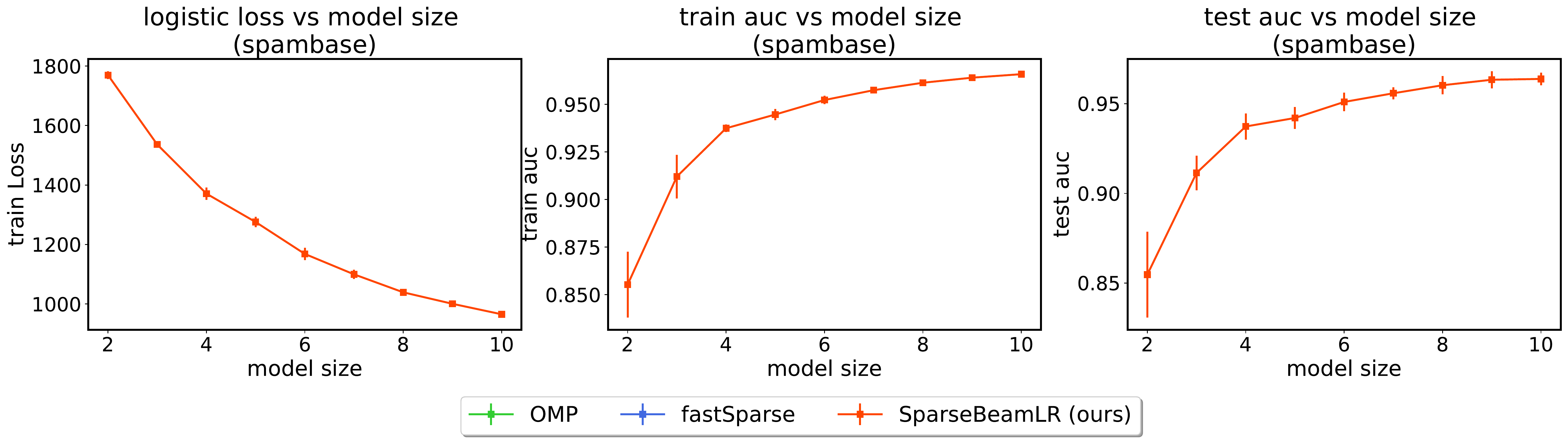}
    \includegraphics[width=\textwidth]{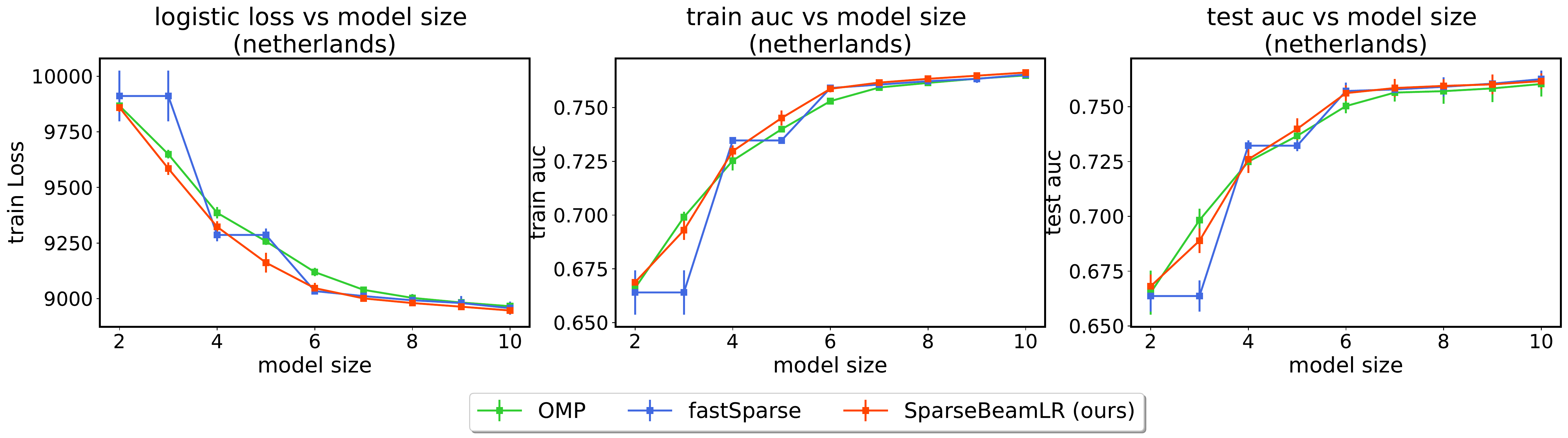}
    
    \caption{Sparse continuous solutions on the breastcancer, spambase, and Netherlands datasets. Left column is loss (lower is better), middle column is training AUC (higher is better) and right column is test AUC (higher is better).
    \ourmethodLR{} consistently produces high-quality continuous sparse solutions. The solution coefficients of OMP and fastSparse violate the box constraints on the spambase dataset, so we omit the results in the plot.}
    \label{fig:continuous_beamSearch_investigation_breastcancer_spambase_netherlands}
\end{figure}

\clearpage
\subsection{Comparison of \ourmethod{} with OMP (or fastSparse) \texorpdfstring{$+$}{+} Sequential Rounding}

Having compared the continuous sparse solutions, we next compare the integer sparse solutions produced by OMP, fastSparse, and \ourmethod{}.
After obtaining the continuous sparse solutions from OMP and fastSparse from Section \ref{app:compare_sparseBeamLR_with_OMP_and_fastSparse}, we round the continuous coefficients to integers using the Sequential Rounding method as stated in Method 5 of \ref{app:baseline_specifications}.

The results are shown in Figure \ref{fig:integer_beamSearch_investigation_adult_bank_mammo}-\ref{fig:integer_beamSearch_investigation_breastcancer_spambase_netherlands}. \ourmethod{} consistently outperforms the other two methods, due to higher quality of continuous sparse solutions and the use of multipliers.

\begin{figure}[ht] 
    \centering
    \includegraphics[width=\textwidth]{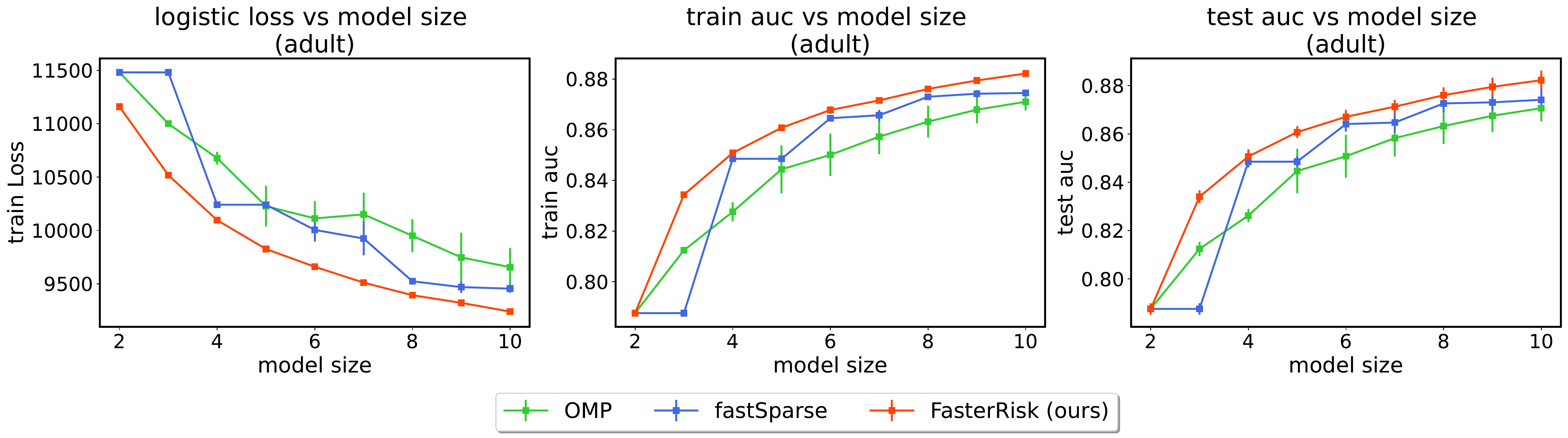}
    \includegraphics[width=\textwidth]{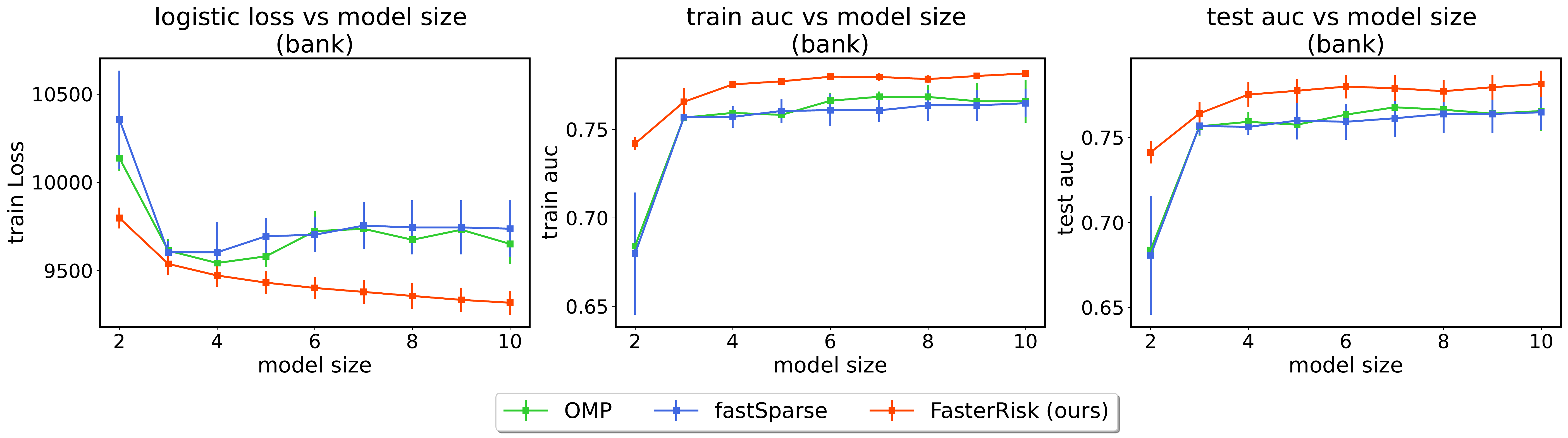}
    \includegraphics[width=\textwidth]{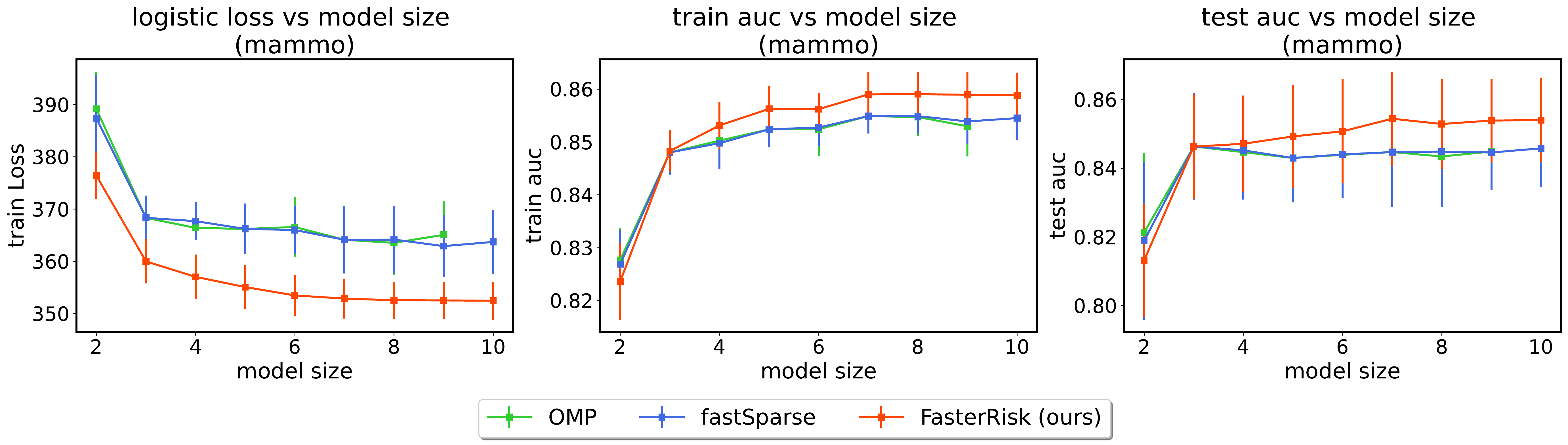}
    
    \caption{Sparse integer solutions on the adult, bank, and mammo datasets. Left column is loss (lower is better), middle column is training AUC (higher is better) and right column is test AUC (higher is better).
    \ourmethod{} consistently outperforms the other two methods, due to higher quality of continuous sparse solutions and the use of multipliers.
    }
    \label{fig:integer_beamSearch_investigation_adult_bank_mammo}
\end{figure}

\newpage
\begin{figure}[ht] 
    \centering
    \includegraphics[width=\textwidth]{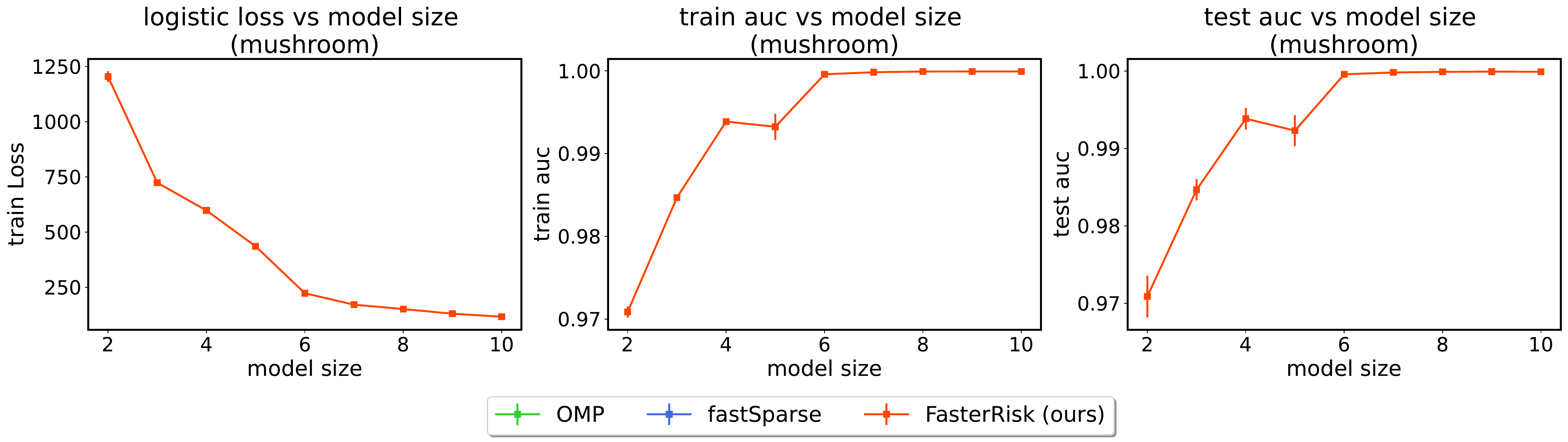}
    \includegraphics[width=\textwidth]{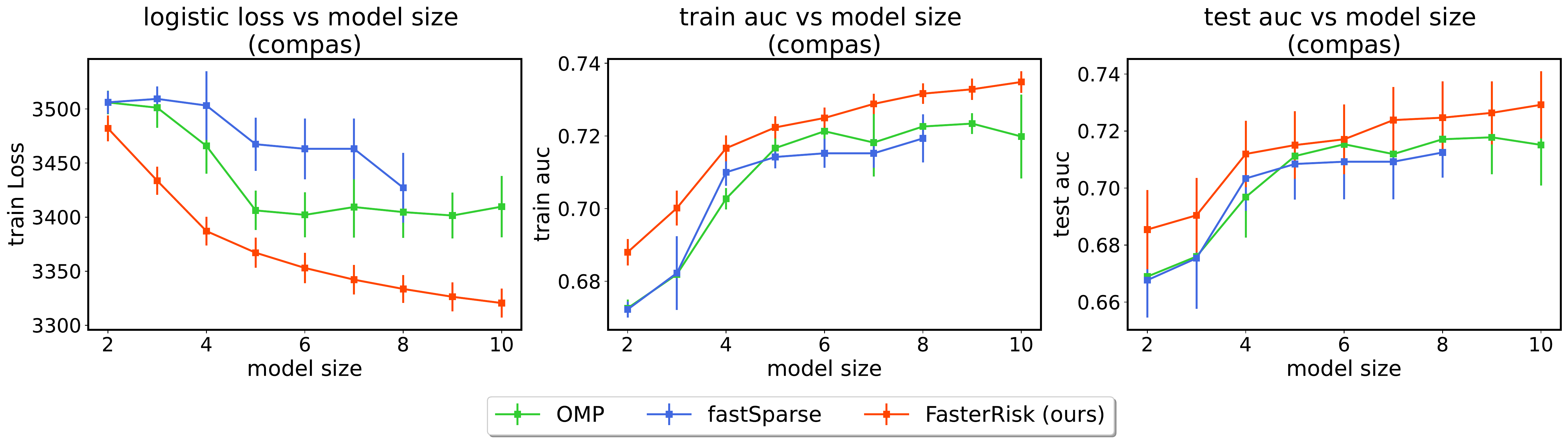}
    \includegraphics[width=\textwidth]{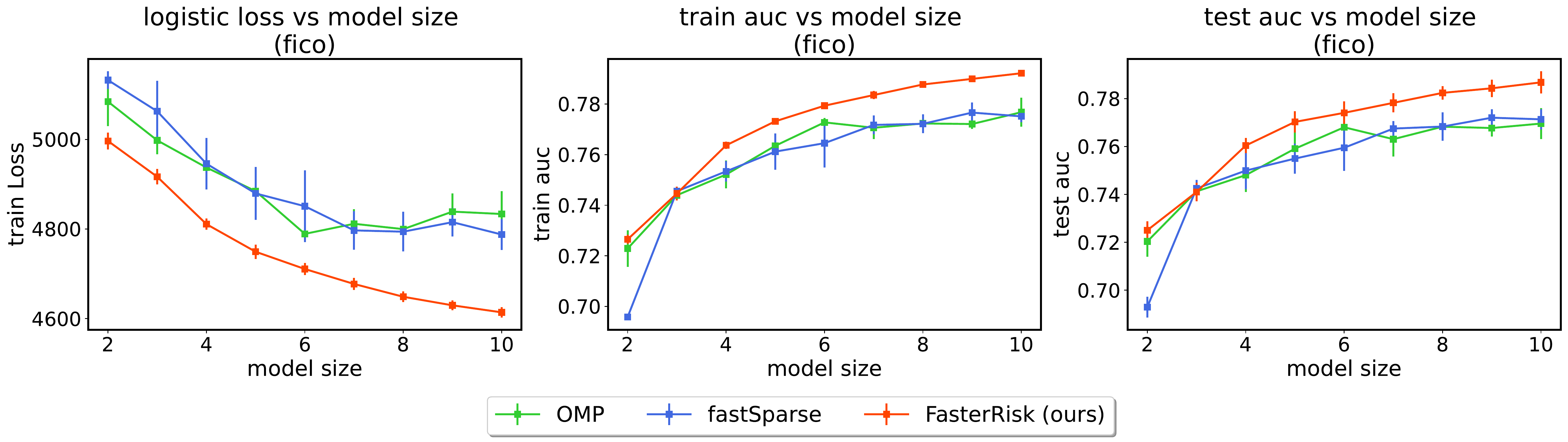}
    
    \caption{Sparse integer solutions on the mushroom, COMPAS, and FICO datasets. Left column is loss (lower is better), middle column is training AUC (higher is better) and right column is test AUC (higher is better).
    The solution coefficients from the OMP and fastSparse methods violate the box constraints on the mushroom dataset, so we omit the results on the plot.
    fastSparse cannot obtain solutions with model size equal to $9$ or $10$ on the COMPAS dataset, so we do not show those points on the plot.
    }
    \label{fig:integer_beamSearch_investigation_mushroom_compas_fico}
\end{figure}

\newpage
\begin{figure}[ht] 
    \centering
    \includegraphics[width=\textwidth]{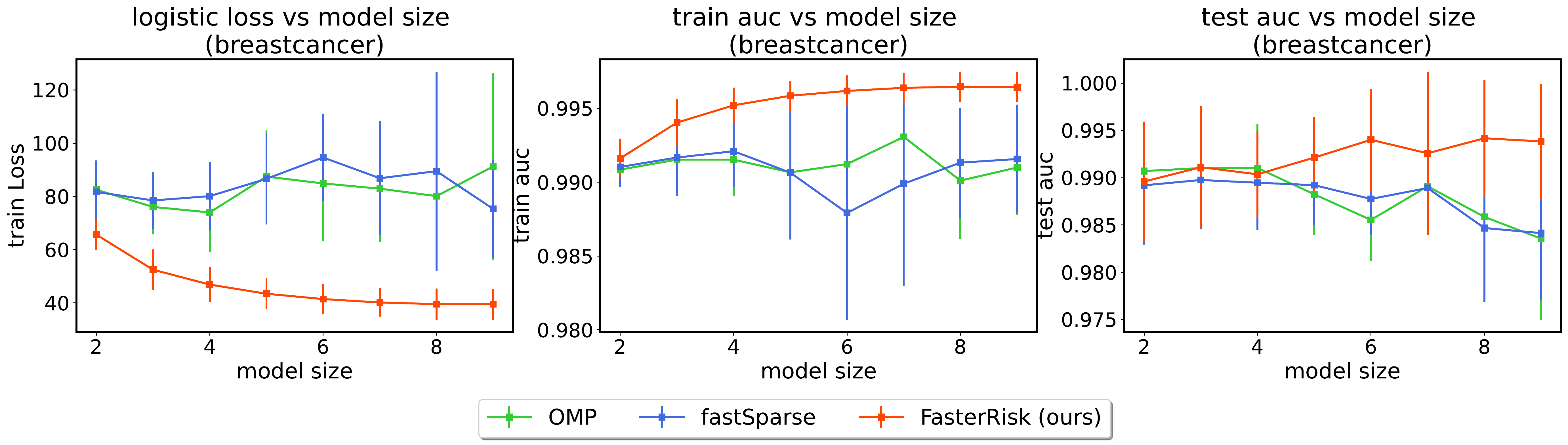}
    \includegraphics[width=\textwidth]{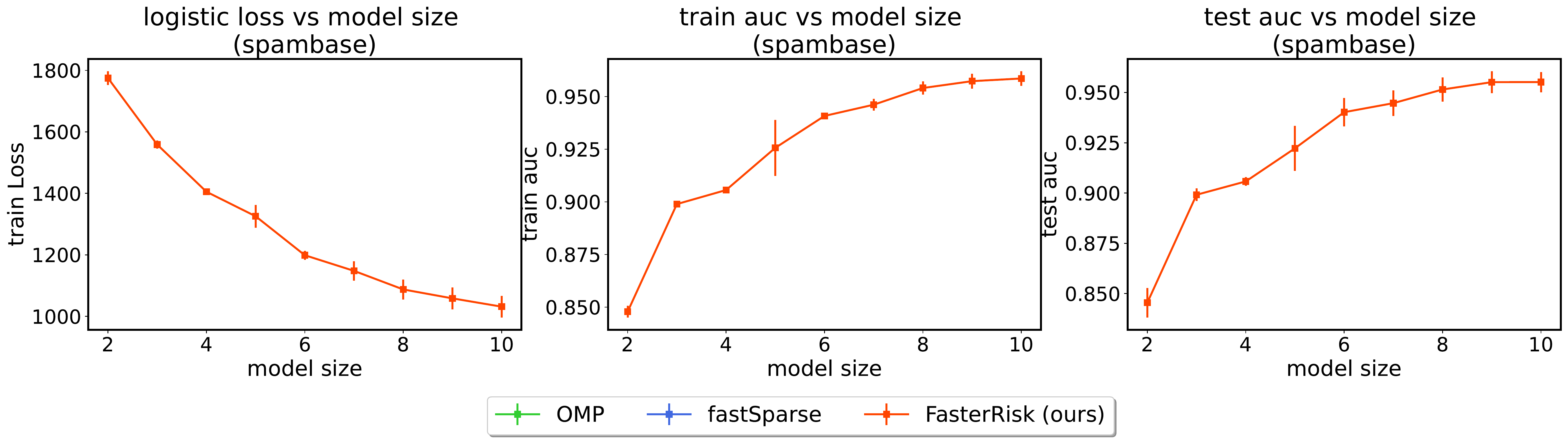}
    \includegraphics[width=\textwidth]{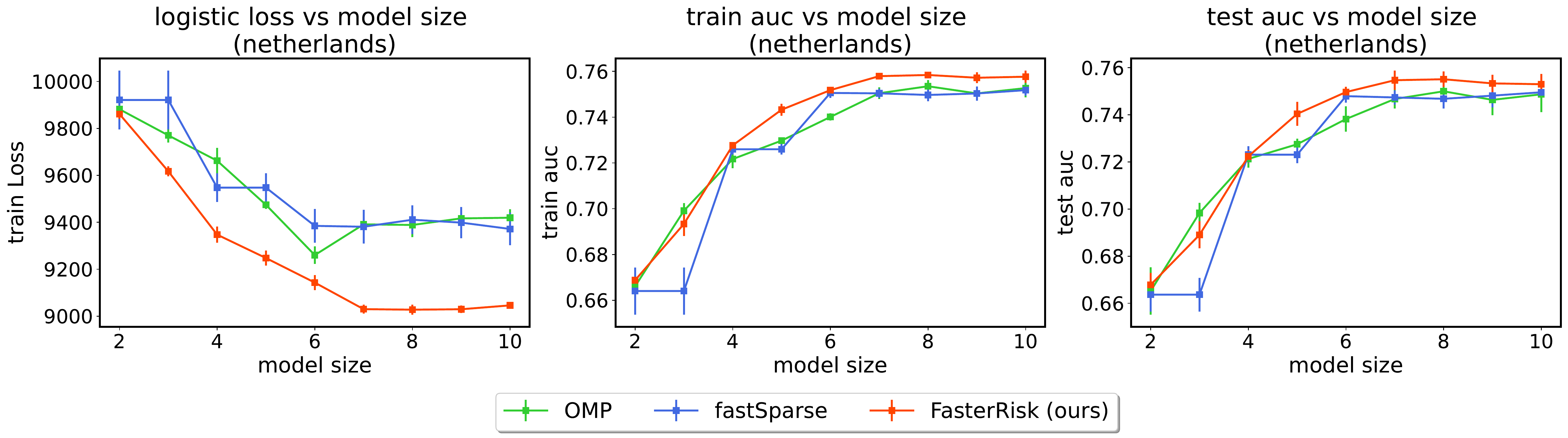}
    
    \caption{Sparse integer solutions on the breastcancer, spambase, and Netherlands datasets. Left column is loss (lower is better), middle column is training AUC (higher is better) and right column is test AUC (higher is better).
    \ourmethod{} consistently outperforms the other two methods, due to the higher quality of the continuous sparse solutions and the use of multipliers.
    The solution coefficients by the OMP and fastSparse methods violate the box constraints on the spambase dataset, so we omit the results in the plot.
    }
    \label{fig:integer_beamSearch_investigation_breastcancer_spambase_netherlands}
\end{figure}

\clearpage
\subsection{Running RiskSLIM Longer}
\label{app:running_riskslim_longer}
The experiments in Section~\ref{sec:experiments} imposed a 900-second timeout, and RiskSLIM frequently did not complete within the 900 seconds. Here, we run RiskSLIM with  longer timeouts (1 hour, and 4 days). We find that even with these long runtimes, \ourmethod{} still outperforms RiskSLIM in both solution quality and runtime.

Runtime is important for two reasons: (1) We may not be able to compute the answer at all using the slow method because it does not scale to reasonably-sized datasets. It could take a week or more to compute the solution for even reasonably small datasets. We will show this shortly through experiments. (2) Machine learning in the wild is never a single run of an algorithm. Often, users want to explore the data and adjust various constraints as they become more familiar with possible models. A fast speed allows users to go through this iteration process many times without lengthy interruptions between runs. This is where \ourmethod{} will be very useful in high stakes offline settings. \ourmethod{}'s pool of models is generated within 5 minutes, and interacting with the pool is essentially instantaneous after it is generated.

\subsubsection{Solution Quality of Running RiskSLIM for 1 hour}

We ran RiskSLIM for a time limit of 1 hour on all 5 folds and all model sizes (2-10). Thus, we ran experiments for 2 days per dataset. As a reminder, our method FasterRisk runs in less than 5 minutes (on all datasets). 
The results of logistic loss on the training set, AUC on the training set, and AUC on the test set are in Figures~\ref{fig:RiskSLIM_1h_adult_bank_mammo}-\ref{fig:RiskSLIM_1h_breastcancer_spambase_netherlands}. FasterRisk still outperforms RiskSLIM in almost all cases, because it uses a larger search space.

\begin{figure}[ht] 
    \centering
    \includegraphics[width=\textwidth]{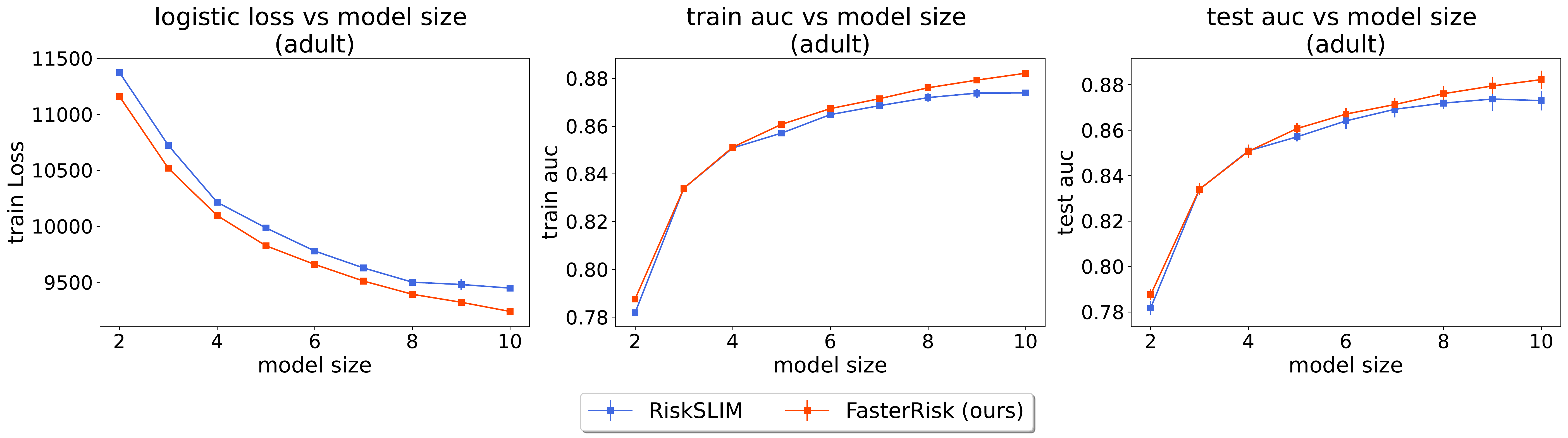}
    
    \includegraphics[width=\textwidth]{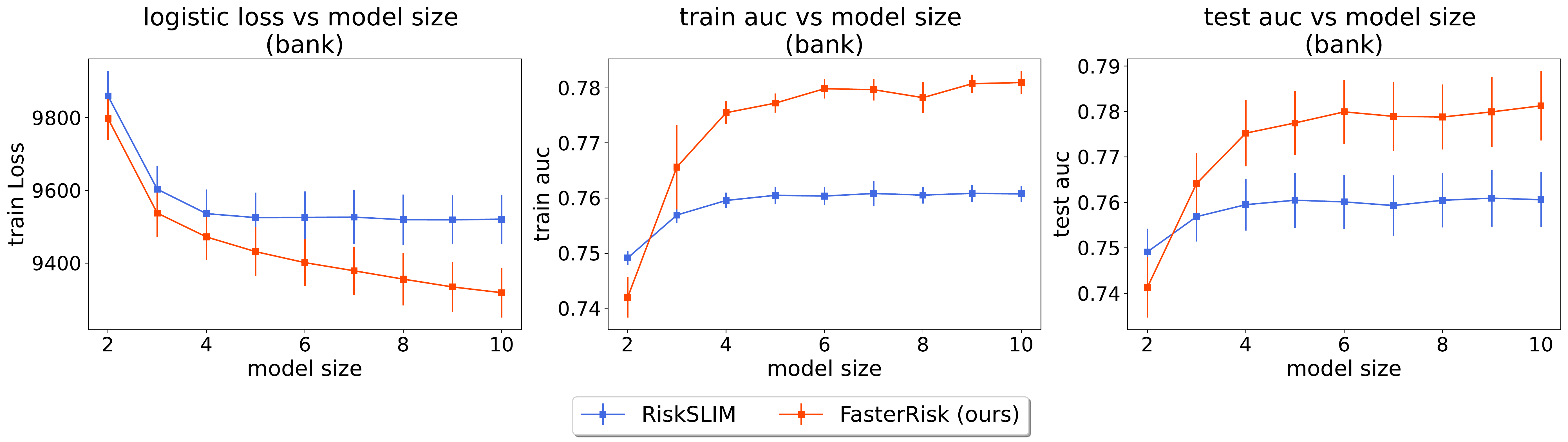}
    
    \includegraphics[width=\textwidth]{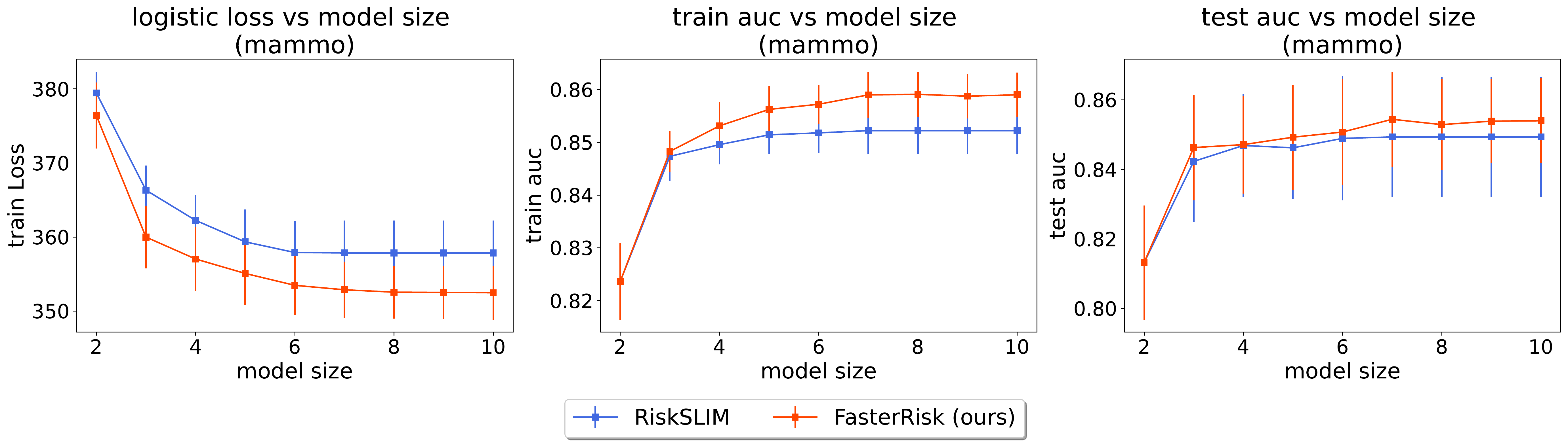}
    
    \caption{Comparison with the state-of-the-art baseline RiskSLIM (running for 1 hour) on the adult, bank, and mammo datasets. The left column is loss (lower is better), the middle column is training AUC (higher is better) and the right column is test AUC (higher is better).
    }
    \label{fig:RiskSLIM_1h_adult_bank_mammo}
\end{figure}

\begin{figure}[ht] 
    \centering
    \includegraphics[width=\textwidth]{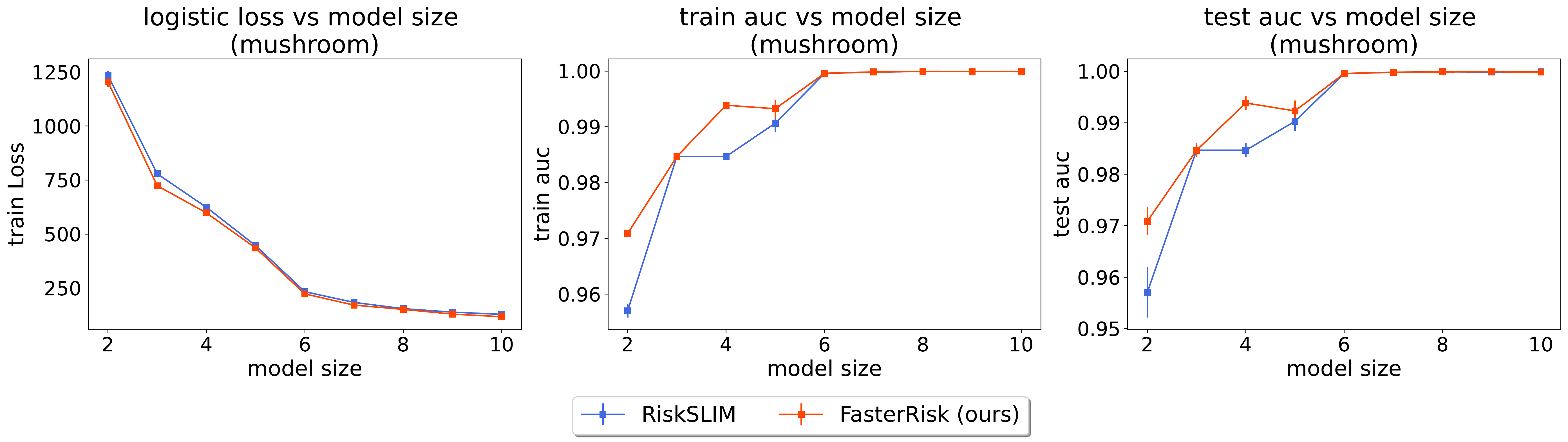}
    
    \includegraphics[width=\textwidth]{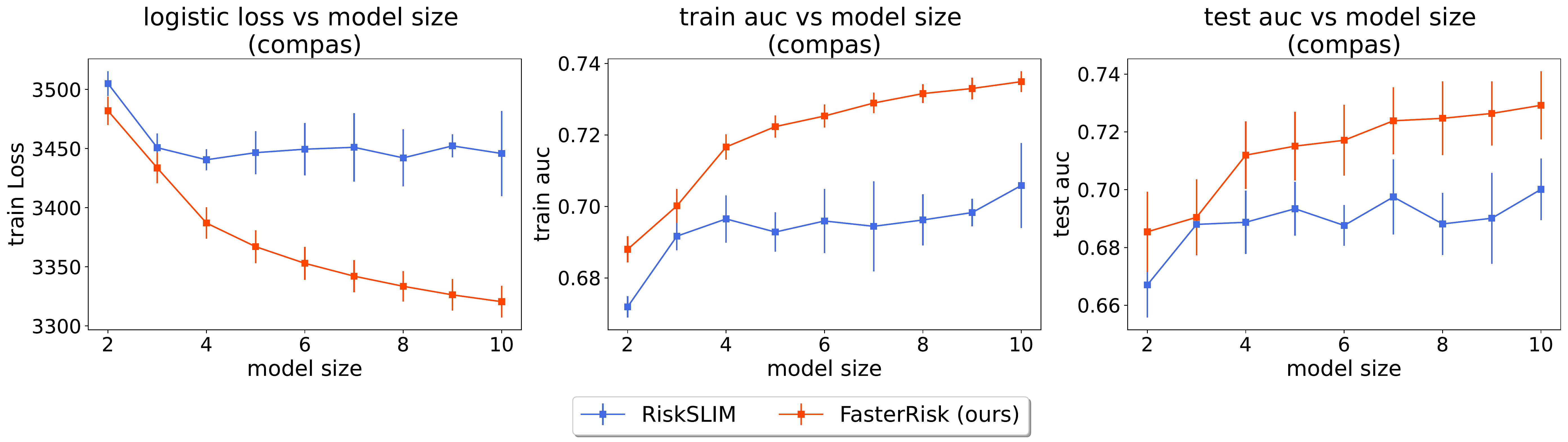}
    
    \includegraphics[width=\textwidth]{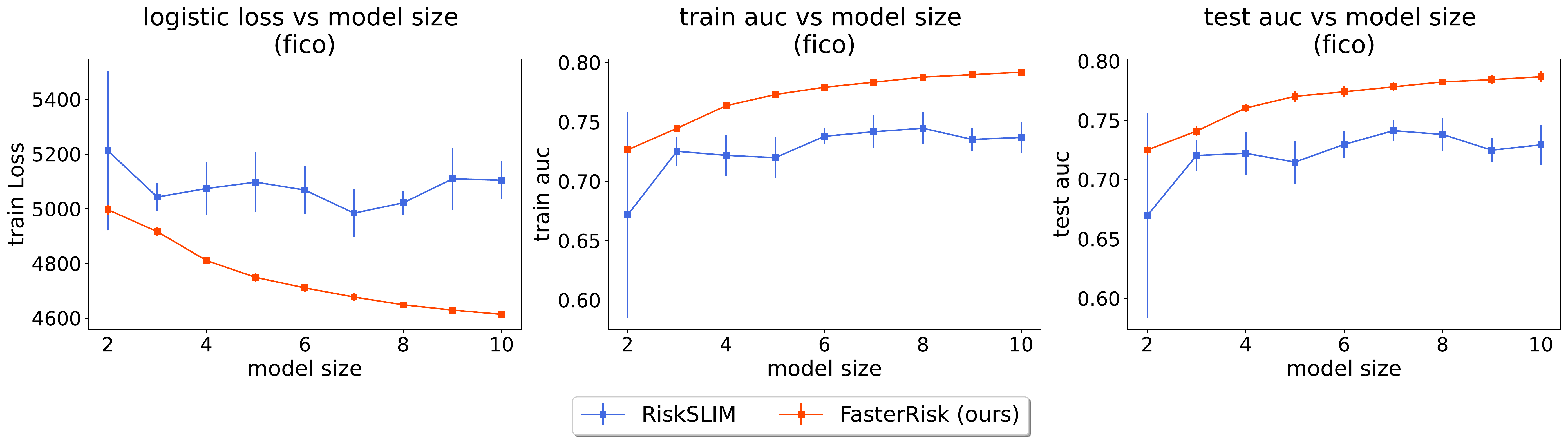}
    
    \caption{Comparison with the state-of-the-art baseline RiskSLIM (running for 1 hour) on the mushroom, COMPAS, and FICO datasets. The left column is loss (lower is better), the middle column is training AUC (higher is better) and the right column is test AUC (higher is better).
    }
    \label{fig:RiskSLIM_1h_mushroom_compas_fico}
\end{figure}

\begin{figure}[ht] 
    \centering
    \includegraphics[width=\textwidth]{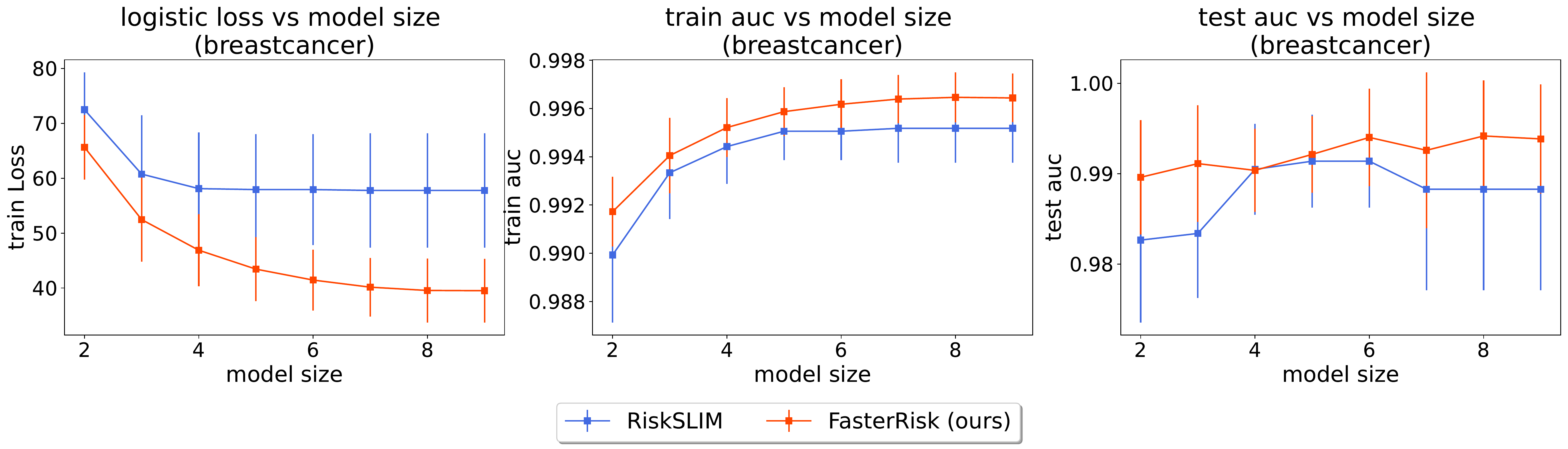}
    
    \includegraphics[width=\textwidth]{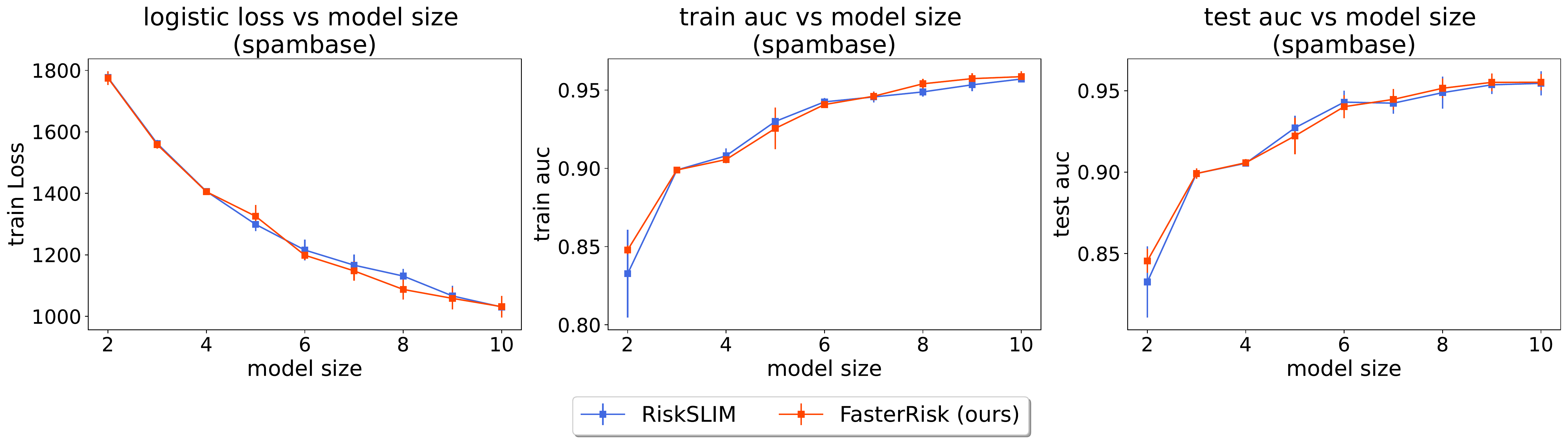}
    
    \includegraphics[width=\textwidth]{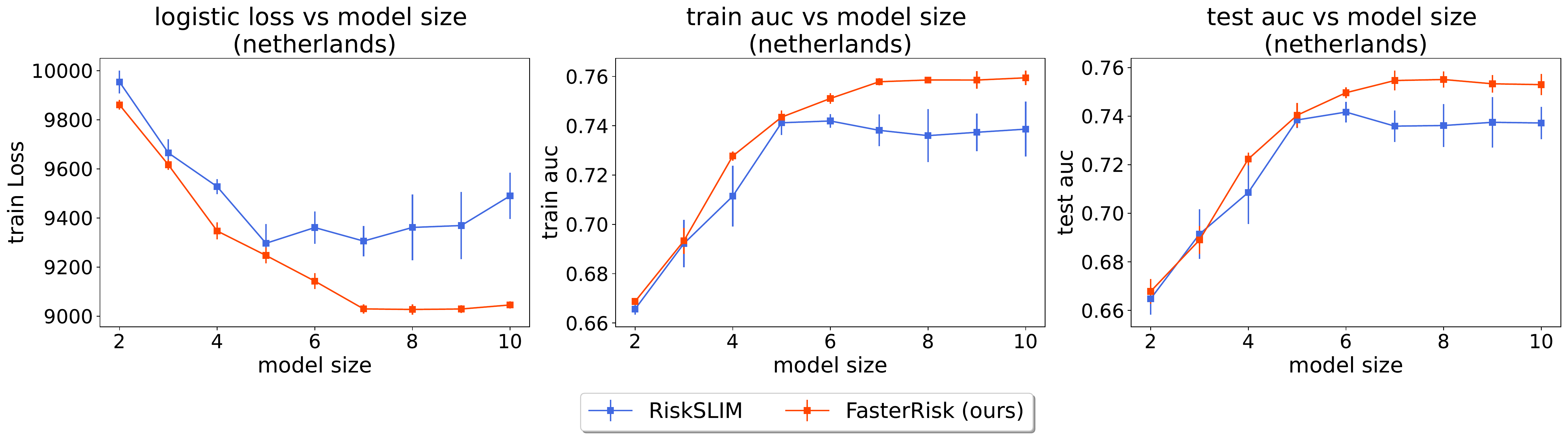}
    
    \caption{Comparison with the state-of-the-art baseline RiskSLIM (running for 1h) on the breastcancer, spambase, and Netherlands datasets. The left column is loss (lower is better), the middle column is training AUC (higher is better) and the right column is test AUC (higher is better).
    }
    \label{fig:RiskSLIM_1h_breastcancer_spambase_netherlands}
\end{figure}

\clearpage
\subsubsection{Time Comparison of Running RiskSLIM for 1 hour}

We plot the running time comparison between FasterRisk and RiskSLIM (with a time limit of 1 hour). The original time results with the 15-minute time limit are shown in Figure~\ref{fig:time_comparison} and Figure~\ref{fig:breastcancer_spambase_netherlands_runtimeCoparison}.

\begin{figure}[ht] 
    \centering
    \includegraphics[width=\textwidth]{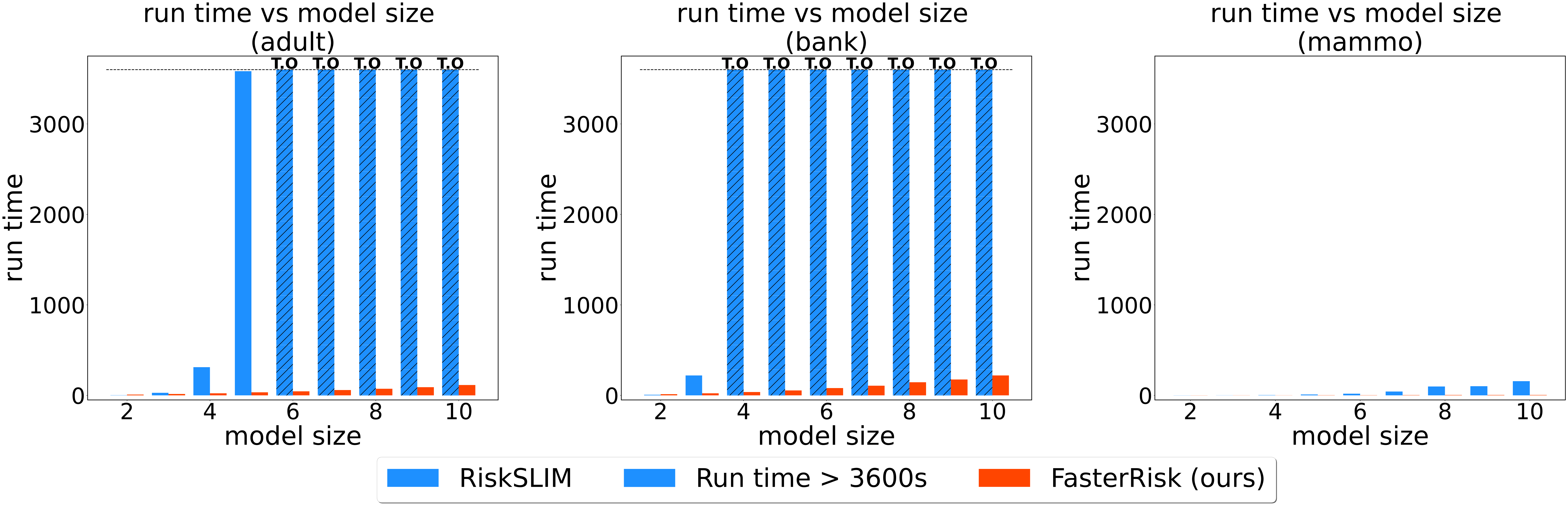}
    
    \includegraphics[width=\textwidth]{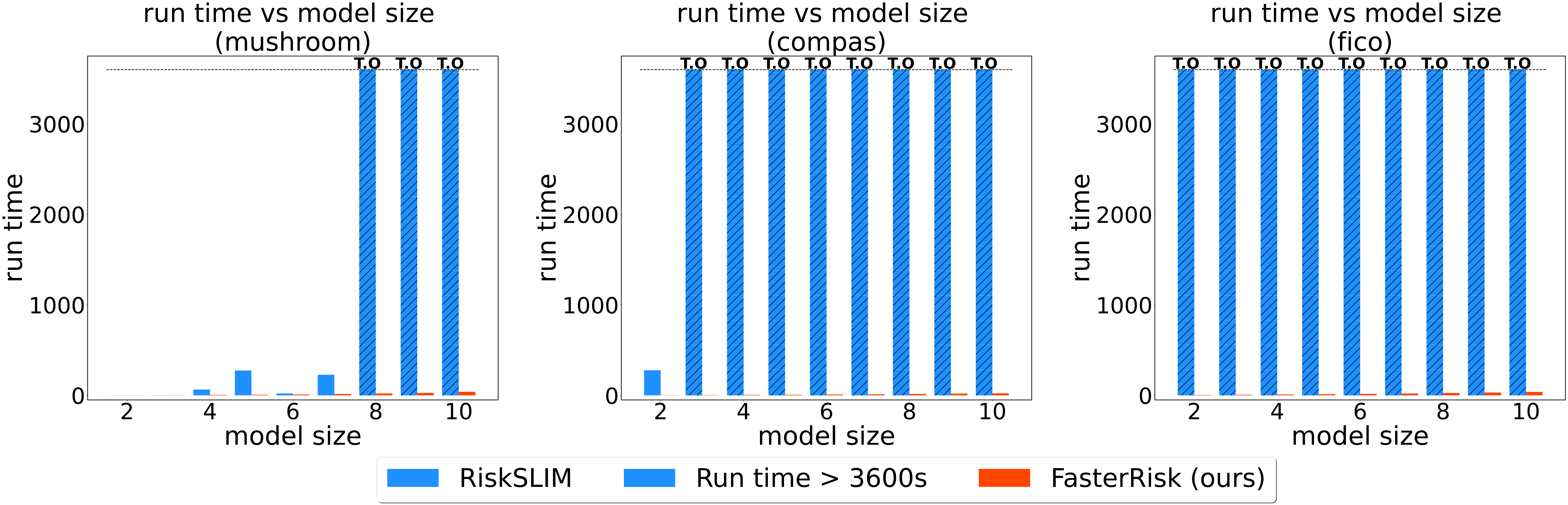}
    
    \includegraphics[width=\textwidth]{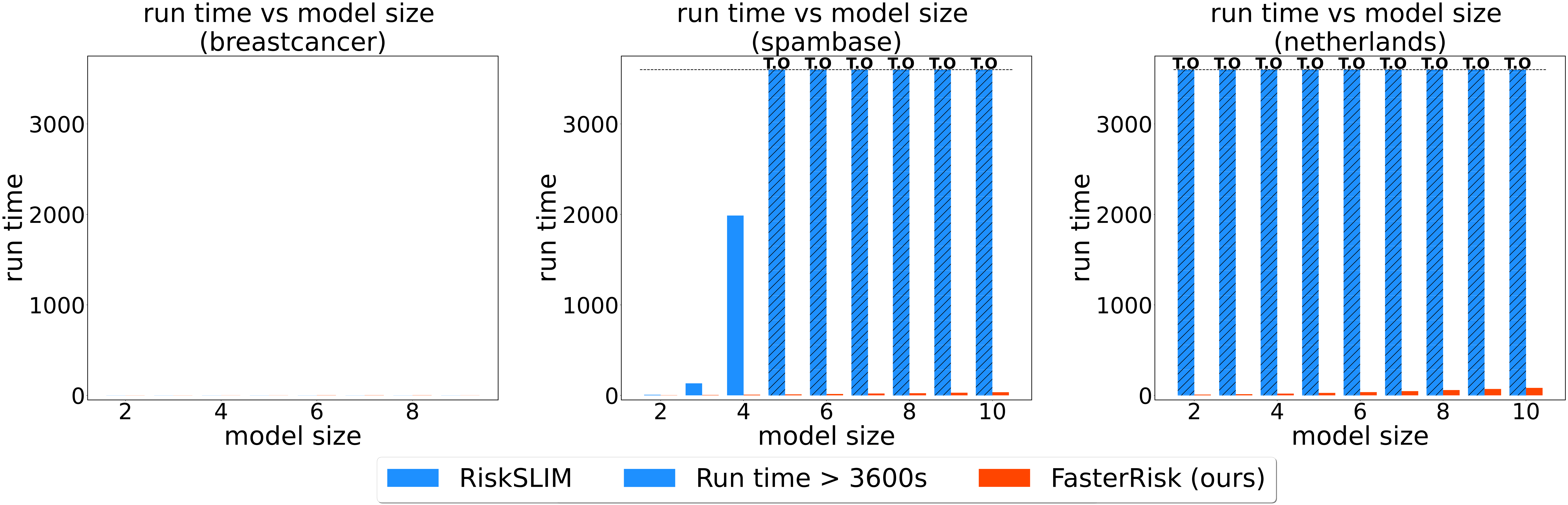}

    \caption{Runtime Comparison. Runtime (in seconds) versus model size for our method \ourmethod{} (in \textcolor{red}{red}) and the RiskSLIM (in \textcolor{blue}{blue}). The \textcolor{blue}{shaded blue} bars indicate cases that timed out at 1 hour. 
    Breastcancer is a small dataset so it takes approximately 2 seconds for both algorithms. For more zoomed-in results on the breastcancer and mammo datasets, please refer to Figure \ref{fig:time_comparison} and Figure~\ref{fig:breastcancer_spambase_netherlands_runtimeCoparison}.
    }
    \label{fig:running_time_1h}
\end{figure}

\clearpage
\subsubsection{Solution Quality of Running RiskSLIM for Days}

We report results of running the baseline RiskSLIM for 4 days. Due to this long running time demand on our servers, we could not run this experiment on all folds and all model sizes, so we only run on the 3rd fold of the 5-CV split. We plot the logistic loss progression over time.

The results are shown in Figure~\ref{fig:RiskSLIM_training_curve_vs_time}. We see that FasterRisk still achieves lower loss than RiskSLIM even after letting RiskSLIM run for 4 days, again because FasterRisk uses a larger model class. The only exceptions are on the Mushroom and the Spambase datasets, where the logistic losses are close to each other.

The major disadvantage of letting an algorithm run for days is that it is challenging to interact with the algorithm, because one has to wait for the results between interactions -- ideally this process would be instantaneous. Furthermore, there could be memory issues for the MIP solver if we let it run for days since the branch-and-bound tree could become too large.

\begin{figure}[ht] 
    \centering
    \includegraphics[width=\textwidth]{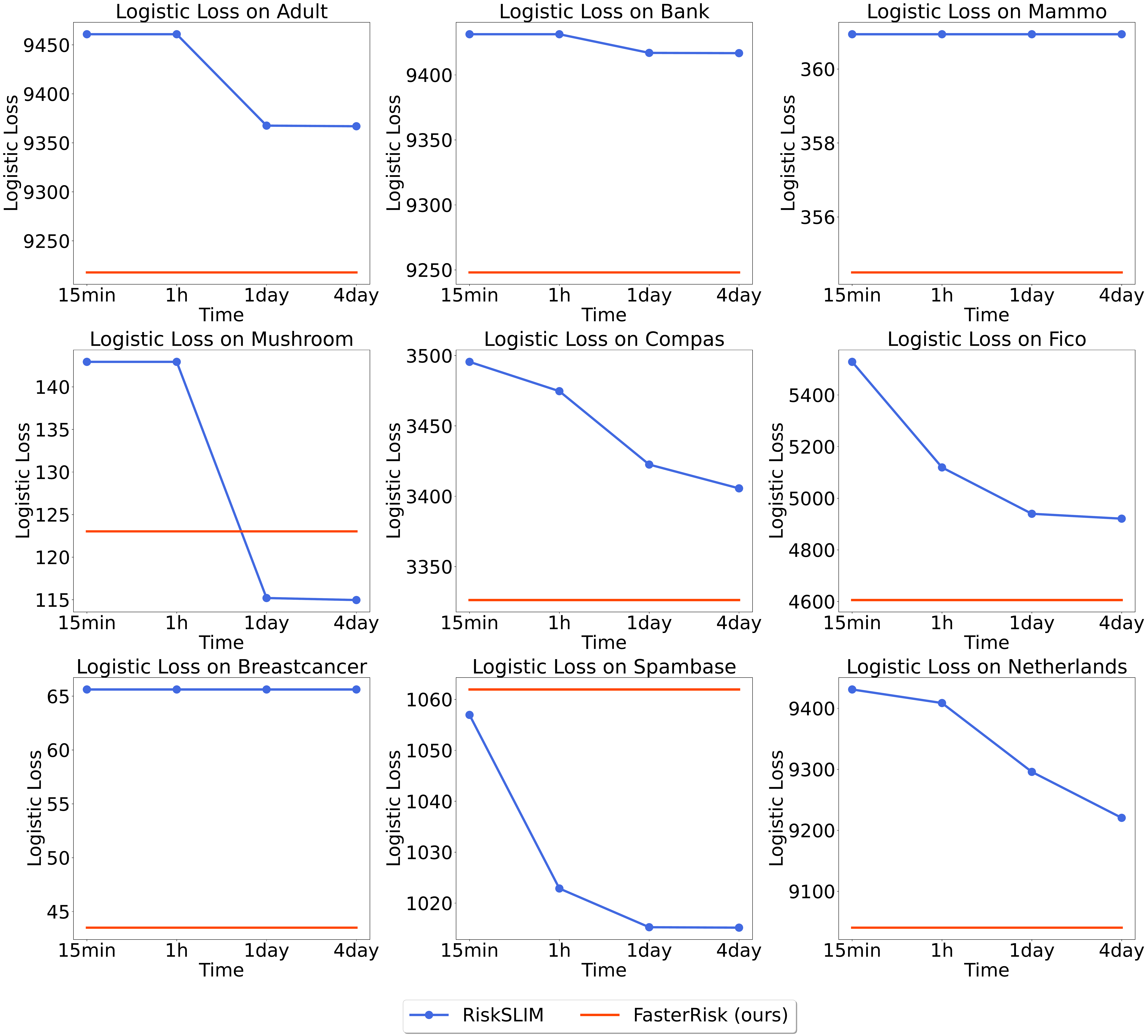}
    
    \caption{Curves of logistic loss vs$.$ training time for the RiskSLIM model on the 3rd fold of the 5-CV split with model size equal to 10. All plots report logistic loss (lower is better).
    }
    \label{fig:RiskSLIM_training_curve_vs_time}
\end{figure}

\clearpage
\subsection{Calibration Curves}
\label{appendix:calibration_curves}
The calibration curves for RiskSLIM and FasterRisk are shown in Figures ~\ref{fig:calibration_curve_model_size_3}-\ref{fig:calibration_curve_model_size_7} with model sizes equal to 3, 5, and 7, respectively. We use the sklearn package\footnote{\url{https://scikit-learn.org/stable/modules/generated/sklearn.calibration.calibration_curve.html}} from python to plot the figures. We use the default value for the number of bins (number of bins is 5) and the default strategy to define the widths of the bins (the strategy is ``uniform'').

The calibration curves on the breastcancer and mammo datasets are more spread out than those on the other datasets. This is perhaps due to the limited number of samples in these datasets (both datasets have fewer than 1000 samples in total; see Table~\ref{tab:data_info}), which increases the variance in the calculation of the curves.

On other datasets, both methods have good calibration curves, showing consistency between predicted score and actual risk.
However, as shown in Figures ~\ref{fig:RiskSLIM_1h_adult_bank_mammo}-\ref{fig:RiskSLIM_1h_breastcancer_spambase_netherlands}, FasterRisk has higher AUC scores, which means our method has higher discrimination ability than RiskSLIM.

\begin{figure}[ht] 
    \centering
    \includegraphics[width=\textwidth]{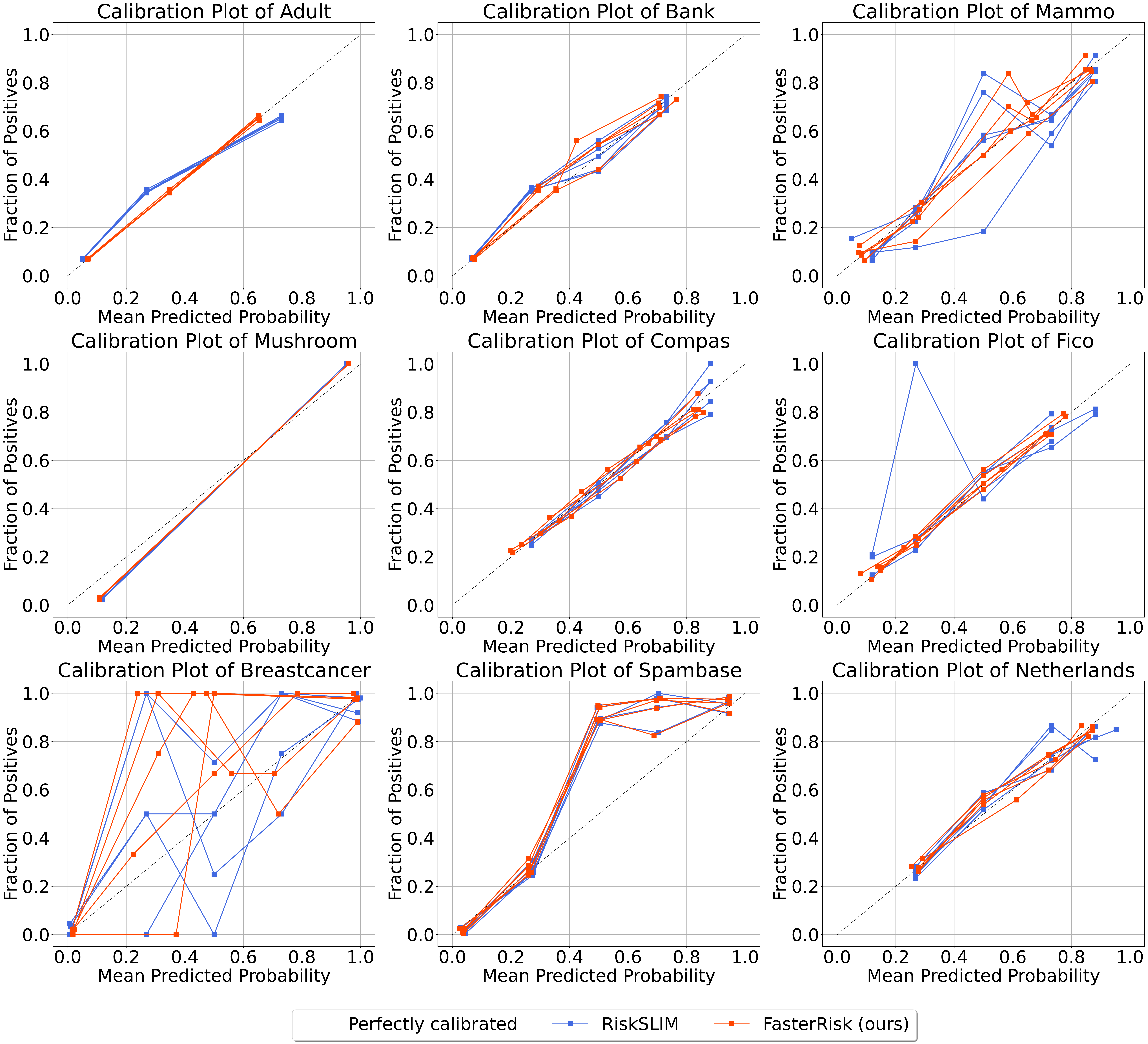}
    
    \caption{Calibration curves for RiskSLIM and \ourmethod{} with model size equal to 3. We plot results from each test fold. The \ourmethod{} model  selected from the pool is that with the smallest logistic loss on the training set.
    }
    \label{fig:calibration_curve_model_size_3}
\end{figure}

\begin{figure}[ht] 
    \centering
    \includegraphics[width=\textwidth]{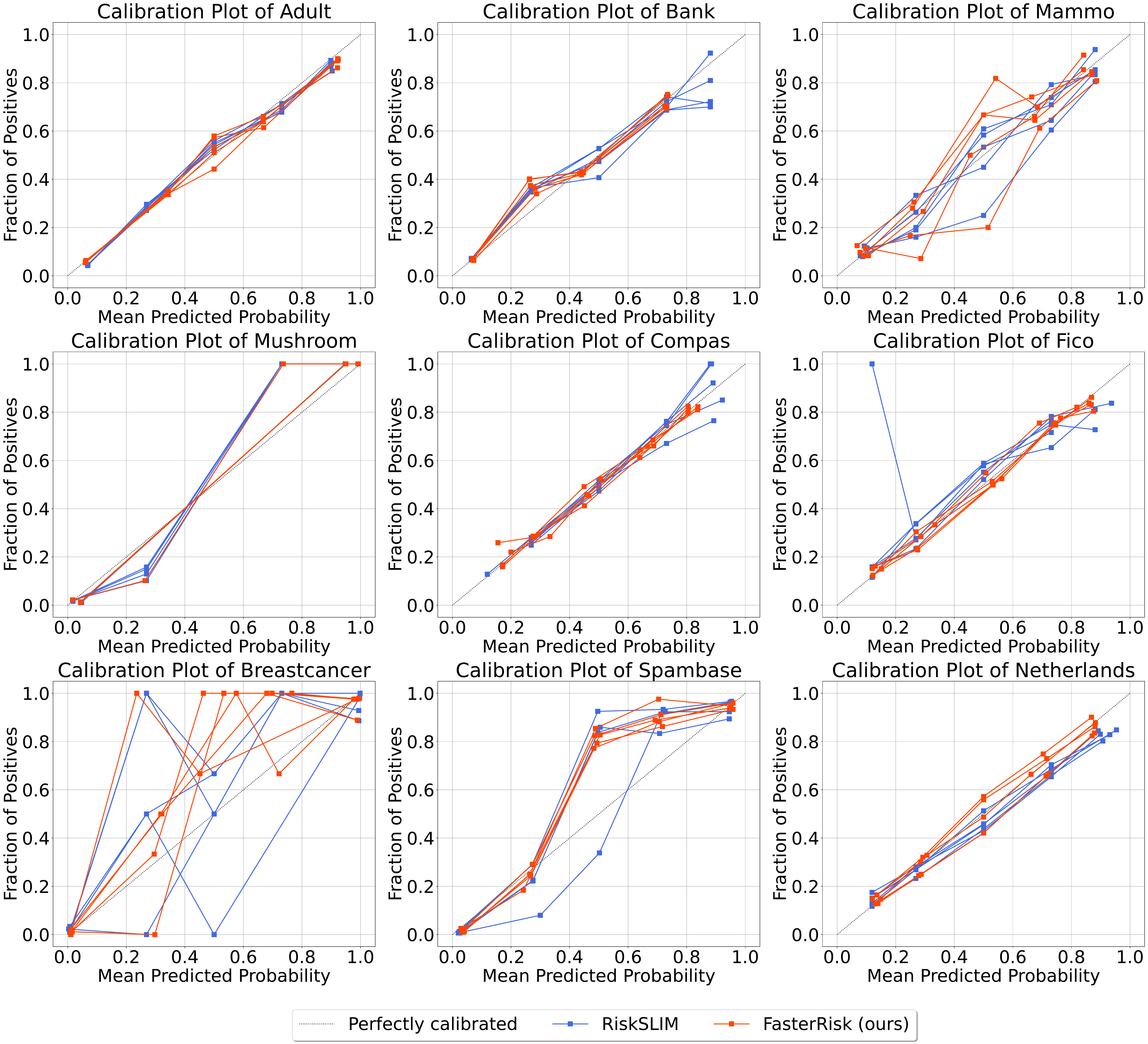}
    
    \caption{Calibration curves for RiskSLIM and FasterRisk with model size equal to 5. We plot results from each test fold. The \ourmethod{} model  selected from the pool is that with the smallest logistic loss on the training set.
    }
    \label{fig:calibration_curve_model_size_5}
\end{figure}

\begin{figure}[ht] 
    \centering
    \includegraphics[width=\textwidth]{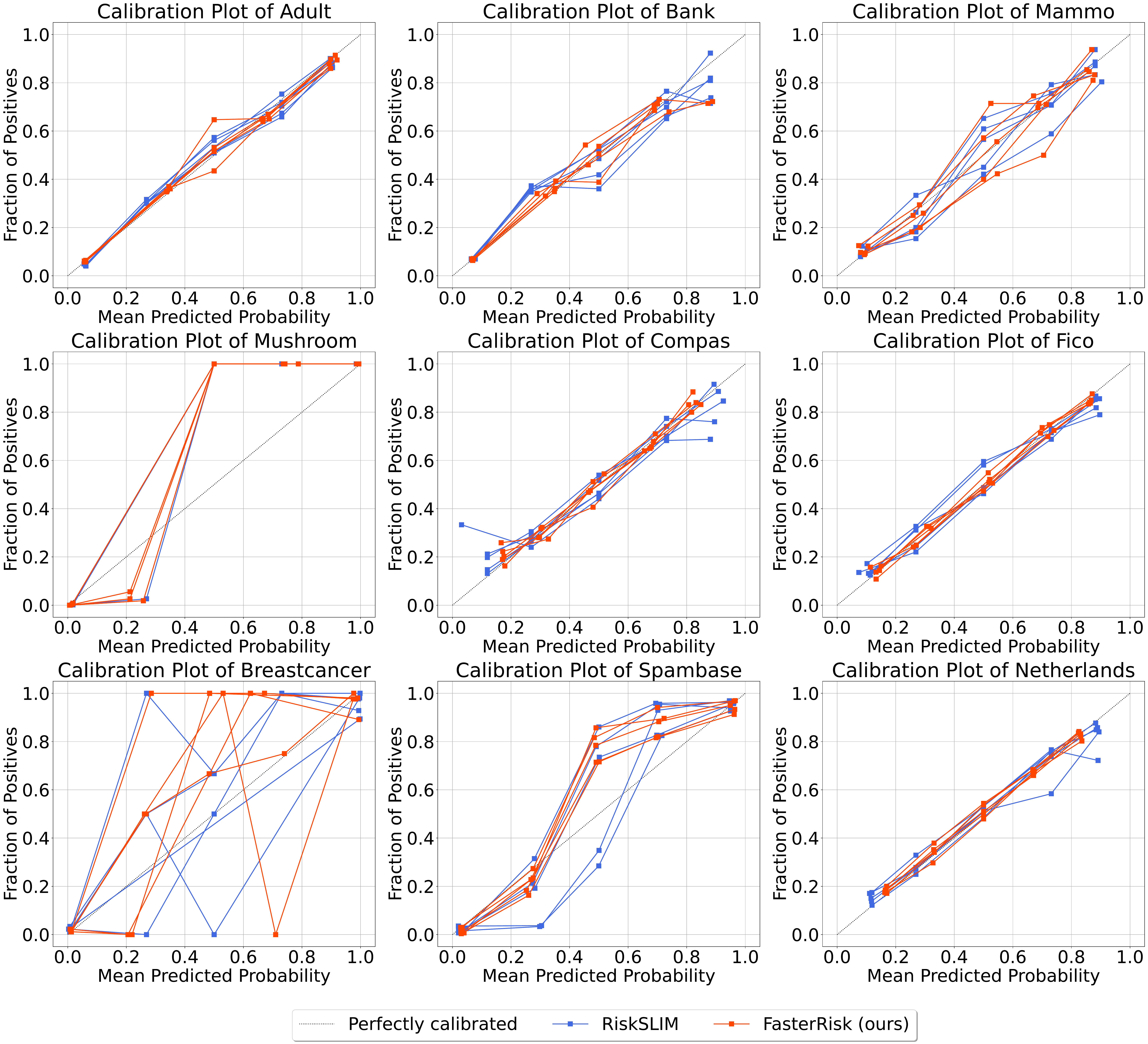}
    
    \caption{Calibration curves for RiskSLIM and FasterRisk with model size equal to 7. We plot results from each fold on the test set. The \ourmethod{} model  selected from the pool is that with the smallest logistic loss on the training set.
    }
    \label{fig:calibration_curve_model_size_7}
\end{figure}

\clearpage
\subsection{Hyperparameter Perturbation Study}
\label{appendix:hyperparameter_perturbation_study}

\subsubsection{Perturbation Study on Beam Size \texorpdfstring{$B$}{B}}

We perform a perturbation study on the hyperparameter beam size $B$ as mentioned in Appendix~\ref{sec:hyperparameter_specification}. We set the beam size to 5, 10, and 15, respectively. The results are shown in Figures~\ref{fig:hyperparameter_beam_size_adult_bank_mammo}-\ref{fig:hyperparameter_beam_size_breastcancer_spambase_netherlands}. The curves greatly overlap, confirming our previous claim that the performance is not particularly sensitive to the choice of $B$.

\begin{figure}[ht] 
    \centering
    \includegraphics[width=\textwidth]{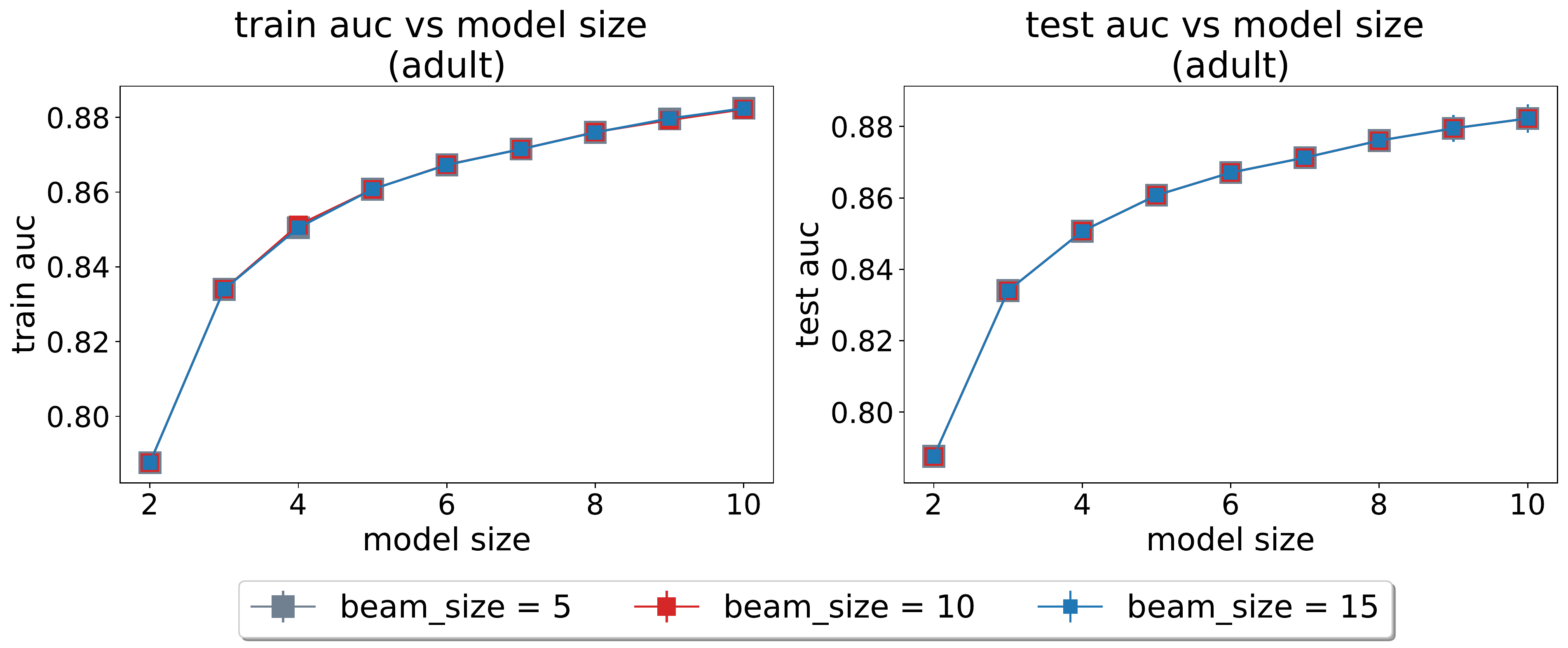}
    
    \includegraphics[width=\textwidth]{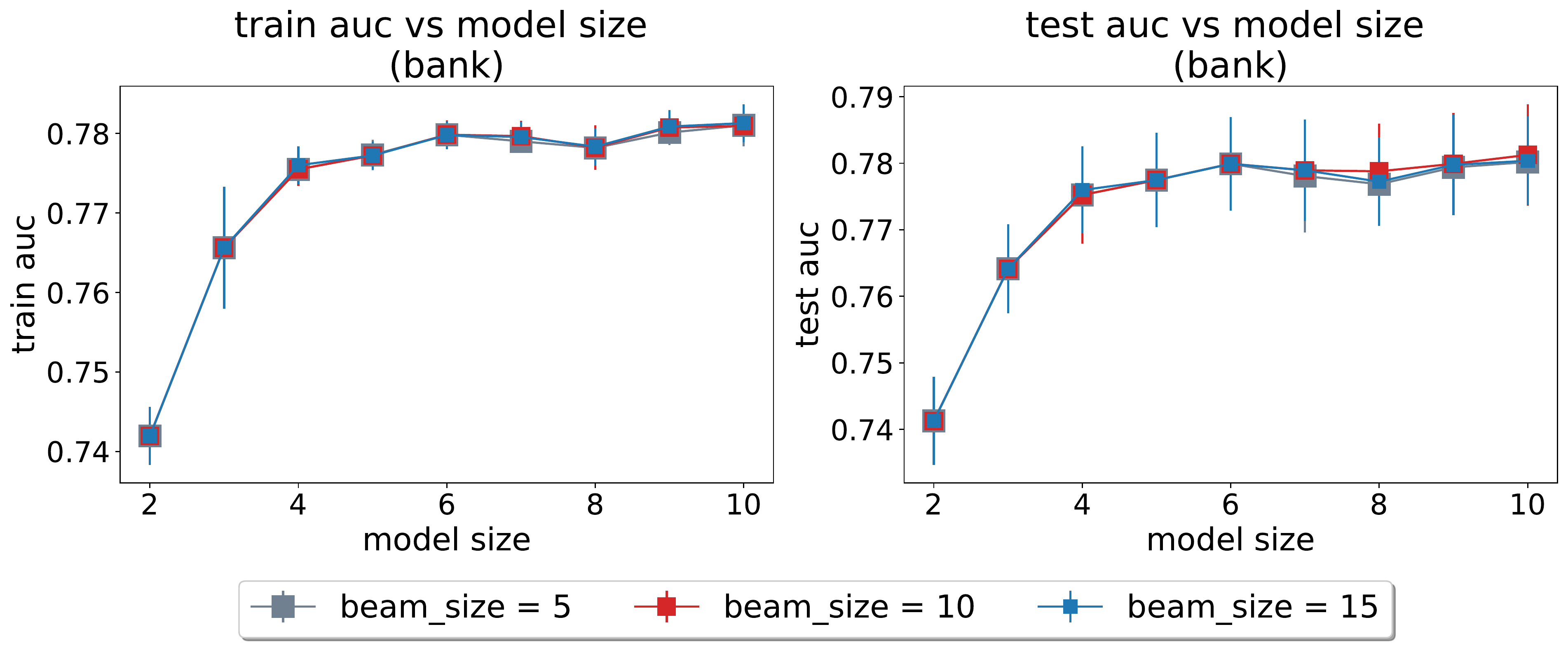}
    
    \includegraphics[width=\textwidth]{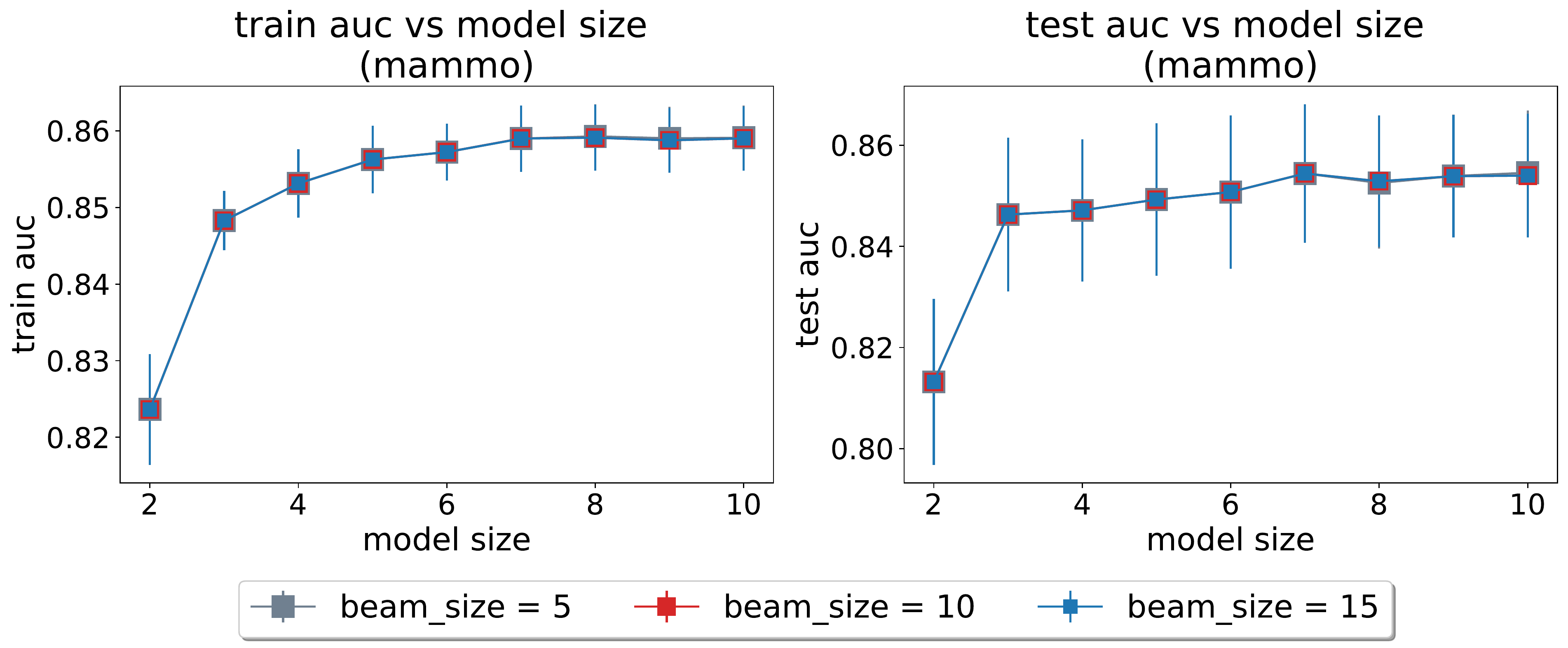}
    
    \caption{Perturbation study for beam size, $B$, on the adult, bank, and mammo datasets. The default value used in the paper is 10.
    }
    \label{fig:hyperparameter_beam_size_adult_bank_mammo}
\end{figure}

\begin{figure}[ht] 
    \centering
    \includegraphics[width=\textwidth]{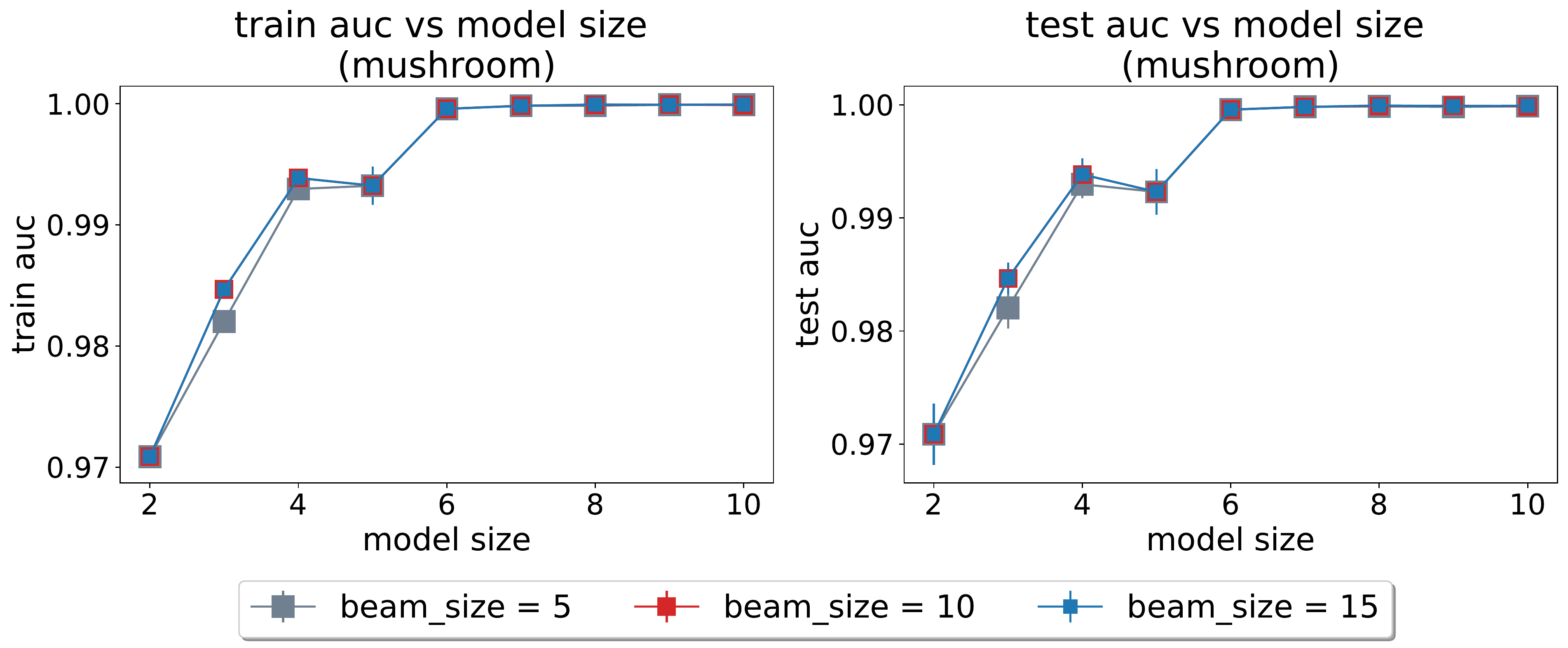}
    
    \includegraphics[width=\textwidth]{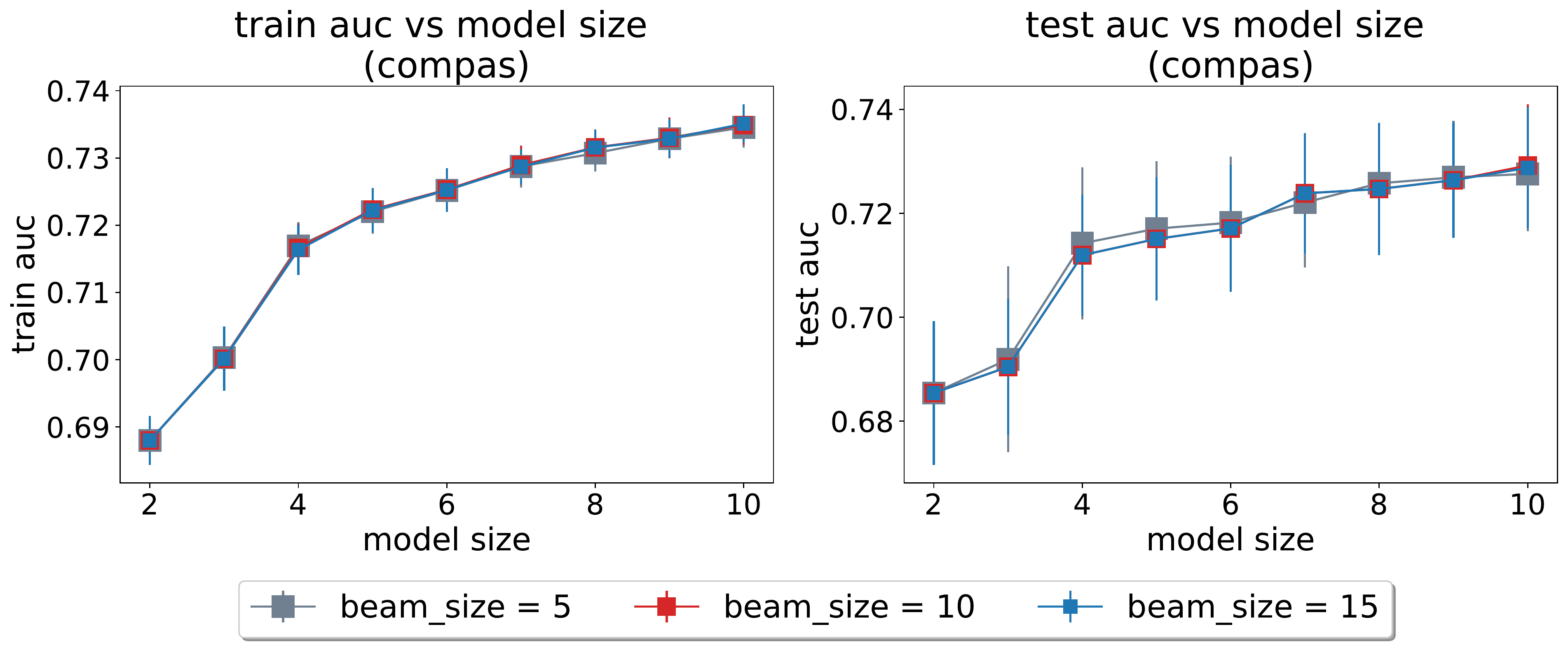}
    
    \includegraphics[width=\textwidth]{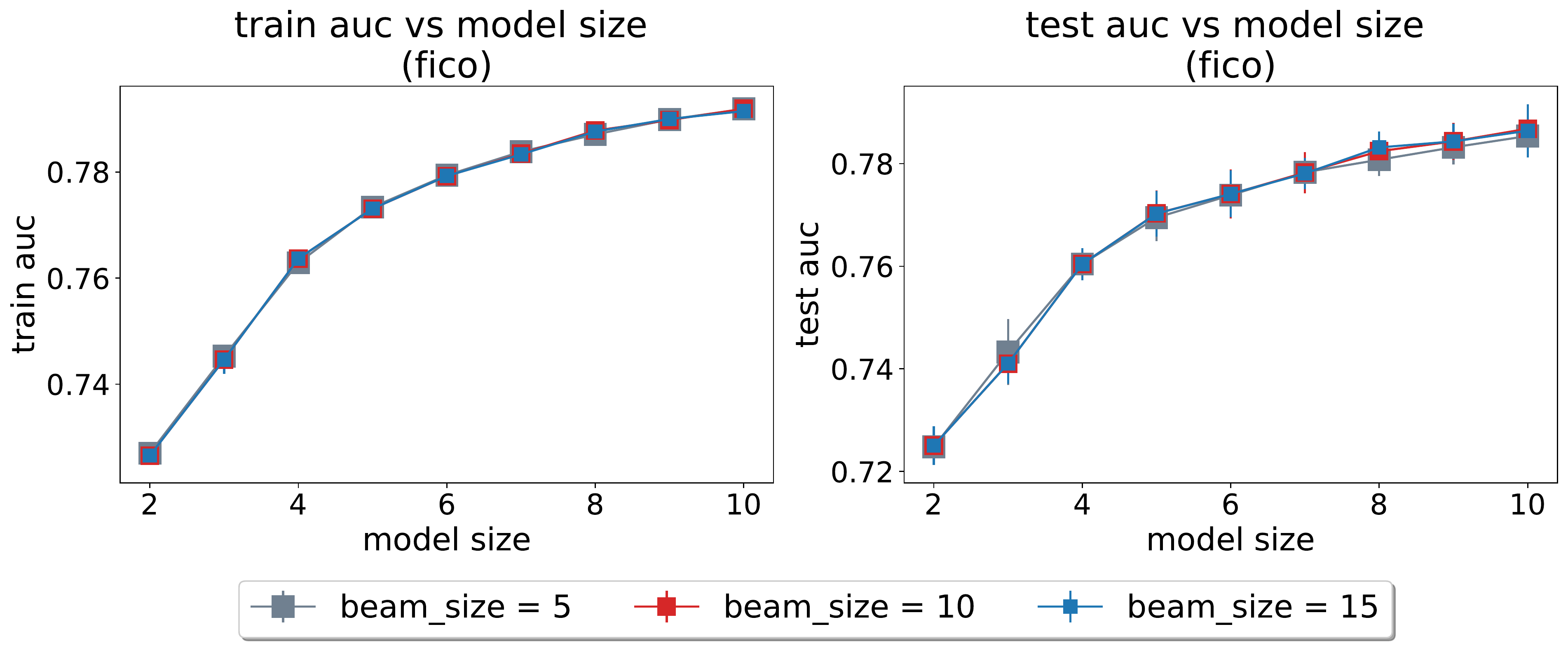}
    
    \caption{Perturbation study for beam size, $B$, on the mushroom, COMPAS, and FICO datasets. The default value used in the paper is 10.
    }
    \label{fig:hyperparameter_beam_size_mushroom_compas_fico}
\end{figure}

\begin{figure}[ht] 
    \centering
    \includegraphics[width=\textwidth]{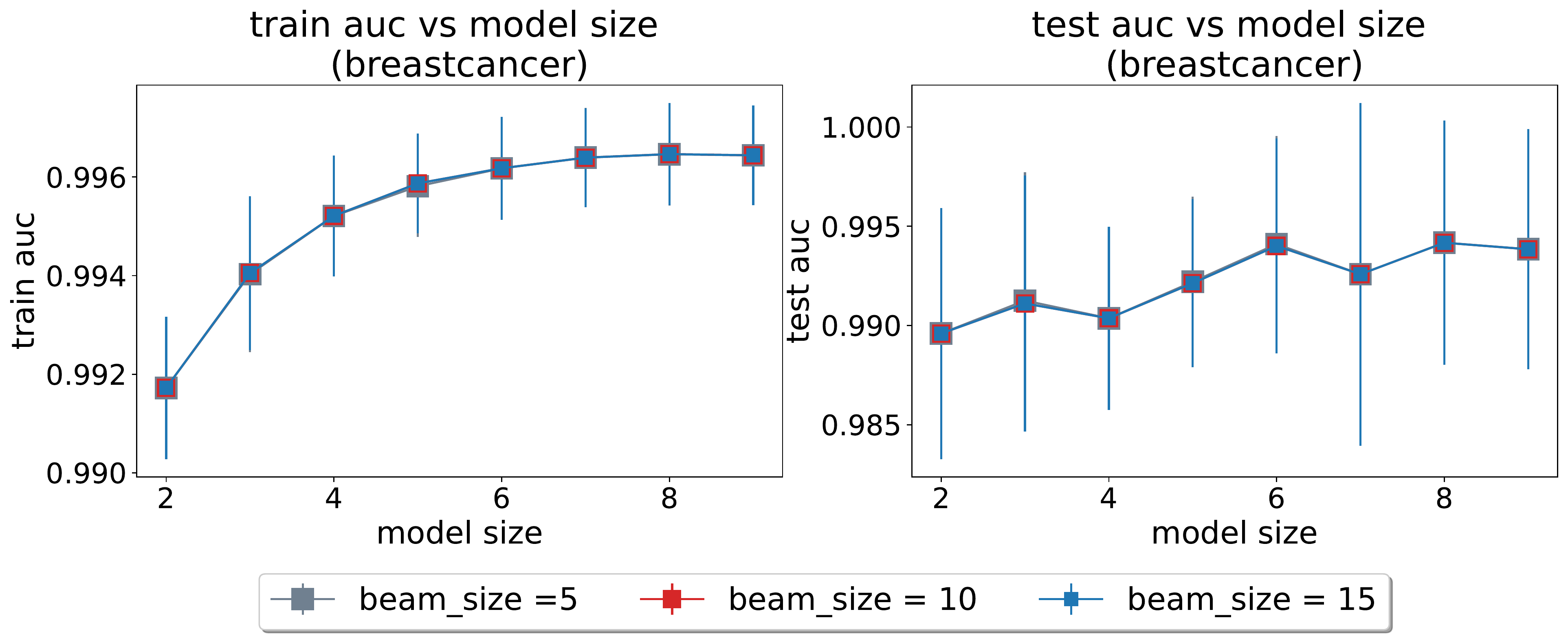}
    
    \includegraphics[width=\textwidth]{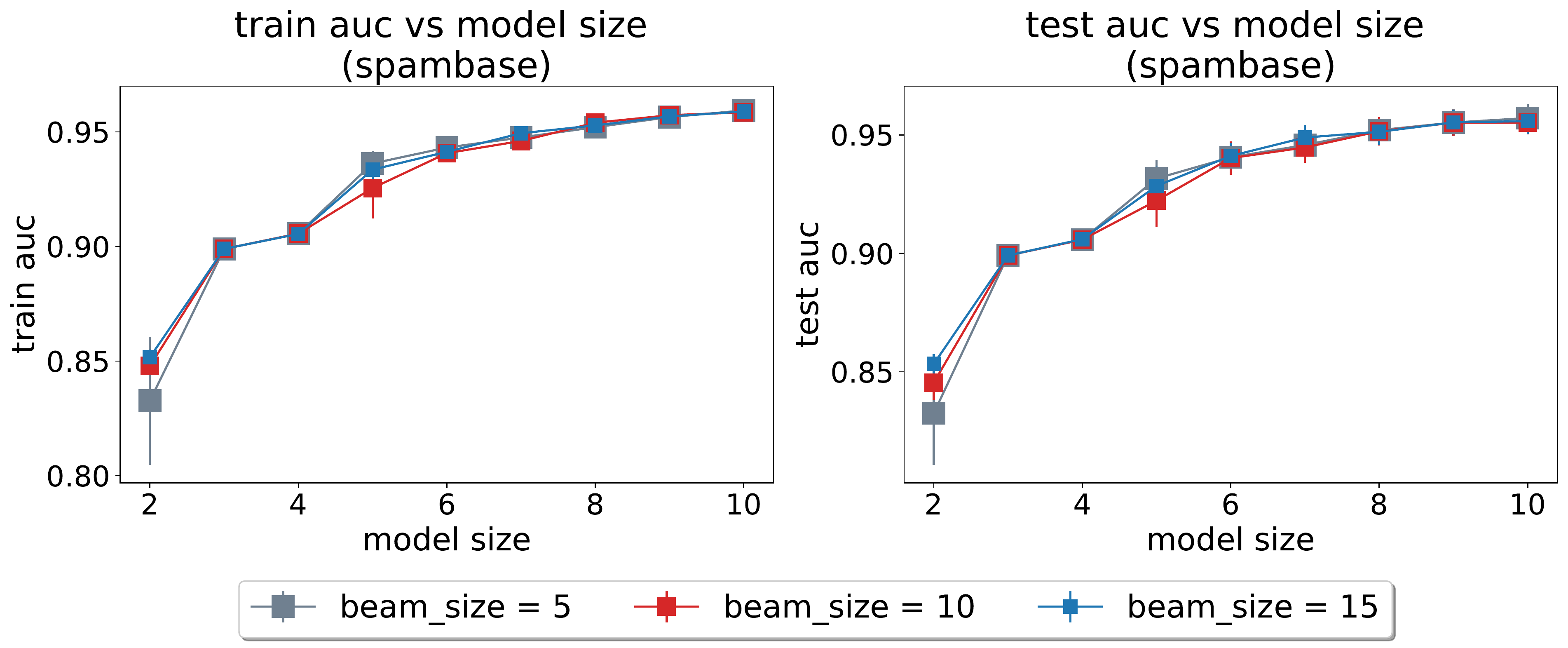}
    
    \includegraphics[width=\textwidth]{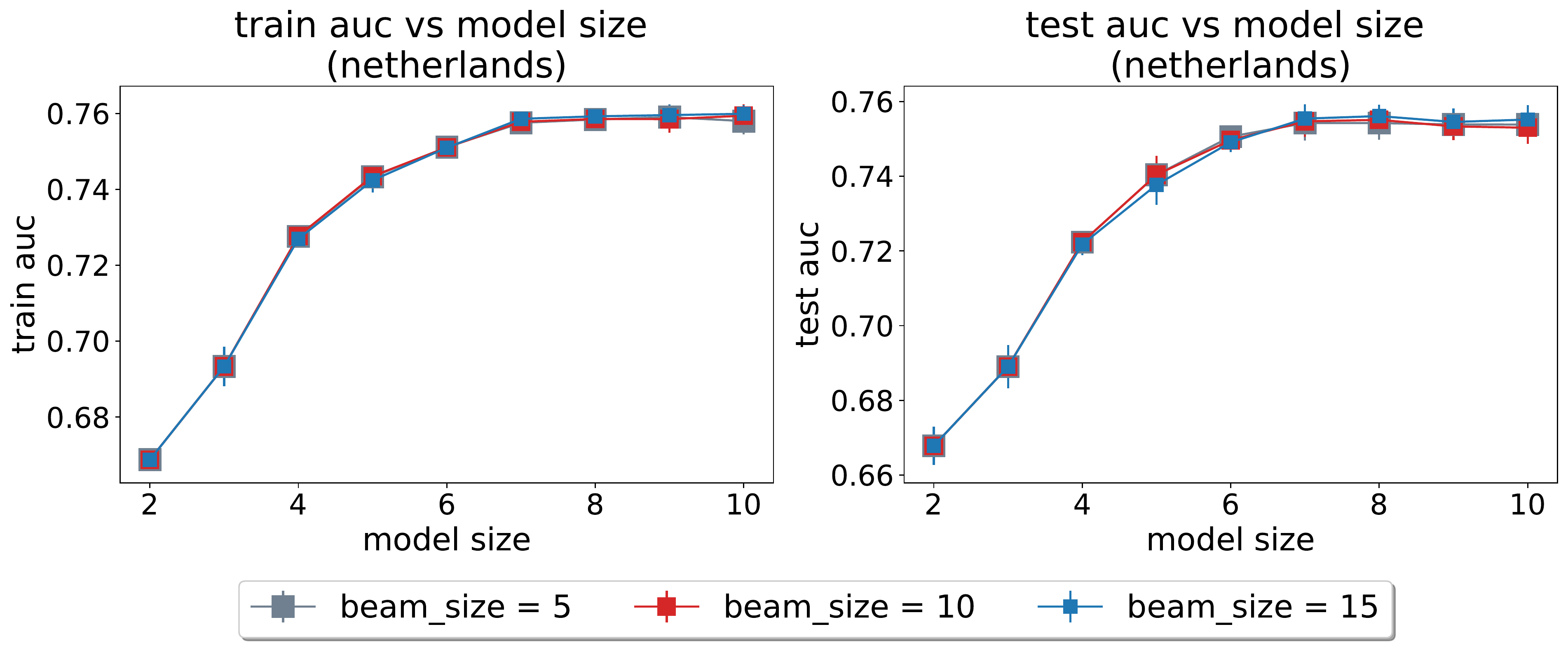}
    
    \caption{Perturbation study for beam size, $B$, on the breastcancer, spambase, and Netherlands datasets. The default value used in the paper is 10.
    }
    \label{fig:hyperparameter_beam_size_breastcancer_spambase_netherlands}
\end{figure}

\clearpage
\subsubsection{Perturbation Study on Tolerance Level \texorpdfstring{$\epsilon$}{epsilon} for Sparse Diverse Pool }

We perform a perturbation study on the hyperparameter tolerance level, $\epsilon$, as mentioned in Appendix~\ref{sec:hyperparameter_specification}. We set the tolerance level to 0.1, 0.3, and 0.5, respectively. The results are shown in Figures~\ref{fig:hyperparameter_gap_tolerance_adult_bank_mammo}-\ref{fig:hyperparameter_gap_tolerance_breastcancer_spambase_netherlands}. The curves greatly overlap, confirming our previous claim that the performance is not particularly sensitive to the choice of value.

\begin{figure}[ht] 
    \centering
    \includegraphics[width=\textwidth]{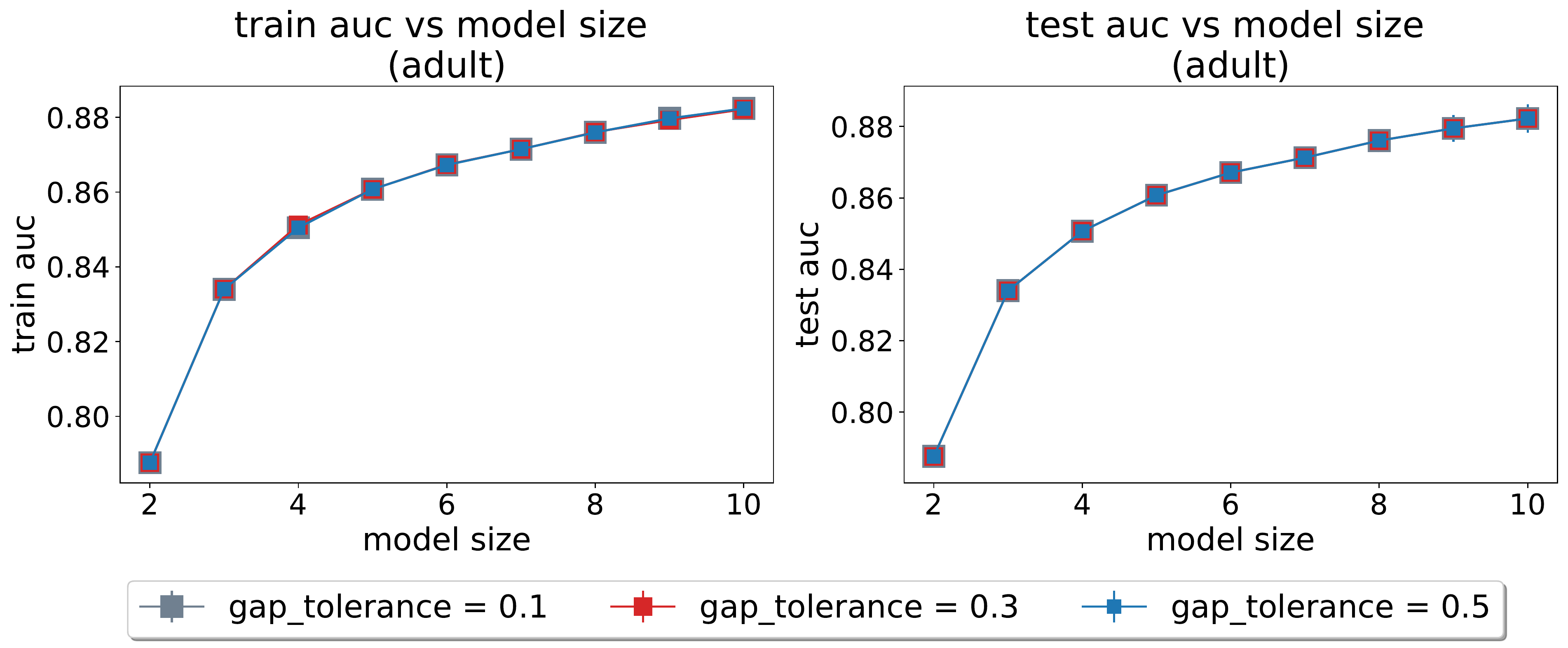}
    
    \includegraphics[width=\textwidth]{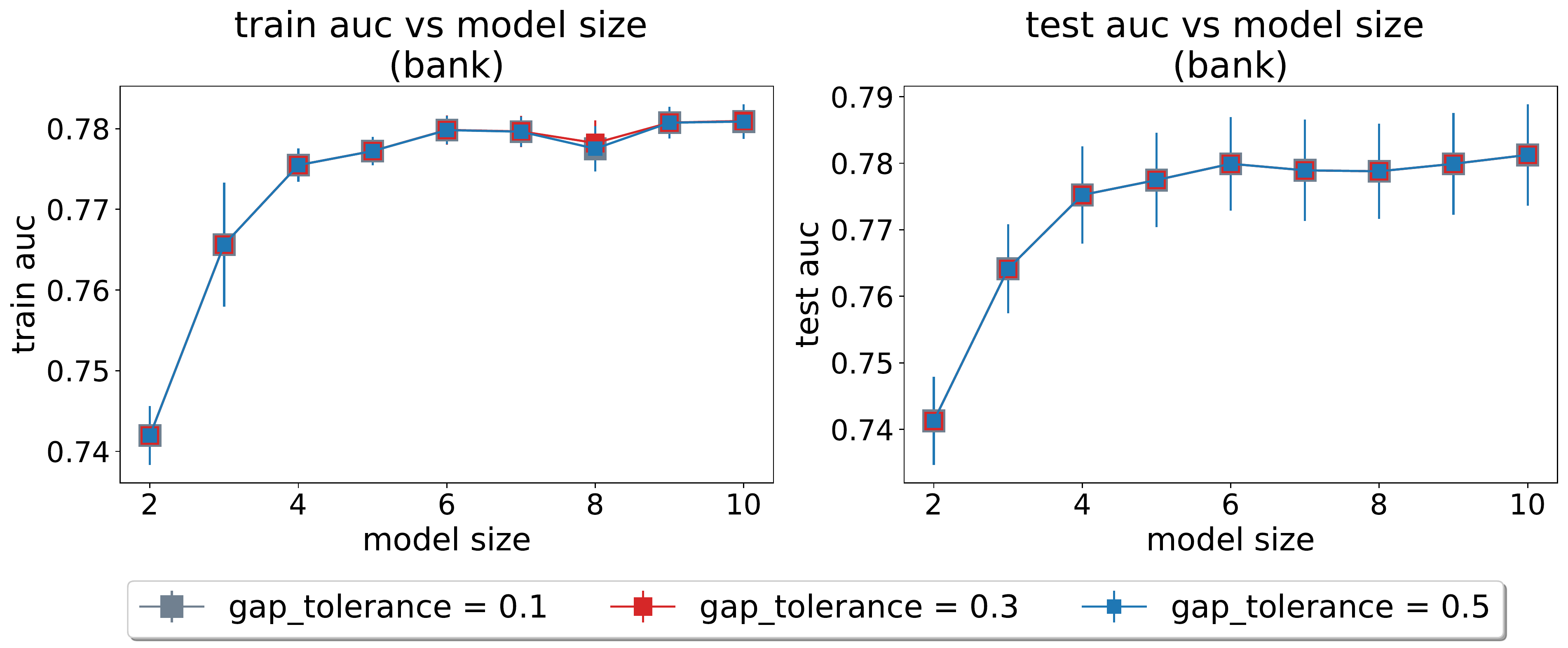}
    
    \includegraphics[width=\textwidth]{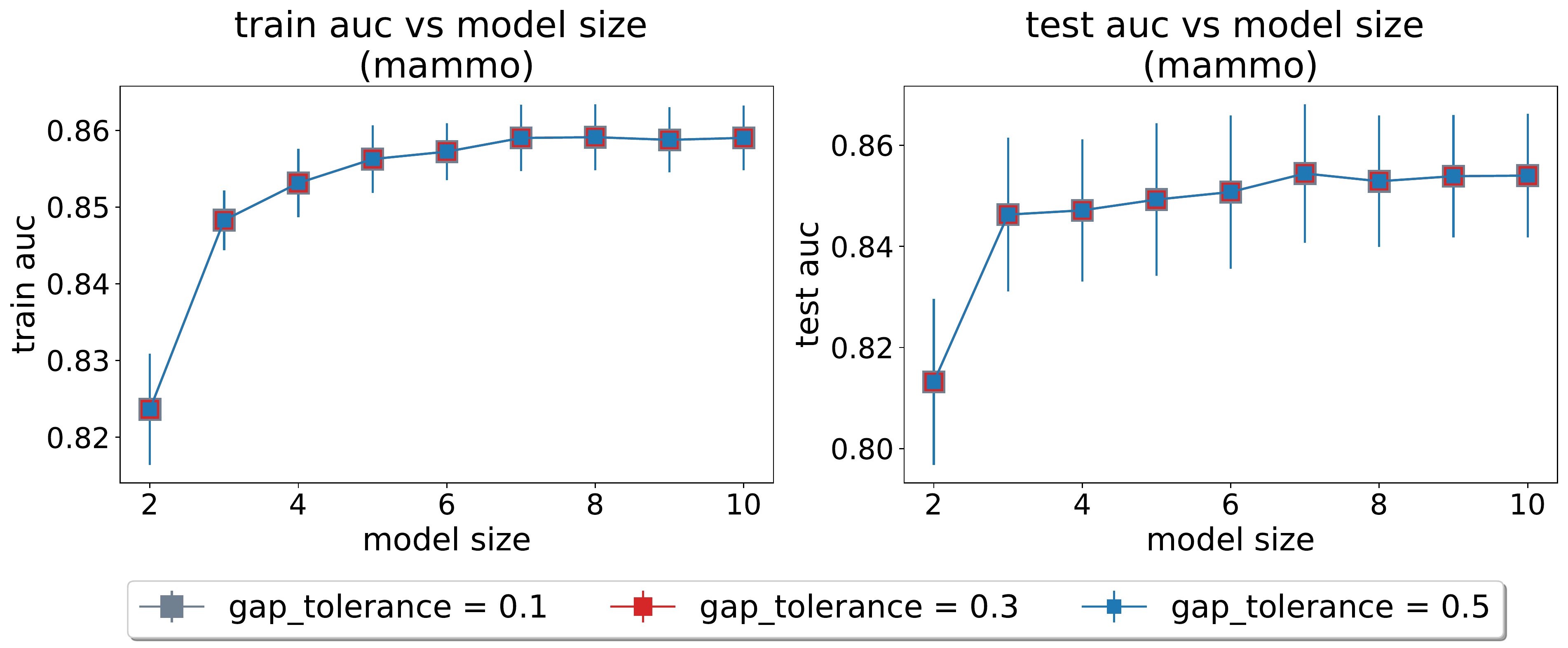}
    
    \caption{Perturbation study on tolerance level, $\epsilon$, for sparse diverse pool on the adult, bank, and mammo datasets. The default value used in the paper is 0.3.
    }
    \label{fig:hyperparameter_gap_tolerance_adult_bank_mammo}
\end{figure}

\begin{figure}[ht] 
    \centering
    \includegraphics[width=\textwidth]{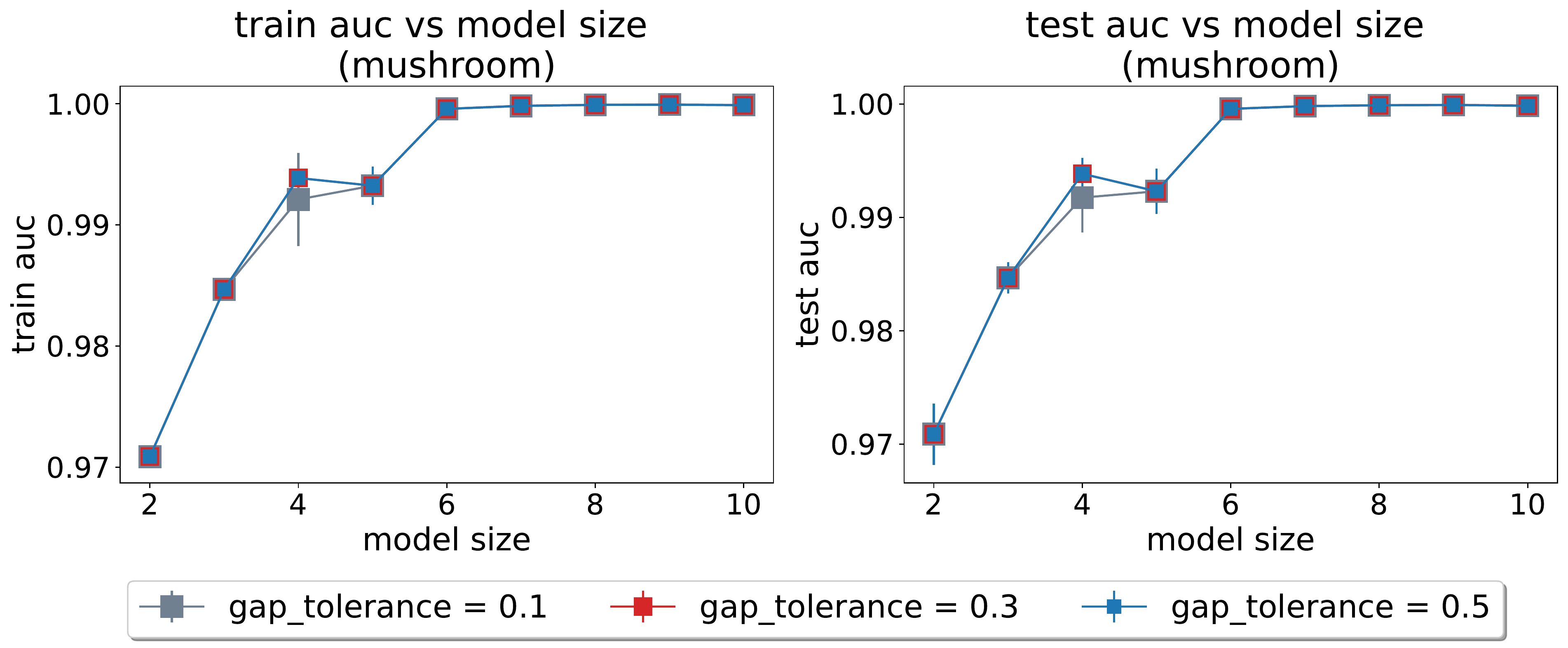}
    
    \includegraphics[width=\textwidth]{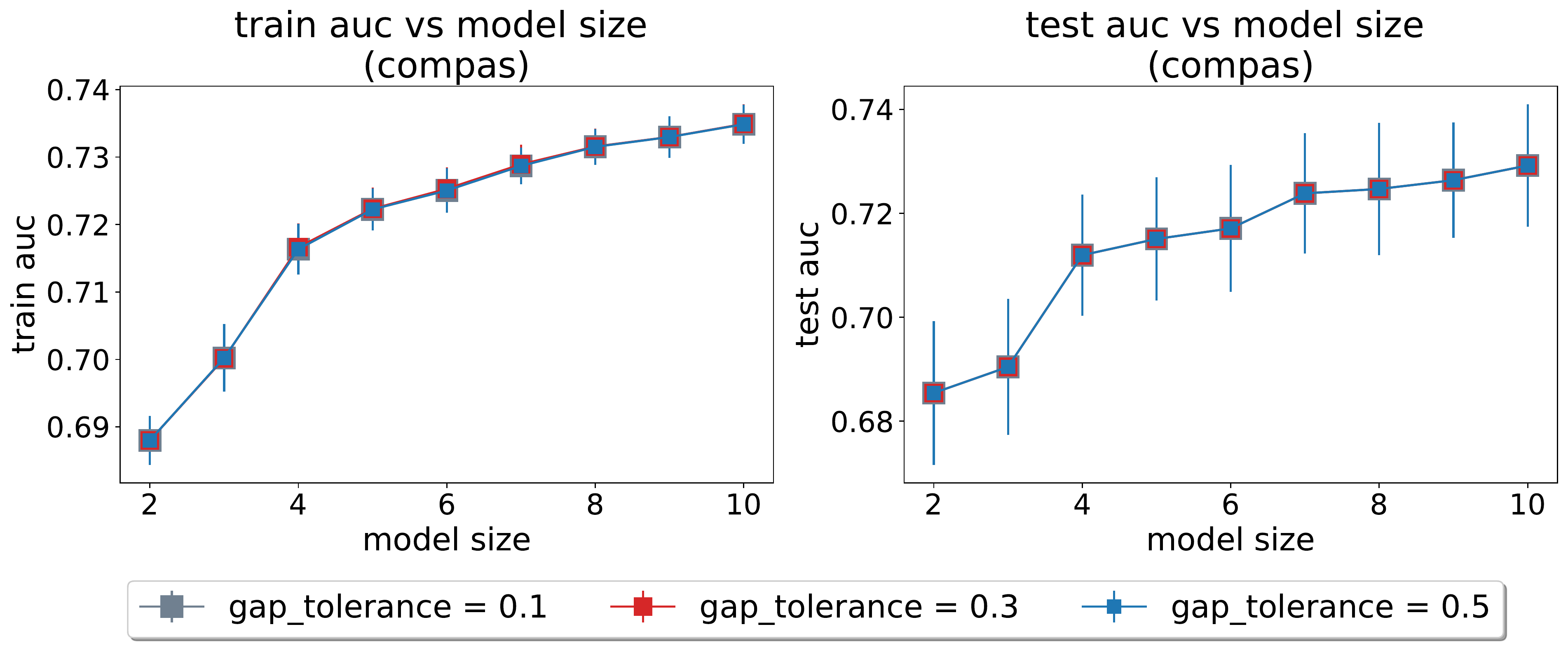}
    
    \includegraphics[width=\textwidth]{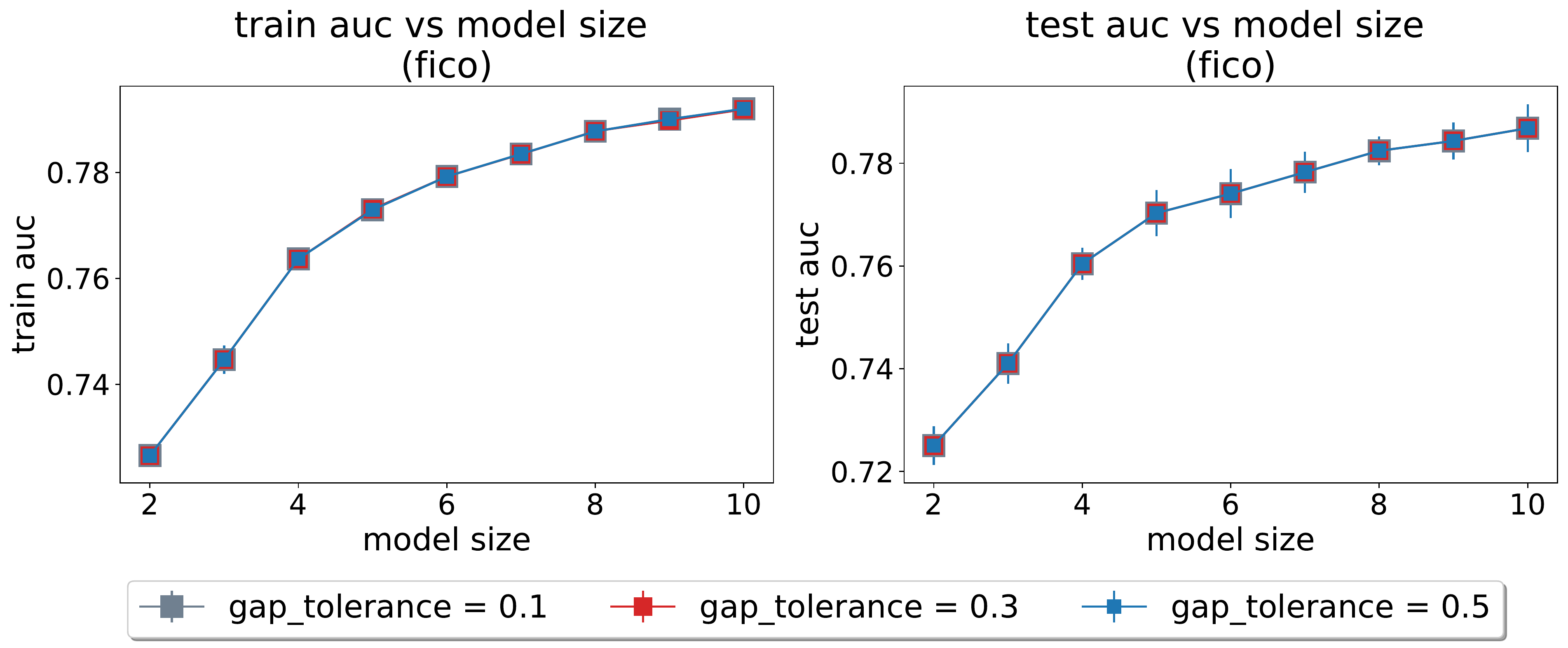}
    
    \caption{Perturbation study on tolerance level, $\epsilon$, for sparse diverse pool on the mushroom, COMPAS and FICO datasets. The default value used in the paper is 0.3.
    }
    \label{fig:hyperparameter_gap_tolerance_mushroom_compas_fico}
\end{figure}

\begin{figure}[ht] 
    \centering
    \includegraphics[width=\textwidth]{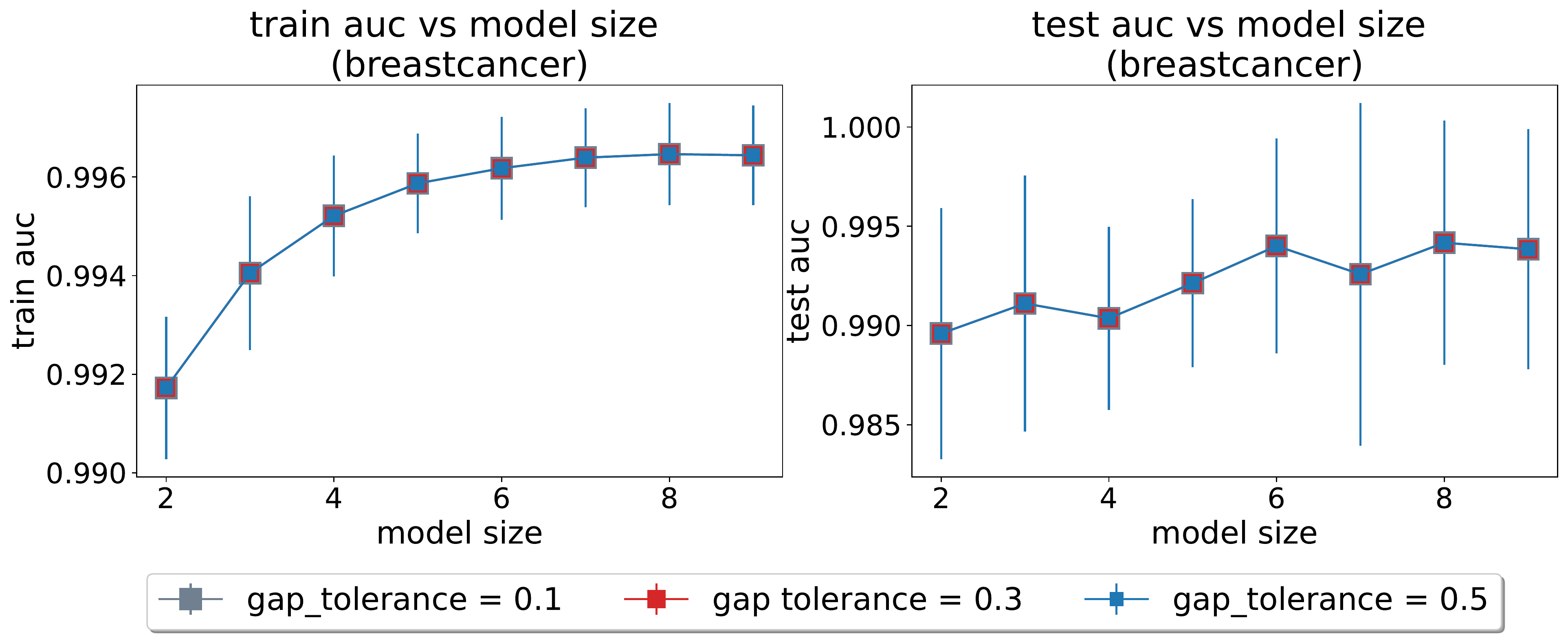}
    
    \includegraphics[width=\textwidth]{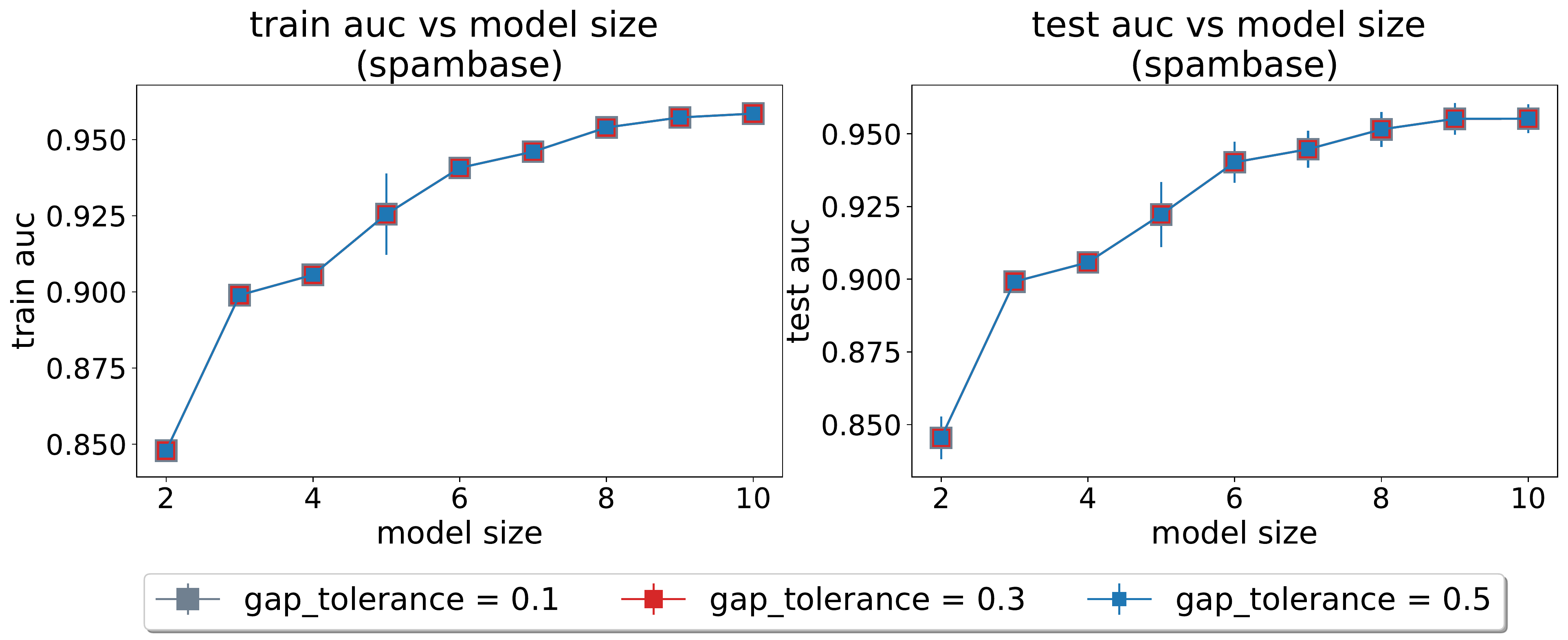}
    
    \includegraphics[width=\textwidth]{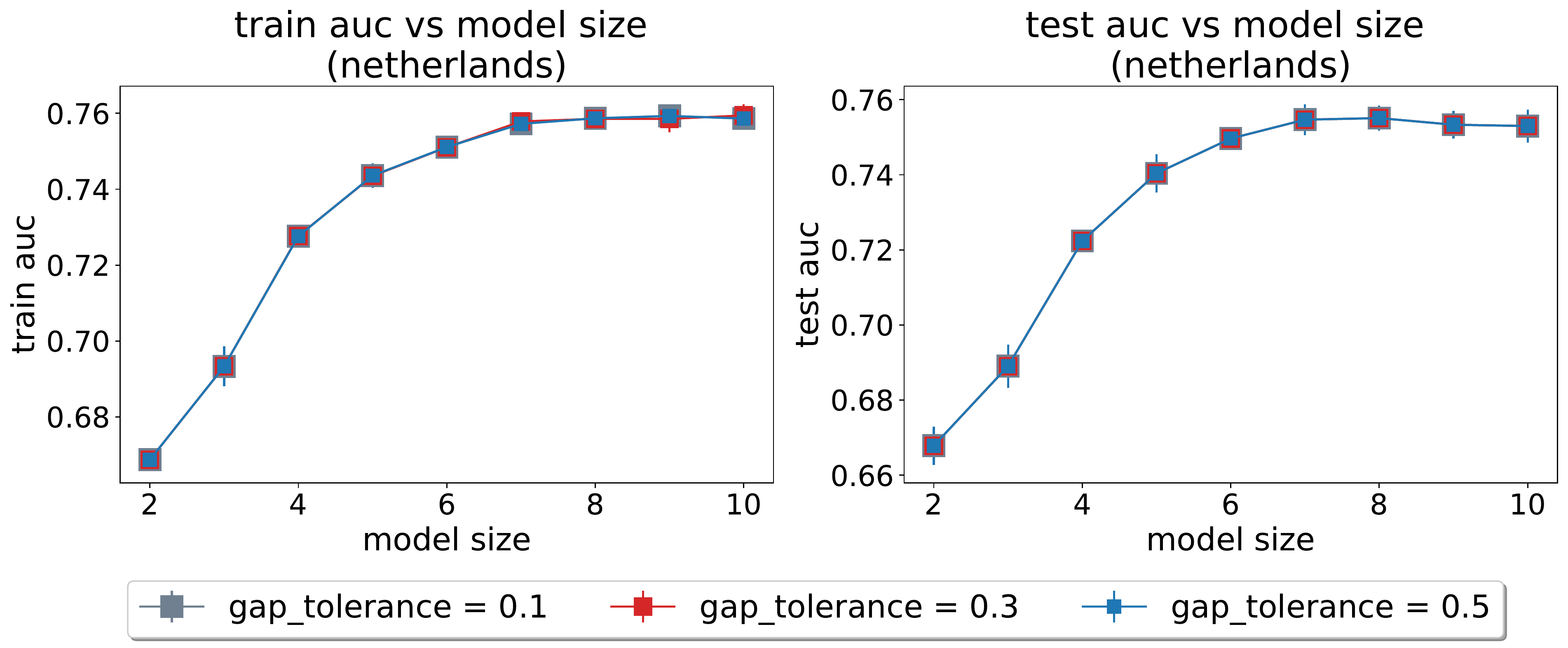}
    
    \caption{Perturbation study on tolerance level, $\epsilon$, for sparse diverse pool on the breastcancer, spambase, and Netherlands datasets. The default value used in the paper is 0.3.
    }
    \label{fig:hyperparameter_gap_tolerance_breastcancer_spambase_netherlands}
\end{figure}

\clearpage
\subsubsection{Perturbation Study on Number of Attempts \texorpdfstring{$T$}{T} for Sparse Diverse Pool}

We perform a perturbation study on the hyperparameter for the number of attempts, $T$, as mentioned in Appendix~\ref{sec:hyperparameter_specification}. We have set the number of attempts to 35, 50, and 65, respectively. The results are shown in Figures~\ref{fig:hyperparameter_max_attempt_adult_bank_mammo}-\ref{fig:hyperparameter_max_attempt_breastcancer_spambase_netherlands}. The curves greatly overlap, confirming our previous claim that the performance is not particularly sensitive to the choice of value for the hyperparameter.

\begin{figure}[ht] 
    \centering
    \includegraphics[width=\textwidth]{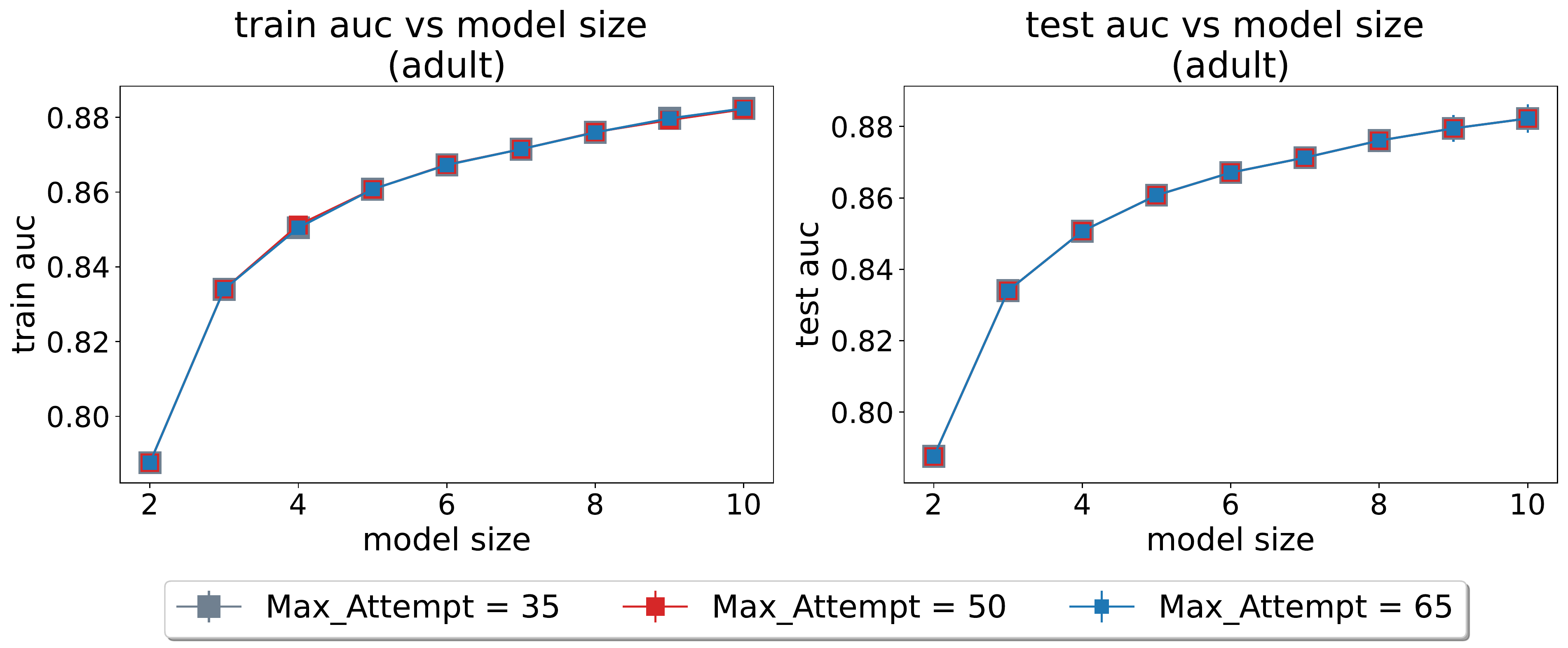}
    
    \includegraphics[width=\textwidth]{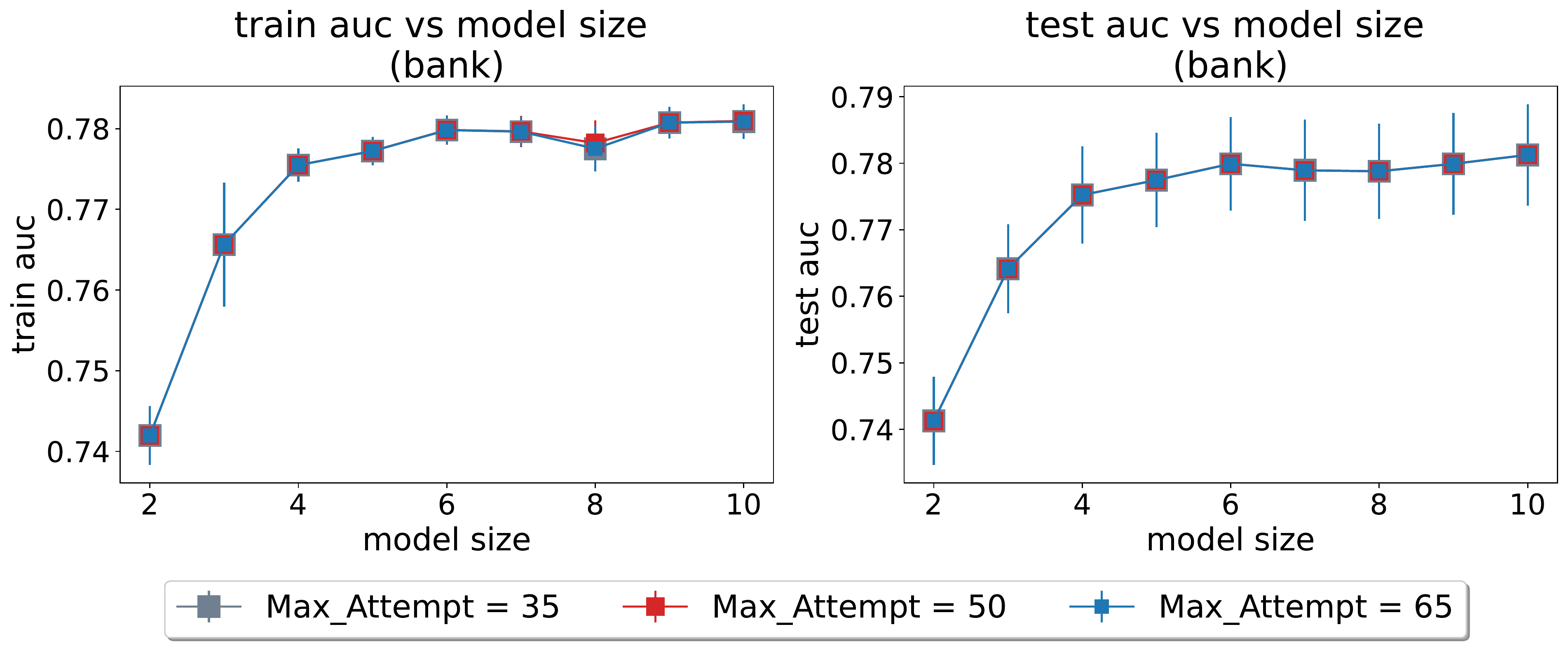}
    
    \includegraphics[width=\textwidth]{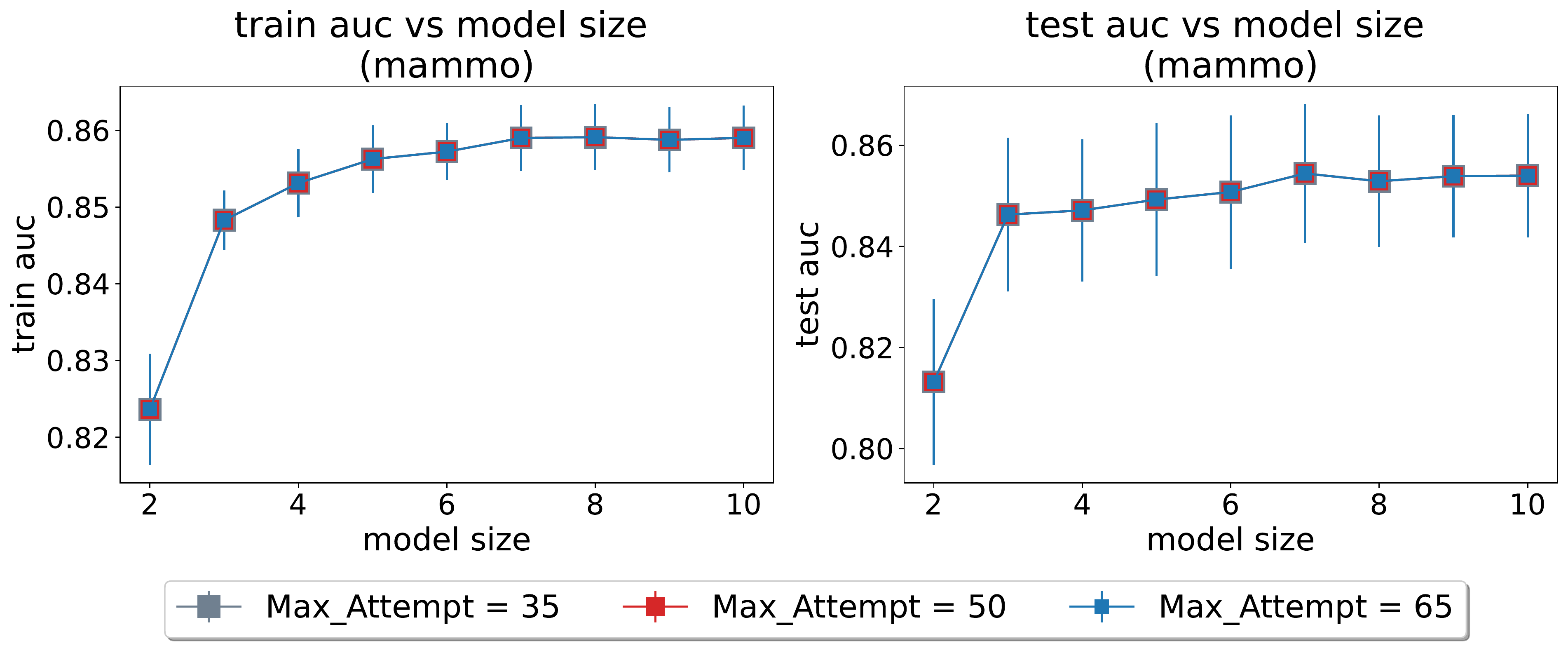}
    
    \caption{Perturbation study on number of attempts parameter, $T$, for sparse diverse pool on the adult, bank, and mammo datasets. The default value used in the paper is 50.
    }
    \label{fig:hyperparameter_max_attempt_adult_bank_mammo}
\end{figure}

\begin{figure}[ht] 
    \centering
    \includegraphics[width=\textwidth]{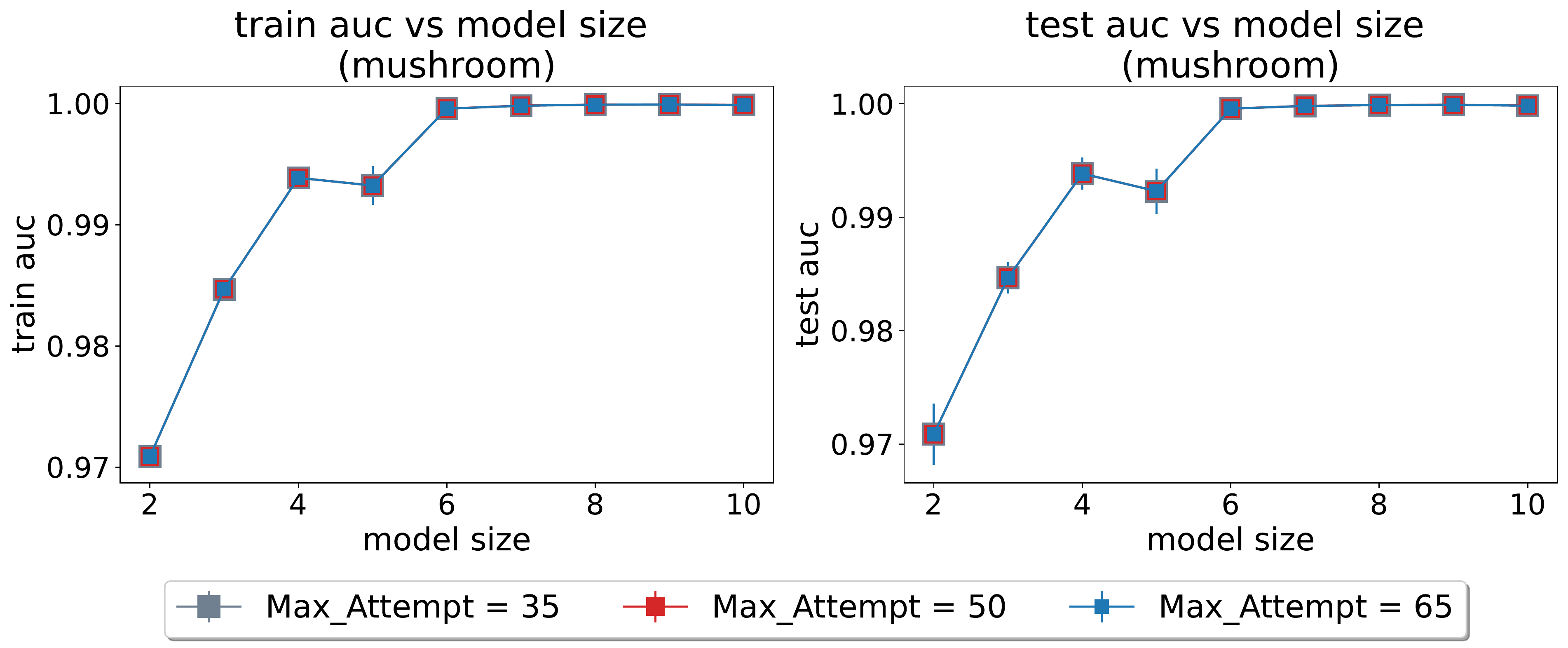}
    
    \includegraphics[width=\textwidth]{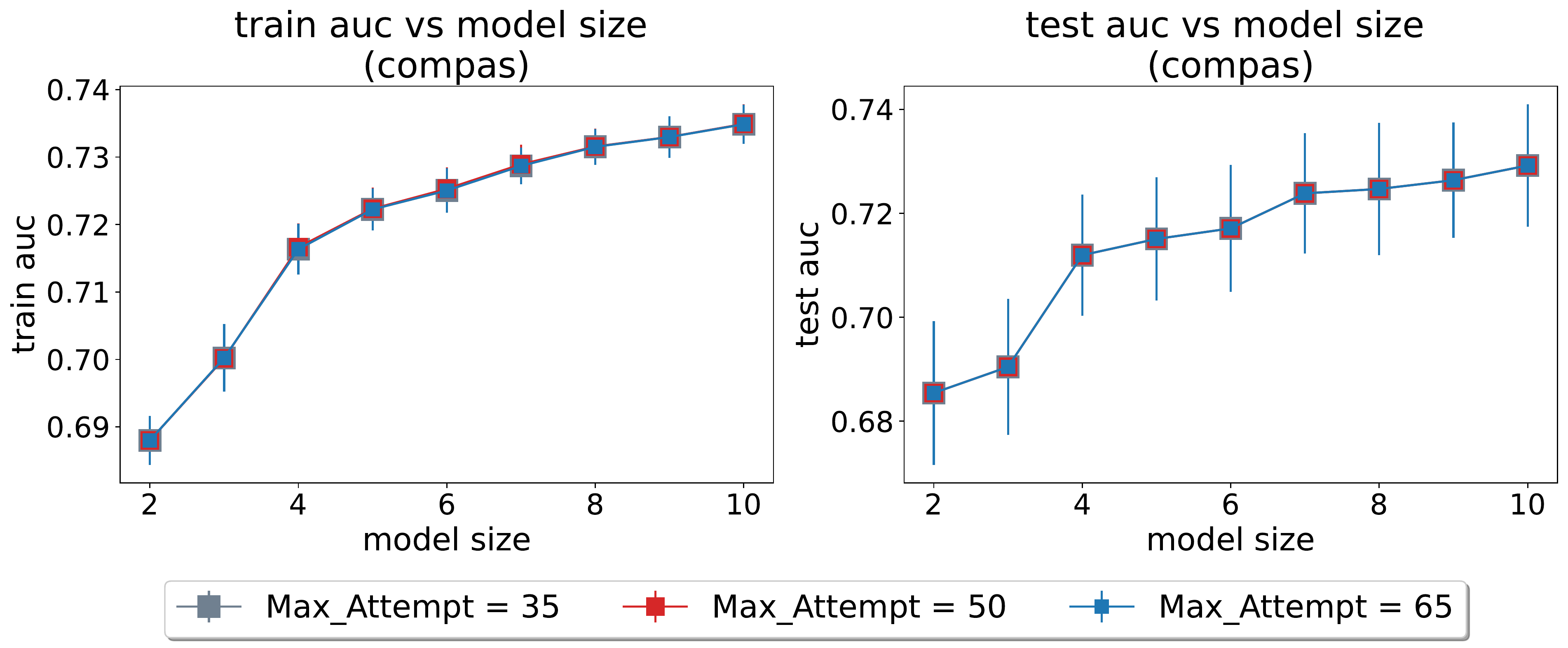}
    
    \includegraphics[width=\textwidth]{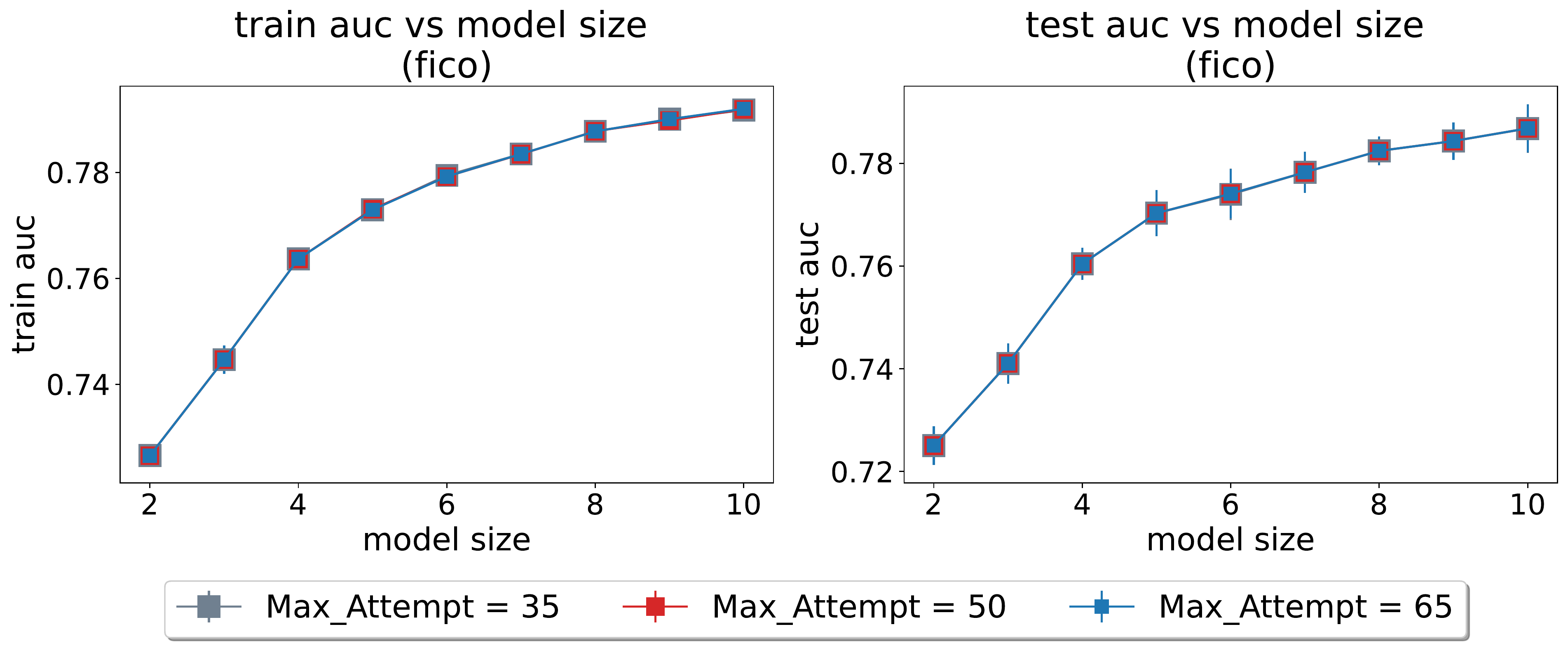}
    
    \caption{Perturbation study on number of attempts parameter, $T$, for sparse diverse pool on the mushroom, COMPAS, and FICO datasets. The default value used in the paper is 50.
    }
    \label{fig:hyperparameter_max_attempt_mushroom_compas_fico}
\end{figure}

\begin{figure}[ht] 
    \centering
    \includegraphics[width=\textwidth]{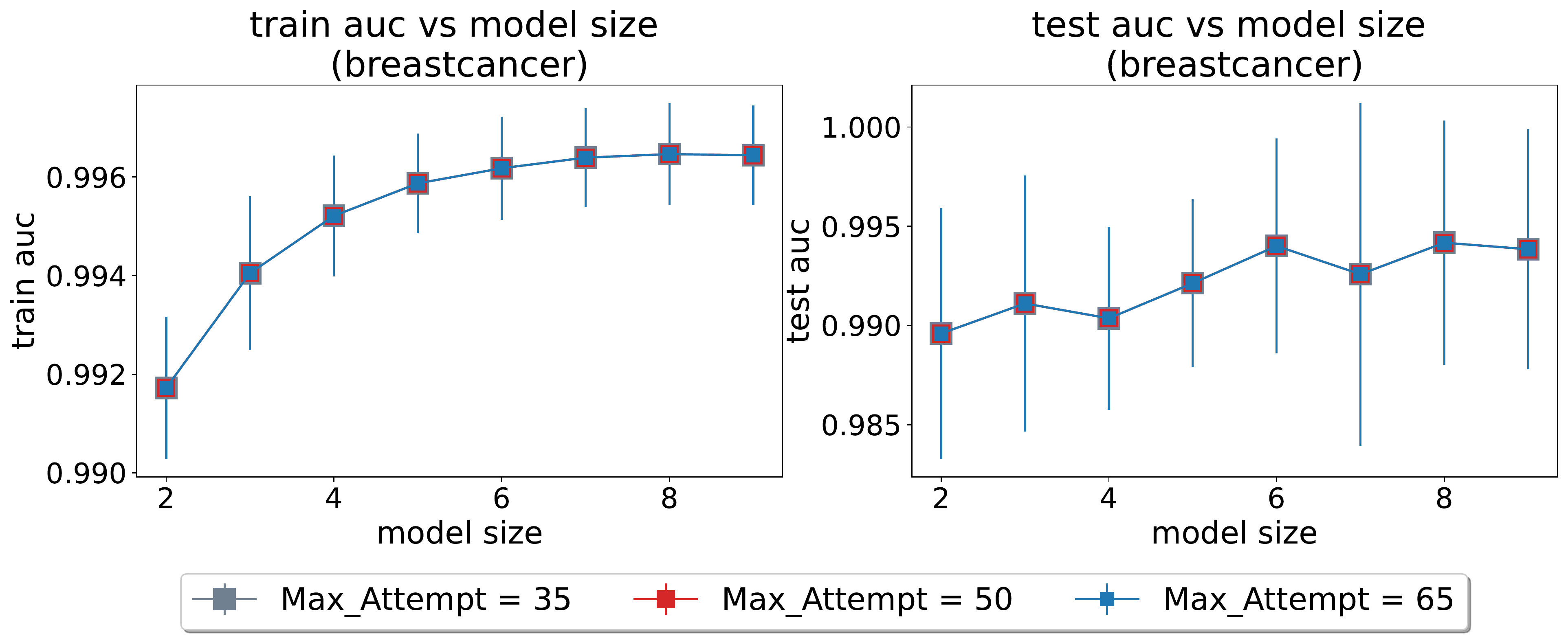}
    
    \includegraphics[width=\textwidth]{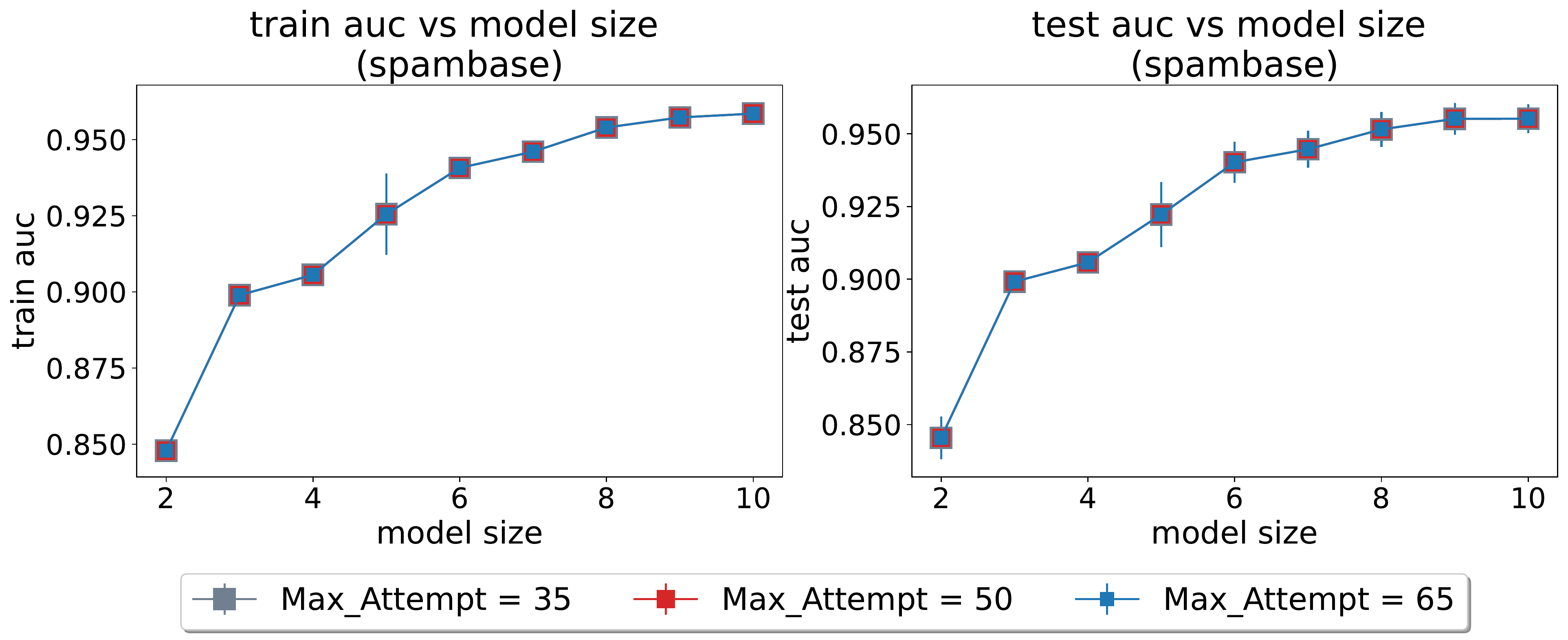}
    
    \includegraphics[width=\textwidth]{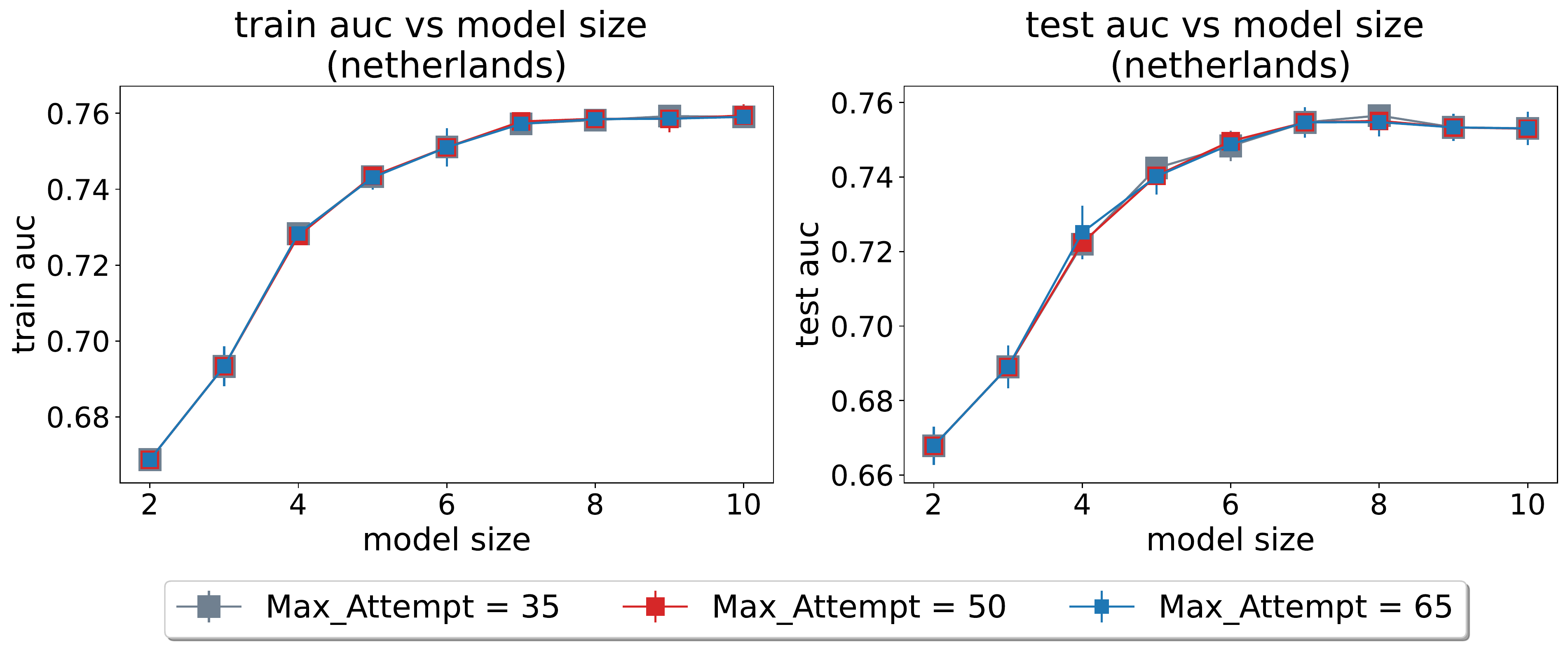}
    
    \caption{Perturbation study on number of attempts parameter, $T$, for sparse diverse pool on the breastcancer, spambase, and Netherlands datasets. The default value used in the paper is 50.
    }
    \label{fig:hyperparameter_max_attempt_breastcancer_spambase_netherlands}
\end{figure}

\clearpage
\subsubsection{Perturbation Study on Number of Multipliers \texorpdfstring{$N_m$}{Nm}}

We perform a perturbation study on the hyperparameter for the number of multipliers, $N_m$, as mentioned in Appendix~\ref{sec:hyperparameter_specification}. We have set the number of multipliers to 10, 20, and 30, respectively. The results are shown in Figures~\ref{fig:hyperparameter_num_ray_search_adult_bank_mammo}-\ref{fig:hyperparameter_num_ray_search_breastcancer_spambase_netherlands}. The curves greatly overlap, confirming our previous claim that the performance is not particularly sensitive to the choice of values for $N_m$.

\begin{figure}[ht] 
    \centering
    \includegraphics[width=\textwidth]{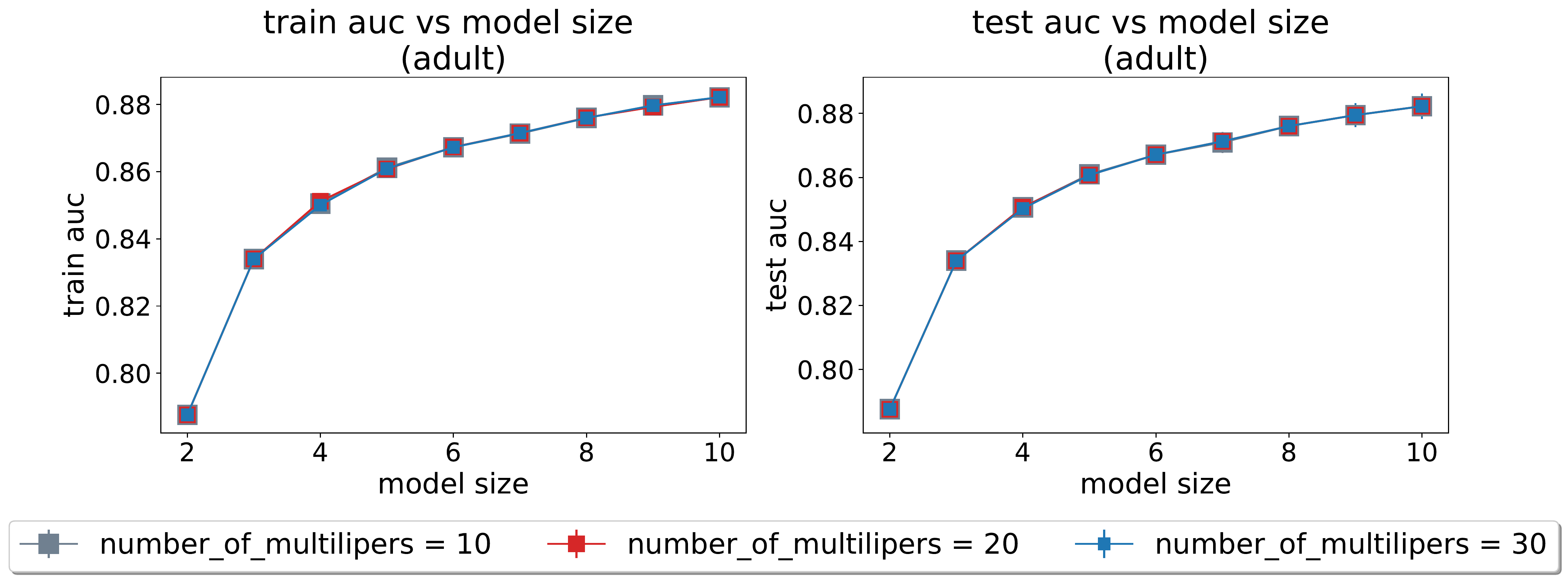}
    
    \includegraphics[width=\textwidth]{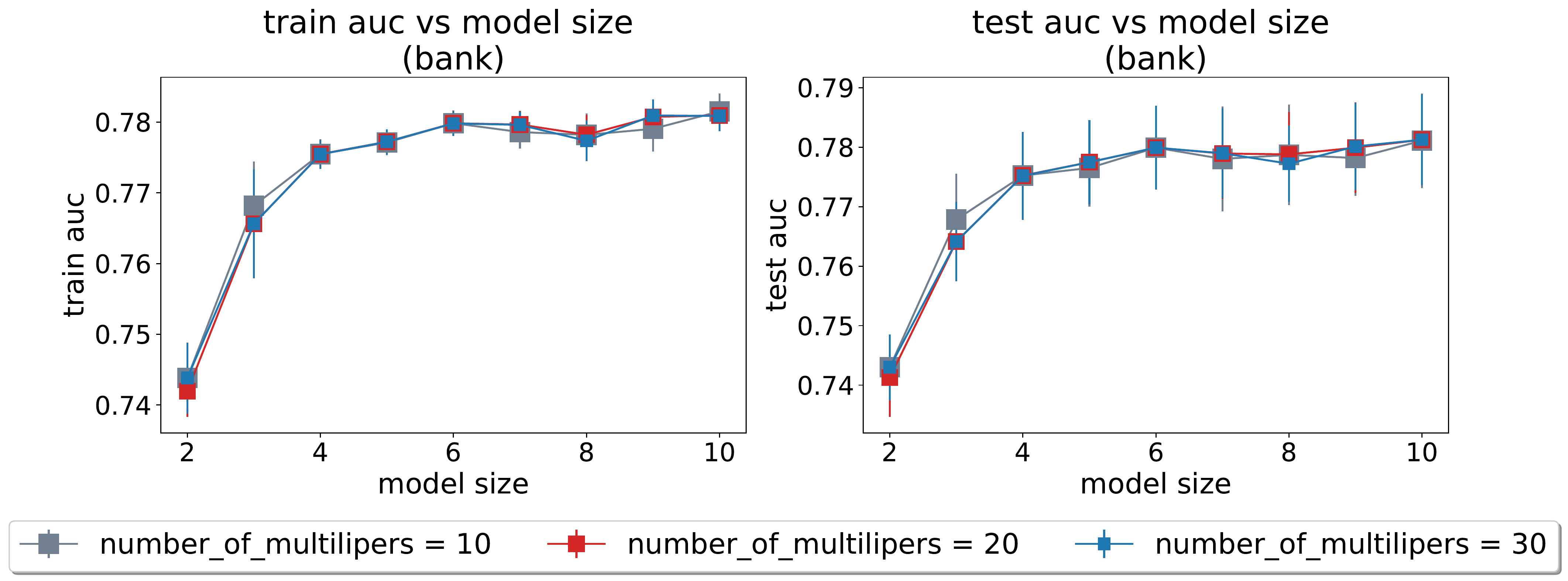}
    
    \includegraphics[width=\textwidth]{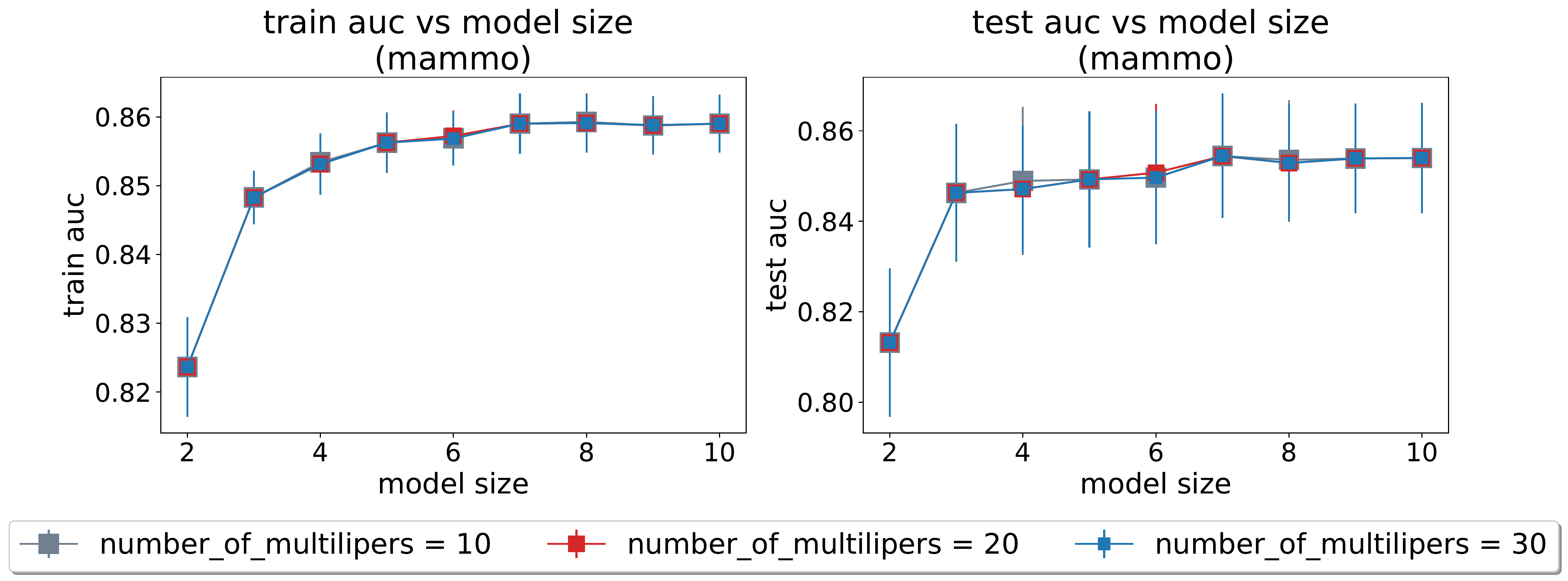}
    
    \caption{Perturbation study on number of multipliers, $N_m$, on the adult, bank, and mammo datasets. The default value used in the paper is 20.
    }
    \label{fig:hyperparameter_num_ray_search_adult_bank_mammo}
\end{figure}

\begin{figure}[ht] 
    \centering
    \includegraphics[width=\textwidth]{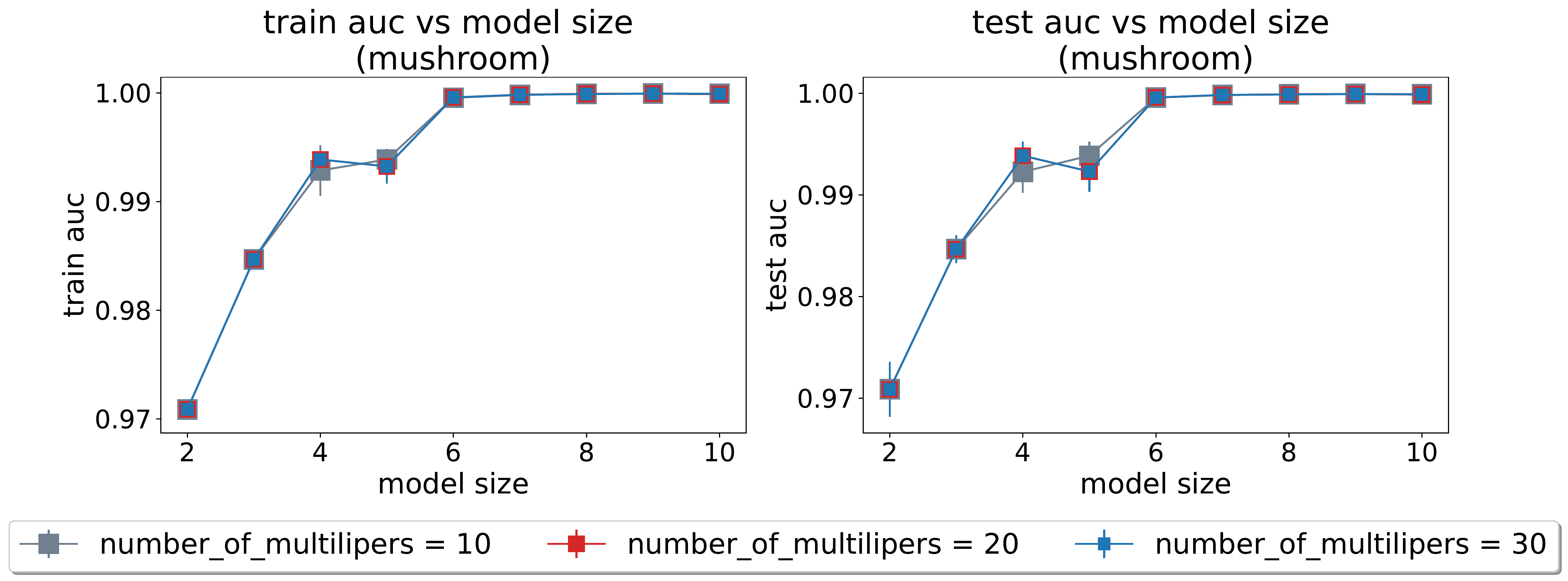}
    
    \includegraphics[width=\textwidth]{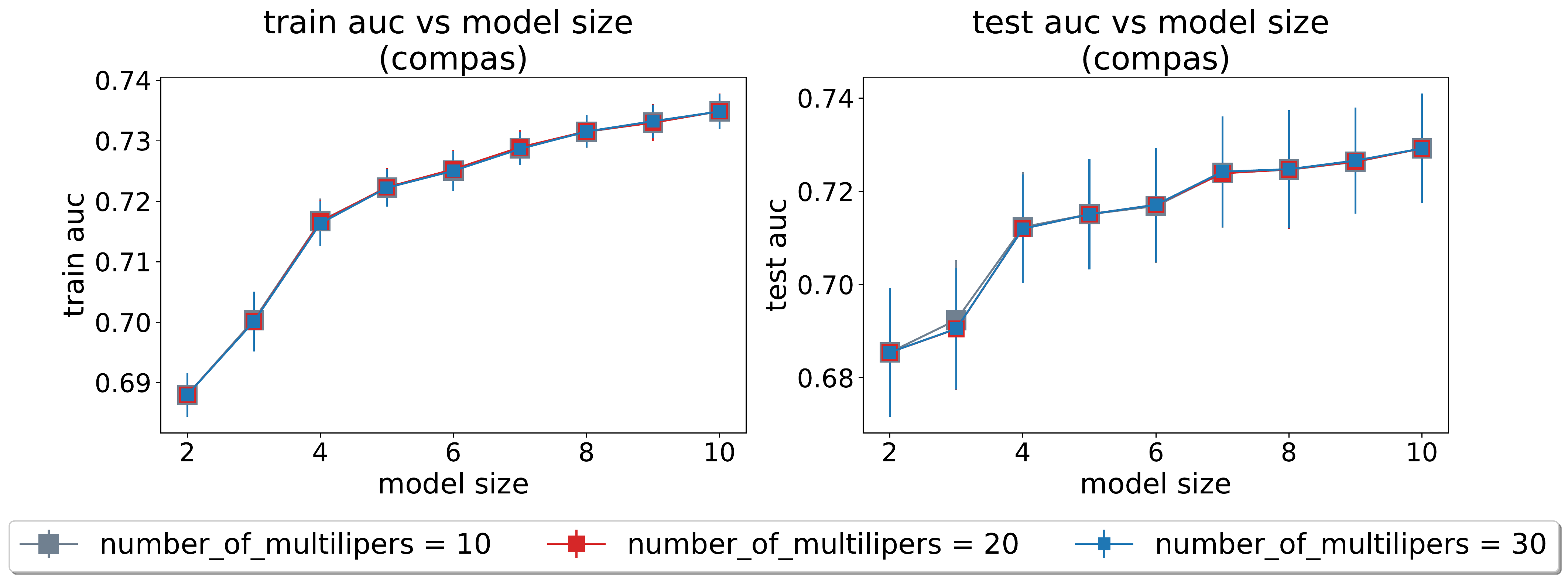}
    
    \includegraphics[width=\textwidth]{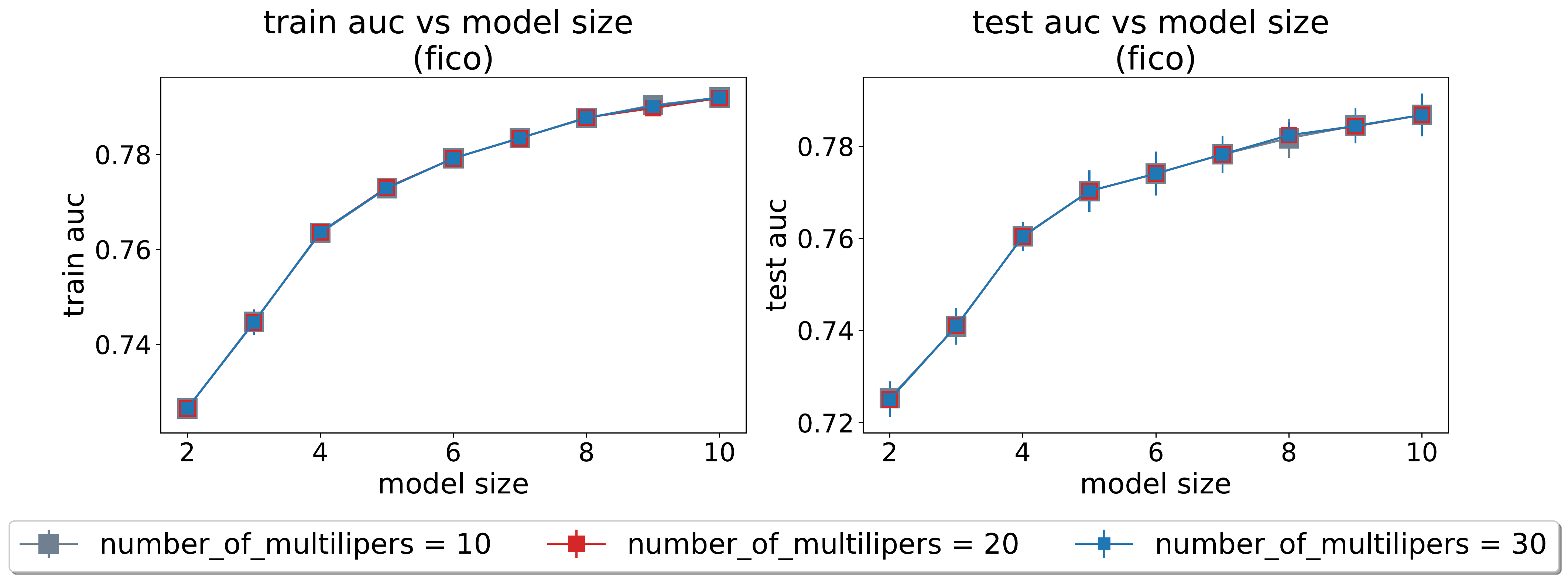}
    
    \caption{Perturbation study on number of multipliers, $N_m$, for sparse diverse pool on the mushroom, COMPAS and FICO datasets. The default value used in the paper is 20.
    }
    \label{fig:hyperparameter_num_ray_search_mushroom_compas_fico}
\end{figure}

\begin{figure}[ht] 
    \centering
    \includegraphics[width=\textwidth]{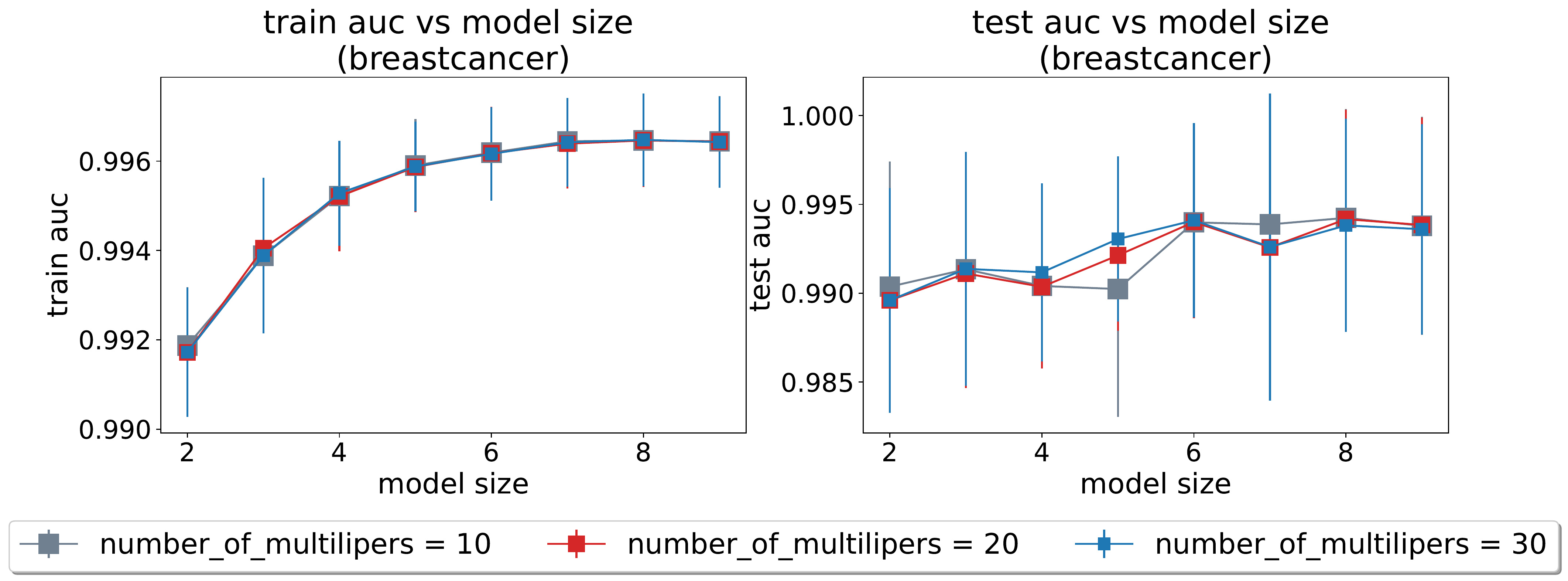}
    
    \includegraphics[width=\textwidth]{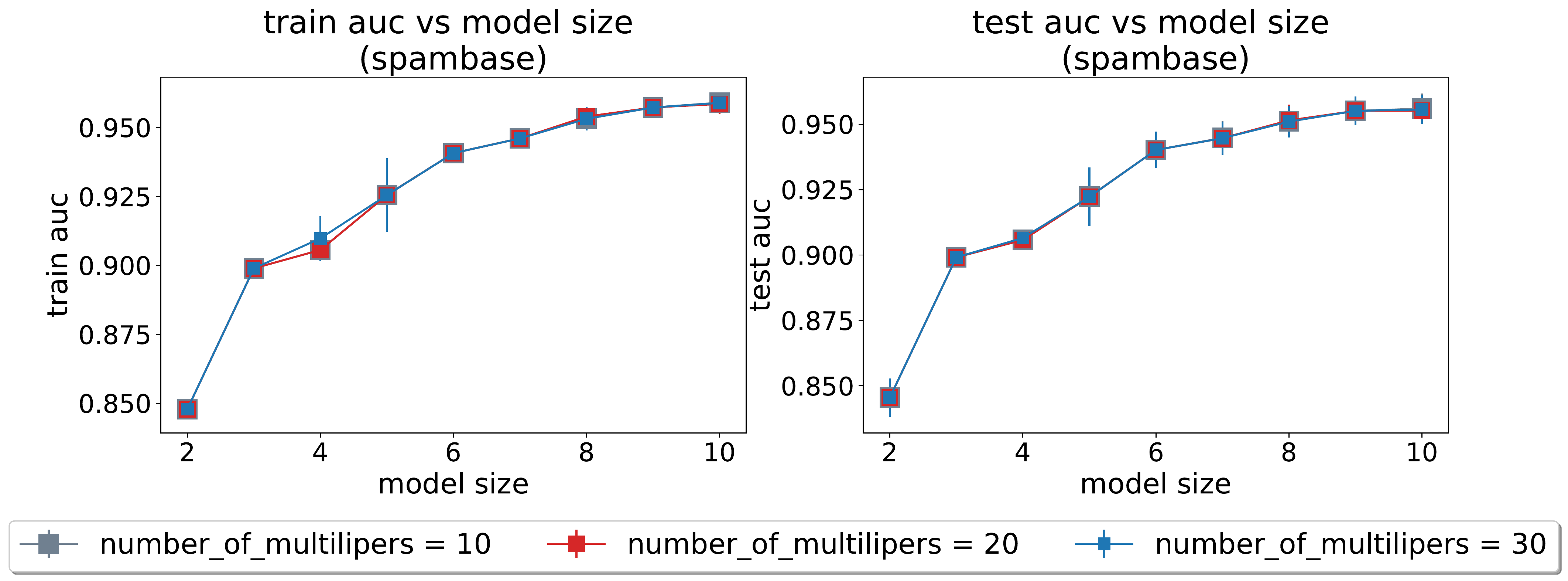}
    
    \includegraphics[width=\textwidth]{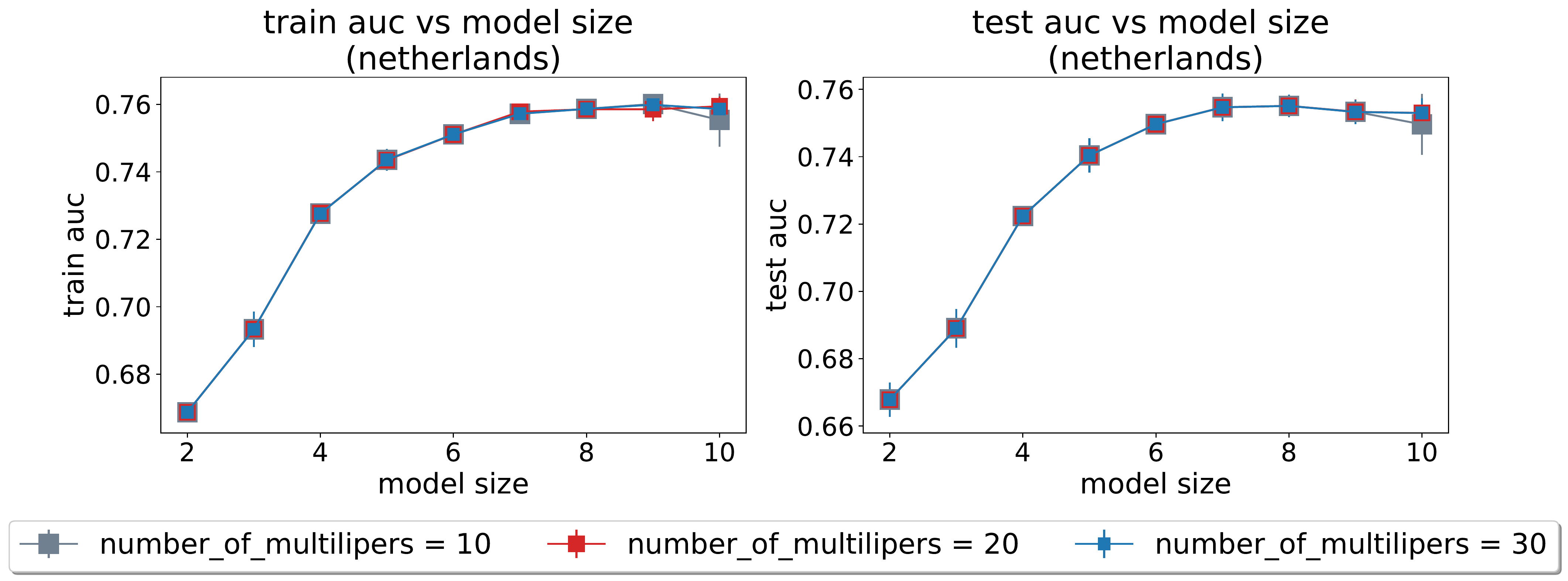}
    
    \caption{Perturbation study on number of multipliers, $N_m$,  for sparse diverse pool on the breastcancer, spambase, and Netherlands datasets. The default value used in the paper is 20.
    }
    \label{fig:hyperparameter_num_ray_search_breastcancer_spambase_netherlands}
\end{figure}

\clearpage
\subsection{Comparison with Baseline AutoScore}
\label{app:comparison_with_baseline}

We compare with the baseline AutoScore~\cite{xie2020autoscore}. We set the number of features from 2 to 10 and use all other hyperparameters in the default setting. The results of training AUC and test AUC are shown in Figures~\ref{fig:AutoScore_adult_bank_mammo}-\ref{fig:AutoScore_breastcancer_spambase_netherlands}. The plots of RiskSLIM are from experiments where we let RiskSLIM run for 1 hour. FasterRisk outperforms both RiskSLIM and AutoScore.

\begin{figure}[ht] 
    \centering
    \includegraphics[width=\textwidth]{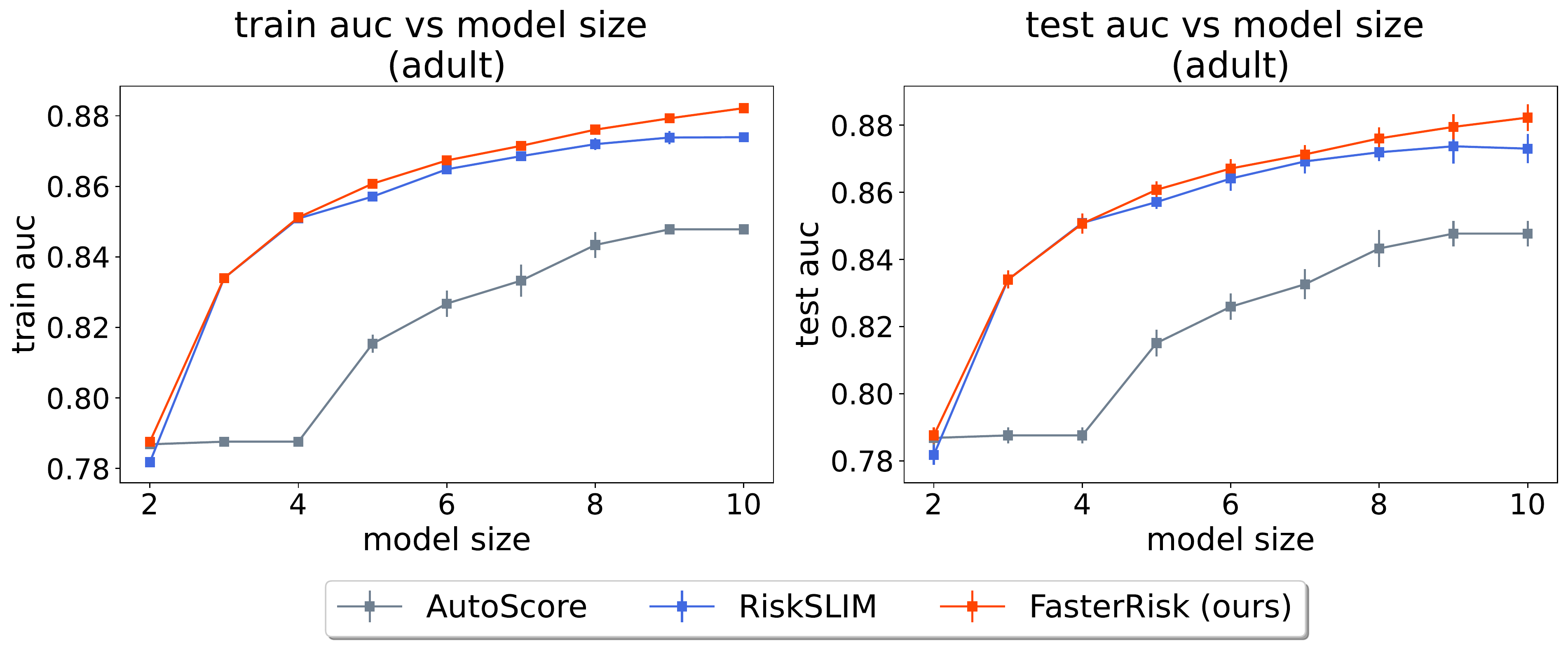}
    
    \includegraphics[width=\textwidth]{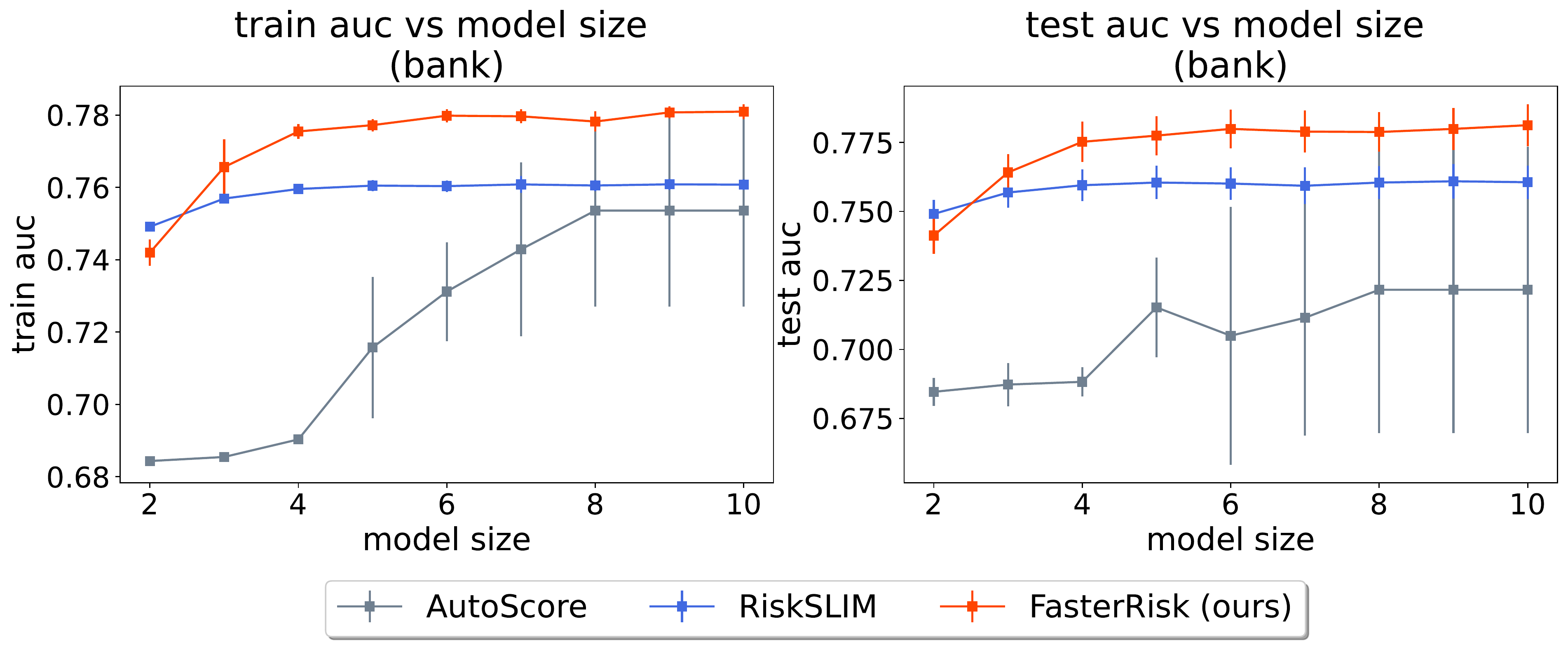}
    
    \includegraphics[width=\textwidth]{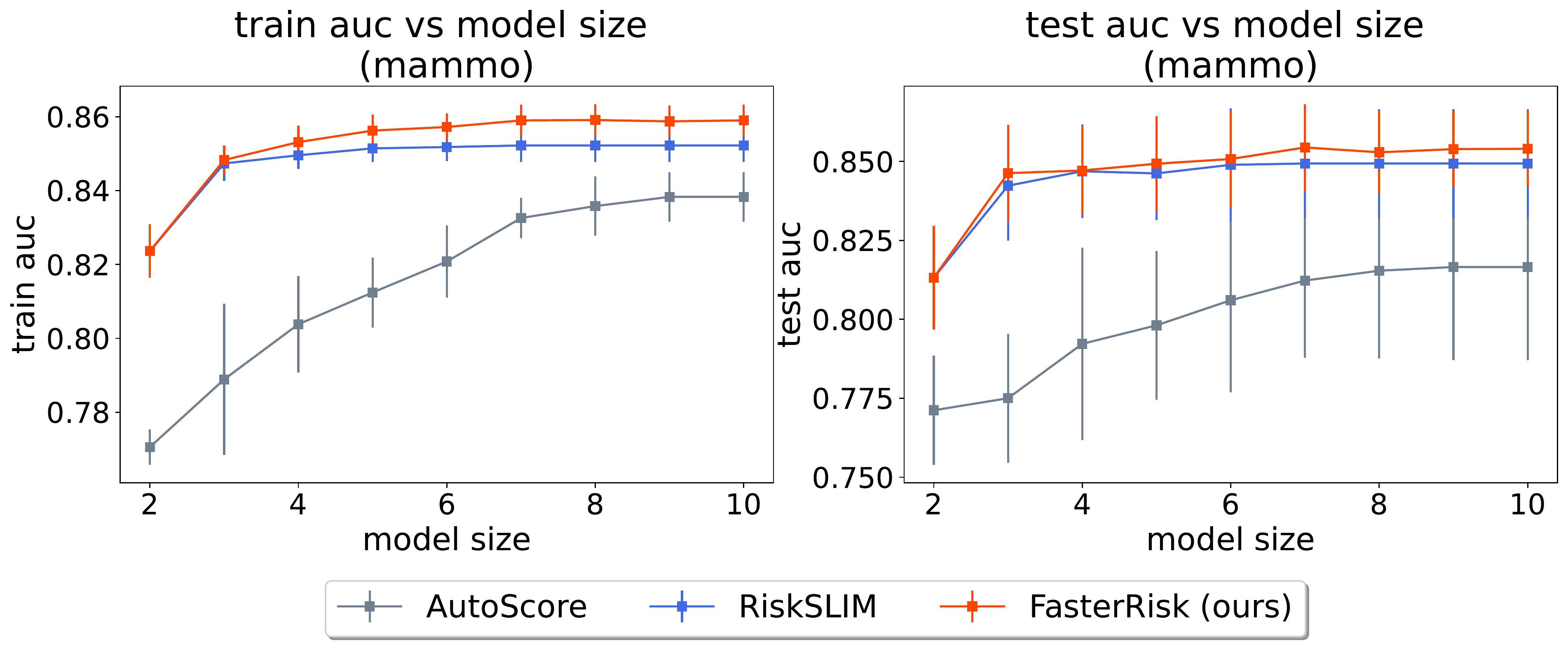}
    
    \caption{Comparison with the new baseline AutoScore on the adult, bank, and mammo datasets. The left column is training AUC (higher is better), and the right column is test AUC (higher is better).
    }
    \label{fig:AutoScore_adult_bank_mammo}
\end{figure}

\begin{figure}[ht] 
    \centering
    \includegraphics[width=\textwidth]{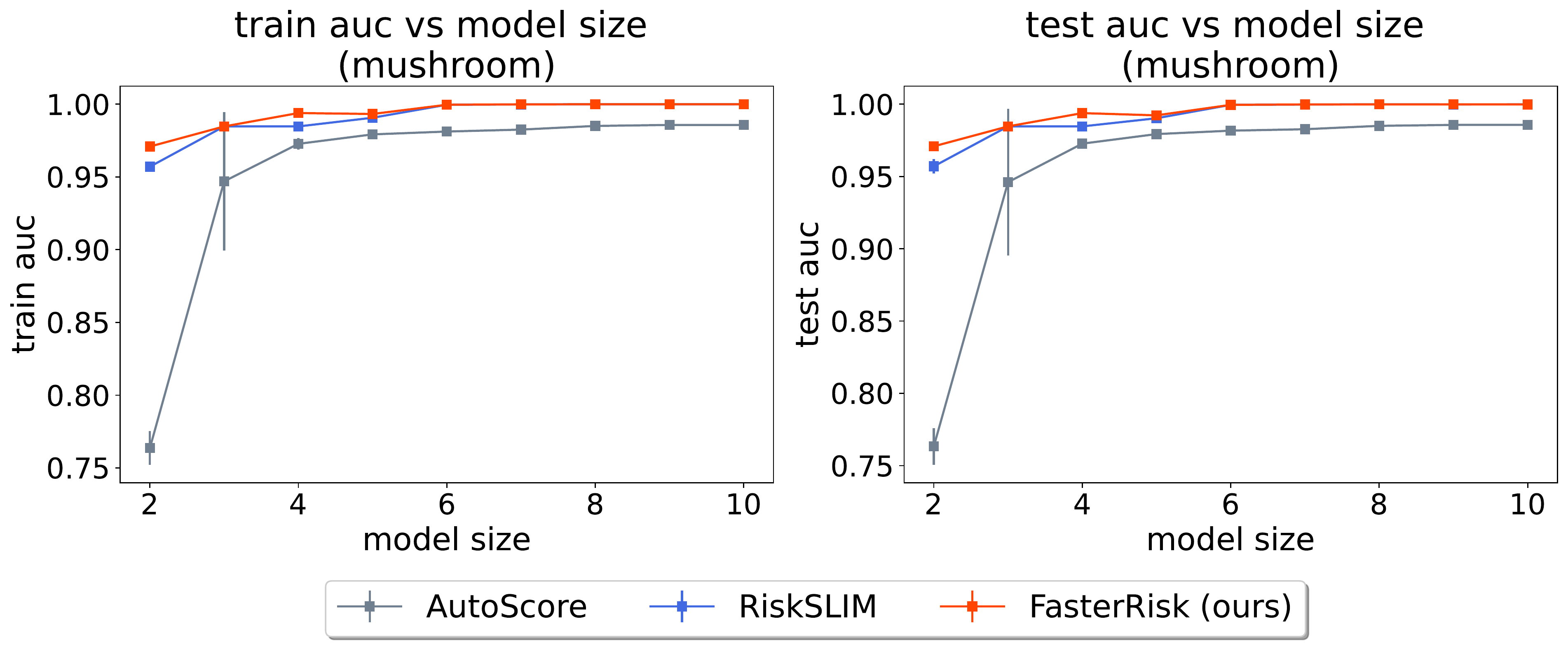}
    
    \includegraphics[width=\textwidth]{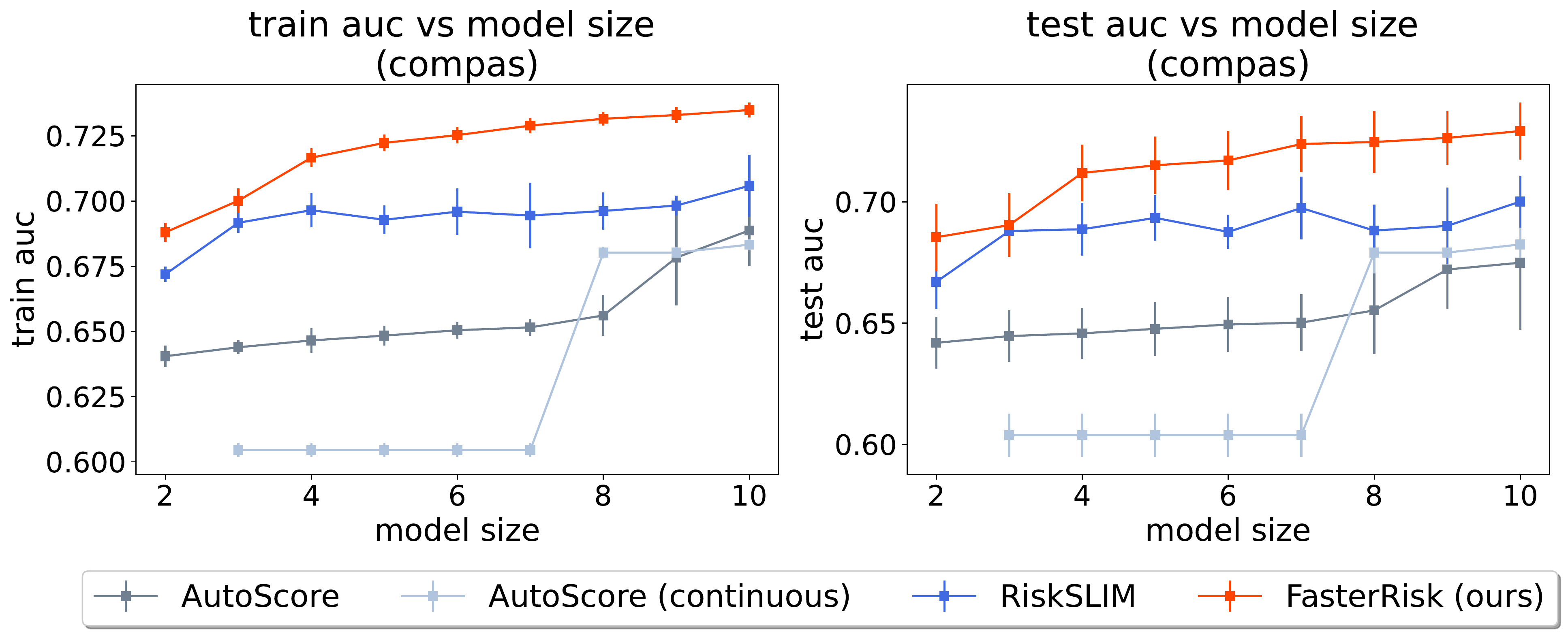}
    
    \includegraphics[width=\textwidth]{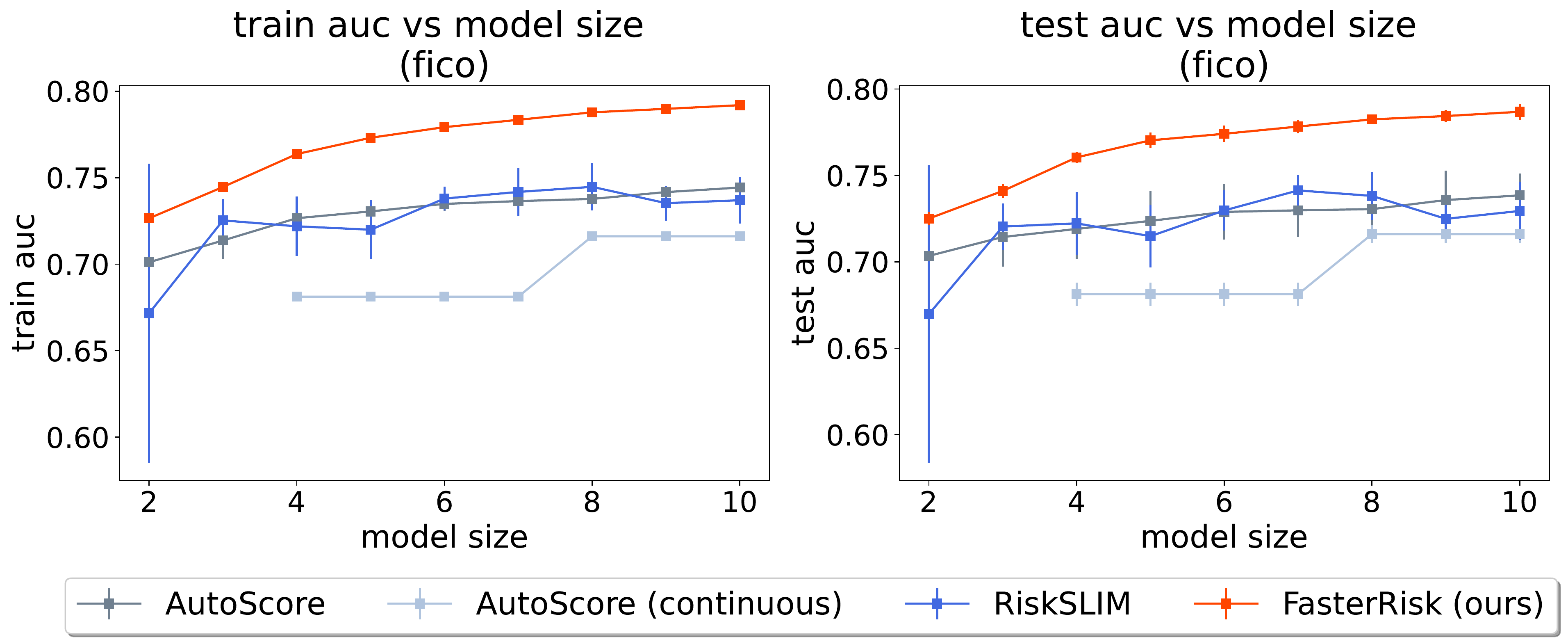}
    
    \caption{Comparison with the new baseline on the mushroom, Compas, and FICO datasets. The AutoScore (continuous) baseline is another method where AutoScore is applied to the original continuous features instead of the binary features as detailed in Appendix~\ref{app:dataset_information}. Not every model size can be obtained by the AutoScore (continuous) method. The left column is training AUC (higher is better), and the right column is test AUC (higher is better).
    }
    \label{fig:AutoScore_mushroom_compas_fico}
\end{figure}

\begin{figure}[ht] 
    \centering
    \includegraphics[width=\textwidth]{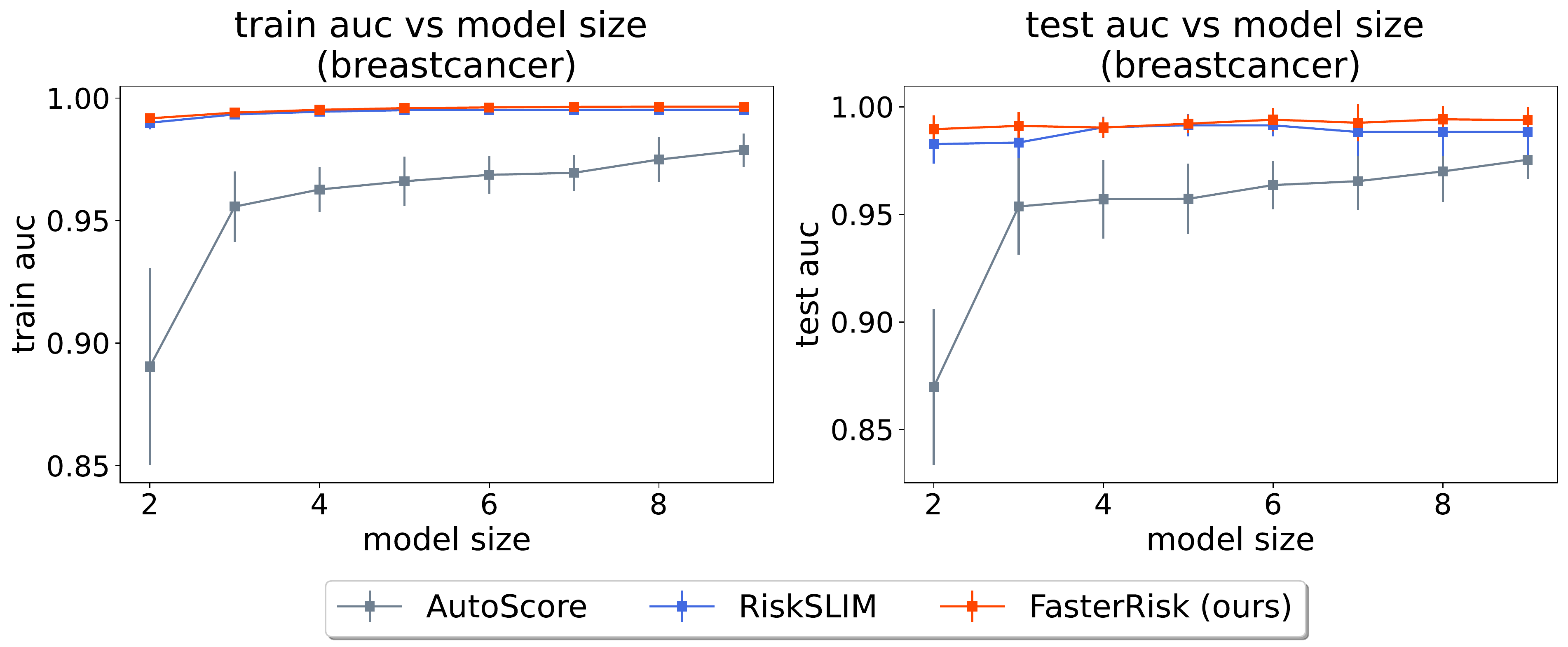}
    
    \includegraphics[width=\textwidth]{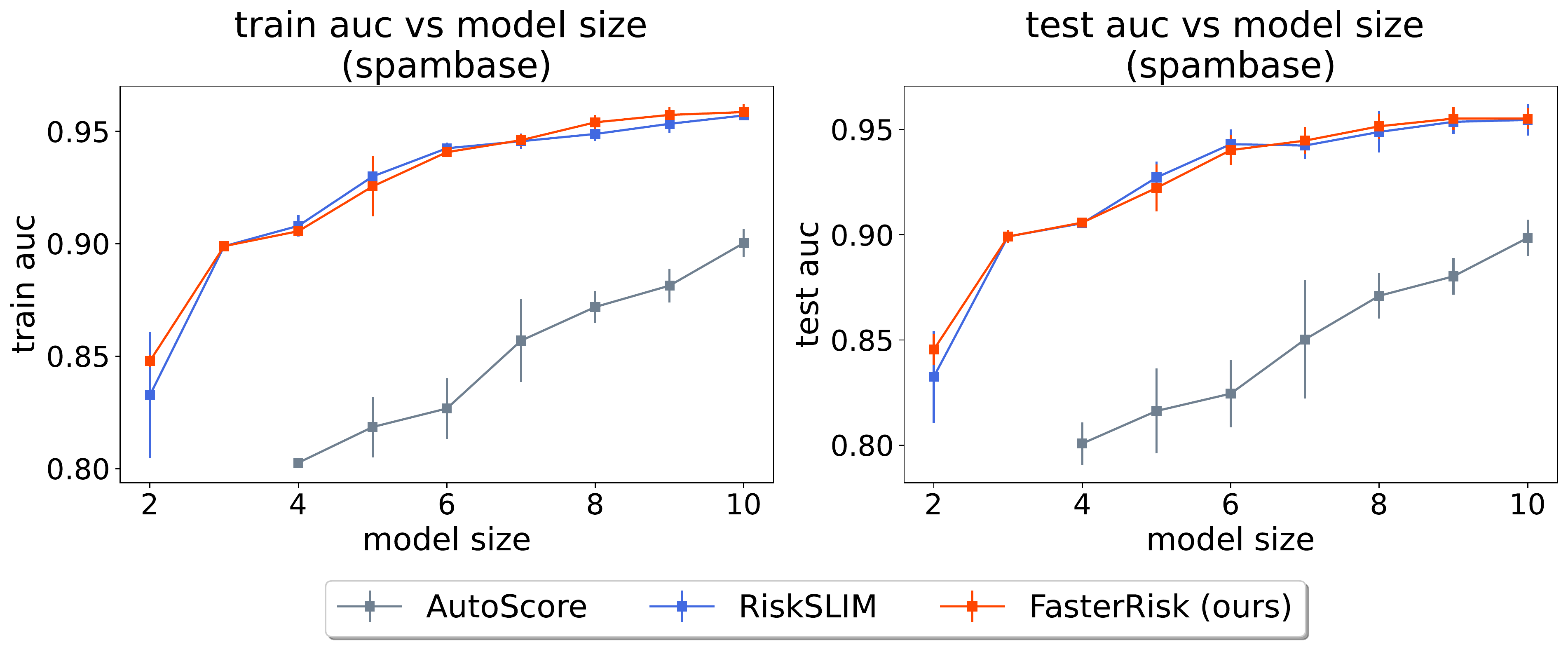}
    
    \includegraphics[width=\textwidth]{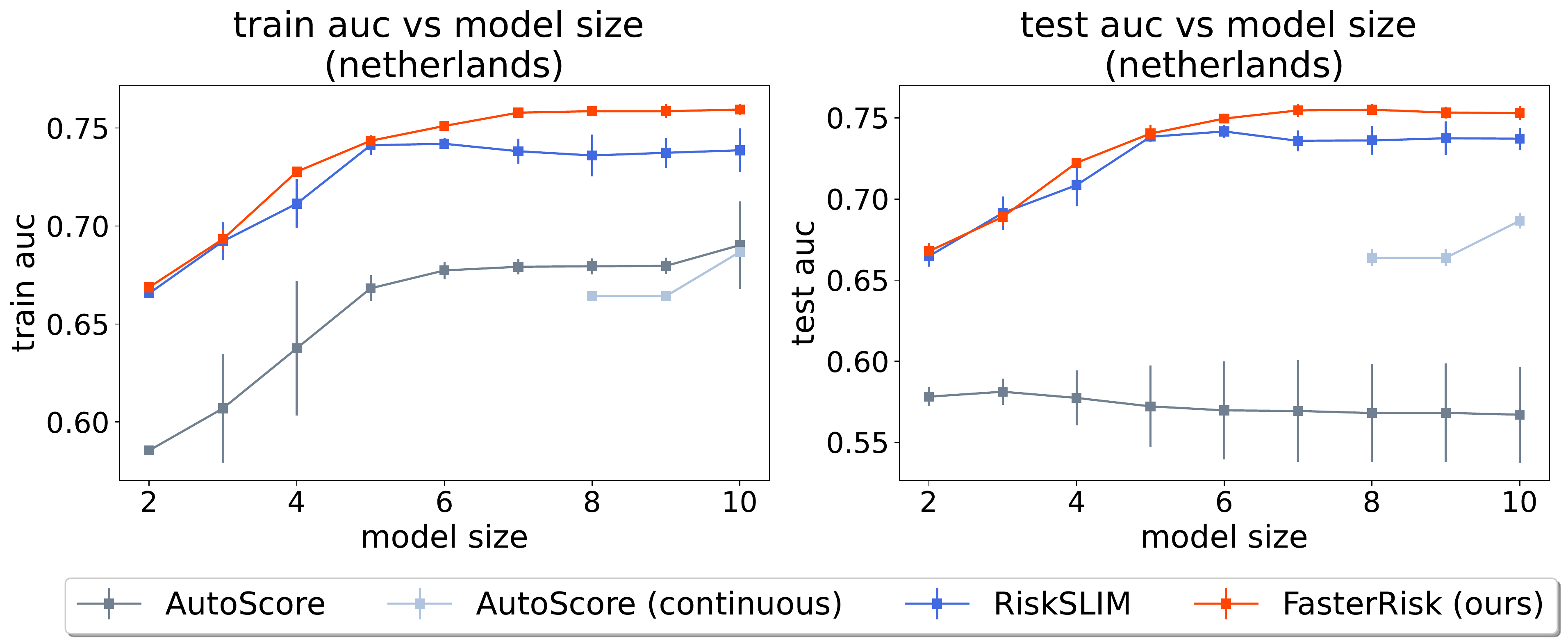}
    
    \caption{Comparison with the new baseline on the breastcancer, spambase, and Netherlands datasets. The AutoScore (continuous) baseline is another method where AutoScore is applied to the original continuous features instead of the binary features as detailed in Appendix~\ref{app:dataset_information}. Not every model size can be obtained by the AutoScore (continuous) method. The left column is training AUC (higher is better), and the right column is test AUC (higher is better).
    }
    \label{fig:AutoScore_breastcancer_spambase_netherlands}
\end{figure}

\clearpage
\section{Additional Risk Score Models}
\label{app:additional_risk_score_models}

We provide additional risk score models for the readers to inspect.

Appendix~\ref{app:risk_score_models_with_different_sizes} shows risk scores with different model sizes on different datasets.

Appendix~\ref{app:examples_from_the_pool} shows different risk scores with the same size from the diverse pool of solutions.

Specifically, Appendix~\ref{app:examples_from_the_pool_bank} shows different risk scores on the bank dataset (financial application), Appendix~\ref{app:examples_from_the_pool_mammo} shows different risk scores on the mammo dataset (medical application), and Appendix~\ref{app:examples_from_the_pool_netherlands} shows different risk scores on the Netherlands dataset (criminal justice application).

\subsection{Risk Score Models with Different Sizes}
\label{app:risk_score_models_with_different_sizes}

For model size $=3$, please see Tables \ref{fig:MoreExampleRiskScore_adult_modelSize_3}-\ref{fig:MoreExampleRiskScore_netherlands_modelSize_3}.

For model size $=5$, please see Tables \ref{fig:MoreExampleRiskScore_adult_modelSize_5}-\ref{fig:MoreExampleRiskScore_netherlands_modelSize_5}.

For model size $=7$, please see Tables \ref{fig:MoreExampleRiskScore_adult_modelSize_7}-\ref{fig:MoreExampleRiskScore_netherlands_modelSize_7}.

We also include a large model with size $=10$ on the FICO dataset, please see Table \ref{fig:MoreExampleRiskScore_fico_modelSize_10}.


\begin{table}[htbp]
\centering
{
\begin{tabular}{|llr|ll|}
\hline
1. &no high school diploma & -4 points &   & ... \\
2. &high school diploma only & -2 points  & + & ... \\
3. &married & 4 points & + & ...\\\hline
     &     & \multicolumn{1}{l|}{\textbf{SCORE}} & = &     \\ \hline
\end{tabular}
}
{
\risktable{}
\begin{tabular}{|r|c|c|c|c|c|}
\hline \rowcolor{scorecolor}\scorelabel{} & -4 & -2 & 0 & 2 & 4 \\
\hline \rowcolor{riskcolor}\risklabel{}  & 1.2\% & 4.1\% & 13.1\% & 34.7\% & 65.3\% \\
\hline
\end{tabular}
}

\medskip\caption{\ourmethod{} model for the adult dataset, predicting salary$>50$K.}
\label{fig:MoreExampleRiskScore_adult_modelSize_3}
\end{table}

\begin{table}[htbp]
\centering
{
\begin{tabular}{|llr|ll|}
\hline
1. &Call in Second Quarter & -2 points &   & ... \\
2. &Previous Call Was Successful & 4 points  & + & ... \\
3. &Employment Indicator $<$ 5100 & 4 points & + & ...\\\hline
     &     & \multicolumn{1}{l|}{\textbf{SCORE}} & = &     \\ \hline
\end{tabular}
}
{
\risktable{}
\begin{tabular}{|r|c|c|c|c|c|c|}
\hline \rowcolor{scorecolor}\scorelabel{} & -2 & 0 & 2 & 6 & 8 \\
\hline \rowcolor{riskcolor}\risklabel{}  & 2.8\% & 6.5\% & 14.5\% & 50.0\% & 70.8\%\\ 
\hline

\end{tabular}
}
\medskip\caption{\ourmethod{} model for the bank dataset, predicting if a person opens a bank account after a marketing call.
}

\label{fig:MoreExampleRiskScore_bank_modelSize_3}
\end{table}

\begin{table}[htbp]
\centering
{
\begin{tabular}{|llr|ll|}
\hline
1. &Irregular Shape & 4 points &   & ... \\
2. &Circumscribed Margin & -5 points  & + & ... \\
3. &Age $\geq$ 60 & 3 points & + & ...\\\hline
     &     & \multicolumn{1}{l|}{\textbf{SCORE}} & = &     \\ \hline
\end{tabular}
}
{
\risktable{}
\begin{tabular}{|r|c|c|c|c|c|c|}
\hline \rowcolor{scorecolor}\scorelabel{} & -5 & -2 & -1 & 2 \\
\hline \rowcolor{riskcolor}\risklabel{}  & 8.2\% & 20.1\% & 26.2\% & 50.0\% \\ 
\hline
\end{tabular} \\
\begin{tabular}{|r|c|c|c|c|c|}
\hline
\rowcolor{scorecolor}\scorelabel{} & 3 & 4 & 7 \\\hline
\rowcolor{riskcolor}\risklabel{} & 58.5\% & 66.6\% & 84.9\%  \\ 
\hline

\end{tabular}
}
\medskip\caption{\ourmethod{} model for the mammo dataset, predicting malignancy of a breast lesion.
}
\label{fig:MoreExampleRiskScore_mammo_modelSize_3}
\end{table}

\begin{table}[htbp]

\centering
{
\begin{tabular}{|llr|ll|}
\hline
1. &odor$=$almond & -5 points &   & ... \\
2. &odor$=$anise & -5 points  & + & ... \\
3. &odor$=$none & -5 points   & + & ... \\ \hline
     &     & \multicolumn{1}{l|}{\textbf{SCORE}} & = &     \\ \hline
\end{tabular}
}
{
\risktable{}
\begin{tabular}{|r|c|c|}
\hline \rowcolor{scorecolor}\scorelabel{} & -5 & 0 \\
\hline \rowcolor{riskcolor}\risklabel{} & 10.8\% & 96.0\% \\ 
\hline

\end{tabular}
}

\medskip\caption{\ourmethod{} model for the mushroom dataset, predicting whether a mushroom is poisonous.
}
\label{fig:MoreExampleRiskScore_mushroom_modelSize_3}
\end{table}

\begin{table}[htbp]

\centering
{
\begin{tabular}{|llr|ll|}
\hline
1. & prior\_counts $\leq 2$ & -4 points &   & ... \\
2. & prior\_counts $\leq 7$ & -4 points  & + & ... \\
3. & age $\leq 31$ & 4 points   & + & ... \\ \hline
     &     & \multicolumn{1}{l|}{\textbf{SCORE}} & = &     \\ \hline
\end{tabular}
}
{
\risktable{}
\begin{tabular}{|r|c|c|c|c|c|}
\hline \rowcolor{scorecolor}\scorelabel{} & -8 & -4 & 0 & 4 \\
\hline \rowcolor{riskcolor}\risklabel{}  & 23.6\% & 44.1\% & 67.0\% & 83.9\% \\ 
\hline
\end{tabular}
}

\medskip\caption{\ourmethod{} model for the COMPAS dataset, predicting whether individuals are arrested within two years of release.
}
\label{fig:MoreExampleRiskScore_compas_modelSize_3}
\end{table}

\begin{table}[htbp]

\centering
{
\begin{tabular}{|llr|ll|}
\hline
1. & MSinceMostRecentInqexcl7days$\leq 0$ & 3 points &   & ... \\
2. & ExternalRiskEstimate$\leq 70$ & 5 points  & + & ... \\
3. & ExternalRiskEstimate$\leq 79$ & 5 points   & + & ... \\ \hline
     &     & \multicolumn{1}{l|}{\textbf{SCORE}} & = &     \\ \hline
\end{tabular}
}
{
\risktable{}
\begin{tabular}{|r|c|c|c|c|c|c|}
\hline \rowcolor{scorecolor}\scorelabel{} & 0 & 3 & 5 & 8  & $\geq$ 10 \\
\hline \rowcolor{riskcolor}\risklabel{}  & 13.7\% & 24.0\% & 33.4\% & 50.0\% & $\geq$ 61.3\% \\ 
\hline
\end{tabular}
}

\medskip\caption{\ourmethod{} model for the FICO dataset, predicting whether an individual will default on a loan.
}
\label{fig:MoreExampleRiskScore_fico_modelSize_3}
\end{table}

\begin{table}[htbp]
\centering
{
\begin{tabular}{|llr|ll|}
\hline
1. & Clump Thickness & $\times$3 points &   & ... \\
2. & Uniformity of Cell Size & $\times$5 points  & + & ... \\
3. & Bare Nuclei & $\times$3 points   & + & ... \\ \hline
     &     & \multicolumn{1}{l|}{\textbf{SCORE}} & = &     \\ \hline
\end{tabular}
}
{
\risktable{}
\begin{tabular}{|r|c|c|c|c|c|c|}
\hline \rowcolor{scorecolor}\scorelabel{} & $\leq$ 33 & 36 & 39 & 42 & 45\\
\hline \rowcolor{riskcolor}\risklabel{}  & $\leq$ 3.3\% & 6.1\% & 10.8\% & 18.6\% & 30.1\\ 
\hline
\hline
\rowcolor{scorecolor}\scorelabel{} & 48 & 51 & 54 & 57 & $\geq$ 60\\\hline
\rowcolor{riskcolor}\risklabel{} & 67.0\% & 77.6\% & 85.5\% & 91.0 & $\geq$ 94.5\%\\ 
\hline


\end{tabular}
}

\medskip\caption{\ourmethod{} model for the breastcancer dataset, predicting whether there is breast cancer using a biopsy.
}

\label{fig:MoreExampleRiskScore_mammo_breastcancer_modelSize_3}
\end{table}

\begin{table}[htbp]

\centering
{
\begin{tabular}{|llr|ll|}
\hline
1. & WordFrequency\_Remove & $\times$5 points &   & ... \\
2. & WordFrequency\_HP & $\times$-2 points  & + & ... \\
3. & CharacterFrequency\_\$ & $\times$5 points   & + & ... \\ \hline
     &     & \multicolumn{1}{l|}{\textbf{SCORE}} & = &     \\ \hline
\end{tabular}
}
{
\risktable{}
\begin{tabular}{|r|c|c|c|c|c|c|}
\hline \rowcolor{scorecolor}\scorelabel{} & $\leq$ -4 & -3 & -2 & -1 & 0 \\
\hline \rowcolor{riskcolor}\risklabel{}  & $\leq$0.4\% & 1.3\% & 3.7\% & 10.2\% & 25.2\%\\ 
\hline
\hline
\rowcolor{scorecolor}\scorelabel{} & 1  & 2 & 3 & 4 & $\geq$ 5\\\hline
\rowcolor{riskcolor}\risklabel{} & 50.0\% & 74.8\% & 89.8\% & 96.3\% & $\geq$ 98.7\%\\ 
\hline


\end{tabular}
}

\medskip\caption{\ourmethod{} model for the spambase dataset, predicting if an e-mail is spam.
}
\label{fig:MoreExampleRiskScore_spambase_modelSize_3}
\end{table}

\begin{table}[htbp]

\centering
{
\begin{tabular}{|llr|ll|}
\hline
1. &  previous case $\leq 20$ & -5 points &   & ... \\
2. &  previous case $ \leq 10 $ or previous case $\geq 21$ & -4 points  & + & ... \\
3. & \# of previous penal cases$\leq 3$ & -2 points   & + & ... \\ \hline
     &     & \multicolumn{1}{l|}{\textbf{SCORE}} & = &     \\ \hline
\end{tabular}
}
{
\risktable{}
\begin{tabular}{|r|c|c|c|c|c|}
\hline \rowcolor{scorecolor}\scorelabel{} & $\leq -9$ & -7 & -6 & 0 \\
\hline \rowcolor{riskcolor}\risklabel{}  & $\leq$ 50\% & 74.6\% & 83.4\% & 99.2\% \\ 
\hline
\end{tabular}
}

\medskip\caption{\ourmethod{} model for the Netherlands dataset, predicting whether defendants have any type of charge within four years.
}
\label{fig:MoreExampleRiskScore_netherlands_modelSize_3}
\end{table}


\begin{table}[htbp]
\centering
{
\begin{tabular}{|llr|ll|}
\hline
1. &no high school diploma & -4 points &   & ... \\
2. &high school diploma only & -2 points  & + & ... \\
3. &age 22 to 29 & -2 points   & + & ... \\ 
4. &any capital gains & 3 points & + & ...\\
5. &married & 4 points & + & ...\\\hline
     &     & \multicolumn{1}{l|}{\textbf{SCORE}} & = &     \\ \hline
\end{tabular}
}
{
\risktable{}
\begin{tabular}{|r|c|c|c|c|c|}
\hline \rowcolor{scorecolor}\scorelabel{} & $<$-4 & -3 & -2 & -1 & 0 \\
\hline \rowcolor{riskcolor}\risklabel{}  & $<$1.3\% & 2.4\% & 4.4\% & 7.8\% & 13.6\%\\ 
\hline
\hline
\rowcolor{scorecolor}\scorelabel{} & 1 & 2 & 3 & 4 & 7  \\\hline
\rowcolor{riskcolor}\risklabel{} & 22.5\% & 35.0\% & 50.5\% & 65.0\% & 92.2\%  \\ 
\hline
\end{tabular}
}

\medskip\caption{\ourmethod{} model for the adult dataset, predicting salary$>50$K. This table has already been shown in the main paper.}
\label{fig:MoreExampleRiskScore_adult_modelSize_5}
\end{table}

\begin{table}[htbp]
\centering
{
\begin{tabular}{|llr|ll|}
\hline
1. &Call in Second Quarter & -2 points &   & ... \\
2. &Previous Call Was Successful & 4 points  & + & ... \\
3. &Previous Marketing Campaign Failed & -1 points & + & ...\\ 
4. &Employment Indicator $>$ 5100 & -5 points & + & ...\\
5. &3 Month Euribor Rate $\geq$ 100 & -2 points & + & ...\\\hline
     &     & \multicolumn{1}{l|}{\textbf{SCORE}} & = &     \\ \hline
\end{tabular}
}
{
\risktable{}
\begin{tabular}{|r|c|c|c|c|c|c|c|c|c|c|}
\hline \rowcolor{scorecolor}\scorelabel{} & $\leq$-5 & -4 & -3 & -2 & -1 \\
\hline \rowcolor{riskcolor}\risklabel{}  & $\leq$ 11.2\% & 15.1\% & 20.1\% & 26.2\% & 33.4\%\\ 
\hline

\hline \rowcolor{scorecolor}\scorelabel{} & 0 & 1 & 2 & 3 & 4 \\
\hline \rowcolor{riskcolor}\risklabel{}  & 41.5\% & 50.0\% & 58.5\% & 66.6\% & 73.8\%\\ 
\hline

\end{tabular}
}
\medskip\caption{\ourmethod{} model for the bank dataset, predicting if a person opens a bank account after a marketing call.
}

\label{fig:MoreExampleRiskScore_bank_modelSize_5}
\end{table}

\begin{table}[htbp]
\centering
{
\begin{tabular}{|l l  r | l |}
   \hline
1.   & \textssm{Oval Shape} & -2 points & $\phantom{+}\prow{}$ \\ 
  2. & \textssm{Irregular Shape}  & 4 points & $+\prow{}$ \\ 
  3. & \textssm{Circumscribed Margin} & -5 points & $+\prow{}$ \\ 
  4. & \textssm{Spiculated Margin} & 2 points & $+\prow{}$ \\ 
  5. & \textssm{Age $\geq$ 60}  & 3 points & $+\phantom{\prow{}}$ \\[0.1em]
   \hline
 & \instruction{1}{5} & \scorelabel{} & $=\phantom{\prow{}}$ \\ 
   \hline
\end{tabular}
}
{%
\risktable{}
\begin{tabular}{|r|c|c|c|c|c|c|}
\hline \rowcolor{scorecolor}\scorelabel{} & -7     & -5     & -4     & -3     & -2     & -1  \\
\hline 
\rowcolor{riskcolor}\risklabel{} & 6.0\%  & 10.6\% & 13.8\% & 17.9\% & 22.8\% & 28.6\%  \\ 
\hline
\end{tabular}
\begin{tabular}{|r|c|c|c|c|c|c|}
\hline \rowcolor{scorecolor}\scorelabel{} & 0 & 1      & 2      & 3      & 4      & $\geq$ 5  \\
\hline 
\rowcolor{riskcolor}\risklabel{} & 35.2\% & 42.4\% & 50.0\% & 57.6\% & 64.8\% & 71.4\% \\ 
\hline
\end{tabular}
}
\medskip\caption{\ourmethod{} model for the mammo dataset, predicting malignancy of a breast lesion. This table has already been shown in the main paper.
}
\label{fig:MoreExampleRiskScore_mammo_modelSize_5}
\end{table}


\begin{table}[htbp]

\centering
{
\begin{tabular}{|llr|ll|}
\hline
1. &odor$=$almond & -5 points &   & ... \\
2. &odor$=$anise & -5 points  & + & ... \\
3. &odor$=$none & -5 points   & + & ... \\ 
4. &odor$=$foul & 5 points & + & ...\\
5. &gill size$=$broad & -3 points & + & ...\\\hline
     &     & \multicolumn{1}{l|}{\textbf{SCORE}} & = &     \\ \hline
\end{tabular}
}
{
\risktable{}
\begin{tabular}{|r|c|c|c|c|}
\hline \rowcolor{scorecolor}\scorelabel{} & -8 & -5 & -3 & $\geq$2\\
\hline \rowcolor{riskcolor}\risklabel{} & 1.62\% & 26.4\% & 73.6\% & >99.8\% \\ 
\hline

\end{tabular}
}

\medskip\caption{\ourmethod{} model for the mushroom dataset, predicting whether a mushroom is poisonous. This table has already been shown in the main paper.
}
\label{fig:MoreExampleRiskScore_mushroom_modelSize_5}
\end{table}

\begin{table}[htbp]

\centering
{
\begin{tabular}{|llr|ll|}
\hline
1. & prior\_counts $\leq 7$ & -5 points &   & ... \\
2. & prior\_counts $\leq 2$ & -5 points &  + & ... \\
3. & prior\_counts $\leq 0$ & -3 points  & + & ... \\
4. & age $\leq$ 33 & 4 points  & + & ... \\
5. & age $\leq$ 23 & 5 points   & + & ... \\ \hline
     &     & \multicolumn{1}{l|}{\textbf{SCORE}} & = &     \\ \hline
\end{tabular}
}
{
\risktable{}
\begin{tabular}{|r|c|c|c|c|c|c|}
\hline \rowcolor{scorecolor}\scorelabel{} & $\leq$ -10 & -9 & -6 & -5 & -4 \\
\hline \rowcolor{riskcolor}\risklabel{}  & $\leq$25.9\% & 29.4\% & 41.3\% & 45.6\% & 50.0\%\\ 
\hline
\hline
\rowcolor{scorecolor}\scorelabel{} & -2  & -1 & 3 & 4 & 9\\\hline
\rowcolor{riskcolor}\risklabel{} & 58.7\% & 62.8\% & 77.3\% & 80.2\% & $\geq$ 90.7\%\\ 
\hline


\end{tabular}
}

\medskip\caption{\ourmethod{} model for the COMPAS dataset, predicting whether individuals are arrested within two years of release.
}
\label{fig:MoreExampleRiskScore_compas_modelSize_5}
\end{table}

\begin{table}[htbp]

\centering
{
\begin{tabular}{|llr|ll|}
\hline
1. & MSinceMostRecentInqexcl7days $\leq -8$ & -4 points &   & ... \\
2. & MSinceMostRecentInqexcl7days $\leq 0$ & 2 points &  + & ... \\
3. & NumSatisfactoryTrades $\leq 12$ & 2 points  & + & ... \\
4. & ExternalRiskEstimate $\leq 70$ & 3 points  & + & ... \\
5. & ExternalRiskEstimate $\leq 79$ & 3 points   & + & ... \\ \hline
     &     & \multicolumn{1}{l|}{\textbf{SCORE}} & = &     \\ \hline
\end{tabular}
}
{
\risktable{}
\begin{tabular}{|r|c|c|c|c|c|c|}
\hline \rowcolor{scorecolor}\scorelabel{} & -2 & 0 & 1 & 2 & 3 \\
\hline \rowcolor{riskcolor}\risklabel{}  & 6.7\% & 13.2\% & 18.2\% & 24.4\% & 32.0\%\\ 
\hline
\hline
\rowcolor{scorecolor}\scorelabel{} & 4  & 5 & 6 & 8 & 10\\\hline
\rowcolor{riskcolor}\risklabel{} & 40.7\% & 50.0\% & 59.3\% & 75.5\% & 86.8\%\\ 
\hline


\end{tabular}
}

\medskip\caption{\ourmethod{} model for the FICO dataset, predicting whether an individual will default on a loan. $-8$ means a missing value on the FICO dataset.
}
\label{fig:MoreExampleRiskScore_fico_modelSize_5}
\end{table}

\begin{table}[htbp]
\centering

{
\begin{tabular}{|llr|ll|}
\hline
1. & Clump Thickness & $\times$5 points &   & ... \\
2. & Uniformity of Cell Size & $\times$4 points  & + & ... \\
3. & Marginal Adhesion & $\times$3 points  & + & ... \\
4. & Bare Nuclei & $\times$4 points   & + & ... \\
5. & Normal Nucleoli & $\times$3 points   & + & ... \\ \hline
     &     & \multicolumn{1}{l|}{\textbf{SCORE}} & = &     \\ \hline
\end{tabular}
}
{
\risktable{}
\begin{tabular}{|r|c|c|c|c|c|c|}
\hline \rowcolor{scorecolor}\scorelabel{} & $\leq$ 55 & 60 & 65 & 70 & 75\\
\hline \rowcolor{riskcolor}\risklabel{}  & $\leq$ 8.6\% & 14.6\% & 23.5\% & 35.7\% & 50.0\\ 
\hline
\hline
\rowcolor{scorecolor}\scorelabel{} & 80 & 85 & 90 & 95 & $\geq$ 100\\\hline
\rowcolor{riskcolor}\risklabel{} & 64.3\% & 76.5\% & 85.4\% & 91.4 & $\geq$ 95.0\%\\ 
\hline


\end{tabular}
}

\medskip\caption{\ourmethod{} model for the breastcancer dataset, predicting whether there is breast cancer using a biopsy.
}

\label{fig:MoreExampleRiskScore_mammo_breastcancer_modelSize_5}
\end{table}

\begin{table}[htbp]

\centering
{
\begin{tabular}{|llr|ll|}
\hline
1. & WordFrequency\_Remove & $\times$5 points &   & ... \\
2. & WordFrequency\_Free & $\times$2 points &   & ... \\
3. & WordFrequency\_0 & $\times$5 points  & + & ... \\
4. & WordFrequency\_HP & $\times$-2 points  & + & ... \\
5. & WordFrequency\_George & $\times$-2 points   & + & ... \\ \hline
     &     & \multicolumn{1}{l|}{\textbf{SCORE}} & = &     \\ \hline
\end{tabular}
}
{
\risktable{}
\begin{tabular}{|r|c|c|c|c|c|c|}
\hline \rowcolor{scorecolor}\scorelabel{} & $\leq$ -4 & -3 & -2 & -1 & 0 \\
\hline \rowcolor{riskcolor}\risklabel{}  & $\leq$0.6\% & 1.6\% & 4.4\% & 11.4\% & 26.4\%\\ 
\hline
\hline
\rowcolor{scorecolor}\scorelabel{} & 1  & 2 & 3 & 4 & $\geq$ 5\\\hline
\rowcolor{riskcolor}\risklabel{} & 50.0\% & 73.6\% & 88.6\% & 95.6\% & $\geq$ 98.4\%\\ 
\hline


\end{tabular}
}

\medskip\caption{\ourmethod{} model for the spambase dataset, predicting if an e-mail is spam.
}
\label{fig:MoreExampleRiskScore_spambase_modelSize_5}
\end{table}

\begin{table}[htbp]

\centering
{
\begin{tabular}{|llr|ll|}
\hline
1. & previous case $\leq 20$ & -5 points &   & ... \\
2. & previous case $\leq 10$ or previous case $\geq 21$ & -3 points &  + & ... \\
3. & \# of previous penal cases $\leq 2$ & -2 points  & + & ... \\
4. & age in years $\leq 38.06$ & 1 points  & + & ... \\
5. & age at first penal case $\leq 22.63$ & 1 points   & + & ... \\ \hline
     &     & \multicolumn{1}{l|}{\textbf{SCORE}} & = &     \\ \hline
\end{tabular}
}
{
\risktable{}
\begin{tabular}{|r|c|c|c|c|c|c|c|}
\hline \rowcolor{scorecolor}\scorelabel{} & $\leq$-9 & -8 & -7 & -6 & -5 & -4 \\
\hline \rowcolor{riskcolor}\risklabel{}  & $\leq$ 23.8\% & 35.8\% & 50.0\% & 64.2\% & 76.2\% & 85.1\%\\ 
\hline
\hline
\rowcolor{scorecolor}\scorelabel{} & -3  & -2 & -1 & 0 & 1 & 2\\\hline
\rowcolor{riskcolor}\risklabel{} & 91.1\% & 94.8\% & 97.0\% & 98.3\% & 99.1\% & 99.5\%\\ 
\hline


\end{tabular}
}

\medskip\caption{\ourmethod{} model for the Netherlands dataset, predicting whether defendants have any type of charge within four years.
}
\label{fig:MoreExampleRiskScore_netherlands_modelSize_5}
\end{table}


\begin{table}[htbp]
\centering
{
\begin{tabular}{|llr|ll|}
\hline
1. &Age 22 to 29 & -2 points &   & ... \\
2. &High School Diploma Only & -2 points  & + & ... \\
3. &No High school Diploma & -4 points &   & ... \\
4. &Married & 4 points & + & ...\\
5. &Work Hours Per Week $<$ 50 & -2 points & + & ...\\
6. &Any Capital Gains & 3 points & + & ...\\
7. &Any Capital Loss & 2 points & + & ...\\\hline
     &     & \multicolumn{1}{l|}{\textbf{SCORE}} & = &     \\ \hline
\end{tabular}
}
{
\risktable{}
\begin{tabular}{|r|c|c|c|c|c|}
\hline \rowcolor{scorecolor}\scorelabel{} & $\leq$-5 & -4 & -3 & -2 & -1 \\
\hline \rowcolor{riskcolor}\risklabel{}  & $\leq$0.8\% & 1.4\% & 2.6\% & 4.6\% & 8.1\% \\ 
\hline

\hline \rowcolor{scorecolor}\scorelabel{} & 0 & 2 & 3 & 4 & 7 \\
\hline \rowcolor{riskcolor}\risklabel{}  & 14.0\% & 35.3\% & 50.0\% & 64.7\% & 91.9\% \\ 
\hline
\end{tabular}
}

\medskip\caption{\ourmethod{} model for the adult dataset, predicting salary$>50$K.}
\label{fig:MoreExampleRiskScore_adult_modelSize_7}
\end{table}

\begin{table}[htbp]
\centering
{
\begin{tabular}{|llr|ll|}
\hline
1. &Blue Collar Job & -1 points &   & ... \\
2. &Call in Second Quarter & -2 points  & + & ... \\
3. &Previous Call Was Successful & 3 points  & + & ... \\
4. &Previous Marketing Campaign Failed & -1 points  & + & ... \\
5. &Employment Indicator $>$ 5100 & -5 points & + & ...\\
6. &Consumer Price Index $\geq$ 93.5 & 1 points & + & ...\\
7. &3 Month Euribor Rate $\geq$ 100 & -1 points & + & ...\\\hline
     &     & \multicolumn{1}{l|}{\textbf{SCORE}} & = &     \\ \hline
\end{tabular}
}
{
\risktable{}
\begin{tabular}{|r|c|c|c|c|c|c|}
\hline \rowcolor{scorecolor}\scorelabel{} & $\leq$-5 & -4 & -3 & -2 & -1 \\
\hline \rowcolor{riskcolor}\risklabel{}  & $\leq$ 7.9\% & 11.5\% & 16.3\% & 22.7\% & 30.6\%\\ 
\hline

\hline \rowcolor{scorecolor}\scorelabel{} & 0 & 1 & 2 & 3 & 4 \\
\hline \rowcolor{riskcolor}\risklabel{}  & 39.9\% & 50.0\% & 60.1\% & 69.4\% & 77.3\%\\ 
\hline

\end{tabular}
}
\medskip\caption{\ourmethod{} model for the bank dataset, predicting if a person opens bank account after marketing call.
}

\label{fig:MoreExampleRiskScore_bank_modelSize_7}
\end{table}

\begin{table}[htbp]
\centering
{
\begin{tabular}{|llr|ll|}
\hline
1. &Lobular Shape & 2 points &   & ... \\
2. &Irregular Shape & 5 points  & + & ... \\
3. &Circumscribed Margin & -4 points  & + & ... \\
4. &Obscured Margin & -1 points  & + & ... \\
5. &Spiculated Margin & 1 points  & + & ... \\
6. &Age $<$ 30 & -5 points  & + & ... \\
7. &Age $\geq$ 60 & 3 points & + & ...\\\hline
     &     & \multicolumn{1}{l|}{\textbf{SCORE}} & = &     \\ \hline
\end{tabular}
}
{
\risktable{}
\begin{tabular}{|r|c|c|c|c|c|}
\hline \rowcolor{scorecolor}\scorelabel{} & $\leq$-1 & 0 & 1 & 2 & 3 \\
\hline \rowcolor{riskcolor}\risklabel{}  & 19.8\% & 25.9\% & 33.2\% & 41.3\% & 50.0\% \\ 
\hline
\hline
\rowcolor{scorecolor}\scorelabel{} & 4 & 5 & 6 & 8 &  9\\\hline
\rowcolor{riskcolor}\risklabel{} & 58.7\% & 66.8\% & 74.1\%  & 85.2\% & 89.1\%\\ 
\hline

\end{tabular}
}
\medskip\caption{\ourmethod{} model for the mammo dataset, predicting malignancy of a breast lesion.
}
\label{fig:MoreExampleRiskScore_mammo_modelSize_7}
\end{table}

\begin{table}[htbp]

\centering
{
\begin{tabular}{|llr|ll|}
\hline
1. &odor$=$anise & -5 points &   & ... \\
2. &odor$=$none & -5 points  & + & ... \\
3. &odor$=$foul & 5 points   & + & ... \\
4. &gill size$=$narrow & 4 points   & + & ... \\ 
5. &stalk surface above ring$=$grooves & 2 points   & + & ... \\ 
6. &spore print color$=$green & 5 points   & + & ... \\ \hline
     &     & \multicolumn{1}{l|}{\textbf{SCORE}} & = &     \\ \hline
\end{tabular}
}
{
\risktable{}
\begin{tabular}{|r|c|c|c|c|c|}
\hline \rowcolor{scorecolor}\scorelabel{} & -5 & 0 & 2 & 4 & $\geq$ 5\\
\hline \rowcolor{riskcolor}\risklabel{} & 0.5\% & 50.0\% & 89.2\% & 98.6\% & 99.5\%\\ 
\hline

\end{tabular}
}

\medskip\caption{\ourmethod{} model for the mushroom dataset, predicting whether a mushroom is poisonous.
}
\label{fig:MoreExampleRiskScore_mushroom_modelSize_7}
\end{table}

\begin{table}[htbp]

\centering
{
\begin{tabular}{|llr|ll|}
\hline
1. & prior\_counts $\leq 7$ & -3 points &   & ... \\
2. & prior\_counts $\leq 2$ & -3 points &  + & ... \\
3. & prior\_counts $\leq 0$ & -2 points  & + & ... \\
4. & age $\leq$ 52 & 2 points  & + & ... \\
5. & age $\leq$ 33 & 2 points  & + & ... \\
6. & age $\leq$ 23 & 2 points  & + & ... \\
7. & age $\leq$ 20 & 4 points  & + & ... \\
\hline

     &     & \multicolumn{1}{l|}{\textbf{SCORE}} & = &     \\ \hline
\end{tabular}
}
{
\risktable{}
\begin{tabular}{|r|c|c|c|c|c|c|c|c|}
\hline \rowcolor{scorecolor}\scorelabel{} & -8 & -6 & -4 & -3 & -2 & -1 & 0 \\
\hline \rowcolor{riskcolor}\risklabel{}  & 11.3\% & 18.7\% & 29.3\% & 35.7\% & 42.7\% & 50.0\% & 57.3\% \\ 
\hline
\hline
\rowcolor{scorecolor}\scorelabel{} & 1  & 2 & 3 & 4 & 6 & 7 & 10\\\hline
\rowcolor{riskcolor}\risklabel{} & 64.3\% & 70.7\% & 76.4\% & 81.3\% & 88.7\% & 91.3\% & 96.2\%\\ 
\hline


\end{tabular}
}

\medskip\caption{\ourmethod{} model for the COMPAS dataset, predicting whether individuals are arrested within two years of release.
}
\label{fig:MoreExampleRiskScore_compas_modelSize_7}
\end{table}

\begin{table}[htbp]

\centering
{
\begin{tabular}{|llr|ll|}
\hline
1. & MSinceMostRecentInqexcl7days $\leq -8$ & -4 points &   & ... \\
2. & MSinceMostRecentInqexcl7days $\leq 0$ & 2 points &  + & ... \\
3. & NetFractionRevolvingBurden $\leq 37$ & -2 points  & + & ... \\
4. & ExternalRiskEstimate $\leq 70$ & 2 points  & + & ... \\
5. & ExternalRiskEstimate $\leq 78$ & 2 points  & + & ... \\
6. & AverageMInFile $\leq 60$ & 2 points  & + & ... \\
7. & PercentTradesNeverDelq $\leq 85$ & 2 points  & + & ... \\
\hline

     &     & \multicolumn{1}{l|}{\textbf{SCORE}} & = &     \\ \hline
\end{tabular}
}
{
\risktable{}
\begin{tabular}{|r|c|c|c|c|c|c|c|c|c|}
\hline \rowcolor{scorecolor}\scorelabel{} & -4 & -2 & 0 & 2 & 4 & 6 & 8 & 10 \\
\hline \rowcolor{riskcolor}\risklabel{}  & 8.0\% & 14.9\% & 26.0\% & 41.4\% & 58.6\% & 74.0\% & 85.1\% & 92.0\%\\ 
\hline


\end{tabular}
}

\medskip\caption{\ourmethod{} model for the FICO dataset, predicting whether an individual will default on a loan. $-8$ means a missing value on the FICO dataset.
}
\label{fig:MoreExampleRiskScore_fico_modelSize_7}
\end{table}

\begin{table}[htbp]
\centering

{
\begin{tabular}{|llr|ll|}
\hline
1. & Clump Thickness & $\times$4 points &   & ... \\
2. & Uniformity of Cell Shape & $\times$3 points  & + & ... \\
3. & Marginal Adhesion & $\times$3 points  & + & ... \\
4. & Bare Nuclei & $\times$3 points  & + & ... \\
5. & Bland Chromatin & $\times$3 points  & + & ... \\
6. & Normal Nucleoli & $\times$2 points  & + & ... \\
7. & Mitoses & $\times$4 points   & + & ... \\ \hline
     &     & \multicolumn{1}{l|}{\textbf{SCORE}} & = &     \\ \hline
\end{tabular}
}
{
\risktable{}
\begin{tabular}{|r|c|c|c|c|c|c|}
\hline \rowcolor{scorecolor}\scorelabel{} & $\leq$ 55 & 60 & 65 & 70 & 75\\
\hline \rowcolor{riskcolor}\risklabel{}  & $\leq$ 5.1\% & 9.3\% & 16.2\% & 26.6\% & 40.6\\ 
\hline
\hline
\rowcolor{scorecolor}\scorelabel{} & 80 & 85 & 90 & 95 & $\geq$ 100\\\hline
\rowcolor{riskcolor}\risklabel{} & 56.3\% & 70.8\% & 82.1\% & 89.6 & $\geq$ 94.2\%\\ 
\hline


\end{tabular}
}

\medskip\caption{\ourmethod{} model for the breastcancer dataset, predicting whether there is breast cancer using a biopsy.
}

\label{fig:MoreExampleRiskScore_mammo_breastcancer_modelSize_7}
\end{table}

\begin{table}[htbp]

\centering
{
\begin{tabular}{|llr|ll|}
\hline
1. & WordFrequency\_Remove & $\times$4 points &   & ... \\
2. & WordFrequency\_Free & $\times$2 points &   & ... \\
3. & WordFrequency\_Business & $\times$1 points  & + & ... \\
4. & WordFrequency\_0 & $\times$4 points  & + & ... \\
5. & WordFrequency\_HP & $\times$-2 points  & + & ... \\
6. & WordFrequency\_George & $\times$-2 points   & + & ...\\
7. & CharacterFrequency\_\$ & $\times$5 points   & + & ...\\ \hline
     &     & \multicolumn{1}{l|}{\textbf{SCORE}} & = &     \\ \hline
\end{tabular}
}
{
\risktable{}
\begin{tabular}{|r|c|c|c|c|c|c|}
\hline \rowcolor{scorecolor}\scorelabel{} & $\leq$ -4 & -3 & -2 & -1 & 0 \\
\hline \rowcolor{riskcolor}\risklabel{}  & $\leq$0.4\% & 1.3\% & 3.7\% & 10.2\% & 25.2\%\\ 
\hline
\hline
\rowcolor{scorecolor}\scorelabel{} & 1  & 2 & 3 & 4 & $\geq$ 5\\\hline
\rowcolor{riskcolor}\risklabel{} & 50.0\% & 74.8\% & 89.8\% & 96.3\% & $\geq$ 98.7\%\\ 
\hline


\end{tabular}
}

\medskip\caption{\ourmethod{} model for the spambase dataset, predicting if an e-mail is spam.
}
\label{fig:MoreExampleRiskScore_spambase_modelSize_7}
\end{table}

\begin{table}[htbp]

\centering
{
\begin{tabular}{|llr|ll|}
\hline
1. & previous case $\leq 20$ & -5 points &   & ... \\
2. & previous case $\leq 10$ or previous case $\geq 21$ & -4 points &  + & ... \\
3. & \# of previous penal cases $\leq 1$ & -1 points  & + & ... \\
4. & \# of previous penal cases $\leq 3$ & -1 points  & + & ... \\
5. & \# of previous penal cases $\leq 5$ & -1 points  & + & ... \\
6. & age in years $\leq 21.80$ & 1 points  & + & ... \\
7. & age in years $\leq 38.05$ & 1 points  & + & ... \\
\hline

     &     & \multicolumn{1}{l|}{\textbf{SCORE}} & = &     \\ \hline
\end{tabular}
}
{
\risktable{}
\begin{tabular}{|r|c|c|c|c|c|c|}
\hline \rowcolor{scorecolor}\scorelabel{} & $\leq$-10 & -9 & -8 & -7 & -6 & -5  \\
\hline \rowcolor{riskcolor}\risklabel{}  & $\leq$ 33.1\% & 50.0\% & 66.9\% & 80.3\% & 89.2\% & 94.3\%  \\ 
\hline
\hline
\rowcolor{scorecolor}\scorelabel{} & -4 & -3  & -2 & -1 & 0 & $\geq$ 1  \\\hline
\rowcolor{riskcolor}\risklabel{} & 97.1\% & 98.6\% & 99.3\% & 99.6\% & 99.8\% & 99.9\% \\ 
\hline


\end{tabular}
}

\medskip\caption{\ourmethod{} model for the Netherlands dataset, predicting whether defendants have any type of charge within four years.
}
\label{fig:MoreExampleRiskScore_netherlands_modelSize_7}
\end{table}

\begin{table}[htbp]

\centering
{
\begin{tabular}{|llr|ll|}
\hline
1. & ExternalRiskEstimate$\leq$63 & 1 points &   & ... \\
2. & ExternalRiskEstimate$\leq$70 & 2 points & + & ... \\
3. & ExternalRiskEstimate$\leq$79 & 2 points & + & ... \\
4. & AverageMInFile$\leq$59 & 2 points & + & ... \\
5. & NumSatisfactoryTrades$\leq$13 & 2 points & + & ... \\
6. & PercentTradesNeverDelq$\leq$95 & 1 points & + & ... \\
7. & PercentInstallTrades$\leq$46 & -1 points & + & ... \\
8. & MSinceMostRecentInqexcl7days$\leq$-8 & -5 points & + & ... \\
9. & MSinceMostRecentInqexcl7days$\leq$0 & 2 points & + & ... \\
10. & NetFractionRevolvingBurden$\leq$37 & -2 points & + & ... \\ \hline
     &     & \multicolumn{1}{l|}{\textbf{SCORE}} & = &     \\ \hline
\end{tabular}
}
{
\risktable{}
\begin{tabular}{|r|c|c|c|c|c|c|c|c|c|c|c|}
\hline \rowcolor{scorecolor}\scorelabel{} & -8 & -7 & -6 & -5 & -4 & -3 & -2 & -1 & 0 & 1 & 2\\
\hline
\rowcolor{riskcolor}\risklabel{} & 2.7\% & 3.7\% & 5.0\% & 6.8\% & 9.2\% & 12.4\% & 16.4\% & 21.3\% & 27.3\% & 34.2\% & 41.9\% \\ 
\hline
\end{tabular}
\begin{tabular}{|r|c|c|c|c|c|c|c|c|c|c|c|}
\hline \rowcolor{scorecolor}\scorelabel{} & 3 & 4 & 5 & 6 & 7 & 8 & 9 & 10 & 11 & 12\\
\hline
\rowcolor{riskcolor}\risklabel{} & 50.0\% & 58.1\% & 65.8\% & 72.7\% & 78.7\% & 83.6\% & 87.6\% & 90.8\% & 93.2\% & 95.0\% \\
\hline


\end{tabular}
}

\medskip\caption{\ourmethod{} model for the FICO dataset, predicting whether an individual will default on a loan. $-8$ means a missing value on the FICO dataset. 
}
\label{fig:MoreExampleRiskScore_fico_modelSize_10}
\end{table}

\clearpage
\subsection{Risk Score Models from the Pool of Solutions}
\label{app:examples_from_the_pool}
\subsubsection{Examples from the Pool of Solutions (Bank Dataset)}
\label{app:examples_from_the_pool_bank}
The extra risk score examples from the pool of solutions are shown in Tables~\ref{fig:RiskScore_bank_modelSize_5_pool_1}-\ref{fig:RiskScore_bank_modelSize_5_pool_12}. All models were from the pool of the third fold on the bank dataset, and we show the top 12 models, provided in  ascending order of the logistic loss on the training set (the model with the smallest logistic loss comes first).

\begin{table}[h]
\centering
{
\begin{tabular}{|llr|ll|}
\hline
1. &Call in Second Quarter & -2 points &   & ... \\
2. &Previous Call Was Successful & 4 points  & + & ... \\
3. &Previous Marketing Campaign Failed & -1 points & + & ...\\ 
4. &Employment Indicator $>$ 5100 & -5 points & + & ...\\
5. &3 Month Euribor Rate $\geq$ 100 & -2 points & + & ...\\\hline
     &     & \multicolumn{1}{l|}{\textbf{SCORE}} & = &     \\ \hline
\end{tabular}
}
{
\risktable{}
\begin{tabular}{|r|c|c|c|c|c|c|c|c|c|c|}
\hline \rowcolor{scorecolor}\scorelabel{} & $\leq$-5 & -4 & -3 & -2 & -1 \\
\hline \rowcolor{riskcolor}\risklabel{}  & $\leq$ 11.2\% & 15.1\% & 20.1\% & 26.2\% & 33.4\%\\ 
\hline

\hline \rowcolor{scorecolor}\scorelabel{} & 0 & 1 & 2 & 3 & 4 \\
\hline \rowcolor{riskcolor}\risklabel{}  & 41.5\% & 50.0\% & 58.5\% & 66.6\% & 73.8\%\\ 
\hline

\end{tabular}
}
\medskip\caption{\ourmethod{} model for the bank dataset, predicting if a person opens a bank account after a marketing call. The logistic loss on the training set is 9352.39. The AUCs on the training and test sets are 0.779 and 0.770, respectively.}

\label{fig:RiskScore_bank_modelSize_5_pool_1}
\end{table}

\begin{table}[h]
\centering
{
\begin{tabular}{|llr|ll|}
\hline
1. &Call in Second Quarter & -2 points &   & ... \\
2. &Previous Call Was Successful & 4 points  & + & ... \\
3. &Previous Marketing Campaign Failed & -1 points & + & ...\\ 
4. &Employment Variation Rate $<$ -1 & 5 points & + & ...\\
5. &3 Month Euribor Rate $\geq$ 100 & -2 points & + & ...\\\hline
     &     & \multicolumn{1}{l|}{\textbf{SCORE}} & = &     \\ \hline
\end{tabular}
}
{
\risktable{}
\begin{tabular}{|r|c|c|c|c|c|c|c|c|c|c|}
\hline \rowcolor{scorecolor}\scorelabel{} & $\leq$0 & 1 & 2 & 3 & 4 \\
\hline \rowcolor{riskcolor}\risklabel{}  & $\leq$ 11.2\% & 15.1\% & 20.1\% & 26.2\% & 33.4\%\\ 
\hline

\hline \rowcolor{scorecolor}\scorelabel{} & 5 & 6 & 7 & 8 & 9 \\
\hline \rowcolor{riskcolor}\risklabel{}  & 41.5\% & 50.0\% & 58.5\% & 66.6\% & 73.8\%\\ 
\hline

\end{tabular}
}
\medskip\caption{\ourmethod{} model for the bank dataset, predicting if a person opens a bank account after a marketing call. The logistic loss on the training set is 9352.39. The AUCs on the training and test sets are 0.779 and 0.770, respectively.}

\label{fig:RiskScore_bank_modelSize_5_pool_2}
\end{table}

\begin{table}[h]
\centering
{
\begin{tabular}{|llr|ll|}
\hline
1. &Call in Second Quarter & -2 points &   & ... \\
2. &Previous Call Was Successful & 4 points  & + & ... \\
3. &Previous Marketing Campaign Failed & -1 points & + & ...\\ 
4. &3 Month Euribor Rate $\geq$ 100 & -2 points & + & ...\\
5. &3 Month Euribor Rate $\geq$ 200 & -5 points & + & ...\\\hline
     &     & \multicolumn{1}{l|}{\textbf{SCORE}} & = &     \\ \hline
\end{tabular}
}
{
\risktable{}
\begin{tabular}{|r|c|c|c|c|c|c|c|c|c|c|}
\hline \rowcolor{scorecolor}\scorelabel{} & $\leq$-5 & -4 & -3 & -2 & -1 \\
\hline \rowcolor{riskcolor}\risklabel{}  & $\leq$ 11.2\% & 15.2\% & 20.1\% & 26.3\% & 33.4\%\\ 
\hline

\hline \rowcolor{scorecolor}\scorelabel{} & 0 & 1 & 2 & 3 & 4 \\
\hline \rowcolor{riskcolor}\risklabel{}  & 41.5\% & 50.0\% & 58.5\% & 66.6\% & 73.7\%\\ 
\hline

\end{tabular}
}
\medskip\caption{\ourmethod{} model for the bank dataset, predicting if a person opens a bank account after a marketing call. The logistic loss on the training set is 9352.86. The AUCs on the training and test sets are 0.779 and 0.769, respectively.}

\label{fig:RiskScore_bank_modelSize_5_pool_3}
\end{table}


\begin{table}[htbp]
\centering
{
\begin{tabular}{|llr|ll|}
\hline
1. &Call in Second Quarter & -2 points &   & ... \\
2. &Previous Marketing Campaign Failed & -1 points  & + & ... \\
3. &Previous Marketing Campaign Succeeded & 4 points & + & ...\\ 
4. &3 Month Euribor Rate $\geq$ 100 & -2 points & + & ...\\
5. &3 Month Euribor Rate $\geq$ 200 & -5 points & + & ...\\\hline
     &     & \multicolumn{1}{l|}{\textbf{SCORE}} & = &     \\ \hline
\end{tabular}
}
{
\risktable{}
\begin{tabular}{|r|c|c|c|c|c|c|c|c|c|c|}
\hline \rowcolor{scorecolor}\scorelabel{} & $\leq$-5 & -4 & -3 & -2 & -1 \\
\hline \rowcolor{riskcolor}\risklabel{}  & $\leq$ 11.3\% & 15.3\% & 20.2\% & 26.3\% & 33.5\%\\ 
\hline

\hline \rowcolor{scorecolor}\scorelabel{} & 0 & 1 & 2 & 3 & 4 \\
\hline \rowcolor{riskcolor}\risklabel{}  & 41.5\% & 50.0\% & 58.5\% & 66.5\% & 73.6\%\\ 
\hline

\end{tabular}
}
\medskip\caption{\ourmethod{} model for the bank dataset, predicting if a person opens a bank account after a marketing call. The logistic loss on the training set is 9363.40. The AUCs on the training and test sets are 0.779 and 0.769, respectively. Note that some customers do not have previous marketing campaigns, so for these customers, neither of conditions 2 nor 3 are satisfied.}

\label{fig:RiskScore_bank_modelSize_5_pool_4}
\end{table}


\begin{table}[htbp]
\centering
{
\begin{tabular}{|llr|ll|}
\hline
1. &Call in Second Quarter & -2 points &   & ... \\
2. &Previous Call Was Successful & 4 points  & + & ... \\
3. &Consumer Price Index $>$ 93.5 & 1 points & + & ...\\ 
4. &3 Month Euribor Rate $\geq$ 100 & -1 points & + & ...\\
5. &3 Month Euribor Rate $\geq$ 200 & -5 points & + & ...\\\hline
     &     & \multicolumn{1}{l|}{\textbf{SCORE}} & = &     \\ \hline
\end{tabular}
}
{
\risktable{}
\begin{tabular}{|r|c|c|c|c|c|c|c|c|c|c|}
\hline \rowcolor{scorecolor}\scorelabel{} & $\leq$-4 & -3 & -2 & -1 & 0 \\
\hline \rowcolor{riskcolor}\risklabel{}  & $\leq$ 9.6\% & 13.4\% & 18.4\% & 24.6\% & 32.2\%\\ 
\hline

\hline \rowcolor{scorecolor}\scorelabel{} & 1 & 2 & 3 & 4 & 5\\
\hline \rowcolor{riskcolor}\risklabel{}  & 40.8\% & 50.0\% & 59.2\% & 67.8\% & 75.4\%\\ 
\hline

\end{tabular}
}
\medskip\caption{\ourmethod{} model for the bank dataset, predicting if a person opens a bank account after a marketing call. The logistic loss on the training set is 9365.51. The AUCs on the training and test sets are 0.778 and 0.769, respectively.}

\label{fig:RiskScore_bank_modelSize_5_pool_5}
\end{table}


\begin{table}[htbp]
\centering
{
\begin{tabular}{|llr|ll|}
\hline
1. &Call in First Quarter & 2 points &   & ... \\
2. &Call in Second Quarter & -1 points &   & ... \\
3. &Previous Call Was Successful & 3 points  & + & ... \\
4. &3 Month Euribor Rate $\geq$ 100 & -2 points & + & ...\\
5. &3 Month Euribor Rate $\geq$ 200 & -3 points & + & ...\\\hline
     &     & \multicolumn{1}{l|}{\textbf{SCORE}} & = &     \\ \hline
\end{tabular}
}
{
\risktable{}
\begin{tabular}{|r|c|c|c|c|c|c|c|c|c|c|}
\hline \rowcolor{scorecolor}\scorelabel{} & $\leq$-4 & -3 & -2 & -1 & 0\\
\hline \rowcolor{riskcolor}\risklabel{}  & $\leq$ 8.8\% & 13.4\% & 19.8\% & 28.2\% & 38.5\%\\ 
\hline

\hline \rowcolor{scorecolor}\scorelabel{} & 1 & 2 & 3 & 4 & 5\\
\hline \rowcolor{riskcolor}\risklabel{}  & 50.0\% & 61.5\% & 71.8\% & 80.2\% & 86.6\%\\ 
\hline

\end{tabular}
}
\medskip\caption{\ourmethod{} model for the bank dataset, predicting if a person opens a bank account after a marketing call. The logistic loss on the training set is 9365.57. The AUCs on the training and test sets are 0.776 and 0.766, respectively.}

\label{fig:RiskScore_bank_modelSize_5_pool_6}
\end{table}


\begin{table}[htbp]
\centering
{
\begin{tabular}{|llr|ll|}
\hline
1. &Call in Second Quarter & -2 points &   & ... \\
2. &Previous Call Was Successful & 3 points  & + & ... \\
3. &Previous Marketing Campaign Failed & -1 points & + & ...\\ 
4. &Consumer Price Index $\geq$ 93.5 & 1 points & + & ...\\
5. &3 Month Euribor Rate $\geq$ 200 & -4 points & + & ...\\\hline
     &     & \multicolumn{1}{l|}{\textbf{SCORE}} & = &     \\ \hline
\end{tabular}
}
{
\risktable{}
\begin{tabular}{|r|c|c|c|c|c|c|c|c|c|c|}
\hline \rowcolor{scorecolor}\scorelabel{} & $\leq$-5 & -4 & -3 & -2 & -1 \\
\hline \rowcolor{riskcolor}\risklabel{}  & $\leq$ 3.0\% & 5.2\% & 9.0\% & 15.0\% & 23.9\%\\ 
\hline

\hline \rowcolor{scorecolor}\scorelabel{} & 0 & 1 & 2 & 3 & 4 \\
\hline \rowcolor{riskcolor}\risklabel{}  & 35.9\% & 50.0\% & 64.1\% & 76.1\% & 85.0\%\\ 
\hline

\end{tabular}
}
\medskip\caption{\ourmethod{} model for the bank dataset, predicting if a person opens a bank account after a marketing call. The logistic loss on the training set is 9367.20. The AUCs on the training and test sets are 0.781 and 0.772, respectively.}

\label{fig:RiskScore_bank_modelSize_5_pool_7}
\end{table}


\begin{table}[htbp]
\centering
{
\begin{tabular}{|llr|ll|}
\hline
1. &Call in Second Quarter & -1 points &   & ... \\
2. &Previous Call Was Successful & 5 points  & + & ... \\
3. &Calls Before Campaign Succeeded & -1 points & + & ...\\ 
4. &3 Month Euribor Rate $\geq$ 100 & -2 points & + & ...\\
5. &3 Month Euribor Rate $\geq$ 200 & -4 points & + & ...\\\hline
     &     & \multicolumn{1}{l|}{\textbf{SCORE}} & = &     \\ \hline
\end{tabular}
}
{
\risktable{}
\begin{tabular}{|r|c|c|c|c|c|c|c|c|c|c|}
\hline \rowcolor{scorecolor}\scorelabel{} & $\leq$-4 & -3 & -2 & -1 & 0 \\
\hline \rowcolor{riskcolor}\risklabel{}  & $\leq$ 11.4\% & 16.3\% & 22.6\% & 30.6\% & 39.9\%\\ 
\hline

\hline \rowcolor{scorecolor}\scorelabel{} & 1 & 2 & 3 & 4 & 5\\
\hline \rowcolor{riskcolor}\risklabel{}  & 50.0\% & 60.1\% & 69.4\% & 77.4\% & 83.7\%\\ 
\hline

\end{tabular}
}
\medskip\caption{\ourmethod{} model for the bank dataset, predicting if a person opens a bank account after a marketing call. The logistic loss on the training set is 9367.93. The AUCs on the training and test sets are 0.780 and 0.769, respectively.}

\label{fig:RiskScore_bank_modelSize_5_pool_8}
\end{table}


\begin{table}[htbp]
\centering
{
\begin{tabular}{|llr|ll|}
\hline
1. &Call in Second Quarter & -1 points &   & ... \\
2. &Any Prior Calls Before Campaign & 4 points  & + & ... \\
3. &Previous Marketing Campaign Failed & -5 points & + & ...\\ 
4. &3 Month Euribor Rate $\geq$ 100 & -2 points & + & ...\\
5. &3 Month Euribor Rate $\geq$ 200 & -4 points & + & ...\\\hline
     &     & \multicolumn{1}{l|}{\textbf{SCORE}} & = &     \\ \hline
\end{tabular}
}
{
\risktable{}
\begin{tabular}{|r|c|c|c|c|c|c|c|c|c|c|}
\hline \rowcolor{scorecolor}\scorelabel{} & $\leq$-5 & -4 & -3 & -2 & -1 \\
\hline \rowcolor{riskcolor}\risklabel{}  & $\leq$ 7.8\% & 11.4\% & 16.2\% & 22.6\% & 30.5\%\\ 
\hline

\hline \rowcolor{scorecolor}\scorelabel{} & 0 & 1 & 2 & 3 & 4 \\
\hline \rowcolor{riskcolor}\risklabel{}  & 39.9\% & 50.0\% & 60.1\% & 69.5\% & 77.4\%\\ 
\hline

\end{tabular}
}
\medskip\caption{\ourmethod{} model for the bank dataset, predicting if a person opens a bank account after a marketing call. The logistic loss on the training set is 9371.75. The AUCs on the training and test sets are 0.779 and 0.769, respectively.}

\label{fig:RiskScore_bank_modelSize_5_pool_9}
\end{table}


\begin{table}[htbp]
\centering
{
\begin{tabular}{|llr|ll|}
\hline
1. &Called via Landline Phone & -2 points &   & ... \\
2. &Previous Call Was Successful & 5 points  & + & ... \\
3. &Previous Marketing Campaign Failed & -2 points & + & ...\\ 
4. &3 Month Euribor Rate $\geq$ 100 & -4 points & + & ...\\
5. &3 Month Euribor Rate $\geq$ 200 & -4 points & + & ...\\\hline
     &     & \multicolumn{1}{l|}{\textbf{SCORE}} & = &     \\ \hline
\end{tabular}
}
{
\risktable{}
\begin{tabular}{|r|c|c|c|c|c|c|c|c|c|c|}
\hline \rowcolor{scorecolor}\scorelabel{} & $\leq$-6 & -5 & -4 & -3 & -2 \\
\hline \rowcolor{riskcolor}\risklabel{}  & $\leq$ 10.9\% & 14.2\% & 18.3\% & 23.2\% & 28.9\%\\ 
\hline

\hline \rowcolor{scorecolor}\scorelabel{} & -1 & 0 & 1 & 3 & 5 \\
\hline \rowcolor{riskcolor}\risklabel{}  & 35.4\% & 42.6\% & 50.0\% & 64.6\% & 76.8\%\\ 
\hline

\end{tabular}
}
\medskip\caption{\ourmethod{} model for the bank dataset, predicting if a person opens a bank account after a marketing call. The logistic loss on the training set is 9376.52. The AUCs on the training and test sets are 0.776 and 0.765, respectively.}

\label{fig:RiskScore_bank_modelSize_5_pool_10}
\end{table}


\begin{table}[htbp]
\centering
{
\begin{tabular}{|llr|ll|}
\hline
1. &Job Is Retired & 1 points &   & ... \\
2. &Call in Second Quarter & -2 points  & + & ... \\
3. &Previous Call Was Successful & 5 points & + & ...\\ 
4. &3 Month Euribor Rate $\geq$ 100 & -2 points & + & ...\\
5. &3 Month Euribor Rate $\geq$ 200 & -5 points & + & ...\\\hline
     &     & \multicolumn{1}{l|}{\textbf{SCORE}} & = &     \\ \hline
\end{tabular}
}
{
\risktable{}
\begin{tabular}{|r|c|c|c|c|c|c|c|c|c|c|}
\hline \rowcolor{scorecolor}\scorelabel{} & $\leq$-3 & -2 & -1 & 0 & 1\\
\hline \rowcolor{riskcolor}\risklabel{}  & $\leq$ 17.2\% & 22.2\% & 28.1\% & 34.8\% & 42.2\%\\ 
\hline

\hline \rowcolor{scorecolor}\scorelabel{} & 2 & 3 & 4 & 5 & 6 \\
\hline \rowcolor{riskcolor}\risklabel{}  & 50.0\% & 57.8\% & 65.2\% & 72.0\% & 77.8\% \\ 
\hline

\end{tabular}
}
\medskip\caption{\ourmethod{} model for the bank dataset, predicting if a person opens a bank account after a marketing call. The logistic loss on the training set is 9378.18. The AUCs on the training and test sets are 0.777 and 0.766, respectively.}

\label{fig:RiskScore_bank_modelSize_5_pool_11}
\end{table}


\begin{table}[htbp]
\centering
{
\begin{tabular}{|llr|ll|}
\hline
1. &Age $\geq$ 60 & 1 points &   & ... \\
2. &Call in Second Quarter & -2 points  & + & ... \\
3. &Previous Call Was Successful & 5 points & + & ...\\ 
4. &3 Month Euribor Rate $\geq$ 100 & -2 points & + & ...\\
5. &3 Month Euribor Rate $\geq$ 200 & -5 points & + & ...\\\hline
     &     & \multicolumn{1}{l|}{\textbf{SCORE}} & = &     \\ \hline
\end{tabular}
}
{
\risktable{}
\begin{tabular}{|r|c|c|c|c|c|c|c|c|c|c|}
\hline \rowcolor{scorecolor}\scorelabel{} & $\leq$-3 & -2 & -1  & 0 & 1\\
\hline \rowcolor{riskcolor}\risklabel{}  & $\leq$ 17.3\% & 22.2\% & 28.1\% & 34.8\% & 42.2\%\\ 
\hline

\hline \rowcolor{scorecolor}\scorelabel{} & 2 & 3 & 4 & 5 & 6\\
\hline \rowcolor{riskcolor}\risklabel{}  & 50.0\% & 57.8\% & 65.2\% & 71.9\% & 77.8\%\\ 
\hline


\end{tabular}
}
\medskip\caption{\ourmethod{} model for the bank dataset, predicting if a person opens a bank account after a marketing call. The logistic loss on the training set is 9378.68. The AUCs on the training and test sets are 0.777 and 0.767, respectively.}

\label{fig:RiskScore_bank_modelSize_5_pool_12}
\end{table}


\clearpage
\subsubsection{Examples from the Pool of Solutions (Mammo Dataset)}
\label{app:examples_from_the_pool_mammo}
The extra risk score examples from the pool of solutions are shown in Tables~\ref{fig:RiskScore_mammo_modelSize_5_pool_1}-\ref{fig:RiskScore_mammo_modelSize_5_pool_12}. All models were from the pool of the third fold on the mammo dataset, and we show the top 12 models, provided in  ascending order of the logistic loss on the training set (the model with the smallest logistic loss comes first).

\begin{table}[h]
\centering
{
\begin{tabular}{|llr|ll|}
\hline
1. &Oval Shape  & -2 points   &  &... \\
2. &Irregular Shape  & 4 points  & + & ... \\
3. &Circumscribed Margin  & -5 points & + & ...\\ 
4. &Spiculated Margin  & 2 points & + & ...\\
5. &Age $\geq$ 60   & 3 points & + & ...\\\hline
     &     & \multicolumn{1}{l|}{\textbf{SCORE}} & = &     \\ \hline
\end{tabular}
}
{
\risktable{}
\begin{tabular}{|r|c|c|c|c|c|c|}
\hline \rowcolor{scorecolor}\scorelabel{} & -7     & -5     & -4     & -3     & -2     & -1  \\
\hline 
\rowcolor{riskcolor}\risklabel{} & 6.0\%  & 10.6\% & 13.8\% & 17.9\% & 22.8\% & 28.6\%  \\ 
\hline
\end{tabular}
\begin{tabular}{|r|c|c|c|c|c|c|}
\hline \rowcolor{scorecolor}\scorelabel{} & 0 & 1      & 2      & 3      & 4      & $\geq$ 5  \\
\hline 
\rowcolor{riskcolor}\risklabel{} & 35.2\% & 42.4\% & 50.0\% & 57.6\% & 64.8\% & 71.4\% \\ 
\hline

\end{tabular}
}
\medskip\caption{\ourmethod{} model for the mammo dataset, predicting the risk of malignancy of a breast lesion. The logistic loss on the training set is 357.77. The AUCs on the training and test sets are 0.854 and 0.853, respectively.}  

\label{fig:RiskScore_mammo_modelSize_5_pool_1}
\end{table}

\begin{table}[h]
\centering
{
\begin{tabular}{|llr|ll|}
\hline
1. &Lobular Shape & 1 point\text{ } &   & ... \\
2. &Irregular Shape & 3 points  & + & ... \\
3. &Circumscribed Margin & -3 points & + & ...\\ 
4. &Spiculated Margin & 1 point\text{ } & + & ...\\
5. &Age $\geq$ 60 & 2 points & + & ...\\\hline
     &     & \multicolumn{1}{l|}{\textbf{SCORE}} & = &     \\ \hline
\end{tabular}
}
{
\risktable{}
\begin{tabular}{|r|c|c|c|c|c|c|}
\hline \rowcolor{scorecolor}\scorelabel{} & -3 & -2 & -1 & 0 & 1 \\
\hline \rowcolor{riskcolor}\risklabel{}  & 7.5\% & 11.8\% & 18.1\% & 26.8\% & 37.7\%\\ 
\hline

\hline \rowcolor{scorecolor}\scorelabel{} & 2 & 3 & 4 & 5 & 6 \\
\hline \rowcolor{riskcolor}\risklabel{}  & 50.0\% & 62.3\% & 73.2\% & 81.9\% & 88.2\%\\ 
\hline

\end{tabular}
}
\medskip\caption{\ourmethod{} model for the mammo dataset, predicting the risk of malignancy of a breast lesion. The logistic loss on the training set is 357.86. The AUCs on the training and test sets are 0.854 and 0.857, respectively.}  

\label{fig:RiskScore_mammo_modelSize_5_pool_2}
\end{table}

\begin{table}[h]
\centering
{
\begin{tabular}{|llr|ll|}
\hline
1. & Lobular Shape & 2 points &   & ... \\
2. & Irregular Shape & 5 points  & + & ... \\
3. & Circumscribed Margin & -4 points & + & ...\\ 
4. & Age $<$ 30 & -5 points & + & ...\\
5. & Age $\geq$ 60 & 3 points & + & ...\\\hline
     &     & \multicolumn{1}{l|}{\textbf{SCORE}} & = &     \\ \hline
\end{tabular}
}
{
\risktable{}
\begin{tabular}{|r|c|c|c|c|c|c|c|c|c|}
\hline \rowcolor{scorecolor}\scorelabel{} & -9 & -7 & -6 & -5 & -4 & -3 & -2 & -1 & 0 \\
\hline \rowcolor{riskcolor}\risklabel{}  & 1.3\% & 2.6\% & 3.7\% & 5.2\% & 7.3\% & 10.1\% & 14.0\% & 18.9\% & 25.1\%\\ 
\hline

\hline \rowcolor{scorecolor}\scorelabel{} & 1 & 2 & 3 & 4 & 5 & 6 & 7 & 8 & 10 \\
\hline \rowcolor{riskcolor}\risklabel{}  & 32.6\% & 41.0\% & 50.0\% & 59.0\% & 67.4\% & 74.9\% & 81.1\% & 86.0\% & 92.7\% \\ 
\hline

\end{tabular}
}
\medskip\caption{\ourmethod{} model for the mammo dataset, predicting the risk of malignancy of a breast lesion. The logistic loss on the training set is 358.24. The AUCs on the training and test sets are 0.852 and 0.854, respectively.}  



\label{fig:RiskScore_mammo_modelSize_5_pool_3}
\end{table}

\begin{table}[htbp]
\centering
{
\begin{tabular}{|llr|ll|}
\hline
1. & Lobular Shape & 2 points &   & ... \\
2. & Irregular Shape & 5 points  & + & ... \\
3. & Circumscribed Margin & -4 points & + & ...\\ 
4. & Age $\geq$ 30 & 5 points & + & ...\\
5. & Age $\geq$ 60 & 3 points & + & ...\\\hline
     &     & \multicolumn{1}{l|}{\textbf{SCORE}} & = &     \\ \hline
\end{tabular}
}
{
\risktable{}
\begin{tabular}{|r|c|c|c|c|c|c|c|c|c|}
\hline \rowcolor{scorecolor}\scorelabel{} & -4 & -2 & -1 & 0 & 1 & 2 & 3 & 4 & 5  \\
\hline \rowcolor{riskcolor}\risklabel{}  & 1.3\% & 2.6\% & 3.7\% & 5.2\% & 7.3\% & 10.1\% & 14.0\% & 18.9\% & 25.1\%\\ 
\hline

\hline \rowcolor{scorecolor}\scorelabel{} & 6 & 7 & 8 & 9 & 10 & 11 & 12 & 13 & 15  \\
\hline \rowcolor{riskcolor}\risklabel{}  & 32.6\% & 41.0\% & 50.0\% & 59.0\% & 67.4\% & 74.9\% & 81.1\% & 86.0\% & 92.7\% \\ 
\hline

\end{tabular}
}
\medskip\caption{\ourmethod{} model for the mammo dataset, predicting the risk of malignancy of a breast lesion. The logistic loss on the training set is 358.24. The AUCs on the training and test sets are 0.852 and 0.854, respectively.}  

\label{fig:RiskScore_mammo_modelSize_5_pool_4}
\end{table}

\begin{table}[htbp]
\centering
{
\begin{tabular}{|llr|ll|}
\hline
1. & Irregular Shape & 2 points &   & ... \\
2. & Circumscribed Margin & -2 points  & + & ... \\
3. & Spiculated Margin & 1 point\text{ } & + & ...\\ 
4. & Age $\geq$ 30 & 2 points & + & ...\\
5. & Age $\geq$ 60 & 1 point\text{ } & + & ...\\\hline
     &     & \multicolumn{1}{l|}{\textbf{SCORE}} & = &     \\ \hline
\end{tabular}
}
{
\risktable{}
\begin{tabular}{|r|c|c|c|c|c|c|}
\hline \rowcolor{scorecolor}\scorelabel{} & -2 & -1 & 0 & 1 & 2 \\
\hline \rowcolor{riskcolor}\risklabel{}  & 2.3\% & 4.7\% & 9.5\% & 18.2\% & 32.0\%\\ 
\hline
\end{tabular} \\
\begin{tabular}{|r|c|c|c|c|}
\hline \rowcolor{scorecolor}\scorelabel{} & 3 & 4 & 5 & 6 \\
\hline \rowcolor{riskcolor}\risklabel{}  & 50.0\% & 68.0\% & 81.8\% & 90.5\% \\ 
\hline

\end{tabular}
}
\medskip\caption{\ourmethod{} model for the mammo dataset, predicting the risk of malignancy of a breast lesion. The logistic loss on the training set is 358.59. The AUCs on the training and test sets are 0.852 and 0.857, respectively.}  

\label{fig:RiskScore_mammo_modelSize_5_pool_5}
\end{table}

\begin{table}[htbp]
\centering
{
\begin{tabular}{|llr|ll|}
\hline
1. & Irregular Shape & 2 points &   & ... \\
2. & Circumscribed Margin & -2 points  & + & ... \\
3. & Spiculated Margin & 1 point\text{ } & + & ...\\ 
4. & Age $<$ 30 &  -2 points & + & ...\\
5. & Age $\geq$ 60 & 1 point\text{ } & + & ...\\\hline
     &     & \multicolumn{1}{l|}{\textbf{SCORE}} & = &     \\ \hline
\end{tabular}
}
{
\risktable{}
\begin{tabular}{|r|c|c|c|c|c|c|}
\hline \rowcolor{scorecolor}\scorelabel{} & -4 & -3 & -2 & -1 & 0 \\
\hline \rowcolor{riskcolor}\risklabel{}  & 2.3\% & 4.7\% & 9.5\% & 18.2\% & 32.0\%\\ 
\hline
\end{tabular} \\
\begin{tabular}{|r|c|c|c|c|}
\hline \rowcolor{scorecolor}\scorelabel{} & 1 & 2 & 3 & 4 \\
\hline \rowcolor{riskcolor}\risklabel{}  & 50.0\% & 68.0\% & 81.8\% & 90.5\% \\ 
\hline

\end{tabular}
}
\medskip\caption{\ourmethod{} model for the mammo dataset, predicting the risk of malignancy of a breast lesion. The logistic loss on the training set is 358.59. The AUCs on the training and test sets are 0.852 and 0.857, respectively.}  

\label{fig:RiskScore_mammo_modelSize_5_pool_6}
\end{table}

\begin{table}[htbp]
\centering
{
\begin{tabular}{|llr|ll|}
\hline
1. & Lobular Shape & 2 points &   & ... \\
2. & Irregular Shape & 5 points  & + & ... \\
3. & Circumscribed Margin & -4 points & + & ...\\ 
4. & Obscure Margin & -1 point\text{ } & + & ...\\
5. & Age $\geq$ 60 & 3 points & + & ...\\\hline
     &     & \multicolumn{1}{l|}{\textbf{SCORE}} & = &     \\ \hline
\end{tabular}
}
{
\risktable{}
\begin{tabular}{|r|c|c|c|c|c|c|c|c|}
\hline \rowcolor{scorecolor}\scorelabel{} & -5 & -4 & -3 & -2 & -1 & 0 & 1 & 2\\
\hline \rowcolor{riskcolor}\risklabel{}  & 5.3\% & 7.4\% & 10.3\% & 14.1\% & 19.1\% & 25.3\% & 32.7\% & 41.1\%\\ 
\hline

\hline \rowcolor{scorecolor}\scorelabel{}  & 3 & 4 & 5 & 6 & 7 & 8 & 9 & 10\\
\hline \rowcolor{riskcolor}\risklabel{}  & 50.0\% & 58.9\% & 67.3\% & 74.7\% & 80.9\% & 85.9\% & 89.7\% & 92.6\%\\ 
\hline

\end{tabular}
}
\medskip\caption{\ourmethod{} model for the mammo dataset, predicting the risk of malignancy of a breast lesion. The logistic loss on the training set is 358.71. The AUCs on the training and test sets are 0.852 and 0.857, respectively.}  

\label{fig:RiskScore_mammo_modelSize_5_pool_7}
\end{table}

\begin{table}[htbp]
\centering
{
\begin{tabular}{|llr|ll|}
\hline
1. & Irregular Shape & 5 points &   & ... \\
2. & Circumscribed Margin & -5 points  & + & ... \\
3. & Microlobulated Margin & 2 points & + & ...\\ 
4. & Spiculated Margin & 2 points & + & ...\\
5. & Age $\geq$ 60 & 3 points & + & ...\\\hline
     &     & \multicolumn{1}{l|}{\textbf{SCORE}} & = &     \\ \hline
\end{tabular}
}
{
\risktable{}
\begin{tabular}{|r|c|c|c|c|c|c|c|}
\hline \rowcolor{scorecolor}\scorelabel{} & -5 & -3 & -2 & -1 & 0 & 2 & 3 \\
\hline \rowcolor{riskcolor}\risklabel{}  & 8.6\% & 14.6\% & 18.6\% & 23.5\% & 29.2\% & 42.7\% & 50.0\%\\ 
\hline

\hline \rowcolor{scorecolor}\scorelabel{} & 4 & 5 & 7 & 8 & 9 & 10 & 12 \\
\hline \rowcolor{riskcolor}\risklabel{}  & 57.3\% & 64.3\% & 76.5\% & 81.4\% & 85.4\% & 88.7\% & 93.4\%\\ 
\hline

\end{tabular}
}
\medskip\caption{\ourmethod{} model for the mammo dataset, predicting the risk of malignancy of a breast lesion. The logistic loss on the training set is 358.98. The AUCs on the training and test sets are 0.852 and 0.852, respectively.}  

\label{fig:RiskScore_mammo_modelSize_5_pool_8}
\end{table}

\begin{table}[htbp]
\centering
{
\begin{tabular}{|llr|ll|}
\hline
1. & Irregular Shape & 4 points &   & ... \\
2. & Circumscribed Margin & -5 points  & + & ... \\
3. & SpiculatedMargin & 2 points & + & ...\\ 
4. & Age $\geq$ 45 & 1 point\text{ } & + & ...\\
5. & Age $\geq$ 60 & 3 points & + & ...\\\hline
     &     & \multicolumn{1}{l|}{\textbf{SCORE}} & = &     \\ \hline
\end{tabular}
}
{
\risktable{}
\begin{tabular}{|r|c|c|c|c|c|c|c|c|}
\hline \rowcolor{scorecolor}\scorelabel{} & -5 & -4 & -3 & -2 & -1 & 0 & 1 & 2  \\
\hline \rowcolor{riskcolor}\risklabel{}  & 7.3\% & 9.7\% & 12.9\% & 16.9\% & 21.9\% & 27.8\% & 34.6\% & 42.1\% \\ 
\hline

\hline \rowcolor{scorecolor}\scorelabel{} & 3 & 4 & 5 & 6 & 7 & 8 & 9 & 10  \\
\hline \rowcolor{riskcolor}\risklabel{}  & 50.0\% & 57.9\% & 65.4\% & 72.2\% & 78.1\% & 83.1\% & 87.1\% & 90.3\% \\ 
\hline

\end{tabular}
}
\medskip\caption{\ourmethod{} model for the mammo dataset, predicting the risk of malignancy of a breast lesion. The logistic loss on the training set is 359.10. The AUCs on the training and test sets are 0.855 and 0.859, respectively.}  

\label{fig:RiskScore_mammo_modelSize_5_pool_9}
\end{table}

\begin{table}[htbp]
\centering
{
\begin{tabular}{|llr|ll|}
\hline
1. & Irregular Shape & 4 points &   & ... \\
2. & Circumscribed Margin & -5 points  & + & ... \\
3. & Obscure Margin &  -1 points & + & ...\\ 
4. & Spiculated Margin &  2 points & + & ...\\
5. & Age $\geq$ 60 & 3 points & + & ...\\\hline
     &     & \multicolumn{1}{l|}{\textbf{SCORE}} & = &     \\ \hline
\end{tabular}
}
{
\risktable{}
\begin{tabular}{|r|c|c|c|c|c|c|c|c|}
\hline \rowcolor{scorecolor}\scorelabel{} & -6 & -5 & -4 & -3 & -2 & -1 & 0 & 1  \\
\hline \rowcolor{riskcolor}\risklabel{}  & 6.8\% & 9.2\% & 12.3\% & 16.3\% & 21.3\% & 27.3\% & 34.2\% & 41.9\%\\ 
\hline

\hline \rowcolor{scorecolor}\scorelabel{} & 2 & 3 & 4 & 5 & 6 & 7 & 8 & 9  \\
\hline \rowcolor{riskcolor}\risklabel{}  & 50.0\% & 58.1\% & 65.8\% & 72.7\% & 78.7\% & 83.7\% & 87.7\% & 90.8\% \\ 
\hline

\end{tabular}
}
\medskip\caption{\ourmethod{} model for the mammo dataset, predicting the risk of malignancy of a breast lesion. The logistic loss on the training set is 359.34. The AUCs on the training and test sets are 0.852 and 0.862, respectively.}  

\label{fig:RiskScore_mammo_modelSize_5_pool_10}
\end{table}

\begin{table}[htbp]
\centering
{
\begin{tabular}{|llr|ll|}
\hline
1. & Oval Shape & -1 point\text{ } &   & ... \\
2. & Lobular Shape & 1 point\text{ }  & + & ... \\
3. & Irregular Shape & 4 points & + & ...\\ 
4. & Circumscribed Margin & -4 points & + & ...\\
5. & Age $\geq$ 60 & 3 points & + & ...\\\hline
     &     & \multicolumn{1}{l|}{\textbf{SCORE}} & = &     \\ \hline
\end{tabular}
}
{
\risktable{}
\begin{tabular}{|r|c|c|c|c|c|c|c|}
\hline \rowcolor{scorecolor}\scorelabel{} & -5 & -4 & -3 & -2 & -1 & 0 & 1  \\
\hline \rowcolor{riskcolor}\risklabel{}  & 7.0\% & 9.8\% & 13.6\% & 18.5\% & 24.8\% & 32.3\% & 40.8\%\\ 
\hline

\hline \rowcolor{scorecolor}\scorelabel{} & 2 & 3 & 4 & 5 & 6 & 7 & 8 \\
\hline \rowcolor{riskcolor}\risklabel{}  & 50.0\% & 59.2\% & 67.7\% & 75.2\% & 81.5\% & 86.4\% & 90.2\% \\ 
\hline

\end{tabular}
}
\medskip\caption{\ourmethod{} model for the mammo dataset, predicting the risk of malignancy of a breast lesion. The logistic loss on the training set is 359.53. The AUCs on the training and test sets are 0.850 and 0.849, respectively.}  

\label{fig:RiskScore_mammo_modelSize_5_pool_11}
\end{table}

\begin{table}[htbp]
\centering
{
\begin{tabular}{|llr|ll|}
\hline
1. & Lobular Shape & 1 point\text{ } &   & ... \\
2. & Irregular Shape & 4 points  & + & ... \\
3. & Circumscribed Margin & -3 points & + & ...\\ 
4. & Age $\geq$ 45 & 1 point\text{ } & + & ...\\
5. & Age $\geq$ 60 & 2 points & + & ...\\\hline
     &     & \multicolumn{1}{l|}{\textbf{SCORE}} & = &     \\ \hline
\end{tabular}
}
{
\risktable{}
\begin{tabular}{|r|c|c|c|c|c|c|}
\hline \rowcolor{scorecolor}\scorelabel{} & -3 & -2 & -1 & 0 & 1 & 2  \\
\hline \rowcolor{riskcolor}\risklabel{}  & 6.3\% & 9.5\% & 14.1\% & 20.5\% & 28.9\% & 38.9\% \\ 
\hline

\hline \rowcolor{scorecolor}\scorelabel{} & 3 & 4 & 5 & 6 & 7 & 8 \\
\hline \rowcolor{riskcolor}\risklabel{}  & 50.0\% & 61.1\% & 71.1\% & 79.5\% & 85.9\% & 90.5\% \\ 
\hline

\end{tabular}
}
\medskip\caption{\ourmethod{} model for the mammo dataset, predicting the risk of malignancy of a breast lesion. The logistic loss on the training set is 359.53. The AUCs on the training and test sets are 0.852 and 0.850, respectively.}  

\label{fig:RiskScore_mammo_modelSize_5_pool_12}
\end{table}

\clearpage
\subsubsection{Examples from the Pool of Solutions (Netherlands Dataset)}
\label{app:examples_from_the_pool_netherlands}
The extra risk score examples from the pool of solutions are shown in Tables~\ref{fig:RiskScore_netherlands_modelSize_5_pool_1}-\ref{fig:RiskScore_netherlands_modelSize_5_pool_12}. All models were from the pool of the third fold on the Netherlands dataset, and we show the top 12 models, provided in  ascending order of the logistic loss on the training set (the model with the smallest logistic loss comes first).

\begin{table}[h]
\centering
{
\begin{tabular}{|llr|ll|}
\hline
1. & \# of previous penal cases $\leq$ 2 & -2 points &   & ... \\
2. & age in years $\leq$ 38.052 & 1 point\text{ } & + & ... \\
3. & age at first penal case $\leq$ 22.633 & 1 point\text{ } & + & ... \\
4. & previous case $\leq$ 10 or > 20 & -3 points & + & ... \\
5. & previouse case $\leq$ 20 & -5 points & + & ... \\ \hline
     &     & \multicolumn{1}{l|}{\textbf{SCORE}} & = &     \\ \hline
\end{tabular}
}
{
\risktable{}
\begin{tabular}{|r|c|c|c|c|c|c|c|}
\hline \rowcolor{scorecolor}\scorelabel{} & -10 & -9 & -8 & -7 & -6 & -5 & -4\\
\hline
\rowcolor{riskcolor}\risklabel{} & 14.9\% & 23.8\% & 35.8\% & 50.0\% & 64.2\% & 76.2\% & 85.1\% \\ 
\hline
\end{tabular} \\
\begin{tabular}{|r|c|c|c|c|c|c|}
\hline \rowcolor{scorecolor}\scorelabel{} & -3 & -2 & -1 & 0 & 1 & 2 \\
\hline
\rowcolor{riskcolor}\risklabel{} & 91.1\% & 94.8\% & 97.0\% & 98.3\% & 99.1\% & 99.5\% \\ 
\hline

\end{tabular}
}
\medskip\caption{\ourmethod{} model for the Netherlands dataset, predicting whether defendants have any type of charge within four years. The logistic loss on the training set is 9226.84. The AUCs on the training and test sets are 0.743 and 0.742, respectively.}  



\label{fig:RiskScore_netherlands_modelSize_5_pool_1}
\end{table}

\begin{table}[h]
\centering
{
\begin{tabular}{|llr|ll|}
\hline
1. & \# of previous penal cases $\leq$ 1 & -1 point\text{ } &   & ... \\
2. & \# of previous penal cases $\leq$ 3 & -1 point\text{ } & + & ... \\
3. & age in years $\leq$ 38.052 & 1 point\text{ } & + & ... \\
4. & previous case $\leq$ 10 or > 20 & -3 points & + & ... \\
5. & previouse case $\leq$ 20 & -3 points & + & ... \\ \hline
     &     & \multicolumn{1}{l|}{\textbf{SCORE}} & = &     \\ \hline
\end{tabular}
}
{
\risktable{}
\begin{tabular}{|r|c|c|c|c|c|}
\hline \rowcolor{scorecolor}\scorelabel{} & -8 & -7 & -6 & -5 & -4\\
\hline
\rowcolor{riskcolor}\risklabel{} & 12.4\% & 27.4\% & 50.0\% & 72.6\% & 87.6\% \\ 
\hline
\hline \rowcolor{scorecolor}\scorelabel{} & -3 & -2 & -1 & 0 & 1  \\
\hline
\rowcolor{riskcolor}\risklabel{} & 94.9\% & 98.0\% & 99.2\% & 99.7\% & 99.9\%  \\
\hline

\end{tabular}
}
\medskip\caption{\ourmethod{} model for the Netherlands dataset, predicting whether defendants have any type of charge within four years. The logistic loss on the training set is 9232.51. The AUCs on the training and test sets are 0.744 and 0.739, respectively.}  

\label{fig:RiskScore_netherlands_modelSize_5_pool_2}
\end{table}

\begin{table}[h]
\centering
{
\begin{tabular}{|llr|ll|}
\hline
1. & \# of previous penal cases $\leq$ 3 & -2 points &   & ... \\
2. & age in years $\leq$ 38.052 & 1 point\text{ } & + & ... \\
3. & age at first penal case $\leq$ 23.265 & 1 point\text{ } & + & ... \\
4. & previous case $\leq$ 10 or > 20 & -3 points & + & ... \\
5. & previouse case $\leq$ 20 & -5 points & + & ... \\ \hline
     &     & \multicolumn{1}{l|}{\textbf{SCORE}} & = &     \\ \hline
\end{tabular}
}
{
\risktable{}
\begin{tabular}{|r|c|c|c|c|c|c|c|}
\hline \rowcolor{scorecolor}\scorelabel{} & -10 & -9 & -8 & -7 & -6 & -5 & -4\\
\hline
\rowcolor{riskcolor}\risklabel{} & 16.1\% & 24.9\% & 36.6\% & 50.0\% & 63.4\% & 75.1\% & 83.9\% \\ 
\hline
\end{tabular} \\
\begin{tabular}{|r|c|c|c|c|c|c|}
\hline \rowcolor{scorecolor}\scorelabel{} & -3 & -2 & -1 & 0 & 1 & 2 \\
\hline
\rowcolor{riskcolor}\risklabel{} & 90.1\% & 94.0\% & 96.5\% & 97.9\% & 98.8\% & 99.3\% \\ 
\hline

\end{tabular}
}
\medskip\caption{\ourmethod{} model for the Netherlands dataset, predicting whether defendants have any type of charge within four years. The logistic loss on the training set is 9250.94. The AUCs on the training and test sets are 0.739 and 0.739, respectively.}  

\label{fig:RiskScore_netherlands_modelSize_5_pool_3}
\end{table}

\begin{table}[htbp]
\centering
{
\begin{tabular}{|llr|ll|}
\hline
1. & \# of previous penal cases $\leq$ 3 & -2 points &   & ... \\
2. & age in years $\leq$ 38.052 & 1 point\text{ } & + & ... \\
3. & age at first penal case $\leq$ 22.989 & 1 point\text{ } & + & ... \\
4. & previous case $\leq$ 10 or > 20 & -3 points & + & ... \\
5. & previouse case $\leq$ 20 & -5 points & + & ... \\ \hline
     &     & \multicolumn{1}{l|}{\textbf{SCORE}} & = &     \\ \hline
\end{tabular}
}
{
\risktable{}
\begin{tabular}{|r|c|c|c|c|c|c|c|}
\hline \rowcolor{scorecolor}\scorelabel{} & -10 & -9 & -8 & -7 & -6 & -5 & -4\\
\hline
\rowcolor{riskcolor}\risklabel{} & 16.1\% & 25.0\% & 36.6\% & 50.0\% & 63.4\% & 75.0\% & 83.9\% \\ 
\hline
\end{tabular} \\
\begin{tabular}{|r|c|c|c|c|c|c|}
\hline \rowcolor{scorecolor}\scorelabel{} & -3 & -2 & -1 & 0 & 1 & 2 \\
\hline
\rowcolor{riskcolor}\risklabel{} & 90.0\% & 94.0\% & 96.5\% & 97.9\% & 98.8\% & 99.3\%  \\ 
\hline

\end{tabular}
}
\medskip\caption{\ourmethod{} model for the Netherlands dataset, predicting whether defendants have any type of charge within four years. The logistic loss on the training set is 9250.95. The AUCs on the training and test sets are 0.738 and 0.739, respectively. Note that this risk score is slightly different from that of Table \ref{fig:RiskScore_netherlands_modelSize_5_pool_3}
in Condition 3.}  

\label{fig:RiskScore_netherlands_modelSize_5_pool_4}
\end{table}

\begin{table}[htbp]
\centering
{
\begin{tabular}{|llr|ll|}
\hline
1. & \# of previous penal cases $\leq$ 3 & -2 points &   & ... \\
2. & age in years $\leq$ 38.052 & 1 point\text{ } & + & ... \\
3. & age at first penal case $\leq$ 23.283 & 1 point\text{ } & + & ... \\
4. & previous case $\leq$ 10 or > 20 & -3 points & + & ... \\
5. & previouse case $\leq$ 20 & -5 points & + & ... \\ \hline
     &     & \multicolumn{1}{l|}{\textbf{SCORE}} & = &     \\ \hline
\end{tabular}
}
{
\risktable{}
\begin{tabular}{|r|c|c|c|c|c|c|c|}
\hline \rowcolor{scorecolor}\scorelabel{} & -10 & -9 & -8 & -7 & -6 & -5 & -4\\
\hline
\rowcolor{riskcolor}\risklabel{} & 16.0\% & 24.9\% & 36.6\% & 50.0\% & 63.4\% & 75.1\% & 84.0\% \\ 
\hline
\end{tabular} \\
\begin{tabular}{|r|c|c|c|c|c|c|}
\hline \rowcolor{scorecolor}\scorelabel{} & -3 & -2 & -1 & 0 & 1 & 2 \\
\hline
\rowcolor{riskcolor}\risklabel{} & 90.1\% & 94.0\% & 96.5\% & 97.9\% & 98.8\% & 99.3\%  \\ 
\hline

\end{tabular}
}
\medskip\caption{\ourmethod{} model for the Netherlands dataset, predicting whether defendants have any type of charge within four years. The logistic loss on the training set is 9251.14. The AUCs on the training and test sets are 0.739 and 0.739, respectively. Note that this risk score is slightly different from that of Table \ref{fig:RiskScore_netherlands_modelSize_5_pool_3}
in Condition 3.}  

\label{fig:RiskScore_netherlands_modelSize_5_pool_5}
\end{table}

\begin{table}[htbp]
\centering
{
\begin{tabular}{|llr|ll|}
\hline
1. & \# of previous penal cases $\leq$ 3 & -2 points &   & ... \\
2. & age in years $\leq$ 38.052 & 1 point\text{ } & + & ... \\
3. & age at first penal case $\leq$ 22.934 & 1 point\text{ } & + & ... \\
4. & previous case $\leq$ 10 or > 20 & -3 points & + & ... \\
5. & previouse case $\leq$ 20 & -5 points & + & ... \\ \hline
     &     & \multicolumn{1}{l|}{\textbf{SCORE}} & = &     \\ \hline
\end{tabular}
}
{
\risktable{}
\begin{tabular}{|r|c|c|c|c|c|c|c|}
\hline \rowcolor{scorecolor}\scorelabel{} & -10 & -9 & -8 & -7 & -6 & -5 & -4\\
\hline
\rowcolor{riskcolor}\risklabel{} & 16.1\% & 25.0\% & 36.6\% & 50.0\% & 63.4\% & 75.0\% & 83.9\% \\ 
\hline
\end{tabular} \\
\begin{tabular}{|r|c|c|c|c|c|c|}
\hline \rowcolor{scorecolor}\scorelabel{} & -3 & -2 & -1 & 0 & 1 & 2 \\
\hline
\rowcolor{riskcolor}\risklabel{} & 90.0\% & 94.0\% & 96.5\% & 97.9\% & 98.8\% & 99.3\%  \\ 
\hline

\end{tabular}
}
\medskip\caption{\ourmethod{} model for the Netherlands dataset, predicting whether defendants have any type of charge within four years. The logistic loss on the training set is 9251.39. The AUCs on the training and test sets are 0.739 and 0.740, respectively. Note that this risk score is slightly different from that of Table \ref{fig:RiskScore_netherlands_modelSize_5_pool_3}
in Condition 3.}  

\label{fig:RiskScore_netherlands_modelSize_5_pool_6}
\end{table}

\begin{table}[htbp]
\centering
{
\begin{tabular}{|llr|ll|}
\hline
1. & \# of previous penal cases $\leq$ 3 & -2 points &   & ... \\
2. & age in years $\leq$ 38.052 & 1 point\text{ } & + & ... \\
3. & age at first penal case $\leq$ 22.907 & 1 point\text{ } & + & ... \\
4. & previous case $\leq$ 10 or > 20 & -3 points & + & ... \\
5. & previouse case $\leq$ 20 & -5 points & + & ... \\ \hline
     &     & \multicolumn{1}{l|}{\textbf{SCORE}} & = &     \\ \hline
\end{tabular}
}
{
\risktable{}
\begin{tabular}{|r|c|c|c|c|c|c|c|}
\hline \rowcolor{scorecolor}\scorelabel{} & -10 & -9 & -8 & -7 & -6 & -5 & -4\\
\hline
\rowcolor{riskcolor}\risklabel{} & 16.1\% & 24.9\% & 36.6\% & 50.0\% & 63.4\% & 75.1\% & 83.9\% \\ 
\hline
\end{tabular} \\
\begin{tabular}{|r|c|c|c|c|c|c|}
\hline \rowcolor{scorecolor}\scorelabel{} & -3 & -2 & -1 & 0 & 1 & 2 \\
\hline
\rowcolor{riskcolor}\risklabel{} & 90.0\% & 94.0\% & 96.5\% & 97.9\% & 98.8\% & 99.3\%  \\ 
\hline

\end{tabular}
}
\medskip\caption{\ourmethod{} model for the Netherlands dataset, predicting whether defendants have any type of charge within four years. The logistic loss on the training set is 9251.53. The AUCs on the training and test sets are 0.739 and 0.740, respectively. Note that this risk score is slightly different from that of Table \ref{fig:RiskScore_netherlands_modelSize_5_pool_3}
in Condition 3.}  

\label{fig:RiskScore_netherlands_modelSize_5_pool_7}
\end{table}

\begin{table}[htbp]
\centering
{
\begin{tabular}{|llr|ll|}
\hline
1. & \# of previous penal cases $\leq$ 3 & -2 points &   & ... \\
2. & age in years $\leq$ 38.052 & 1 point\text{ } & + & ... \\
3. & age at first penal case $\leq$ 23.328 & 1 point\text{ } & + & ... \\
4. & previous case $\leq$ 10 or > 20 & -3 points & + & ... \\
5. & previouse case $\leq$ 20 & -5 points & + & ... \\ \hline
     &     & \multicolumn{1}{l|}{\textbf{SCORE}} & = &     \\ \hline
\end{tabular}
}
{
\risktable{}
\begin{tabular}{|r|c|c|c|c|c|c|c|}
\hline \rowcolor{scorecolor}\scorelabel{} & -10 & -9 & -8 & -7 & -6 & -5 & -4\\
\hline
\rowcolor{riskcolor}\risklabel{} & 16.0\% & 24.9\% & 36.5\% & 50.0\% & 63.5\% & 75.1\% & 84.0\% \\ 
\hline
\end{tabular} \\
\begin{tabular}{|r|c|c|c|c|c|c|}
\hline \rowcolor{scorecolor}\scorelabel{} & -3 & -2 & -1 & 0 & 1 & 2 \\
\hline
\rowcolor{riskcolor}\risklabel{} & 90.1\% & 94.0\% & 96.5\% & 97.9\% & 98.8\% & 99.3\%  \\ 
\hline

\end{tabular}
}
\medskip\caption{\ourmethod{} model for the Netherlands dataset, predicting whether defendants have any type of charge within four years. The logistic loss on the training set is 9252.07. The AUCs on the training and test sets are 0.738 and 0.739, respectively. Note that this risk score is slightly different from that of Table \ref{fig:RiskScore_netherlands_modelSize_5_pool_3}
in Condition 3.}  

\label{fig:RiskScore_netherlands_modelSize_5_pool_8}
\end{table}

\begin{table}[htbp]
\centering
{
\begin{tabular}{|llr|ll|}
\hline
1. & \# of previous penal cases $\leq$ 3 & -2 points &   & ... \\
2. & age in years $\leq$ 38.052 & 1 point\text{ } & + & ... \\
3. & age at first penal case $\leq$ 22.965 & 1 point\text{ } & + & ... \\
4. & previous case $\leq$ 10 or > 20 & -3 points & + & ... \\
5. & previouse case $\leq$ 20 & -5 points & + & ... \\ \hline
     &     & \multicolumn{1}{l|}{\textbf{SCORE}} & = &     \\ \hline
\end{tabular}
}
{
\risktable{}
\begin{tabular}{|r|c|c|c|c|c|c|c|}
\hline \rowcolor{scorecolor}\scorelabel{} & -10 & -9 & -8 & -7 & -6 & -5 & -4\\
\hline
\rowcolor{riskcolor}\risklabel{} & 16.1\% & 24.9\% & 36.6\% & 50.0\% & 63.4\% & 75.1\% & 83.9\% \\ 
\hline
\end{tabular} \\
\begin{tabular}{|r|c|c|c|c|c|c|}
\hline \rowcolor{scorecolor}\scorelabel{} & -3 & -2 & -1 & 0 & 1 & 2 \\
\hline
\rowcolor{riskcolor}\risklabel{} & 90.1\% & 94.0\% & 96.5\% & 97.9\% & 98.8\% & 99.3\%  \\ 
\hline

\end{tabular}
}
\medskip\caption{\ourmethod{} model for the Netherlands dataset, predicting whether defendants have any type of charge within four years. The logistic loss on the training set is 9252.13. The AUCs on the training and test sets are 0.738 and 0.740, respectively. Note that this risk score is slightly different from that of Table \ref{fig:RiskScore_netherlands_modelSize_5_pool_3}
in Condition 3.}  

\label{fig:RiskScore_netherlands_modelSize_5_pool_9}
\end{table}

\begin{table}[htbp]
\centering
{
\begin{tabular}{|llr|ll|}
\hline
1. & \# of previous penal cases $\leq$ 3 & -2 points &   & ... \\
2. & age in years $\leq$ 38.052 & 1 point\text{ } & + & ... \\
3. & age at first penal case $\leq$ 22.830 & 1 point\text{ } & + & ... \\
4. & previous case $\leq$ 10 or > 20 & -3 points & + & ... \\
5. & previouse case $\leq$ 20 & -5 points & + & ... \\ \hline
     &     & \multicolumn{1}{l|}{\textbf{SCORE}} & = &     \\ \hline
\end{tabular}
}
{
\risktable{}
\begin{tabular}{|r|c|c|c|c|c|c|c|}
\hline \rowcolor{scorecolor}\scorelabel{} & -10 & -9 & -8 & -7 & -6 & -5 & -4\\
\hline
\rowcolor{riskcolor}\risklabel{} & 16.1\% & 24.9\% & 36.6\% & 50.0\% & 63.4\% & 75.1\% & 83.9\% \\ 
\hline
\end{tabular} \\
\begin{tabular}{|r|c|c|c|c|c|c|}
\hline \rowcolor{scorecolor}\scorelabel{} & -3 & -2 & -1 & 0 & 1 & 2 \\
\hline
\rowcolor{riskcolor}\risklabel{} & 90.1\% & 94.0\% & 96.5\% & 97.9\% & 98.8\% & 99.3\%  \\ 
\hline

\end{tabular}
}
\medskip\caption{\ourmethod{} model for the Netherlands dataset, predicting whether defendants have any type of charge within four years. The logistic loss on the training set is 9252.19. The AUCs on the training and test sets are 0.739 and 0.740, respectively. Note that this risk score is slightly different from that of Table \ref{fig:RiskScore_netherlands_modelSize_5_pool_3}
in Condition 3.}  

\label{fig:RiskScore_netherlands_modelSize_5_pool_10}
\end{table}

\begin{table}[htbp]
\centering
{
\begin{tabular}{|llr|ll|}
\hline
1. & \# of previous penal cases $\leq$ 3 & -2 points &   & ... \\
2. & age in years $\leq$ 38.052 & 1 point\text{ } & + & ... \\
3. & age at first penal case $\leq$ 22.870 & 1 point\text{ } & + & ... \\
4. & previous case $\leq$ 10 or > 20 & -3 points & + & ... \\
5. & previouse case $\leq$ 20 & -5 points & + & ... \\ \hline
     &     & \multicolumn{1}{l|}{\textbf{SCORE}} & = &     \\ \hline
\end{tabular}
}
{
\risktable{}
\begin{tabular}{|r|c|c|c|c|c|c|c|}
\hline \rowcolor{scorecolor}\scorelabel{} & -10 & -9 & -8 & -7 & -6 & -5 & -4\\
\hline
\rowcolor{riskcolor}\risklabel{} & 16.1\% & 24.9\% & 36.6\% & 50.0\% & 63.4\% & 75.1\% & 83.9\% \\ 
\hline
\end{tabular} \\
\begin{tabular}{|r|c|c|c|c|c|c|}
\hline \rowcolor{scorecolor}\scorelabel{} & -3 & -2 & -1 & 0 & 1 & 2 \\
\hline
\rowcolor{riskcolor}\risklabel{} & 90.1\% & 94.0\% & 96.5\% & 97.9\% & 98.8\% & 99.3\%  \\ 
\hline

\end{tabular}
}
\medskip\caption{\ourmethod{} model for the Netherlands dataset, predicting whether defendants have any type of charge within four years. The logistic loss on the training set is 9252.25. The AUCs on the training and test sets are 0.739 and 0.740, respectively. Note that this risk score is slightly different from that of Table \ref{fig:RiskScore_netherlands_modelSize_5_pool_3}
in Condition 3.}  

\label{fig:RiskScore_netherlands_modelSize_5_pool_11}
\end{table}

\begin{table}[htbp]
\centering
{
\begin{tabular}{|llr|ll|}
\hline
1. & \# of previous penal cases $\leq$ 3 & -2 points &   & ... \\
2. & age in years $\leq$ 38.052 & 1 point\text{ } & + & ... \\
3. & age at first penal case $\leq$ 23.233 & 1 point\text{ } & + & ... \\
4. & previous case $\leq$ 10 or > 20 & -3 points & + & ... \\
5. & previouse case $\leq$ 20 & -5 points & + & ... \\ \hline
     &     & \multicolumn{1}{l|}{\textbf{SCORE}} & = &     \\ \hline
\end{tabular}
}
{
\risktable{}
\begin{tabular}{|r|c|c|c|c|c|c|c|}
\hline \rowcolor{scorecolor}\scorelabel{} & -10 & -9 & -8 & -7 & -6 & -5 & -4\\
\hline
\rowcolor{riskcolor}\risklabel{} & 16.0\% & 24.9\% & 36.5\% & 50.0\% & 63.5\% & 75.1\% & 84.0\% \\ 
\hline
\end{tabular} \\
\begin{tabular}{|r|c|c|c|c|c|c|}
\hline \rowcolor{scorecolor}\scorelabel{} & -3 & -2 & -1 & 0 & 1 & 2 \\
\hline
\rowcolor{riskcolor}\risklabel{} & 90.1\% & 94.0\% & 96.5\% & 97.9\% & 98.8\% & 99.3\%  \\ 
\hline

\end{tabular}
}
\medskip\caption{\ourmethod{} model for the Netherlands dataset, predicting whether defendants have any type of charge within four years. The logistic loss on the training set is 9252.27. The AUCs on the training and test sets are 0.738 and 0.739, respectively. Note that this risk score is slightly different from that of Table \ref{fig:RiskScore_netherlands_modelSize_5_pool_3}
in Condition 3.}  

\label{fig:RiskScore_netherlands_modelSize_5_pool_12}
\end{table}

\clearpage
\section{Model Reduction}
\subsection{Reducing Models to Relatively Prime Coefficients}
\label{app:model_reduction}
If the coefficients of a model are not relatively prime, one can divide all the coefficients by any common prime factors without changing any of the predicted risks. Table \ref{fig:ReducedMushroom}(left), copied from Table \ref{fig:MoreExampleRiskScore_mushroom_modelSize_3},
 is reduced in this way to produce Table \ref{fig:ReducedMushroom}(right). Table \ref{fig:ReducedCOMPAS}(left), copied from Table \ref{fig:MoreExampleRiskScore_compas_modelSize_3}, is reduced in this way to produce Table \ref{fig:ReducedCOMPAS}(right).

\begin{table}[htbp]
\small
\centering
\begin{subtable}[t]{0.48\linewidth}
{
\begin{tabular}{|llr|ll|}
\hline
1. &odor$=$almond & -5 points &   & ... \\
2. &odor$=$anise & -5 points  & + & ... \\
3. &odor$=$none & -5 points   & + & ... \\ \hline
     &     & \multicolumn{1}{l|}{\textbf{SCORE}} & = &     \\ \hline
\end{tabular}
}
{
\risktable{}
\begin{tabular}{|r|c|c|}
\hline \rowcolor{scorecolor}\scorelabel{} & -5 & 0 \\
\hline \rowcolor{riskcolor}\risklabel{} & 10.8\% & 96.0\% \\ 
\hline

\end{tabular}
}
\end{subtable}
\hfill
\begin{subtable}[t]{0.48\linewidth}
{
\begin{tabular}{|llr|ll|}
\hline
1. &odor$=$almond & -1 points &   & ... \\
2. &odor$=$anise & -1 points  & + & ... \\
3. &odor$=$none & -1 points   & + & ... \\ \hline
     &     & \multicolumn{1}{l|}{\textbf{SCORE}} & = &     \\ \hline
\end{tabular}
}
{
\risktable{}
\begin{tabular}{|r|c|c|}
\hline \rowcolor{scorecolor}\scorelabel{} & -1 & 0 \\
\hline \rowcolor{riskcolor}\risklabel{} & 10.8\% & 96.0\% \\ 
\hline

\end{tabular}
}

\end{subtable}
\vspace{2mm}
\caption{\textit{Left:} \ourmethod{} model for the Mushroom dataset, predicting whether a mushroom is poisonous. Copy of Table \ref{fig:MoreExampleRiskScore_mushroom_modelSize_3}. \textit{Right:} Reduction to have relatively prime coefficients.}
\label{fig:ReducedMushroom}
\end{table}


\begin{table}[htbp]
\small
\centering
\begin{subtable}[t]{0.48\linewidth}
{
\begin{tabular}{|llr|ll|}
\hline
1. & prior\_counts $\leq 2$ & -4 points &   & ... \\
2. & prior\_counts $\leq 7$ & -4 points  & + & ... \\
3. & age $\leq 31$ & 4 points   & + & ... \\ \hline
     &     & \multicolumn{1}{l|}{\textbf{SCORE}} & = &     \\ \hline
\end{tabular}
}
{
\risktable{}
\begin{tabular}{|r|c|c|c|c|c|}
\hline \rowcolor{scorecolor}\scorelabel{} & -8 & -4 & 0 & 4 \\
\hline \rowcolor{riskcolor}\risklabel{}  & 23.6\% & 44.1\% & 67.0\% & 83.9\% \\ 
\hline
\end{tabular}
}
\end{subtable}
\hfill
\begin{subtable}[t]{0.48\linewidth}
{
\begin{tabular}{|llr|ll|}
\hline
1. & prior\_counts $\leq 2$ & -1 points &   & ... \\
2. & prior\_counts $\leq 7$ & -1 points  & + & ... \\
3. & age $\leq 31$ & 1 points   & + & ... \\ \hline
     &     & \multicolumn{1}{l|}{\textbf{SCORE}} & = &     \\ \hline
\end{tabular}
}
{
\risktable{}
\begin{tabular}{|r|c|c|c|c|c|}
\hline \rowcolor{scorecolor}\scorelabel{} & -2 & -1 & 0 & 1 \\
\hline \rowcolor{riskcolor}\risklabel{}  & 23.6\% & 44.1\% & 67.0\% & 83.9\% \\ 
\hline
\end{tabular}
}
\end{subtable}
\vspace{2mm}
\caption{\textit{Left:} \ourmethod{} model for the COMPAS dataset, predicting whether individuals are arrested within two years of release. Copy of Table \ref{fig:MoreExampleRiskScore_compas_modelSize_3}. \textit{Right:} Reduction to have relatively prime coefficients.}
\label{fig:ReducedCOMPAS}
\end{table}

\clearpage
\subsection{Transforming Features for Better Interpretability}
Sometimes the original features are not as interpretable as they could be with some minor postprocessing. 
For example, Table \ref{fig:unpostprocessed_netherlands} has features "previous case $\leq$ 10 or $>$ 20" and "previous case $\leq$ 20". We can transform them into more interpretable and user-friendly features as "previous case $\leq$ 10", "10 $<$ previous case $\leq$ 20", and "previous case $>$ 20". The transformed model is shown in Table~\ref{fig:postprocessed_netherlands}.
\begin{table}[h]
\centering
{
\begin{tabular}{|llr|ll|}
\hline
1. & \# of previous penal cases $\leq$ 2 & -2 points &   & ... \\
2. & age in years $\leq$ 38.052 & 1 point\text{ } & + & ... \\
3. & age at first penal case $\leq$ 22.633 & 1 point\text{ } & + & ... \\
4. & previous case $\leq$ 10 or > 20 & -3 points & + & ... \\
5. & previouse case $\leq$ 20 & -5 points & + & ... \\ \hline
     &     & \multicolumn{1}{l|}{\textbf{SCORE}} & = &     \\ \hline
\end{tabular}
}
{
\risktable{}
\begin{tabular}{|r|c|c|c|c|c|c|c|}
\hline \rowcolor{scorecolor}\scorelabel{} & -10 & -9 & -8 & -7 & -6 & -5 & -4\\
\hline
\rowcolor{riskcolor}\risklabel{} & 14.9\% & 23.8\% & 35.8\% & 50.0\% & 64.2\% & 76.2\% & 85.1\% \\ 
\hline
\end{tabular} \\
\begin{tabular}{|r|c|c|c|c|c|c|}
\hline \rowcolor{scorecolor}\scorelabel{} & -3 & -2 & -1 & 0 & 1 & 2 \\
\hline
\rowcolor{riskcolor}\risklabel{} & 91.1\% & 94.8\% & 97.0\% & 98.3\% & 99.1\% & 99.5\% \\ 
\hline

\end{tabular}
}
\medskip\caption{Original \ourmethod{} model for the Netherlands dataset, predicting whether defendants have any type of charge within four years.} 

\label{fig:unpostprocessed_netherlands}
\end{table}

\begin{table}[h]
\centering
{
\begin{tabular}{|llr|ll|}
\hline
1. & \# of previous penal cases $\leq$ 2 & -2 points &   & ... \\
2. & age in years $\leq$ 38.052 & 1 point\text{ } & + & ... \\
3. & age at first penal case $\leq$ 22.633 & 1 point\text{ } & + & ... \\
4. & previous case $\leq$ 10 & -8 points & + & ... \\
5. & 10 $<$ previouse case $\leq$ 20 & -5 points & + & ... \\
6. & previouse case $>$ 20 & -3 points & + & ... \\ \hline
     &     & \multicolumn{1}{l|}{\textbf{SCORE}} & = &     \\ \hline
\end{tabular}
}
{
\risktable{}
\begin{tabular}{|r|c|c|c|c|c|c|c|}
\hline \rowcolor{scorecolor}\scorelabel{} & -10 & -9 & -8 & -7 & -6 & -5 & -4\\
\hline
\rowcolor{riskcolor}\risklabel{} & 14.9\% & 23.8\% & 35.8\% & 50.0\% & 64.2\% & 76.2\% & 85.1\% \\ 
\hline
\end{tabular} \\
\begin{tabular}{|r|c|c|c|c|c|c|}
\hline \rowcolor{scorecolor}\scorelabel{} & -3 & -2 & -1 & 0 & 1 & 2 \\
\hline
\rowcolor{riskcolor}\risklabel{} & 91.1\% & 94.8\% & 97.0\% & 98.3\% & 99.1\% & 99.5\% \\ 
\hline

\end{tabular}
}
\medskip\caption{Postprocessed \ourmethod{} model for the Netherlands dataset, predicting whether defendants have any type of charge within four years. We have transformed the "previous case" feature for better interpretability. Note that in the original model, samples with previous case values less than 10 accumulate -8 points, -3 for the 4th line and -5 for the 5th line.  In the transformed model, this case is more clearly stated in line 4.} 

\label{fig:postprocessed_netherlands}
\end{table}

\clearpage
\section{Discussion of Limitations}
\label{app:discussion_of_limitation}
 \ourmethod{} does not provide provably optimal solutions to an NP-hard problem, which is how it is able to perform in reasonable time. \ourmethod's models should not be interpreted as causal. \ourmethod{} creates very sparse, generalized, additive models, and thus has limited model capacity. \ourmethod's models inherit flaws from data on which it was trained. \ourmethod{} is not yet customized to a given application, which can be done in future work. We note that even if a model is interpretable, it can still have negative societal bias. (Generally, it is easier to check for such biases with scoring systems than with black box models). Looking at a variety of models from the diverse pool can help users to find models that are more fair.

\end{document}